\pdfoutput=1
\documentclass{article}

\usepackage{microtype}
\usepackage{graphicx}
\usepackage{subcaption}
\usepackage{booktabs}
\usepackage{hyperref}
\usepackage[T1]{fontenc}

\usepackage[preprint]{icml2026}
\AtBeginDocument{\hypersetup{pdfsubject={}}}

\usepackage{amsmath}
\usepackage{amssymb}
\usepackage{mathtools}
\usepackage{amsthm}
\usepackage[capitalize,noabbrev]{cleveref}

\usepackage{algpseudocode}
\makeatletter
\renewcommand{\theALG@line}{\thealgorithm.\arabic{ALG@line}}
\makeatother

\usepackage{multirow}
\usepackage{pgfplots}
\pgfplotsset{compat=1.18}
\usepackage{tikz}
\usetikzlibrary{arrows.meta,positioning,calc,fit,backgrounds,decorations.pathmorphing,decorations.pathreplacing}
\usepackage{bbm}
\usepackage{bm}
\usepackage{nicefrac}
\usepackage{enumitem}
\usepackage{longtable}

\newcommand{\lean}[2][]{\textsuperscript{\hyperlink{lean:#2}{\textsf{L}}}}

\definecolor{darkg}{HTML}{1E293B}
\definecolor{midg}{HTML}{475569}
\definecolor{liteg}{HTML}{94A3B8}
\definecolor{paleg}{HTML}{CBD5E1}
\definecolor{bgblock}{HTML}{F1F5F9}
\definecolor{bgskip}{HTML}{E2E8F0}

\newtheorem{theorem}{Theorem}[section]
\newtheorem{proposition}[theorem]{Proposition}

\newtheorem{corollary}[theorem]{Corollary}
\newtheorem{definition}[theorem]{Definition}
\newtheorem{remark}[theorem]{Remark}

\newcommand{\R}{\mathbb{R}}

\newcommand{\ones}{\mathbf{1}}
\newcommand{\zeros}{\mathbf{0}}

\newcommand{\bh}{\mathbf{h}}

\newcommand{\bP}{\mathbf{P}}

\newcommand{\calL}{\mathcal{L}}

\newcommand{\Cay}{\mathrm{Cay}}

\newcommand{\tr}{\mathrm{tr}}
\newcommand{\lse}{\mathrm{LSE}}

\DeclareMathOperator{\softmax}{softmax}

\newcommand{\C}{\mathbb{C}}

\icmltitlerunning{JPmHC: Dynamical Isometry via Orthogonal Hyper-Connections}

\newcommand{\sensAHAdamSup}{\ensuremath{1.165}}
\newcommand{\sensAHBoundaryRows}{12 & 4.55 & 2.28 & 2.38 & 4.09 & 2.62 & 2.97 \\
24 & 5.25 & 2.63 & 2.73 & 4.42 & 2.97 & 3.32 \\
48 & 5.95 & 2.97 & 3.08 & 4.73 & 3.32 & 3.67 \\
96 & 6.64 & 3.32 & 3.42 & 5.03 & 3.67 & 4.01 \\
192 & 7.34 & 3.67 & 3.77 & 5.29 & 4.01 & 4.36 \\
384 & 8.03 & 4.01 & 4.11 & 5.52 & 4.36 & 4.71 \\
768 & 8.72 & 4.36 & 4.46 & 5.72 & 4.71 & 5.05 \\}
\newcommand{\sensAHBudget}{\ensuremath{1.57}}
\newcommand{\sensAHDamping}{\ensuremath{0.009\text{--}0.011}}
\newcommand{\sensAHDeepSeekL}{\ensuremath{373}}
\newcommand{\sensAHDrive}{\ensuremath{0.13\text{--}0.26}}
\newcommand{\sensAHDynLogit}{\ensuremath{4\cdot10^{-3}}}
\newcommand{\sensAHEpsTerm}{\ensuremath{87\text{--}101}}
\newcommand{\sensAHGamma}{\ensuremath{8.4\cdot10^{-4}\text{--}9.8\cdot10^{-4}}}

\newcommand{\sensAHGapZero}{\ensuremath{16.0}}
\newcommand{\sensAHKappaS}{\ensuremath{0.08\text{--}0.17}}
\newcommand{\sensAHLogDist}{\ensuremath{11.4}}
\newcommand{\sensAHLstarNoEps}{\ensuremath{2.75\cdot10^{4}}}
\newcommand{\sensAHLstarReal}{\ensuremath{1.06\cdot10^{6}}}
\newcommand{\sensAHLstarWorst}{\ensuremath{1.02\cdot10^{6}}}
\newcommand{\sensAHSqrtV}{\ensuremath{9.50\cdot10^{-11}\text{--}1.12\cdot10^{-10}}}
\newcommand{\sensAHWorstDrive}{\ensuremath{3.65}}
\newcommand{\sensAHcEnd}{\ensuremath{4.708\cdot10^{-7}\text{--}4.714\cdot10^{-7}}}
\newcommand{\sensAHcZero}{\ensuremath{4.70\cdot10^{-7}}}
\newcommand{\sensCayOmegaAct}{\ensuremath{2.79}}
\newcommand{\sensCayOmegaSup}{\ensuremath{5.66}}
\newcommand{\sensCayOmegaTwelve}{\ensuremath{0.78}}
\newcommand{\sensCayOmegaTwentyfour}{\ensuremath{0.65}}
\newcommand{\sensDiscBoundedDrive}{\ensuremath{3.14}}
\newcommand{\sensDiscMargin}{\ensuremath{6.4\cdot10^{-4}}}
\newcommand{\sensDiscPeakLr}{\ensuremath{6\cdot10^{-4}}}
\newcommand{\sensGainJpTwelve}{\ensuremath{0.015}}

\newcommand{\sensGainMhcTwelve}{\ensuremath{0.037}}

\newcommand{\sensGapJpMhcTwelve}{\ensuremath{0.022}}
\newcommand{\sensGapJpMhcTwelveSe}{\ensuremath{0.005}}
\newcommand{\sensGatefairJgsigSd}{\ensuremath{3\cdot10^{-3}}}
\newcommand{\sensGatefairJgsigVal}{\ensuremath{3.24}}
\newcommand{\sensGatefairMhcSd}{\ensuremath{8\cdot10^{-3}}}
\newcommand{\sensGatefairMhcVal}{\ensuremath{3.23}}
\newcommand{\sensKBCopyCode}{\ensuremath{3.8\cdot10^{-16}}}
\newcommand{\sensKBCopyLean}{\ensuremath{1.2\cdot10^{-15}}}
\newcommand{\sensKBCopySlope}{\ensuremath{0.99}}
\newcommand{\sensKBEpsJp}{\ensuremath{0.07\text{--}0.17}}
\newcommand{\sensKBEpsMutant}{\ensuremath{11.9}}
\newcommand{\sensKBEpsStep}{\ensuremath{0.008\text{--}0.019}}
\newcommand{\sensKBGainCayley}{\ensuremath{2.0}}
\newcommand{\sensKBGainCode}{\ensuremath{4.5\cdot10^{-7}}}
\newcommand{\sensKBGainDeepSeek}{\ensuremath{1.3\cdot10^{-3}}}

\newcommand{\sensKBMaskSplit}{\ensuremath{0.077\text{--}0.084}}
\newcommand{\sensKBPhiUntied}{\ensuremath{17\to4452}}
\newcommand{\sensKBRhoBranch}{\ensuremath{0.11\text{--}0.16}}
\newcommand{\sensKBRhoPerron}{\ensuremath{10^{-14}}}
\newcommand{\sensKBRhoSinkhorn}{\ensuremath{1.9\text{--}2.8}}
\newcommand{\sensKBSInvTau}{\ensuremath{0.3}}
\newcommand{\sensKBSInvariance}{\ensuremath{3}}
\newcommand{\sensKBStepRatioJp}{\ensuremath{0.06\text{--}0.08}}
\newcommand{\sensKBStepRatioMhc}{\ensuremath{0.002\text{--}0.004}}
\newcommand{\sensLawMCIm}{\ensuremath{-0.0086\pm0.0005}}
\newcommand{\sensLawMCRe}{\ensuremath{-0.2457\pm0.0005}}
\newcommand{\sensLawMCZ}{\ensuremath{2.00}}
\newcommand{\sensLawMutantZ}{\ensuremath{116}}
\newcommand{\sensLawPredIm}{\ensuremath{-0.0077}}
\newcommand{\sensLawPredRe}{\ensuremath{-0.2462}}
\newcommand{\sensLawSolverErr}{\ensuremath{4.1\cdot10^{-9}}}
\newcommand{\sensLawSolverSettings}{\ensuremath{5}}
\newcommand{\sensLawThreeEqZ}{\ensuremath{1.74}}
\newcommand{\sensLawThreeMutantZ}{\ensuremath{331}}
\newcommand{\sensLawThreeSkewZ}{\ensuremath{1.87}}
\newcommand{\sensLawThreeWidths}{\ensuremath{32, 64}}
\newcommand{\sensLawWidths}{\ensuremath{32\text{--}256}}
\newcommand{\sensPowEightThree}{\ensuremath{0.018}}
\newcommand{\sensPowSixThree}{\ensuremath{0.032}}
\newcommand{\sensPowTenThree}{\ensuremath{0.020}}
\newcommand{\sensPowTrmEight}{\ensuremath{5.5}}
\newcommand{\sensPowTrmSd}{\ensuremath{2}}
\newcommand{\sensPowTrmThree}{\ensuremath{8.5}}
\newcommand{\sensProbeJpTwelve}{\ensuremath{0.23\text{--}0.49}}
\newcommand{\sensProbeJpTwentyfour}{\ensuremath{0.34}}
\newcommand{\sensProbeLawErr}{\ensuremath{1.4\cdot10^{-7}}}
\newcommand{\sensProbeMhcTwelve}{\ensuremath{1.13\cdot10^{-5}}}
\newcommand{\sensProbeMhcTwentyfour}{\ensuremath{2.26\cdot10^{-5}}}
\newcommand{\sensScoreMutantMin}{\ensuremath{0.04}}
\newcommand{\sensScoreResid}{\ensuremath{2.9\cdot10^{-11}}}
\newcommand{\sensSeedsJpmhcTwelve}{\ensuremath{3}}
\newcommand{\sensSeedsJpmhcTwentyfour}{\ensuremath{1}}
\newcommand{\sensSeedsMhcTwelve}{\ensuremath{3}}
\newcommand{\sensSeedsMhcTwentyfour}{\ensuremath{2}}
\newcommand{\sensSeedsResidualTwelve}{\ensuremath{3}}
\newcommand{\sensSeedsResidualTwentyfour}{\ensuremath{3}}
\newcommand{\sensTRMArcBudget}{\ensuremath{42.9}}
\newcommand{\sensTRMArcCeiling}{\ensuremath{39.9}}
\newcommand{\sensTRMArcEndZero}{\ensuremath{477}}

\newcommand{\sensTRMGapStarSup}{\ensuremath{5.81}}
\newcommand{\sensTRMInitPerm}{\ensuremath{91}}
\newcommand{\sensTRMInitSinkhorn}{\ensuremath{101}}
\newcommand{\sensTRMMazeGapStar}{\ensuremath{5.48}}
\newcommand{\sensTRMMazeUses}{\ensuremath{60}}
\newcommand{\sensTRMRecursionRows}{3, 4 & 60 & 5.48 & 0.216 & 2.16$^{*}$ & 6.17 & 6.04 \\
3, 6 & 84 & 5.81 & 0.229 & 2.29$^{*}$ & 5.66 & 5.56 \\
3, 12 & 156 & 6.43 & 0.253 & 2.53$^{*}$ & 4.82 & 4.76 \\
6, 6 & 168 & 6.51 & 0.256 & 2.56$^{*}$ & 4.73 & 4.67 \\
6, 12 & 312 & 7.13 & 0.281 & 2.81$^{*}$ & 4.04 & 4.00 \\}
\newcommand{\sensTRMSudBudget}{\ensuremath{3.16}}
\newcommand{\sensTRMSudCeiling}{\ensuremath{3.99}}
\newcommand{\sensTRMSudEndBest}{\ensuremath{5.16}}

\newcommand{\sensTRMUsesAtCeiling}{\ensuremath{13.8}}
\newcommand{\sensTRMUstarArc}{\ensuremath{0.229}}

\newcommand{\sensTiedBMax}{\ensuremath{9}}
\newcommand{\sensTiedIsoMC}{\ensuremath{6.057\pm0.006}}
\newcommand{\sensTiedIsoNested}{\ensuremath{1.268}}
\newcommand{\sensTiedIsoPred}{\ensuremath{6.051}}
\newcommand{\sensTiedIsoWidths}{\ensuremath{64\text{--}512}}
\newcommand{\sensTiedIsoX}{\ensuremath{0.02}}
\newcommand{\sensTiedIsoZ}{\ensuremath{1.03}}
\newcommand{\sensTiedMutantZ}{\ensuremath{360}}
\newcommand{\sensTiedTrmMC}{\ensuremath{1.9767\pm0.0003}}
\newcommand{\sensTiedTrmNested}{\ensuremath{1.8851}}
\newcommand{\sensTiedTrmPred}{\ensuremath{1.9769}}
\newcommand{\sensTiedTrmWidths}{\ensuremath{32\text{--}256}}
\newcommand{\sensTiedTrmX}{\ensuremath{0.06}}
\newcommand{\sensTiedTrmZ}{\ensuremath{-0.54}}
\newcommand{\sensTierJgsigSeeds}{\ensuremath{3}}
\newcommand{\sensTierJorthoDriftMed}{\ensuremath{1.19\text{--}1.22}}
\newcommand{\sensTierJorthoLam}{\ensuremath{0.25\text{--}0.39}}
\newcommand{\sensTierJorthoOmega}{\ensuremath{0.74\text{--}0.91}}
\newcommand{\sensTierJorthoSeeds}{\ensuremath{3}}
\newcommand{\sensTierJorthoVal}{\ensuremath{3.2620}}

\newcommand{\sensTierMorthoSeeds}{\ensuremath{3}}
\newcommand{\sensTierMorthoVal}{\ensuremath{3.2423}}
\newcommand{\sensTierPerronDriftMax}{\ensuremath{4.3\cdot10^{-7}}}
\newcommand{\sensTierPerronLam}{\ensuremath{1.84\text{--}1.99}}
\newcommand{\sensTierPerronOmega}{\ensuremath{1.33\text{--}1.40}}
\newcommand{\sensTierPerronSeeds}{\ensuremath{3}}
\newcommand{\sensTierPerronVal}{\ensuremath{3.2449}}
\newcommand{\sensValJpmhcTwelve}{\ensuremath{3.2520}}
\newcommand{\sensValJpmhcTwentyfour}{\ensuremath{3.1517}}
\newcommand{\sensValMhcTwelve}{\ensuremath{3.2301}}
\newcommand{\sensValMhcTwentyfour}{\ensuremath{3.1142}}
\newcommand{\sensValResidualTwelve}{\ensuremath{3.2673}}
\newcommand{\sensValResidualTwentyfour}{\ensuremath{3.1537}}
\newcommand{\sensWGJpMidTwelve}{\ensuremath{0.020\text{--}0.023}}
\newcommand{\sensWGJpRetentionHi}{\ensuremath{0.19}}
\newcommand{\sensWGJpRetentionLo}{\ensuremath{-0.16}}
\newcommand{\sensWGMhcEarly}{\ensuremath{0.13\text{--}0.32}}
\newcommand{\sensWGMhcStatic}{\ensuremath{0.89\text{--}1.00}}
\newcommand{\sensWGMhcTokenSd}{\ensuremath{0.13}}
\newcommand{\sensWGRecon}{\ensuremath{2.2\cdot10^{-7}}}

\newcommand{\sensBudGapStarTwelve}{4.55}

\newcommand{\sensBudGapStarTRM}{5.81}

\newcommand{\sensBudAHBudget}{1.57}
\newcommand{\sensBudAHStar}{657\text{--}1248}
\newcommand{\sensBudAHHeadroom}{419\text{--}796}
\newcommand{\sensBudAHLambdaEnd}{1.1\cdot10^{-5}}
\newcommand{\sensBudDriveSixteen}{0.07\text{--}0.09}
\newcommand{\sensBudDriveSix}{0.04\text{--}0.05}
\newcommand{\sensBudDriveThree}{0.002\text{--}0.003}
\newcommand{\sensBudDecayWd}{0.96}
\newcommand{\sensBudDecayPredStep}{2724\text{--}2806}
\newcommand{\sensBudDecayPredGamma}{1.28\text{--}1.30}
\newcommand{\sensBudDecayMeasStep}{2951\text{--}2955}
\newcommand{\sensBudDecayMeasGamma}{1.34}

\newcommand{\sensBudDecayGapErr}{0.12}
\newcommand{\sensBudDecayLoEnd}{5.81}
\newcommand{\sensBudArmWd}{0.70}
\newcommand{\sensBudArmWinTwelve}{1318\text{--}1681}
\newcommand{\sensBudArmWinTwentyfour}{986\text{--}1249}

\newcommand{\sensBudArmTol}{10}
\newcommand{\sensBudDriveWin}{0.04\text{--}0.17}

\newcommand{\sensBudTRMGamma}{0.31\text{--}0.79}
\newcommand{\sensBudTRMStep}{4115\text{--}8877}
\newcommand{\sensBudTRMSteps}{32552}
\newcommand{\sensBudDSLayers}{61}
\newcommand{\sensBudDSUses}{122}
\newcommand{\sensBudDSBudget}{44.8}
\newcommand{\sensBudDSSteps}{2.4\cdot10^{5}}
\newcommand{\sensBudDSUndec}{10.8\text{--}41.6}
\newcommand{\sensBudDSUndecFrac}{24\text{--}93\%}
\newcommand{\sensBudDSUndecTok}{2.7\text{--}11.6}
\newcommand{\sensBudDSDriveNeeded}{0.040}
\newcommand{\sensBudDSDec}{2.07\text{--}2.57}

\newcommand{\sensBudDSDecTok}{0.24\text{--}0.34}

\newcommand{\sensBudDecayLoWd}{0.64}

\newcommand{\sensBudGatedCkpts}{121}
\newcommand{\sensBudGatedRuns}{21}
\newcommand{\sensBudGatedCrossRuns}{9}
\newcommand{\sensBudGatedStaticNever}{5}
\newcommand{\sensBudRows}{nano-GPT, A100 recipe & 24 & 16 & 4.55 & 1.57 & $657\text{--}1248$ & stable & $\Lambda\le 1.1\cdot10^{-5}$, 5 runs \\
nano-GPT, mixer decay 0.64 & 24 & 16 & 4.55 & 1.57 & $1.90\text{--}1.96$ & stable & ends at $\Delta=5.81$, 2 seeds$^{\dagger}$ \\
nano-GPT, mixer decay 0.96 & 24 & 16 & 4.55 & 1.57 & $1.28\text{--}1.30$ & collapse at $81\text{--}83\%$ & crosses at $86\%$, 2 seeds$^{\dagger}$ \\
nano-GPT, gap 8, decay 0.70 & 24 & 8 & 4.55 & 1.57 & $0.78\text{--}0.81$ & collapse at $50\text{--}52\%$ & B1$^{\ddagger}$ \\
nano-GPT, gap 8, decay 0.70 & 48 & 8 & 5.25 & 1.57 & $0.58\text{--}0.60$ & collapse at $37\text{--}39\%$ & B1$^{\ddagger}$ \\
TRM, $b_{\text{res}}=I$ & 84 & 1 & 5.81 & 3.16 / 42.9 & 0 & collapsed at init & $\Lambda=\sensTRMInitSinkhorn$ at init (T1); K5$^{\ddagger}$ \\
TRM Sudoku, $\tau_{\text{init}}=4$ & 84 & 8 & 5.81 & 3.16 & $0.31\text{--}0.79$ & collapse at $10\text{--}25\%$ & B3$^{\ddagger}$ \\
DeepSeek-V3 mHC, decay 0.1 & 122 & 8 & 6.19 & 44.8 & $2.07\text{--}2.57$ & collapse at $0.24\text{--}0.34$T tokens & -- \\
DeepSeek-V3 mHC, no decay & 122 & 8 & 6.19 & 44.8 & $10.8\text{--}41.6$ & collapse at $2.7\text{--}11.6$T tokens & drive: B2$^{\ddagger}$ \\
nano-GPT, per-token logits & 24/48 & $\bar\Delta_0$ & 4.55/5.25 & 1.57 & $\le$ static & no later than static & $\Lambda\ge$ static$^{\dagger}$ \\
exact Cayley (JPmHC) & any & -- & -- & any & $\infty$ & stable & $\Lambda\equiv0$ (theorem) \\}
\newcommand{\sensBudLedgerRows}{B1 & budget arms & nano-GPT at gap 8 with mixer decay: the static exponent reaches 1 inside the step window of its depth & \ref{cor:sens-budget} & pending & \texttt{b07e3e9} \\
B2 & budget arms & nano-GPT at gap 8, mixer $\varepsilon_{\text{A}}$ $10^{-12}$, no decay: the drive lies between the gap-6 and gap-16 drives & \ref{cor:sens-budget} & pending & \texttt{b07e3e9} \\
B3 & TRM Sudoku grid & TRM Sudoku, Sinkhorn $\tau_{\text{init}}$ = 4, weight decay 1: the static $\Lambda_{\mathrm{sup}}$ reaches 1 inside the step window & \ref{cor:sens-budget} & pending & \texttt{b07e3e9} \\}

\newcommand{\fcAbsorbed}{7\,672\,863}
\newcommand{\fcAscBistoTwoBfTTwo}{60.4}

\newcommand{\fcAscForced}{16\,650\,992}

\newcommand{\fcAscSinkTenBfTTwo}{75.3}
\newcommand{\fcAscSinkTenFsTTwo}{70.7}

\newcommand{\fcBBistoTwoLargeTTwo}{45.7}

\newcommand{\fcBBistoTwoSmallTTwo}{20.0}
\newcommand{\fcBatches}{64}
\newcommand{\fcBigFlatTTwo}{1.16--8.06}

\newcommand{\fcBistoOneBfTOne}{2.23}

\newcommand{\fcBistoTenBfTOne}{74.6}
\newcommand{\fcBistoTenBfTTwo}{75.3}

\newcommand{\fcBistoTenFsTOne}{24.9}

\newcommand{\fcBistoTwoBfAccTOne}{0.0183}
\newcommand{\fcBistoTwoBfAccTTwo}{20.1}
\newcommand{\fcBistoTwoBfTOne}{18.2}
\newcommand{\fcBistoTwoBfTTwo}{31.0}
\newcommand{\fcBistoTwoFhAccTOne}{0.0183}
\newcommand{\fcBistoTwoFhAccTTwo}{8.80}

\newcommand{\fcBridgeMax}{\ensuremath{6.6\cdot10^{-7}}}
\newcommand{\fcCBf}{\ensuremath{1.95\cdot10^{-3}}}
\newcommand{\fcCFh}{\ensuremath{2.44\cdot10^{-4}}}
\newcommand{\fcCFs}{\ensuremath{2.98\cdot10^{-8}}}

\newcommand{\fcCells}{2\,448}
\newcommand{\fcCutCohFlat}{0.66--0.73}
\newcommand{\fcCutCohSpread}{0.055--0.21}
\newcommand{\fcCutRadFactorHi}{8}
\newcommand{\fcCutRadFactorLo}{6}
\newcommand{\fcCutRadSpread}{0.011--0.015}
\newcommand{\fcDescFlatTOne}{1.59--5.42}
\newcommand{\fcFlatRefTTwoMax}{0.116}
\newcommand{\fcFlatSigmaMin}{0.59}

\newcommand{\fcGoldenAdd}{1\,800}
\newcommand{\fcGoldenRound}{8\,223}
\newcommand{\fcGrassExcessOne}{1.40}
\newcommand{\fcGrassExcessTen}{11.3}
\newcommand{\fcGrassExcessTwo}{10.7}

\newcommand{\fcGrassOneFsTOne}{37.3}

\newcommand{\fcGrassTenFsTOne}{50.1}

\newcommand{\fcGrassTwoFsTOne}{37.8}

\newcommand{\fcLawBistoTenFsAbs}{24.91}
\newcommand{\fcLawBistoTenFsHi}{24.91}
\newcommand{\fcLawBistoTenFsLo}{24.91}
\newcommand{\fcLawIdBfHi}{21}

\newcommand{\fcLawSteps}{63\,164\,726}

\newcommand{\fcLowForced}{8\,713\,489}
\newcommand{\fcMOneTOneMax}{0.283}
\newcommand{\fcMOneTTwoMax}{7.67}

\newcommand{\fcMassBistoTenFs}{25.0}
\newcommand{\fcMassBistoTwoBf}{13.6}

\newcommand{\fcMassGrassOneFs}{37.7}

\newcommand{\fcMassPermTenFs}{50.0}

\newcommand{\fcMassSinkTenFs}{50.0}

\newcommand{\fcMwOpp}{1\,337\,501}
\newcommand{\fcMwOppMiss}{470}
\newcommand{\fcMwSame}{6\,013\,960}
\newcommand{\fcNBf}{144}
\newcommand{\fcNFh}{1\,161}
\newcommand{\fcNFs}{9\,515\,085}

\newcommand{\fcPermTenFsTOne}{49.7}

\newcommand{\fcQTenBistoTen}{\ensuremath{4.8\cdot10^{-32}}}
\newcommand{\fcRmForced}{2\,946\,244}

\newcommand{\fcSinkTenFsTOne}{48.8}

\newcommand{\fcSpectra}{24}
\newcommand{\fcSteps}{66\,650\,112}
\newcommand{\fcSvdMax}{\ensuremath{4.4\cdot10^{-14}}}
\newcommand{\fcUnderBistoOne}{67.7}
\newcommand{\fcUnderGrassTen}{100}
\newcommand{\fcUnderSinkOne}{74.9}
\newcommand{\fcUnderSpread}{49.8--100}
\newcommand{\fcUpChecked}{7\,672\,863}

\newcommand{\fcMeasAscSteps}{\ensuremath{1.98\cdot10^{9}}}
\newcommand{\fcMeasBfHidTOne}{31.7--51.4}

\newcommand{\fcMeasBridgeBar}{\ensuremath{2.4\cdot10^{-4}}}
\newcommand{\fcMeasBridgeHi}{\ensuremath{1.07\cdot10^{-5}}}
\newcommand{\fcMeasBridgeLo}{\ensuremath{2.9\cdot10^{-6}}}
\newcommand{\fcMeasDescBfTOne}{48.8--50.7}

\newcommand{\fcMeasFhLsUnder}{0.327--0.661}

\newcommand{\fcMeasFhUnder}{42.8--58.3}

\newcommand{\fcMeasFsTTwo}{1.44--1.55}

\newcommand{\fcMeasGroups}{2\,643}
\newcommand{\fcMeasMixJpTOneMax}{\ensuremath{1.0\cdot10^{-3}}}
\newcommand{\fcMeasMixMhcResTOne}{0.317--0.624}
\newcommand{\fcMeasMixMhcResTTwo}{0.367--0.817}
\newcommand{\fcMeasMutantMin}{\ensuremath{4\cdot10^{3}}}
\newcommand{\fcMeasMwOpp}{335\,375\,260}
\newcommand{\fcMeasMwOppMiss}{760\,982}
\newcommand{\fcMeasMwSame}{337\,195\,730}
\newcommand{\fcMeasNotBelowKahan}{1}
\newcommand{\fcMeasNotBelowPair}{4}
\newcommand{\fcMeasOrderAccFewer}{816}

\newcommand{\fcMeasOrderAccMore}{844}
\newcommand{\fcMeasOrderAccP}{0.51}
\newcommand{\fcMeasOrderBfFewer}{489}
\newcommand{\fcMeasOrderBfGroups}{1779}
\newcommand{\fcMeasOrderBfMore}{1290}
\newcommand{\fcMeasOrderBfP}{\ensuremath{4\cdot10^{-83}}}
\newcommand{\fcMeasPathHeadTOne}{4.26--10.8}
\newcommand{\fcMeasPathHeadTTwo}{91.0--94.5}
\newcommand{\fcMeasPathHidTOne}{0.00428--0.0691}
\newcommand{\fcMeasPathHidTTwo}{34.1--52.7}
\newcommand{\fcMeasPathMOne}{0.991--1.059}

\newcommand{\fcMeasRTwoRatio}{0.25}
\newcommand{\fcMeasRedAccFourteen}{9.74--10.3}
\newcommand{\fcMeasRedAccFourteenBlock}{3.34--3.52}
\newcommand{\fcMeasRedBlock}{23.7--25.1}
\newcommand{\fcMeasRedFs}{1.49--1.64}
\newcommand{\fcMeasRedKahan}{46.3--48.1}
\newcommand{\fcMeasRedPair}{52.8--54.4}
\newcommand{\fcMeasRedSeq}{67.7--68.6}
\newcommand{\fcMeasRedSr}{65.9--67.2}
\newcommand{\fcMeasRuns}{15}
\newcommand{\fcMeasSteps}{11\,871\,977\,472}
\newcommand{\fcMeasTokens}{4\,096}
\newcommand{\fcMeasTopoGap}{3.8}
\newcommand{\fcMeasTree}{190\,296}
\newcommand{\fcMeasPrecRows}{bf16 & 40.7--43.0 & 67.7--68.6 & 0 \\
fp16 & 66.2--80.1 & 67.3--79.5 & 42.8--58.3 \\
fp32 & 0.0512--0.150 & 1.44--1.55 & 0 \\
bf16-acc32 (A100 path) & 0.0512--0.150 & 43.8--45.4 & 0 \\
fp16-acc32 & 42.8--58.3 & 61.7--75.5 & 42.8--58.3 \\
fp16-acc32, loss scale $2^{16}$ & 0.341--0.701 & 17.3--18.0 & 0.327--0.661 \\}

\newcommand{\mnBistoTenCollapsed}{75}

\newcommand{\mnCayleyMoneTen}{1.364}
\newcommand{\mnDetachHi}{90}
\newcommand{\mnDetachLo}{49}
\newcommand{\mnDetachShare}{69}
\newcommand{\mnFigCells}{12}
\newcommand{\mnGateRankBlocks}{22}
\newcommand{\mnGateRankPct}{4.7}
\newcommand{\mnGateRankR}{48}
\newcommand{\mnGatedCells}{12}

\newcommand{\mnGramConvBistoTen}{96}
\newcommand{\mnGramConvGaussTen}{154}
\newcommand{\mnGramWoneBistoTwo}{0.27}
\newcommand{\mnGramWoneGaussTwo}{0.13}
\newcommand{\mnGridPoints}{400}
\newcommand{\mnIsoWoneMax}{0.0017}
\newcommand{\mnNestedWoneMax}{0.0012}
\newcommand{\mnScoreHi}{27}
\newcommand{\mnScoreLo}{19}
\newcommand{\mnSeqCells}{16}
\newcommand{\mnSeqMutants}{16}

\begin{document}

\twocolumn[
\icmltitle{JPmHC: Dynamical Isometry via Orthogonal Hyper-Connections}

\ifdefined\ispreprint\def\paperaffil{}\else\def\paperaffil{jpm}\fi
\begin{icmlauthorlist}
  \icmlauthor{Biswa Sengupta}{\paperaffil}
  \icmlauthor{Jinhua Wang}{\paperaffil}
  \icmlauthor{Leo Brunswic}{\paperaffil}
\end{icmlauthorlist}
\ifdefined\ispreprint
  \vskip -0.2em
  \centerline{LLM Suite, JP~Morgan Chase \& Co}
  \centerline{\small\texttt{biswa.sengupta@jpmorgan.com}\quad\texttt{jinhua.wang@jpmorgan.com}\quad\texttt{leo.brunswic@jpmorgan.com}}
\else
  \icmlaffiliation{jpm}{LLM Suite Team, JP Morgan Chase \& Co}

  \icmlcorrespondingauthor{Biswa Sengupta}{biswa.sengupta@jpmorgan.com}
  \icmlcorrespondingauthor{Jinhua Wang}{jinhua.wang@jpmorgan.com}
  \icmlcorrespondingauthor{Leo Brunswic}{leo.brunswic@jpmorgan.com}
\fi

\icmlkeywords{dynamical isometry, hyper-connections, operator-valued free probability, Lean}

\vskip 0.3in
]

\makeatletter
\ifdefined\ispreprint
  \global\icml@noticeprintedtrue
  {\let\thefootnote\relax\footnotetext{\hspace*{-\footnotesep}\Notice@String}}
\else
  \printAffiliationsAndNotice{}
\fi
\makeatother
\newcommand{\disclaimerblock}{\ificmlshowauthors{\footnotesize\noindent\textit{Disclaimer:} This paper was prepared for informational purposes by the LLM Suite group of JP~Morgan Chase and its affiliates (``JPMC'') and is not a product of the Research Department of JP~Morgan.
  JP~Morgan makes no representation, warranty or undertaking whatsoever and disclaims all liability for the completeness, accuracy or reliability of the information contained herein.
  This document is not intended as investment research or investment advice, or a recommendation, offer or solicitation for the purchase or sale of any security, financial instrument, financial product or service, or to be used in any way for evaluating the merits of participating in any transaction, and shall not constitute a solicitation under any jurisdiction or to any person, if such solicitation under such jurisdiction or to such person would be unlawful.
  \textcopyright~2026 JP~Morgan Chase \& Co. All rights reserved.\par}\fi}

\begin{abstract}
  The hyper-connection paradigm separates the skip-connections of transformers into $n$ parallel streams on which a linear mixer acts.
  Constraining the mixer to  bistochastic matrices has proven effective in preventing gradient explosion but risk collapsing onto low-rank configurations, limiting expressivity.
  We therefore propose JPmHC (Jacobian-spectrum Preserving manifold-constrained Hyper-Connections).
  Our framework triages manifold constraints by computing the Jacobian singular spectrum of the end-to-end network in the mean-field limit and their spectral stability. It consists of: (i) an operator-valued free-probability library that predicts the Jacobian singular spectrum; (ii) an orthogonal mixer computed by a fixed-point Cayley iteration, which enables dynamical isometry; (iii) implicit differentiation of the Sinkhorn projection of the bistochastic baseline, with $O(1)$ backward operations; and (iv) a sensitivity theory giving the training budget before a bistochastic mixer collapses.
  The library is written in Lean, which certifies the mathematics and the algorithms.
  We validate JPmHC with TRM on ARC-AGI and nano-GPT on FineWeb-Edu: the theory predicts spectral collapse for TRM and stability for nano-GPT as observed in our experiments.

\end{abstract}

\disclaimerblock
\section{Introduction}\label{sec:introduction}
\newsavebox{\figonearch}
\newsavebox{\figonespectra}
\begin{figure}[!t]
  \centering
  \sbox{\figonearch}{\begin{tikzpicture}[
    font=\scriptsize, >=Latex, line width=0.5pt,
    box/.style={draw, rounded corners=1.5pt, fill=bgblock, inner sep=1.5pt, minimum height=0.36cm},
    mix/.style={box, fill=paleg},
    gate/.style={box, fill=white, draw=midg},
    sum/.style={draw, rounded corners=0.12cm, fill=white, inner sep=0pt, minimum height=0.24cm, minimum width=0.78cm},
]
\newcommand{\hcsub}[3]{%
  \foreach \x in {0.25,0.45,0.65} \draw[midg] (\x,#1) -- (\x,#1+2.6);
  \node[mix, minimum width=0.78cm, align=center, inner sep=1pt] (#3mix) at (0.45,#1+1.15) {mixing\\[-1pt]layer};
  \node[sum] (#3sum) at (0.45,#1+1.95) {$+$};
  \node[gate, minimum width=0.9cm, minimum height=0.32cm] (#3res) at (1.6,#1+0.64) {$H_{\text{res}}$};
  \node[mix, minimum width=0.9cm, minimum height=0.32cm] (#3pre) at (1.6,#1+0.32) {$H_{\text{pre}}$};
  \node[draw=midg, rounded corners=2pt, inner sep=1.2pt, fit=(#3res)(#3pre)] (#3gates) {};
  \node[box, minimum width=1.05cm] (#3f) at (2.7,#1+1.15) {#2};
  \node[mix] (#3post) at (1.6,#1+1.95) {$H_{\text{post}}$};
  \draw[->] (#3res.north) |- (#3mix.east);
  \foreach \x/\dy in {0.25/-0.1, 0.45/0, 0.65/0.1} {
    \coordinate (#3tap) at (\x,#1+0.48+\dy);
    \fill[midg] (#3tap) circle (0.7pt);
    \draw[-{Latex[length=1mm,width=0.9mm]}] (#3tap) -- (#3gates.west |- #3tap);
  }
  \draw[->] (#3gates.east |- #3pre) -| (#3f);
  \draw[->] (#3f) |- (#3post);
  \foreach \dy in {-0.08, 0, 0.08}
    \draw[-{Latex[length=1mm,width=0.9mm]}] ([yshift=\dy cm]#3post.west) -- ([yshift=\dy cm]#3sum.east);
}
\hcsub{0}{attention}{a}
\hcsub{2.6}{FFN}{b}
\draw[decorate, decoration={brace, amplitude=3pt}] (3.35,5.2) -- (3.35,0) node[midway, right=2pt] {$\times L$};
\end{tikzpicture}}%
  \sbox{\figonespectra}{\includegraphics[height=\ht\figonearch]{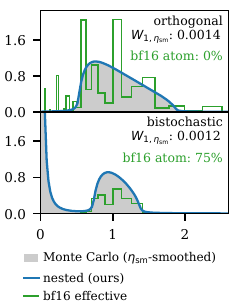}}%
  \resizebox{\columnwidth}{!}{\usebox{\figonearch}\hspace{6pt}\usebox{\figonespectra}}\\[3pt]
  {\small\setlength{\tabcolsep}{5pt}
  \begin{tabular}{lcc}
    \toprule
    & bistochastic & orthogonal \\
    \midrule
    ARC-AGI Pass@1 & 36.5\% & \textbf{40.5\%} \\
    ARC-AGI Pass@10 & 46.8\% & \textbf{53.1\%} \\
    ARC-AGI Pass@1000 & 56.1\% & \textbf{62.7\%} \\
    FineWeb-Edu val.\ loss & $\sensGatefairMhcVal${\scriptsize${}\pm\sensGatefairMhcSd$} & $\sensGatefairJgsigVal${\scriptsize${}\pm\sensGatefairJgsigSd$} \\
    \bottomrule
  \end{tabular}}
  \caption{\textbf{The hyper-connection stack and its Jacobian spectrum.}
  Left: each attention and MLP sublayer reads the streams through $H_{\text{pre}}$, writes to them through $H_{\text{post}}$ and mixes them by $H_{\text{res}}$~\eqref{eq:hc-general}.
  Right: eigenvalue density of $J^\top J$ at $S=10$ for an orthogonal (top) and a bistochastic (bottom) mixer~\eqref{eq:simplified_residual}, nested theory against Monte Carlo (Figure~\ref{fig:ov-vs-mc-panel}); green: the spectrum after bf16 gradient accumulation, whose atom at zero is the share of lost gradient directions (Appendix~\ref{app:effective-spectra}).
  Bottom: ARC-AGI (TRM, native gates, one run each, best checkpoint per metric; Section~\ref{sec:experiments}) and FineWeb-Edu (nano-GPT, $L=12$, sigmoid gates on both arms, mean $\pm$ std over $\sensSeedsMhcTwelve$ and $\sensTierJgsigSeeds$ seeds; Appendix~\ref{app:sensitivity}).}
  \label{fig:nested-l10}\label{fig:arch}
\end{figure}
Vanishing and exploding gradients have hampered the training of deep networks since their analysis in recurrent networks \citep{hochreiter1991untersuchungen,bengio1994learning,hochreiter1998vanishing,pascanu2013difficulty}, motivating gated architectures \citep{hochreiter1997long}, non-saturating activation functions \citep{nair2010rectified,glorot2011deep}, initialization schemes \citep{glorot2010understanding,he2015delving} and normalization \citep{ioffe2015batch}; \citet{schmidhuber2015deep} reviews this history.
Their common diagnostic is the singular spectrum of the input--output Jacobian, which also shapes the loss curvature \citep{li2018visualizing} and the geometry that natural gradient methods correct for \citep{amari1998natural}: singular values much smaller than one cause gradient vanishing and expressivity collapse, much larger ones gradient explosion and training instability.
Dynamical isometry \citep{saxe2014exact,schoenholz2017deep,pennington2017resurrecting} requires the spectrum to concentrate near unity, and skip-connections provide a universal mechanism for it, grounding the empirical success of ResNet \citep{he2016deep} and Transformer \citep{vaswani2017attention} architectures \citep{tarnowski2019dynamical, bachlechner2021rezero}.
\paragraph{The expressivity--stability tradeoff.}Skip-connections enable dynamical isometry but constrain expressivity.
For instance, identity skips give $x_L=x_0+\sum_\ell F_\ell(x_\ell)$, so the isometry $x_L=-x_0$ needs branches summing to $-2x_0$, far outside the small-branch regime in which skips guarantee dynamical isometry.
\paragraph{Hyper-connections.}Hyper-connections (HC) \citep{zhu2024hyperconnections} address this tension by replacing the scalar skip-connection with a learnable mixer acting across multiple streams.
The hidden state $x_\ell \in \mathbb{R}^n \otimes \mathbb{R}^p$ is viewed as $n$ parallel streams of dimension $p$, and $F_\ell:\mathbb{R}^p\to\mathbb{R}^p$ denotes the layer's residual branch (an attention or MLP block).
Three input-dependent linear maps govern the information flow: $H_{\text{res}}(x_\ell) \in \mathbb{R}^{n\times n}$ mixes streams, $H_{\text{pre}}(x_\ell) \in \mathbb{R}^{1\times n}$ projects onto a single stream for the nonlinear block, and $H_{\text{post}}(x_\ell) \in \mathbb{R}^{n\times 1}$ expands the result back to all streams.
The update rule reads
\begin{equation}\label{eq:hc-general}
  \begin{aligned}
    x_{\ell+1}={}&\underbrace{\bigl(H_{\text{res}}(x_\ell)\otimes I_{p}\bigr)\,x_\ell}_{\text{stream-wise mixing}}\\
    &+\underbrace{\bigl(H_{\text{post}}(x_\ell)\otimes I_{p}\bigr)}_{\text{expand}}\,\underbrace{F_\ell\bigl(\bigl(H_{\text{pre}}(x_\ell)\otimes I_{p}\bigr)\,x_\ell\bigr)}_{\text{project-apply}}.
  \end{aligned}
\end{equation}
For $n=1$ the architecture reduces to a rescaled skip-connection; for $n>1$ the mixer adds expressivity but loses the anchoring to the identity.
\paragraph{Manifold-constrained hyper-connections.}Manifold-constrained hyper-connections (mHC) \citep{xie2025mhc} restore stability by constraining $H_{\text{res}}$ to the set of bistochastic matrices, with strong results at 27B-parameter scale.
The geometric intuition is that bistochastic matrices are couplings between uniform marginals on finite sets \citep{villani2021topics}: they preserve the uniform marginal, ensuring no stream is systematically amplified or suppressed.
The deeper algebraic reason for this success is that bistochastic matrices form a \textbf{compact monoid}: a bounded set of matrices closed under composition.
\begin{theorem}\label{theo:bounded_gradient}
  Let $\mathcal M$ be a compact submonoid of $M_d(\mathbb R)$ and $C\coloneqq\sup_{M\in\mathcal M}\|M\|$.
  Consider $L$ residual layers $x_{\ell+1}=M_\ell x_\ell+aF_\ell(x_\ell)$ with $M_\ell\in\mathcal M$, $a>0$ and $K$-Lipschitz branches $F_\ell$.
  Then the skip product $M_{L-1}\cdots M_0$ has norm at most $C$ at every depth\lean[Caveat]{FreeProbReal.MonoidBound.norm_resJac_le}, and $\|dx_L/dx_0\|\le C(1+aCK)^L$; if $a\le a_0/L$, then $\|dx_L/dx_0\|\le C\,e^{a_0CK}$ for every $L$.
\end{theorem}
So $\log\|dx_L/dx_0\|=O(1)$ under the scaling $a=O(1/L)$.
Expanding the product of the layer Jacobians, the mixers between two branch factors multiply to an element of $\mathcal M$, so a term with $k$ branch factors has norm at most $C^{k+1}(aK)^k$.
\Cref{theo:bounded_gradient} applies to mHC with $\mathcal M=\mathrm{Bis}_n\otimes I_p$ ($\mathrm{Bis}_n$: the $n\times n$ bistochastic matrices) and extends to input-dependent mixers.
It controls explosion but says nothing about collapse toward zero, and products of bistochastic mixers do collapse onto the direction $\mathbf 1$.
\begin{theorem}[Bistochastic collapse]\label{thm:bisto-collapse}
  Let $A_1,A_2,\dots$ be bistochastic with all entries at least $a_{\min}>0$, and $B\coloneqq A_S\cdots A_1$ their stream map.
  Then $B\to\tfrac1n\mathbf 1\mathbf 1^\top$ as $S\to\infty$\lean[Formalised]{FreeProbReal.SpectralContraction.layerProd_tendsto}, contracting $\mathbf 1^\perp$ like $(1-2a_{\min})^S$\lean[Formalised]{FreeProbReal.SpectralContraction.layerProd_perp_contract}.
\end{theorem}
All singular values of $B$ but one therefore tend to $0$.
Yet if every $A_\ell$ fixes $\mathbf 1$ (unit row sums, e.g.\ bistochastic), the first moment of the Jacobian composed layer by layer (\emph{nested},~\eqref{eq:nested-m1}) never falls below $1/n$ at any depth: the direction $\mathbf 1$ is never contracted\lean[Formalised]{FreeProbReal.GaussianMoment.nestedM_perron_gauss}.
The product of the per-layer moments, which ignores the nesting, instead decays geometrically once the per-layer moment is below one\lean[Formalised]{FreeProbReal.FlatNested.scalar_collapse}.
\textbf{By \Cref{thm:bisto-collapse}, a deep mHC stack initialised near the identity collapses to a single-stream transformer once training has mixed its streams past a depth-dependent threshold.}
\paragraph{Open questions.}This analysis leaves three questions open:
\begin{enumerate}
  \item \emph{How do matrix skips compose the layer spectra?}
  \item \emph{Which compact monoid also prevents vanishing, and how can it be enforced and differentiated efficiently?}
  \item \emph{How much training can a bistochastic stack absorb before it collapses?}
\end{enumerate}
\paragraph{Contributions.}We answer them as follows.
\textbf{First}, we leverage operator-valued free probability \citep{voiculescu1995operations,dykema2006multilinear} to extend scalar spectral analysis \citep{tarnowski2019dynamical} to the nested spectral analysis of hyper-connection architectures.
Its solver matches Monte Carlo in all $\mnFigCells$ mixer$\times$depth cells, while Gram-based (flat) compositions miss for bistochastic and Gaussian mixers (\Cref{fig:ov-vs-mc-panel}).
Our library is written in Lean, which certifies the mathematics and the algorithms of this paper.
\textbf{Second}, as isometry at every depth excludes monoids with rank-deficient elements such as $\mathrm{Bis}_n$, and every compact group is conjugate into $O(n)$, we constrain the mixers to be orthogonal (JPmHC), computed by a fixed-point Cayley iteration \citep{li2020efficient,lezcano2019trivializations} and differentiated by standard autograd through its matrix products (Section~\ref{sec:cayley-method}).
\textbf{Third}, we differentiate the Sinkhorn projection of the bistochastic baseline implicitly: the closed-form gradient of \citet{eisenberger2022unified}, solved by a Gauss--Seidel iteration in fused Triton kernels, replaces the unrolled $20$-iteration autograd graph by $O(1)$ backward operations (Appendix~\ref{sec:sinkhorn}).
\textbf{Fourth}, we give in closed form the training budget a bistochastic mixer can spend before it collapses (Section~\ref{sec:beyond-init}).
It predicts, as measured, that out of the box TRM's bistochastic mixers are collapsed at initialisation and nano-GPT's stay stable, and predicts collapse for nano-GPT variants (some runs pending) and a DeepSeek-V3-scale model; the predictions hold so far.
Finally, when the spectrum is spread, finite-precision accumulation drops the gradient directions of its small singular values (Section~\ref{sec:effective-spectrum}).

\section{Neural Function Spectrum}\label{sec:spectrum}

The Jacobian of a deep network is a product of layer Jacobians, and its singular spectrum decides whether gradients explode, vanish or stay isometric.
With scalar skip-connections the spectrum of the product follows from the spectra of its factors; with the matrix skips of hyper-connections it does not, and the layers must be composed in a nested way.
We give the exact first moment of this nested composition, its law by operator-valued free probability, and what training does to it.

\paragraph{Model.}
A hyper-connection block (\Cref{fig:arch}) acts on $n$ streams of width $p$, $x\in\mathbb R^{n\times p}\cong\mathbb R^N$, $N\coloneqq np$, as $x\mapsto H_{\text{res}}x+u\,F(v^\top x)^\top$, with the write vector $u\coloneqq H_{\text{post}}$ and the read vector $v\coloneqq H_{\text{pre}}^\top$ of~\eqref{eq:hc-general}; all three gates depend on $x$.
Our $n\times n$ gates act only through the vectors $u=c\,H_{\text{post}}\mathbf 1$, with the write scale $c$, and $v=H_{\text{pre}}^\top\mathbf 1/n$\lean[Formalised]{FreeProbReal.GatedModel.combine_eq} (Appendix~\ref{sec:cayley}).
We want the law of the eigenvalues of $J^\top J$, the squared singular values of the end-to-end Jacobian $J$, as $p\to\infty$ with $n$ fixed.
We measure depth by the number $S$ of mixers that a forward pass applies.
It differs from the depth $L$ of the architecture, since a transformer layer applies two mixers and TRM reuses its layers; Sections~\ref{sec:gated-theory} and~\ref{sec:beyond-init} give $S$ in terms of $L$ for both.
We first take the \emph{ungated} block, a mixer $A_\ell\in\mathbb{R}^{n\times n}$ that does not depend on $x$ and an MLP branch,
\begin{equation}\label{eq:simplified_residual}
  x_{\ell+1} = (A_\ell\otimes I_p)\, x_\ell + \phi(W_\ell\, x_\ell),
\end{equation}
with the activation $\phi$ and $W_\ell \sim \mathcal{N}(0, \sigma_w^2 / N)$.
Its layer Jacobian is $Y_\ell \coloneqq A_\ell\otimes I_p + D_\ell W_\ell$ with $D_\ell \coloneqq \mathrm{diag}(\phi'(W_\ell x_\ell))$; each layer applies one mixer, so $S$ layers give $J = Y_S\cdots Y_1$.
As usual in mean-field analyses~\citep{pennington2017resurrecting,tarnowski2019dynamical} we replace $D_\ell W_\ell$ by isotropic Gaussian noise $\sigma_\ell X_\ell$ of the same second moment, $X_\ell$ with i.i.d.\ $\mathcal N(0,1/N)$ entries and $\sigma_\ell^2\coloneqq\sigma_w^2\psi_\ell$ with the activation gain $\psi_\ell\coloneqq\mathbb{E}[\phi'(h)^2]$ over the pre-activations $h$ of the layer (Appendix~\ref{sec:signal-prop}); \Cref{thm:gated} keeps the mask.

\paragraph{Nested, not flat.}
Free-probability analyses of deep networks compose layers through the per-layer Gram matrices $Y_\ell^\top Y_\ell$: for scalar skips the $Y_\ell$ are asymptotically free and $\mu_{J^\top J}=\boxtimes_\ell\,\mu_{Y_\ell^\top Y_\ell}$, with $\mu_M$ the eigenvalue distribution of $M$ and $\boxtimes$ the free multiplicative convolution~\citep{pennington2017resurrecting,tarnowski2019dynamical}.
With matrix skips $A_\ell\otimes I_p$ this composition is no longer valid, already at the level of finite matrices\lean[Formalised]{FreeProbReal.FlatNested.hgud1}: $J^\top J$ is in general not isospectral to $\prod_\ell Y_\ell^\top Y_\ell$.
The first moment already tells the two apart.
It sees each layer through the \emph{covariance map} of its noise $Y_{\text{br},\ell}\coloneqq Y_\ell-A_\ell\otimes I_p$, $\eta_\ell(b)\otimes I_p\coloneqq\mathbb E[Y_{\text{br},\ell}^\top(b\otimes I_p)Y_{\text{br},\ell}]$, which for isotropic noise is $\eta_\ell(b)=\sigma_\ell^2\,\tfrac{\operatorname{tr}b}{n}\,I_n$.
Define the completely positive maps $\mathcal T_\ell(b)\coloneqq A_\ell^\top b\,A_\ell+\eta_\ell(b)$ on $M_n(\mathbb{R})$ and the per-layer matrices $\Phi_\ell\coloneqq\mathcal T_\ell(I_n)=A_\ell^\top A_\ell+\sigma_\ell^2 I_n$.
Then the nested first moment $m_1$ is, exactly and at every finite width $p$,
\begin{equation}\label{eq:nested-m1}
  m_1\coloneqq\mathbb{E}\,\tfrac1N\operatorname{tr}(J^\top J)=\tfrac1n\operatorname{tr}\bigl(\mathcal T_1\circ\cdots\circ \mathcal T_S(I_n)\bigr),
\end{equation}
for i.i.d.\ Gaussian noise independent across layers\lean[Formalised]{FreeProbReal.GaussianMoment.gauss_m1}, whereas Gram-based compositions, such as the operator-valued S-transform, produce $\tfrac1n\operatorname{tr}(\Phi_1\cdots\Phi_S)$ (flat) or the scalar product $\prod_\ell \tfrac1n\operatorname{tr}\Phi_\ell$.
The three coincide when every mixer is co-isometric, $A_\ell^\top A_\ell\propto I_n$: then $\mathcal T_\ell$ maps scalar matrices to scalar matrices\lean[Formalised]{FreeProbReal.FlatNested.layerT_smul_one} and the nested and flat moments agree at every depth\lean[Formalised]{FreeProbReal.GaussianMoment.iso_of_coisometric_gauss}.

\paragraph{Other branches.}
The nested moment needs only each layer's covariance map.
If the noise of a layer is a sum of routed pieces $\sum_mR_m\otimes X_m$ (routings $R_m$) with $\mathbb E[(X_m)_{ij}(X_{m'})_{kl}]=\Theta_{mm'}\delta_{ik}\delta_{jl}/p$, \eqref{eq:nested-m1} holds at every width with $\eta_\ell(b)=\sum_{m,m'}\Theta_{mm'}R_m^\top bR_{m'}$\lean[Formalised]{FreeProbReal.TokenFFN.integral_jac_gram_corr}, e.g.\ for attention heads or the tokens of a token-wise MLP (Appendix~\ref{app:nested-branches}).

\subsection{Operator-valued free probability}\label{sec:ov-free-prob}

The full law follows from operator-valued free probability over $\mathcal{B}\coloneqq M_n(\mathbb{C})$~\citep{voiculescu1995operations,speicher1998combinatorial}, with the conditional expectation $E_{\mathcal B}\coloneqq\mathrm{id}_n\otimes\operatorname{tr}_p$, $\operatorname{tr}_p\coloneqq\tfrac1p\operatorname{tr}$: the noise of each layer is free over $\mathcal B$ from the mixers and from the other layers (Appendix~\ref{app:pencil}).
The $\mathcal B$-valued Cauchy transform of $J^\top J$ is, for $b\in M_n(\mathbb C)$ with $\operatorname{Im}b>0$,
\begin{equation}\label{eq:ov-cauchy}
  G_{J^\top J}(b)\coloneqq\lim_{p\to\infty}E_{\mathcal B}\bigl[(b\otimes I_p-J^\top J)^{-1}\bigr],
\end{equation}
and the density of the limiting law of $J^\top J$ is recovered by Stieltjes inversion of its trace,
\begin{equation}\label{eq:stieltjes}
  \rho(\lambda)=-\frac1\pi\lim_{\varepsilon\to0^+}\operatorname{Im}\tfrac1n\operatorname{tr}G_{J^\top J}\bigl((\lambda+i\varepsilon)I_n\bigr),\qquad\lambda\in\mathbb R.
\end{equation}
For one layer, $G_{J^\top J}$ solves a subordination equation in $A^\top A$ alone\lean[Formalised]{FreeProbReal.OVDyson.ovSolver_gram_eq}, so every orthogonal mixer gives the law of the identity skip\lean[Formalised]{FreeProbReal.OVDyson.ovSolver_orthogonal}.
For $S\ge2$ the linearisation of the nested product~\citep{belinschi2017analytic} makes $G_{J^\top J}(zI_n)$, $\operatorname{Im}z>0$, the corner block\lean[Formalised]{FreeProbReal.OVLin.corner_pencil_inv} of the solution of a fixed-point (Dyson) equation on $2S\times2S$ blocks over $\mathcal B$, $G=\mathcal F_z(G)\coloneqq\bigl(\mathcal Z(z)-\mathcal P_0-\Sigma(G)\bigr)^{-1}$.
Here $\mathcal Z(z)\coloneqq\mathrm{diag}(zI_n,0,\ldots,0)$, $\mathcal P_0$ is the constant part of the linearisation (the mixers $A_\ell$ and the identity couplings between consecutive layers), and the self-energy $\Sigma(G)$ is block-diagonal: on each of the two blocks of layer $\ell$ it is $\eta_\ell$ of the other diagonal block of $G$ (Appendix~\ref{app:pencil}).
The equation has no closed form; we iterate $G\leftarrow\tfrac12\bigl(G+\mathcal F_z(G)\bigr)$~\citep{helton2007operator} with continuation in $\operatorname{Im}z$ (Appendix~\ref{app:hrfs}).
The first moment of its solution is the nested moment~\eqref{eq:nested-m1}\lean[Caveat]{FreeProbReal.OVLin.tendsto_firstMoment}, and for co-isometric mixers the equation reduces to the scalar free product of the Gram-based compositions\lean[Caveat]{FreeProbReal.OVLin.isSol_coiso_iff}.

Training moves the mixers away from the identity; to first order, the law sees only the trace of the move.
\begin{theorem}[First variation of the Cauchy transform]\label{thm:law-response}
  Let every mixer be the identity except $A_k=I_n+\epsilon\mathcal X$, $\mathcal X\in M_n(\mathbb R)$, with independent Gaussian noise, and let $\operatorname{Im}z>0$.
  \begin{enumerate}[label=(\roman*),leftmargin=*,itemsep=0pt,topsep=1pt]
    \item For one layer with noise variance $\sigma^2$, let $G_s(z)\coloneqq\tfrac1n\operatorname{tr}G_{J^\top J}(zI_n)$, $\vartheta\coloneqq1-\sigma^2G_s(z)$ and $\zeta\coloneqq z\vartheta^2$ at $\epsilon=0$, with $\zeta\ne1$ and $(\zeta-1)^2\ne\sigma^2(\zeta+1)$.
      Then
      \begin{equation*}
        \frac{d}{d\epsilon}\Big|_{\epsilon=0}G_s(z)=\frac{2\vartheta\operatorname{tr}\mathcal X}{n\bigl((\zeta-1)^2-\sigma^2(\zeta+1)\bigr)}\lean[Caveat]{FreeProbReal.LawResponse.law_response}.
      \end{equation*}
    \item At every depth and width, $\frac{d}{d\epsilon}\big|_{\epsilon=0}\mathbb E\,\tfrac1N\operatorname{tr}(zI_N-J^\top J)^{-1}$ is $\operatorname{tr}\mathcal X$ times a function of $z$ that does not depend on $\mathcal X$.
  \end{enumerate}
\end{theorem}
Moves of equal trace therefore have equal response\lean[Caveat]{FreeProbReal.LawResponse.law_response_trace_only}.
Skew-symmetric moves, the tangents of $O(n)$, have none\lean[Caveat]{FreeProbReal.LawResponse.law_response_skew}, while a bistochastic tangent ($\mathcal X\mathbf 1=0$) moves the law through its off-diagonal mass alone, $\operatorname{tr}\mathcal X=-m_{\text{off}}$\lean[Caveat]{FreeProbReal.LawResponse.law_response_ds}: at the identity the spectrum is blind, to first order, to orthogonal training.
The identity base point is the initialisation of both arms; away from it the response depends on more than the trace.
Part (ii) follows from the $O(n)$-invariance of the Gaussian noise; the proof and a Monte-Carlo check are in Appendix~\ref{app:sensitivity}.

\subsection{The gated block}\label{sec:gated-theory}
With its gates held at their values, the hyper-connection block $\ell$ has the \emph{frozen-gate} Jacobian $Y_\ell=A_\ell\otimes I_p+\sum_mR_{\ell m}\otimes X_{\ell m}$, $A_\ell$ the frozen $H_{\text{res},\ell}$: the gates route the branch ($R=u_\ell v_\ell^\top$, $X=D_\ell W_\ell$ with $W_\ell\in\mathbb R^{p\times p}$ for an MLP; $R=\mathcal U\mathcal A_m\mathcal V^\top$ per head of attention pattern $\mathcal A_m$ on a sequence of $T$ tokens, $\mathcal U\coloneqq\bigoplus_{t=1}^Tu_t$, $\mathcal V\coloneqq\bigoplus_{t=1}^Tv_t$ with $u_t,v_t$ the write and read vectors of token $t$).
Let $J_{\text s}\coloneqq Y_S\cdots Y_1$, one factor per block; a transformer wraps the attention and the MLP sublayer of each of its $L$ layers in one block each, so $S=2L$.
The full Jacobian $J_{\text f}$ also differentiates the gates.
\begin{theorem}[Gated stacks]\label{thm:gated}
  Let the gates be frozen and the blocks independent, with centred pieces of cross-covariance $\Theta_\ell$ on $n_q$ streams ($n_q=n$, or $Tn$ on $T$ tokens), $N=n_qp$.
  \begin{enumerate}[label=(\roman*),leftmargin=*,itemsep=0pt,topsep=1pt]
    \item At every width $p$, $\mathbb E[J_{\text s}^\top J_{\text s}]=\mathcal T_1\circ\cdots\circ\mathcal T_S(I_{n_q})\otimes I_p$, with $\mathcal T_\ell$ of~\eqref{eq:nested-m1} and $\eta_\ell$ the covariance map of block $\ell$\lean[Formalised]{FreeProbReal.TokenFFN.integral_jac_gram_corr}.
      For an MLP branch it has rank one, $\eta_\ell(b)=\sigma_\ell^2\,(u_\ell^\top b\,u_\ell)\,v_\ell v_\ell^\top$\lean[Conditional]{FreeProbReal.GatedModel.integral_jac_gram_masked}.
    \item If every $A_\ell$ fixes $\mathbf 1$ and the pieces are uncorrelated, $J_{\text s}$ has $m_1\ge1/n_q$ at every depth\lean[Formalised]{FreeProbReal.RoutedBranch.nestedM_perron_kraus}.
    \item If the gates of block $\ell$ read $x$ through $r_{\text g,\ell}$ logits, then for every weight draw the eigenvalue distribution functions of $J_{\text f}^\top J_{\text f}$ and $J_{\text s}^\top J_{\text s}$ differ by at most $2\sum_\ell r_{\text g,\ell}/N$\lean[Formalised]{FreeProbReal.GateRank.ks_full_static}.
  \end{enumerate}
\end{theorem}
The full statement and its proof are in Appendix~\ref{app:nested-branches}; the proof uses only second moments for (i), with a mask independent of the weights for the MLP and one term per head for attention with frozen patterns\lean[Caveat]{FreeProbReal.SeqAttention.integral_jac_gram_attn}, the positivity of $\eta_\ell$ for (ii) and the rank of the gate term for (iii).
Three consequences follow.
\begin{itemize}
  \item \emph{The gated noise is not isotropic.}
    For uniform gates $uv^\top=\mathbf 1\mathbf 1^\top/n$, so the noise enters only along $\mathbf 1$; one identity layer gives $m_1=1+\sigma^2/n$ instead of $1+\sigma^2$ for isotropic noise\lean[Formalised]{FreeProbReal.GatedModel.etaGate_ne_etaIso}.
  \item \emph{With uniform gates, orthogonal mixers fixing $\mathbf 1$ stay isometric on $\mathbf 1^\perp$:} $m_1=1+\bigl(\prod_\ell(1+\sigma_\ell^2)-1\bigr)/n$\lean[Caveat]{FreeProbReal.GatedModel.m1_uniform_perron}.
    The Cayley mixer of Section~\ref{sec:cayley-method} fixes $\mathbf 1$ only if its skew generator has $\Omega\mathbf 1=0$, e.g.\ $Q\Omega Q$ with $Q\coloneqq I-\mathbf 1\mathbf 1^\top/n$.
  \item \emph{Nested $\neq$ flat even for orthogonal mixers.}
    An orthogonal mixer that moves $\mathbf 1$ already separates the nested and flat moments at $S=3$\lean[Caveat]{FreeProbReal.GatedModel.gated_rot_nested_ne_flat}: the co-isometric criterion needs $\eta_\ell(I)\propto I$, which fails for rank-one noise.
\end{itemize}

\subsection{Validation}\label{sec:validation}

We compare each theory with Monte Carlo by the 1-Wasserstein distance $W_{1,\eta_{\text{sm}}}$ between its density and the density of the sampled eigenvalues of $J^\top J$, both smoothed at the same width $\eta_{\text{sm}}$, under an acceptance rule fixed in advance (Appendix~\ref{app:mc}).
As predicted, the Gram-based composition agrees with Monte Carlo only for the co-isometric (identity, orthogonal) mixers and misses for the bistochastic and normalised-Gaussian ones, while the nested theory passes in all $\mnFigCells$ mixer$\times$depth cells ($W_{1,\eta_{\text{sm}}}\le\mnNestedWoneMax$; \Cref{fig:nested-l10}), including the mass that bistochastic mixers collapse at $0$.
With frozen gates, the autograd Jacobian of a PyTorch block agrees with the gated theory in all $\mnGatedCells$ cells, where isotropic noise of the same variance is rejected; live gates add a term of low rank (\Cref{thm:gated}(iii)) that carries most of the excess first moment of the orthogonal block, so the theory describes the mixer and the branch, not the gate dynamics (Appendix~\ref{app:mc}).

\subsection{Training budget before collapse}\label{sec:beyond-init}
Training moves the mixers away from their initialisation, usually the identity, so the question becomes: \emph{for a given architecture and geometry, what training budget can be spent before collapse is expected?}
We answer it for static mixers, whose logits do not depend on the input, and bound the budget of per-token mixers, whose logits do, by theirs (proofs in Appendix~\ref{app:sensitivity}).

\paragraph{Collapse exponent.}
The stream map $B$ of \Cref{thm:bisto-collapse}, the product of the $S$ mixers of a forward pass, has the collapse exponent
\begin{equation}\label{eq:collapse-exp}
  \Lambda(B)\coloneqq-\log\sigma_{\min}\bigl(B|_{\mathbf 1^\perp}\bigr),
\end{equation}
$\sigma_{\min}$ the smallest singular value, and the $S$ applications are \emph{stable} if $\Lambda<1$ (one $e$-fold, a convention).
Training with learning rate $\gamma_j$ at step $j$ spends the budget $\gamma_\Sigma\coloneqq\sum_j\gamma_j$; TRM has $S=84$ per supervision step (Section~\ref{sec:experiments}).

\paragraph{Sinkhorn mixers.}
For mHC's bistochastic mixer $A=\mathrm{Sinkhorn}(\exp\tilde H_{\text{res}})$ (Appendix~\ref{sec:sinkhorn}) with logits $\tilde H_{\text{res}}\in\mathbb R^{n\times n}$, the contraction is set by the \emph{gap}, the mean diagonal minus the mean off-diagonal logit,
\begin{equation}\label{eq:gap}
  \Delta(\tilde H_{\text{res}})\coloneqq\frac1n\sum_{i}(\tilde H_{\text{res}})_{ii}-\frac1{n(n-1)}\sum_{i\neq j}(\tilde H_{\text{res}})_{ij}.
\end{equation}
Its exponent per application and the collapse gap of $S$ applications are
\begin{equation}\label{eq:gap-star}
  \begin{gathered}
    c_{\text{Sk}}(\Delta)\coloneqq-\log\Bigl(1-\frac n{e^\Delta+n-1}\Bigr)\approx ne^{-\Delta},\\
    \Delta^\ast(S)\coloneqq\log\Bigl(\frac n{1-e^{-1/S}}-(n-1)\Bigr)\approx\log(nS).
  \end{gathered}
\end{equation}
On the gap, the optimiser exerts a drive $\mathsf d$ toward mixing per unit budget, and decoupled weight decay at rate $\lambda_{\text{wd}}$ shrinks the logits (\Cref{thm:sens-adam,cor:sens-budget} give $\mathsf d$ for Adam and AdamAtan2).

\paragraph{Stability criteria.}
For mixers of gap $\Delta$, and for a training run of budget $\gamma_\Sigma$ started at gap $\Delta_0$, we use the two criteria
\begin{equation}\label{eq:criteria}
  \text{(C1)}\quad \Delta>\Delta^\ast(S),\qquad\qquad \text{(C2)}\quad \gamma_\Sigma<\gamma^\ast_\Sigma,
\end{equation}
with the \emph{budget before collapse} $\gamma^\ast_\Sigma$, the budget at which the gap reaches $\Delta^\ast\coloneqq\Delta^\ast(S)$ ($0$ if $\Delta_0\le\Delta^\ast$, $\infty$ if it never does).

\begin{theorem}[Stability criteria]\label{thm:budget}
  \begin{enumerate}[label=(\roman*),leftmargin=*,itemsep=0pt,topsep=1pt]
    \item If $\tilde H_{\text{res}}-\Delta I$ is a multiple of $\mathbf 1\mathbf 1^\top$, Sinkhorn is exact in one pass and $S$ applications have $\Lambda=S\,c_{\text{Sk}}(\Delta)$.
      They are stable if and only if (C1) holds\lean[Formalised]{FreeProbReal.SensCriteria.decay_stable_iff}\lean[Formalised]{FreeProbReal.SensExponent.gapStar_spec}.
    \item If the $S$ mixers share one gap that keeps the form (i), and the optimiser moves it at a constant drive $\mathsf d$, the mixers end training stable if and only if (C2) holds\lean[Caveat]{FreeProbReal.SensCriteria.adam_unstable_iff_depth}\lean[Caveat]{FreeProbReal.SensCriteria.decay_gap_solution}.
      Without $\varepsilon_{\text{A}}$ and for $\Delta_0>\Delta^\ast$,
      \begin{equation}\label{eq:budget}
        \begin{gathered}
          \gamma^\ast_\Sigma=\frac{\Delta_0-\Delta^\ast}{\mathsf d}\quad(\lambda_{\text{wd}}=0),\\
          \gamma^\ast_\Sigma=\frac1{\lambda_{\text{wd}}}\log\frac{\Delta_0+\mathsf d/\lambda_{\text{wd}}}{\Delta^\ast+\mathsf d/\lambda_{\text{wd}}}\quad(\lambda_{\text{wd}}>0),
        \end{gathered}
      \end{equation}
      and $\gamma^\ast_\Sigma=\infty$ when the denominator is not positive.
    \item (Per-token mixers.)
      Let the logits of application $k$ depend on the input $x$, of the form (i) with gaps $\Delta_k(x)>0$, and let $\bar\Delta$ be their mean over the $S$ applications and the inputs.
      Then the input-averaged exponent satisfies $\mathbb E_x\Lambda(x)\ge S\,c_{\text{Sk}}(\bar\Delta)$\lean[Formalised]{FreeProbReal.SensExponent.sinkSym_mulVec_perp}\lean[Formalised]{FreeProbReal.SensDiscrete.perLayer_jensen}.
      If $\bar\Delta$ moves as in (ii), the per-token mixers collapse no later than the static ones: $\gamma^\ast_\Sigma$ is an upper bound of their budget.
  \end{enumerate}
\end{theorem}
Part (i) is \Cref{prop:sens-additive}; the dynamics behind (ii), with Adam's $\varepsilon_{\text{A}}$ and decoupled weight decay, are \Cref{cor:sens-budget}, and (iii) is \Cref{cor:sens-gated}.
The budget is infinite when there is no drive toward mixing without decay, or when decay holds the gap at an equilibrium at or above $\Delta^\ast$.
The gap to cover is linear in $\Delta_0$, while the collapse gap grows only like $\log(nS)$.
Adam's $\varepsilon_{\text{A}}$ adds to the undecayed numerator of~\eqref{eq:budget} a barrier that grows like $e^{\Delta_0}$ (\Cref{cor:sens-budget}); at nano-GPT's initial gap this barrier, not the distance $\Delta_0-\Delta^\ast$, keeps the mixers stable.

\paragraph{The assumptions.}
All initialisations we consider have the form (i): nano-GPT's $\tau_{\text{init}}(2I-\mathbf 1\mathbf 1^\top)$ ($\Delta=2\tau_{\text{init}}$), TRM's $b_{\text{res}}=I$ ($\Delta=1$) and DeepSeek's off-diagonal $-8$ ($\Delta=8$).
Unequal per-layer gaps only add to the exponent, since $c_{\text{Sk}}$ is convex, and by (iii) per-token logits only bring collapse forward: a predicted collapse is robust, while a predicted stability assumes equal static gaps.
The drive model of (ii) reduces the optimiser to a step map on the gap and is not derived from Adam's moment estimates; the criteria survive discrete steps and a step-varying drive (\Cref{prop:sens-discrete}), and \Cref{sec:sens-budget-values} lists the other caveats.
Orthogonal mixers are outside the criteria: the exact Cayley map (Section~\ref{sec:cayley-method}) has $\Lambda\equiv0$ and $\gamma^\ast_\Sigma=\infty$, and the deployed two-step map contracts only at fourth order in its rotation rate (\Cref{prop:sens-cayley2}).

\paragraph{Values.}
\Cref{tab:budget} evaluates (C2) for our recipes and for DeepSeek-V3 with mHC.
It predicts stability for nano-GPT out of the box, and collapse for TRM, for some nano-GPT variants (weight decay on the mixer logits) and for DeepSeek-V3 during pretraining.
TRM's bistochastic ARC arm starts collapsed, at $b_{\text{res}}=I$ ($\Delta_0=1$), and its Sudoku recipe ($\tau_{\text{init}}=4$, weight decay $1$) is predicted to collapse during training (B3).
Part of these runs are pending; the predictions for bistochastic mixers checked so far hold (\Cref{sec:sens-budget-values} gives the measurements, the drives and the DeepSeek-V3 inputs).

\begin{table*}[t]
  \centering\footnotesize\setlength{\tabcolsep}{1.9pt}
  \caption{Training budget before collapse $\gamma^\ast_\Sigma$ (with Adam's barrier, \Cref{cor:sens-budget}) against the budget $\gamma_\Sigma$ of each recipe: $S$ mixer applications per forward pass (per supervision step for TRM), initial gap $\Delta_0$, collapse gap $\Delta^\ast(S)$~\eqref{eq:gap-star}; ``decay'' is $\lambda_{\text{wd}}$ on the mixer logits.
  Ranges span the drives $\mathsf d$ measured on nano-GPT (\Cref{sec:sens-budget-values}); the expected regime follows from (C2), as a share of $\gamma_\Sigma$ where collapse occurs during training.
  Per-token logits: \Cref{thm:budget}(iii) applied to the nano-GPT rows ($L=12$ / $24$), at their mean gap $\bar\Delta_0$; the median per-token exponent is at least the static one at all $\sensBudGatedCkpts$ checkpoints of $\sensBudGatedRuns$ runs.
  TRM's two budgets are Sudoku / ARC-AGI. $^\dagger$Read post hoc. $^\ddagger$Registered before its runs (Appendix~\ref{sec:sens-ledger}), not yet run.}\label{tab:budget}
  \begin{tabular}{@{}lccccccl@{}}
    \toprule
    setting & $S$ & $\Delta_0$ & $\Delta^\ast(S)$ & $\gamma_\Sigma$ & $\gamma^\ast_\Sigma$ & expected regime & validation \\
    \midrule
    \sensBudRows
    \bottomrule
  \end{tabular}
\end{table*}

\subsection{Effective spectrum}\label{sec:effective-spectrum}
\begin{figure}[t]
  \centering
  \includegraphics[width=\columnwidth]{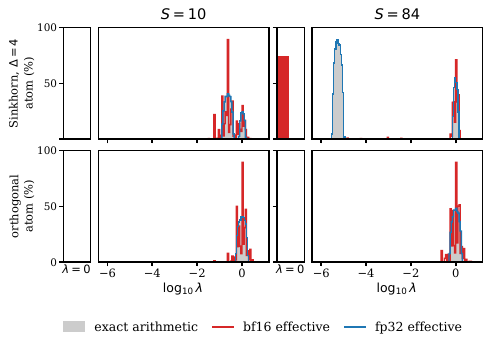}
  \caption{Exact and effective spectra (law of $\log_{10}\lambda$, $\lambda=\varsigma^2$ for the singular values $\varsigma$ of $J$) and the atom at zero, for a Sinkhorn mixer of gap $\Delta=4$ and an orthogonal mixer at $S=10$ and $S=84$ applications (the model of \Cref{fig:nested-l10}).
  Effective: what each sample changes in a gradient summed over $256$ samples in bf16 or fp32.
  At $\Delta=4$, (C1) holds at $S=10$ and fails at $S=84$~\eqref{eq:gap-star}; at $S=84$ the Sinkhorn product also contracts $\mathbf 1^\perp$ below bf16's precision $2^{-8}$, so bf16 rounds those samples to zero, but not below fp32's $2^{-24}$.}\label{fig:effective-main}
\end{figure}
The spectra above hold in exact arithmetic, but training sees them through gradients summed in floating point.
The gradient of a weight is a sum over the samples of a batch, and in the sampling model of Appendix~\ref{app:effective-spectra} each sample enters it through $J$, scaled by one of its singular values.
A floating-point sum keeps a fixed number of significant digits relative to its current value, so an addend much smaller than the running sum falls below its last digit and is rounded off\lean[Formalised]{FreeProbReal.FloatCutoff.step_band}: the sum does not change.
A sample scaled by a small singular value, summed with samples scaled by large ones, therefore never reaches the optimiser.
What reaches it, the \emph{effective spectrum}, is the exact spectrum with its small values moved into an atom at zero.
The cutoff is relative, so the spread of the spectrum costs, not its scale: orthogonal mixers, whose singular values stay together, lose almost nothing, while a mixer that contracts some directions loses those directions once the contraction exceeds the precision of the format (\Cref{fig:effective-main}).
On trained nano-GPT models, whose mixers are far from collapse, the bf16 storage of the gradient, not the mixer, sets the lost share (Appendix~\ref{sec:fc-measured}).

\section{Orthogonal Mixers}\label{sec:cayley-method}

\Cref{theo:bounded_gradient} bounds the gradients for every compact monoid of mixers, but a monoid with rank-deficient elements, such as stochastic matrices or Stiefel embeddings other than $O(n)$ itself, cannot give dynamical isometry at every depth.
Bistochastic matrices fail in this way: positive ones have every eigenvalue other than the eigenvalue $1$ of $\mathbf 1$ strictly inside the unit disk\lean[Formalised]{FreeProbReal.SpectralContraction.posDS_eigenvalue_norm_lt_one}, and their products converge to the rank-one projector $\tfrac1n\mathbf 1\mathbf 1^\top$ (\Cref{thm:bisto-collapse}).
A compact monoid of invertible matrices is a group, and every compact subgroup of $\mathrm{GL}_n(\mathbb R)$ is conjugate to a subgroup of $O(n)$: up to a change of basis, the orthogonal group is the largest admissible family, and products of orthogonal mixers stay orthogonal and contract nothing\lean[Formalised]{FreeProbReal.SpectralContraction.orth_layerProd_mem}.
Since $H_{\text{res}}$ is recomputed for every token from logits $\tilde H_{\text{res}}\in\mathbb R^{n\times n}$, the map from logits to $O(n)$ must also be cheap at batch scale and differentiable.

\paragraph{Cayley transform.}
For a skew matrix $\Omega$ and a step $\alpha>0$, $\Cay(\Omega)\coloneqq(I-\tfrac{\alpha}{2}\Omega)^{-1}(I+\tfrac{\alpha}{2}\Omega)$ is a rotation, and its image is the set of rotations without eigenvalue $-1$, a dense open subset of $SO(n)$: an eigenvalue $i\omega$ of $\Omega$ becomes $(1+i\alpha\omega/2)/(1-i\alpha\omega/2)\neq-1$.
Following \citet{li2020efficient}, we avoid the inverse with the fixed-point iteration $\Cay_0(\Omega)\coloneqq I+\alpha \Omega$, $\Cay_{s+1}(\Omega)\coloneqq I+\tfrac{\alpha}{2}\Omega\,(I+\Cay_s(\Omega))$ (\Cref{alg:cayley}, Appendix~\ref{sec:cayley}), whose $s$ matrix products autograd differentiates directly; it converges to $\Cay(\Omega)$ when $r\coloneqq\|\tfrac{\alpha}{2}\Omega\|_2<1$\lean[Formalised]{FreeProbReal.CayleySpec.cayleyIter_tendsto}.
The iteration comes with an effective bound on its orthogonality defect $\|\Cay_s^\top\Cay_s-I\|_2$ in terms of $r$ and $s$, which fixes the number of iterations for a given tolerance\lean[Formalised]{FreeProbReal.CayleyTolerance.cayleyIter_exec_orth_defect_le_of_tolS} (Appendix~\ref{sec:cayley-tolerance}); the two-step map ($s=2$) of our experiments contracts by a computable amount (Section~\ref{sec:beyond-init}).

\paragraph{The JPmHC mixer.}
We set $H_{\text{res}}\coloneqq\Cay_s(\Omega)$ with $\Omega\coloneqq\tilde H_{\text{res}}-\tilde H_{\text{res}}^\top$; the three gates of~\eqref{eq:hc-general} come from one bias-free projection of the normalised input streams, with softmax read and write gates (\eqref{eq:fused-proj}, Appendix~\ref{sec:cayley}).
The skip path is then isometric but the block is not: the branch adds a term of rank one in stream space (Section~\ref{sec:gated-theory}).
Within the tolerance of the defect bound the mixers are co-isometric, so for non-gated mixers the nested and Gram-based compositions of Section~\ref{sec:spectrum} coincide; with gated branches they differ unless $\Omega\mathbf 1=0$ (Section~\ref{sec:gated-theory}).

\section{Experiments}
\label{sec:experiments}

Section~\ref{sec:beyond-init} predicts collapse for TRM and stability for nano-GPT as both are configured out of the box (\Cref{tab:budget}).
TRM's bistochastic arm starts from the logit bias $b_{\text{res}}=I$ of $\tilde H_{\text{res}}$, gap $\Delta_0=1$, below the collapse gap $\Delta^\ast(84)=\sensBudGapStarTRM$ of its $S=84$ mixers per supervision step: it is collapsed at initialisation, with a measured exponent $\Lambda=\sensTRMInitSinkhorn$ over one supervision step, against $0$ for the orthogonal arm.
In nano-GPT the bistochastic arm starts at gap $16$ and stays stable (measured $\Lambda\le\sensBudAHLambdaEnd$), so no orthogonal advantage is predicted there.
Some variants are predicted to collapse during training: nano-GPT with mixer weight decay, and TRM with its Sudoku recipe ($\tau_{\text{init}}=4$, weight decay $1$); part of these runs is pending, and the collapse predictions tested so far hold.

\subsection{Tasks}

We evaluate on the Abstraction and Reasoning Corpus (ARC-AGI)~\citep{chollet2019arc}: from a few demonstration pairs of grids, the solver must infer a latent rule and produce the exact output grid of a test input (\Cref{fig:arc-examples} in Appendix~\ref{app:training-details}).
Following TRM, we train on ARC-AGI-1 and ConceptARC with augmentations and hold out the 419 test inputs of the evaluation tasks.
The language-model testbed is nano-GPT on FineWeb-Edu (Appendix~\ref{sec:sens-setting}).

\subsection{Model and Training}

We expand each block of the \textbf{Tiny Recursive Model (TRM)}~\citep{jolicoeur2025trm}, a 7M-parameter recursive transformer, to $n=4$ streams of width $p=512$.
Its $L=2$ weight-tied layers run $H(L_c+1)=21$ times per supervision step ($H=3$ outer cycles of $L_c=6$ inner updates and one outer update), so a supervision step applies $S=2LH(L_c+1)=84$ mixers; the gradient flows through the last $L_c+1=7$ runs, a suffix of $2L(L_c+1)=28$ mixers.
Adaptive Computation Time (ACT) halting allows up to 16 supervision steps.
Both arms share one recipe (AdamAtan2~\citep{everett2024scaling}, weight decay $0.1$, bfloat16; Appendix~\ref{app:training-details}).
The orthogonal arm is the mixer of Section~\ref{sec:cayley-method} with $s=2$; the bistochastic arm follows mHC~\citep{xie2025mhc} (RMSNorm, sigmoid gates), with a Sinkhorn backward pass that was unrolled early in the run and implicit~\citep{eisenberger2022unified} later (Appendix~\ref{sec:sinkhorn-run}).

\subsection{Results}
\label{sec:results}

\Cref{fig:nested-l10} (bottom) reports both testbeds.
On FineWeb-Edu, with the sigmoid gates of mHC on both arms, the orthogonal arm does not beat the bistochastic one, as predicted.
On ARC-AGI the orthogonal arm leads at every reported $k$; each arm was trained once and is reported at its best checkpoint per metric on the evaluation split (\Cref{sec:limitations}).
The ARC arms also differ in their read and write gates (\Cref{prop:sens-readout}), and collapse predicts lost streams, not lost accuracy, so the ARC gap is consistent with the prediction but does not test it.

\paragraph{Curves and cost.}
The orthogonal arm reaches the bistochastic arm's final exact match at ${\sim}297$K steps ($58\%$ of the bistochastic run) and a lower minimum of the evaluation loss; both losses rise after their minima (\Cref{fig:eval-accuracy}, Appendix~\ref{app:training-details}).
Per token, the fused gate projection dominates both mixer modules and is twice as large for the orthogonal arm ($3n^2$ against $2n+n^2$ outputs; \Cref{tab:param-counts}).

\section{Related Work}
\label{sec:related-work}

\paragraph{Hyper-connections.}
Beyond the works we build on~\citep{zhu2024hyperconnections,xie2025mhc}, two concurrent works stay within the Birkhoff polytope: \citet{kromhc2026} parametrise bistochastic matrices as Kronecker products of smaller bistochastic factors, and \citet{noguer2026} proposes an operator-constrained framework.
We show that products of positive bistochastic mixers collapse every direction but $\mathbf 1$, and replace the polytope by the orthogonal group, whose products contract nothing.
At the time of submission, \citet{liu2026beyond} proposed spectral-sphere-constrained hyper-connections, which share intuitions with ours without free-probability tools.

\paragraph{Orthogonal and manifold-constrained parametrisations.}
Orthogonal and unitary weights prevent gradient decay across time steps in recurrent networks~\citep{arjovsky2016unitary,wisdom2016full}, and trivializations such as the exponential map and the Cayley transform~\citep{cayley1846,li2020efficient} make hard orthogonal constraints trainable~\citep{lezcano2019trivializations}, within the broader theory of optimisation on matrix manifolds~\citep{absil2008optimization,JMLR_peyre}.
The Sinkhorn operator plays this role for bistochastic matrices~\citep{sinkhorn1967}, with an implicit derivative in closed form~\citep{eisenberger2022unified}, whose linear system we solve by Gauss--Seidel in fused forward and backward GPU kernels (Appendix~\ref{sec:sinkhorn}).
We apply these parametrisations not to weight matrices but to the skip mixers of hyper-connections, recomputed for every token, where orthogonality keeps the skip path isometric at every depth.
For the fixed-point Cayley iteration of \citet{li2020efficient} we bound the orthogonality defect, which fixes the number of iterations.

\paragraph{Dynamical isometry and free probability.}
Mean-field theory tracks scalar order parameters of signal propagation~\citep{schoenholz2017deep,yang2017mean}; free multiplicative convolution of the per-layer Gram matrices gives the whole Jacobian spectrum~\citep{pennington2017resurrecting,pennington2017nonlinear,xiao2018dynamical}, and \citet{tarnowski2019dynamical}, whose lines we follow, showed that for residual networks it depends on the activation only through an effective cumulant.
These compositions assume scalar skips: with the matrix skips of hyper-connections they fail already at the first moment (Section~\ref{sec:spectrum}), and we compose layers in the nested way described by operator-valued freeness, proved for network Jacobians by \citet{yang2020tensor}.
These analyses hold at initialisation; we also price the spectrum that training spends (Section~\ref{sec:beyond-init}).

\paragraph{Recursive reasoning models.}
The Hierarchical Reasoning Model (HRM)~\citep{wang2025hrm} introduced recursive multi-step reasoning with deep supervision for puzzle-solving tasks.
The Tiny Recursive Model (TRM)~\citep{jolicoeur2025trm} simplified HRM to a single 2-layer network with weight-tied recursion, achieving state-of-the-art results on ARC-AGI-1 (45\% accuracy) with only 7M parameters.
Its deep recursion is the regime in which Section~\ref{sec:beyond-init} predicts bistochastic collapse, which makes it our testbed.

\section{Limitations}\label{sec:limitations}

\paragraph{Scope.}
Task-level experiments use a 7M-parameter TRM on ARC-AGI and nano-GPT on FineWeb-Edu; larger models and other tasks are untested.
Only the nano-GPT comparison of \Cref{fig:nested-l10} ($L=12$) gives both arms the same (sigmoid) gates; the other nano-GPT runs ($n=4$, $L\in\{12,24\}$) and the ARC arms keep each arm's native gates, which differ in their write gain (\Cref{prop:sens-readout}).
Scaling with the number of streams is unexplored: at $n=4$ the mixers are small enough for exact spectral analysis.
Scaling to $n\ge8$ may require approximate methods for the operator-valued Dyson pipeline and more Cayley iterations at fixed $\alpha$ (Appendix~\ref{sec:cayley-tolerance}).

\paragraph{Single runs.}
Each ARC-AGI arm was trained once and is reported at its best checkpoint per metric on the evaluation split, which is optimistic for both; at the last checkpoints the pass@1 gap nearly closes, while the exact-match and pass@1000 gaps remain (\Cref{tab:eval-progression,tab:sinkhorn-progression}).
TRM reports a per-seed standard deviation of $\sensPowTrmSd$ accuracy points on Sudoku-Extreme, so single runs do not resolve gaps of a few points; the bistochastic arm's code also changed during its run (Appendix~\ref{sec:sinkhorn-run}).

\paragraph{Assumptions and theory.}
\Cref{thm:gated} fixes the gate values, bounds their derivative by rank only, and assumes independent blocks and a mask independent of the weights.
The gate and mask assumptions are measured (Appendices~\ref{app:mc} and~\ref{app:families}): the mask assumption holds at small depth, while the dropped gate derivatives carry most of the orthogonal block's excess first moment.
For exactly orthogonal mixers and isotropic noise the first moment is $\prod_\ell(1+\sigma_\ell^2)$\lean[Formalised]{FreeProbReal.GaussianMoment.iso_of_coisometric_gauss}, bounded in depth if and only if $\sum_\ell\sigma_\ell^2$ is, i.e.\ for a branch amplitude $O(1/\sqrt S)$ when the blocks are independent; tied blocks, as in TRM, need $O(1/S)$ (\Cref{prop:sens-tied}).
No theorem links the collapse exponent $\Lambda$ to the loss, and three registered predictions fail (\Cref{tab:sens-ledger}): the loss gap of the bistochastic arms to the residual network does not rise with $\Lambda$ (D4), and the exponent of the deployed two-step Cayley mixer neither compounds with depth (X4) nor responds to mixer decay as predicted (D6).
On nano-GPT the exact Cayley map keeps at least half of the deployed orthogonal arm's loss gap to the bistochastic arm (K8, pass), so the two-step truncation explains at most half of that gap.
The ARC prediction (K5: the bistochastic arm ends single-stream) and the gate predictions G1 and G3 are pending, and the ARC comparison has no single-stream ($n=1$) arm.

\section{Conclusion}
\label{sec:conclusion}

JPmHC replaces the bistochastic mixers of mHC~\citep{xie2025mhc} by orthogonal ones, chosen a priori by operator-valued free probability.
The nested theory matches Monte Carlo in all $\mnFigCells$ mixer$\times$depth cells of \Cref{fig:ov-vs-mc-panel}, agrees with it to the eye for gated, activation and threaded branches even where the strict acceptance rule rejects a cell (Appendix~\ref{app:families}), and is consistent with the ARC-AGI gap between the two arms; the training budget of \Cref{thm:budget} predicts the measured collapse of TRM's bistochastic mixers and the stability of nano-GPT's.
Multi-seed and gate-fair ARC studies, trained TRM spectra and gate-fair language models beyond $L=12$ remain open.

\label{endofmainbody}%
\section*{Impact Statement}
This paper presents work whose goal is to advance the field of Machine Learning.
There are many potential societal consequences of our work, none which we feel must be specifically highlighted here.

\bibliography{references}
\bibliographystyle{icml2026}

\newpage
\onecolumn
\appendix
\section{Overview and notation}\label{app:overview}

\paragraph{Map.}
Appendices~\ref{app:architecture}--\ref{app:families} hold the spectrum pipeline of Section~\ref{sec:spectrum}: activation moments; the operator-valued setting, freeness and the linearisation of the nested product (Appendix~\ref{app:pencil}); the nested first moment for general branches and the proof of \Cref{thm:gated}; the solver; the Monte-Carlo validation; and its repetition for gated, activation, threaded-ResNet and threaded-transformer branches.
Appendix~\ref{app:sensitivity} proves the training-time statements of Section~\ref{sec:beyond-init}.
Appendix~\ref{app:mixers} collects the mixer sets and the parameter counts; Appendix~\ref{sec:cayley} gives the gate computation of the Cayley arm, the Cayley iteration and its iteration count (Section~\ref{sec:cayley-method}), and the per-checkpoint evaluation of both ARC runs; Appendix~\ref{sec:sinkhorn} the implicit Sinkhorn backward pass and the trained bistochastic arm; and Appendix~\ref{app:training-details} the configuration and history of the ARC runs (Section~\ref{sec:experiments}).
Appendix~\ref{app:effective-spectra} treats finite-precision gradient accumulation (Section~\ref{sec:effective-spectrum}), and Appendix~\ref{app:lean} indexes every Lean citation with its tier.

\paragraph{Code and data.}
The Lean~4 formalisation with its executables (every algorithm of Appendices~\ref{app:architecture}--\ref{app:mc}, the Monte Carlo of Figure~\ref{fig:ov-vs-mc-panel}, the floating-point emulator of Appendix~\ref{app:effective-spectra}), the PyTorch probes, the generators that typeset the measured numbers of Section~\ref{sec:spectrum} and Appendices~\ref{app:mc}, \ref{app:sensitivity} and~\ref{app:effective-spectra} from their result files, and the training code of Section~\ref{sec:experiments} will be released with the paper.
ARC-AGI-1, ConceptARC and FineWeb-Edu are public.
The checkpoints of the two ARC runs were not retained; their numbers come from the training logs.

{\small
\begin{longtable}{@{}p{0.27\textwidth}p{0.69\textwidth}@{}}
  \caption{Notation.
  Every symbol has one meaning throughout the paper.
  A fixed descriptive subscript or superscript makes a different symbol ($u$ and $u_{\text{fl}}$, $\delta$ and $\delta^{\mathrm{loo}}_k$); an index subscript ($\ell$, $k$, $m$, $t$) does not ($\eta_\ell$ is the covariance map of layer $\ell$).}\label{tab:notation}\\
  \toprule
  Symbol & Meaning \\
  \midrule
  \endfirsthead
  \multicolumn{2}{@{}l}{\emph{\Cref{tab:notation}, continued}}\\
  \toprule
  Symbol & Meaning \\
  \midrule
  \endhead
  \bottomrule
  \endlastfoot
  \multicolumn{2}{@{}l}{\textbf{Conventions}}\\
  $i,j,k,l,m,t$ & bound indices of sums, entries, applications, pieces and tokens or steps; $i$ is also $\sqrt{-1}$ \\
  $e$, $\pi$ & Euler's number; $\pi$ \\
  $I$, $I_n$, $\ones$, $\zeros$ & identity matrix (of size $n$); all-ones and all-zeros vectors \\
  $\R$, $\C$, $\mathbb Z$, $M_n(\cdot)$ & real, complex and integer numbers; $n\times n$ matrices \\
  $\mathbb E$, $\operatorname{tr}$, $\mathrm{tr}_p$ & expectation; trace; normalised trace $\tfrac1p\operatorname{tr}$ \\
  $\mathcal N(0,s^2)$, $\delta_{ik}$, $\delta_0$ & Gaussian law of variance $s^2$; Kronecker delta; unit point mass at $0$ \\
  $O(\cdot)$, $\mathcal O(\cdot)$, $o(\cdot)$ & Landau symbols \\
  $O(n)$, $SO(n)$, $\mathrm{GL}_n$, $\mathfrak{so}(n)$ & orthogonal, special orthogonal and general linear groups; skew-symmetric matrices \\
  $\otimes$, $\oplus$, $\odot$ & Kronecker product; direct sum; entrywise product \\
  $\|\cdot\|_1$, $\|\cdot\|_2$, $\|\cdot\|_\infty$, $\|\cdot\|_F$, $\|\cdot\|_{\max}$ & $\ell^1$ norm; spectral (Euclidean) norm; $\ell^\infty$ norm; Frobenius norm; largest absolute entry \\
  $\mathrm{diag}$, $\mathrm{conv}$, $\mathrm{sym}$ & diagonal matrix; convex hull; symmetric part $\mathrm{sym}\,X\coloneqq(X+X^\top)/2$ \\
  $\mathrm{sig}$, $\mathrm{atan2}$, $\mathrm{clamp}$ & logistic $1/(1+e^{-t})$; two-argument arctangent; entrywise clipping \\
  $\mathrm{LayerNorm}$, $\mathrm{RMSNorm}$, $\mathrm{LSE}_{\text{row}}$, $\mathrm{LSE}_{\text{col}}$ & layer and RMS normalisation; row and column log-sum-exp \\
  $\mathrm{Flatten}$, $\mathrm{Split}$ & concatenation of the $n$ streams into one vector; split of a vector into consecutive blocks \\
  \midrule
  \multicolumn{2}{@{}l}{\textbf{Architecture and mixers} (Sections~\ref{sec:introduction}, \ref{sec:cayley-method}, \ref{sec:experiments}; Appendices~\ref{app:mixers}--\ref{app:training-details})}\\
  $n$ & number of streams \\
  $p$ & width of a stream \\
  $N$ & width of the residual state, $N\coloneqq np$ ($N\coloneqq n_qp$ in \Cref{thm:gated}) \\
  $L$, $\ell$ & number of layers; layer or block index \\
  $T$ & number of tokens \\
  $x$, $x_\ell$, $x_k$ & residual state, $x\in\R^{n\times p}\cong\R^N$ (after layer $\ell$, or after the $k$-th mixer application) \\
  $x_0$, $x_{\text{emb}}$ & input state; token embedding, $x_0\coloneqq\ones\otimes x_{\text{emb}}$ \\
  $x_{\text{in}}$, $x_{\text{flat}}$, $\hat x$ & gate input; its flattening in $\R^{np}$; its RMS normalisation \\
  $\bar x_k$ & common row of a state with equal rows, $x_k=\ones\bar x_k^\top$ \\
  $F$, $F_\ell$, $F_k$ & residual branch (attention or MLP) \\
  $W$, $W_\ell$ & branch weights \\
  $\phi$, $h$, $D_\ell$ & activation; pre-activation; activation mask $\mathrm{diag}(\phi'(h))$ \\
  $H_{\text{res}}$, $H_{\text{pre}}$, $H_{\text{post}}$ & gates: residual mixer, read map, write map \\
  $\tilde H_{\text{res}}$, $\tilde H_{\text{pre}}$, $\tilde H_{\text{post}}$ & gate logits \\
  $W_{\text{fused}}$, $\tau$ & fused gate projection; softmax temperature of the gates \\
  $r_{\text g,\ell}$ & number of gate logits of block $\ell$ (gate rank) \\
  $u$, $v$, $c$ & write vector $u\coloneqq H_{\text{post}}$ and read vector $v\coloneqq H_{\text{pre}}^\top$ for HC's vector gates; for our $n\times n$ gates $u=c\,H_{\text{post}}\ones$, $v=H_{\text{pre}}^\top\ones/n$ (Appendix~\ref{app:cayley-gates}); write scale \\
  $A$, $A_\ell$, $A_k$ & mixer ($H_{\text{res}}$ at frozen gates), of layer $\ell$, or at its $k$-th application \\
  $S$ & mixer applications per forward pass ($S=2L$ for a transformer, $2LH(L_c+1)$ per TRM supervision step) \\
  $B$ & stream map $B\coloneqq A_S\cdots A_1$ \\
  $P$, $Q$ & Perron projector $\ones\ones^\top/n$; $Q\coloneqq I-P$ \\
  $\mathrm{Bis}_n$, $\bP$ & doubly stochastic $n\times n$ matrices (Birkhoff polytope); a doubly stochastic matrix \\
  $\mathrm{Sinkhorn}$, $\Pi_{\mathrm{Bis}}$ & Sinkhorn normalisation of a positive matrix; of a logit matrix, $\Pi_{\mathrm{Bis}}(\tilde H_{\text{res}})=\mathrm{Sinkhorn}(\exp\tilde H_{\text{res}})$ with clamped logits \\
  $\Pi_{\text{swap}}$ & permutation matrix of a transposition \\
  $\Omega$ & skew-symmetric generator of the Cayley mixer \\
  $\alpha$, $\omega$ & Cayley step; rotation rate of $\Omega$ on an invariant plane \\
  $\Cay$, $\Cay_s$, $s$ & Cayley transform; its $s$-step fixed-point iterate; number of Cayley iterations \\
  $r$, $r_{\max}$ & $r\coloneqq\|\tfrac\alpha2\Omega\|_2$; a bound on $r$ \\
  $\Upsilon$, $\varepsilon_{\text{orth}}$ & orthogonality defect matrix $(\tfrac\alpha2\Omega)^\top(\tfrac\alpha2\Omega)$; tolerance on the defect \\
  $s_{\text{Sk}}$, $s_{\text{GS}}$ & Sinkhorn iterations; Gauss--Seidel iterations of its implicit backward \\
  $\mathbf M_{\text{Sk}}$ & log-domain Sinkhorn iterate \\
  $\mathbf M_{\text{GS}}$, $\bh_{\text{row}}$, $\bh_{\text{col}}$ & $\bP\odot\nabla_{\bP}\calL$ and its row and column sums \\
  $\mathbf a_{\text{row}}$, $\mathbf a_{\text{col}}$ & dual variables of the row and column constraints \\
  $\Xi_{\text{pre}}$, $\Xi_{\text{post}}$, $\Xi_{\text{res}}$ & raw gate projections of the Sinkhorn arm \\
  $g_{\text{pre}}$, $g_{\text{post}}$, $g_{\text{res}}$ & learned scalar gains of the gate logits \\
  $b_{\text{pre}}$, $b_{\text{post}}$, $b_{\text{res}}$ & static logit biases of the gates \\
  $H$, $L_c$ & TRM outer and inner cycles ($L$ blocks per update) \\
  $\mathcal M$, $M$ & mixing monoid (\Cref{theo:bounded_gradient}); an element \\
  $C$, $K$ & norm bound $C\coloneqq\sup_{M\in\mathcal M}\|M\|$ of the monoid; Lipschitz constant of the branches \\
  $a$, $d$ & branch scale; dimension ($M\in M_d$) \\
  \midrule
  \multicolumn{2}{@{}l}{\textbf{Spectral theory} (Section~\ref{sec:spectrum}; Appendices~\ref{app:architecture}--\ref{app:mc})}\\
  $J$, $J_{\text s}$, $J_{\text f}$ & end-to-end Jacobian; with the gates held fixed (frozen gates); with the gate derivatives \\
  $Y_\ell$, $Y_{\text{br},\ell}$ & layer (block) Jacobian; its branch part \\
  $J_{\text{br}}$ & single-stream product of branch Jacobians \\
  $\sigma_w$, $\sigma_{w,\ell}$ & weight scale, $W_\ell\sim\mathcal N(0,\sigma_w^2/N)$ \\
  $\sigma$, $\sigma_\ell$, $\sigma_{\ell m}$ & standard deviation of the branch noise ($\sigma^2$ its variance) \\
  $\sigma_{\text{in}}$, $\sigma_{\text{out}}$ & weight scales of the input and output layers of an FFN branch \\
  $X$, $X_\ell$, $X_{\ell m}$ & Gaussian noise of a layer, or of its piece $m$ \\
  $Z$ & standard Gaussian (scalar or matrix) \\
  $\psi$, $\psi_\ell$ & activation gain $\mathbb E[\phi'(h)^2]$ \\
  $q$, $q_\ell$, $\bar q$ & pre-activation variance; signal variance $\tfrac1N\mathbb E\|x\|^2$ \\
  $\kappa(q)$, $\bar\phi(q)$ & Gaussian moments $\mathbb E\,\phi(\sqrt q Z)^2$ and $\mathbb E\,\phi(\sqrt q Z)$ \\
  $\mu_M$, $\boxtimes$ & eigenvalue law of $M$; free multiplicative convolution \\
  $\mathcal D_i$, $w_j$ & subalgebras and their elements in the definition of freeness \\
  $m_1$, $m_{1,\infty}$ & first moment $\mathbb E\tfrac1N\operatorname{tr}J^\top J$; its infinite-width limit \\
  $\eta$, $\eta_\ell$ & covariance map $b\mapsto\eta(b)$ of the layer noise \\
  $\eta^{\mathrm R}_\ell$, $\eta^{\mathrm L}_\ell$ & right and left covariance maps of the pencil \\
  $b$ & argument of maps on $M_n$ \\
  $\mathcal T_\ell$ & layer map $b\mapsto A_\ell^\top bA_\ell+\eta_\ell(b)$ \\
  $\Phi_\ell$, $\Phi_{1:S}$ & $\mathcal T_\ell(I)$; $\mathcal T_1\circ\cdots\circ\mathcal T_S(I)$ \\
  $R_m$, $R_{\ell m}$ & routing matrices of the branch pieces \\
  $\Theta$, $\Theta_\ell$ & cross-covariance of the branch pieces \\
  $n_q$ & number of streams in \Cref{thm:gated} ($n$, or $Tn$ on $T$ tokens) \\
  $\mathcal A_m$, $\mathcal U$, $\mathcal V$ & attention pattern of head $m$; stacked per-token write and read vectors $\bigoplus_tu_t$, $\bigoplus_tv_t$ \\
  $W_O$, $W_V$, $W_{OV}$, $n_{\text{head}}$ & output and value projections, $W_{OV}\coloneqq W_OW_V$; number of heads \\
  $p_{\text{hid}}$, $\dot K_{tt'}$ & hidden width of the FFN; token co-activation \\
  $R_{\text{in}}$, $R_{\text{out}}$, $W_{\text{in}}$, $W_{\text{out}}$, $\Delta_{\text{hid}}$ & routings and weights into and out of the hidden layer; hidden factor \\
  $S_{\text{out}}$, $S_{\text{skip}}$ & second moments of the output and of the skip part \\
  $n_{\text{mask}}$ & number of mask classes \\
  $\mathcal B$, $E_{\mathcal B}$ & $M_n(\C)$; conditional expectation $\mathrm{id}_n\otimes\mathrm{tr}_p$ \\
  $G$, $G_{J^\top J}$ & solution of the Dyson equation~\eqref{eq:lin-dyson} ($\mathcal B\otimes M_{2S}$-valued); $\mathcal B$-valued Cauchy transform of $J^\top J$ (Section~\ref{sec:ov-free-prob}) \\
  $G_s$, $g$ & scalar Cauchy transform of one layer, $G_s(z)\coloneqq\tfrac1n\operatorname{tr}G_{J^\top J}(zI_n)$ (\Cref{thm:law-response}); of the limiting law of $J^\top J$ \\
  $z$, $\lambda$, $\lambda_{\max}$ & spectral parameter, $\operatorname{Im}z>0$; real spectral variable or eigenvalue; largest eigenvalue \\
  $\rho$ & spectral density \\
  $\varepsilon$ & imaginary regulariser of the spectral parameter \\
  $\vartheta$, $\zeta$ & one-layer subordination variables $\vartheta\coloneqq1-\sigma^2G_s(z)$, $\zeta\coloneqq z\vartheta^2$ (\Cref{thm:law-response}) \\
  $\mathcal P$, $\mathcal P_0$, $\mathcal Z$, $\Sigma$ & linearisation pencil; its constant part; its spectral-parameter block; self-energy \\
  $n_{\text{blk}}$, $\varpi$, $y$ & block size of the pencil; a block vector; a right-hand side \\
  $\mathcal F_z$, $G_{\text{new}}$ & Dyson fixed-point map; next iterate \\
  $\eta_{\text{cont}}$, $n_{\text{lev}}$, $s_{\max}$, $\varepsilon_{\text{tol}}$ & continuation height; number of levels; iteration cap; tolerance of the solver \\
  $\eta_{\text{sm}}$, $W_{1,\eta_{\text{sm}}}$ & Poisson smoothing width; 1-Wasserstein distance of the smoothed densities \\
  $\rho^{\mathrm{MC}}_{\eta_{\text{sm}}}$ & smoothed Monte-Carlo density \\
  \midrule
  \multicolumn{2}{@{}l}{\textbf{Training-time sensitivity} (Section~\ref{sec:beyond-init}; Appendices~\ref{app:sensitivity}, \ref{app:training-details})}\\
  $t(B)$ & mean squared singular value of $B$ on $\ones^\perp$, $\operatorname{tr}(QB^\top BQ)/(n-1)$ \\
  $\Lambda(B)$, $\Lambda^\ast$ & collapse exponent $-\log\sigma_{\min}(B|_{\ones^\perp})$; instability threshold $\Lambda^\ast\coloneqq1$ \\
  $\Lambda_{\mathrm{sup}}$, $\Lambda_{\mathrm{st}}$ & exponent of the mixer factors of one supervision step; exponent of the static logits \\
  $\sigma_{\min}$, $\sigma_{\max}$ & smallest and largest singular values \\
  $\delta$ & Frobenius distance of a doubly stochastic stream map to $I$ \\
  $\Delta$, $\bar\Delta$ & logit gap of a Sinkhorn mixer, mean diagonal minus mean off-diagonal logit~\eqref{eq:gap}; mean of per-layer gaps \\
  $\Delta_{\text{unif}}$, $\Delta_{\text{ceil}}$ & uniform gap of equal potential; ceiling gap under weight decay \\
  $c_{\text{Sk}}(\Delta)$, $c_{\text{Cay}}$ & exponent per application of a Sinkhorn mixer; of the deployed Cayley mixer \\
  $\Delta^\ast(S)$, $S^\ast$, $L^\ast$ & smallest stable gap at $S$ applications; unstable number of applications; of blocks \\
  $\omega^\ast$, $\bar u$, $\bar u^\ast$ & threshold rotation rate; mean AdamAtan2 drive on the gap; its threshold \\
  $n_{\mathrm{eff}}$, $c_{\text{mix}}$ & effective number of streams; mixing constant of the depth continuum \\
  $\theta$, $\dot\theta$ & parameters; their flow \\
  $\tilde\theta$, $c_{\text{rp}}$ & reparametrised coordinates $\theta=c_{\text{rp}}\tilde\theta$; the gain \\
  $\Psi$, $\calL$ & model output $\Psi(\theta)$; training loss \\
  $r_{\text{res}}$, $y_{\text{tgt}}$ & residual $\Psi(\theta)-y_{\text{tgt}}$; target \\
  $\mathcal K$, $\mathsf g$ & preconditioner; parameter group \\
  $\mathcal S$ & differentiable functional of the spectrum \\
  $\varphi_{\mathsf g}$, $\varrho_{\mathsf g}$, $\varrho_{\mathrm{branch}}$, $\mathrm{share}_{\mathsf g}$ & functional sensitivity, sensitivity ratio (of the branch group), share of a group \\
  $t_{\text{end}}$ & end of training \\
  $\mathcal X$, $\epsilon$ & tangent direction; curve parameter \\
  $\Pi_{\text G}$, $n_{\text G}$ & Grassmann mixer (an orthogonal projector); its rank \\
  $\dot G_s$, $\dot\vartheta$, $\dot\zeta$ & $\epsilon$-derivatives at $\epsilon = 0$ along a mixer move (\Cref{prop:sens-law}) \\
  $\mathcal R$ & matrix of a first-order law response, $\mathcal X\mapsto\operatorname{tr}(\mathcal R\mathcal X)$ \\
  $\dot x$, $\dot A$, $\dot u$, $\dot v$, $\dot F$ & $\epsilon$-derivatives along a curve of states and gates (\Cref{prop:sens-score}) \\
  $m_{\text{off}}$, $a_{\min}$ & off-diagonal mass; lower bound on the entries \\
  $a_{\text{out},k}$, $\nu$ & readout row; embedding retention \\
  $\gamma$, $\gamma_t$, $\gamma_{\max}$, $\gamma_\Sigma$ & learning rate (at step $t$); its peak; budget $\sum_t\gamma_t$ \\
  $\gamma^\ast_\Sigma$ & training budget before collapse: the budget at which the gap reaches $\Delta^\ast(S)$ (\Cref{thm:budget}, \Cref{cor:sens-budget}) \\
  $\hat m$, $\hat v$, $\beta_1$, $\beta_2$ & Adam's bias-corrected moments; their decay rates \\
  $\varepsilon_{\text{A}}$, $\chi$, $A_{\text{Adam}}$ & Adam's epsilon; the damping it causes; bound $\sup|\hat m|/\sqrt{\hat v}$ \\
  $\Gamma$ & loss gradient per unit of off-diagonal mixing \\
  $s_{\text{coh}}$, $\kappa_{\text{gap}}$ & coherence of the optimiser's drive; gap change per unit coordinate step \\
  $\lambda_{\text{wd}}$, $a_{\text{at}}$, $b_{\text{at}}$ & weight decay; AdamAtan2 constants \\
  $\tau_{\text{init}}$ & logit initialisation scale of nano-GPT's mHC, $\tilde H_{\text{res}}=\tau_{\text{init}}(2I-\ones\ones^\top)$ \\
  $\mathsf d_t$, $\mathsf u_t$ & per-step drive of the gap ($\mathsf d$: constant drive toward mixing per unit budget); per unit learning rate \\
  $\mathcal E$, $\Pi_{\text{dec}}$ & conserved potential of the gap; product of the decay factors \\
  $\iota$, $\mathcal C$ & label map of tied layers; tied transfer products \\
  $b_{\text{fw}}$ & coefficient of the finite-width fit \\
  $n_{\text{tok}}$, $n_{\text{thr}}$, $n_{\text{deg}}$ & tokens per thread; number of threads; mixer degree \\
  \midrule
  \multicolumn{2}{@{}l}{\textbf{Finite-precision accumulation} (Section~\ref{sec:effective-spectrum}; Appendix~\ref{app:effective-spectra})}\\
  $\varsigma$, $\varsigma_i$, $\varsigma_k$ & a singular value of $J$; the $i$-th; the one met by sample $k$ \\
  $\varsigma_{\mathrm{rms}}$, $\bar\varsigma$, $\varsigma_{\max}$ & root mean square, mean and largest sampled singular value \\
  $B_{\text{b}}$, $B_{\text{b}}^\ast$ & batch size; largest batch size at which ascending order absorbs nothing \\
  $g_{\text{scale}}$, $g_{\text{st}}$ & gradient scale; stored gradient \\
  $\mathrm{fl}_{\mathrm{in}}$, $\mathrm{fl}_{\mathrm{acc}}$, $\mathrm{fl}_{\mathrm{out}}$ & input, accumulator and output formats \\
  $v^{\mathrm{in}}_k$, $v^{\mathrm{in}}_{\max}$, $S^{\mathrm{acc}}_k$, $n_{\text{add}}$ & addend; bound on the addends; partial sum; number of addends \\
  $\mathrm{psum}$, $\mathrm{rnd}$, $\mathrm{ulp}$ & recursive floating-point sum; round-to-nearest; unit in the last place \\
  $\delta^{\mathrm{step}}_k$, $\delta^{\mathrm{loo}}_k$ & step and leave-one-out transmissions \\
  $\mu_{\mathrm{eff}}$, $\mu_{\mathrm{tr}}$, $a_{\text{atom}}$ & effective spectrum; its transmitted part; its atom at zero \\
  $p_{\text{fl}}$, $e_{\min}$ & precision (significand bits); minimum exponent \\
  $\xi$, $m_{\text{fl}}$, $e_{\text{fl}}$, $\mathbb F$ & a real number; significand and exponent; a set of floats \\
  $u_{\text{fl}}$, $u'_{\text{fl}}$, $c_{\text{fl}}$ & unit roundoff $2^{-p_{\text{fl}}}$; a relative error bound; $(1-u_{\text{fl}})u_{\text{fl}}/2$ \\
  $W_{\text{lin}}$, $x_{\text{lin},t}$, $g_{\text{lin},t}$, $T_{\text{mb}}$ & weight, input and output gradient of a linear layer; tokens per micro-batch \\
\end{longtable}
}

\section{Spectrum computational pipeline}
\label{app:architecture}

Every algorithm of this appendix is an executable Lean~4 definition, generic over the scalar type and run in IEEE binary32 unless stated otherwise; the superscript \textsuperscript{\textsf{L}} links each one to Appendix~\ref{app:lean}.
The pipeline computes, for the model of Section~\ref{sec:spectrum}, $J=Y_S\cdots Y_1$ with $Y_\ell=A_\ell\otimes I_p+\sigma_\ell X_\ell$:
\begin{enumerate}
  \item the operating point $\sigma_\ell^2=\sigma_{w,\ell}^2\,\psi(q_\ell)$ of each layer, $q_\ell$ its pre-activation variance (Section~\ref{sec:signal-prop});
  \item the exact nested first moment~\eqref{eq:nested-m1}\lean{FreeProb.FlatNested.nestedM};
  \item the limiting law of $J^\top J$ from the linearised operator-valued Dyson equation (Sections~\ref{app:pencil}--\ref{app:structured});
  \item the density by Stieltjes inversion with Richardson extrapolation (Section~\ref{app:density}), and its Monte-Carlo validation (Section~\ref{app:mc}).
\end{enumerate}

\paragraph{Gram-based compositions.}
The S-transform, the $z_1$-mapping and the twisted operator-valued S-transform compose per-layer laws and see each mixer only through its Gram matrix $A_\ell^\top A_\ell$, which no exact formula can do for non-co-isometric mixers (Section~\ref{sec:spectrum}).
They are exact when every mixer is co-isometric, where the linearised equation below reduces to them; Figure~\ref{fig:ov-vs-mc-panel} keeps the Gram-based (flat) composition as the comparator.

\section{Activation functions and signal propagation}\label{sec:signal-prop}

The activation enters through Gaussian moments at pre-activation variance $q$, computed by Gauss--Hermite quadrature~\citep{golub1969calculation} (100 nodes; the rule converges to the Gaussian moments\lean[Formalised]{FreeProbReal.SignalPropSpec.gaussHermite_tendsto}, and an even node count makes the ReLU moments exact\lean[Formalised]{FreeProbReal.SignalPropSpec.psi_relu_ruleOfEig_even}\lean[Formalised]{FreeProbReal.SignalPropSpec.kappa_relu_ruleOfEig}; executables\lean{FreeProb.SignalProp.psi}\lean{FreeProb.SignalProp.kappa}\lean{FreeProb.SignalProp.phiBar}):
\begin{align}
  \psi(q)  &\coloneqq \mathbb{E}_{Z\sim\mathcal{N}(0,1)}\bigl[\phi'(\sqrt{q}\,Z)^2\bigr], \label{eq:psi}\\
  \kappa(q)&\coloneqq \mathbb{E}_{Z\sim\mathcal{N}(0,1)}\bigl[\phi(\sqrt{q}\,Z)^2\bigr], \label{eq:kappa}\\
  \bar\phi(q)&\coloneqq \mathbb{E}_{Z\sim\mathcal{N}(0,1)}\bigl[\phi(\sqrt{q}\,Z)\bigr]. \label{eq:phi_moment}
\end{align}
In a bias-free network the pre-activation variance is $q_\ell=\sigma_{w,\ell}^2\,\bar q_{\ell-1}$, with $\bar q_{\ell-1}\coloneqq\tfrac1N\mathbb{E}\|x_{\ell-1}\|^2$ the signal variance~\citep{schoenholz2017deep}, and the effective noise variance of layer $\ell$ is $\sigma_\ell^2=\sigma_{w,\ell}^2\,\psi_\ell$ with the activation gain $\psi_\ell=\psi(q_\ell)$.
The second moment propagates as
\begin{equation}\label{eq:v_recursion}
  \mathbb{E}\,x_{\ell,i}^2=\kappa(q_\ell)+\mathbb{E}\,(A_\ell x_{\ell-1})_i^2+2\,\bar\phi(q_\ell)\,\mathbb{E}\,(A_\ell x_{\ell-1})_i ,
\end{equation}
where the cross term vanishes for odd $\phi$ but not for ReLU, whose positive mean is carried by the skip.
With Gaussian weights and these $\sigma_\ell^2$, the isotropic noise model reproduces the first moment of the network's Jacobian exactly; its spectrum does not, because the ReLU mask $D_\ell$ is a projection rather than a scalar: Section~\ref{sec:gated-theory} keeps it.

\section{Operator-valued setting and linearisation of the nested product}
\label{app:pencil}

\paragraph{Operator-valued setting.}
We work over $\mathcal{B}=M_n(\mathbb{C})$ with the conditional expectation $E_{\mathcal B}$, the Cauchy transform~\eqref{eq:ov-cauchy} and the Stieltjes inversion~\eqref{eq:stieltjes} of Section~\ref{sec:ov-free-prob}.
Subalgebras $\mathcal{D}_1,\dots,\mathcal{D}_k\supseteq\mathcal{B}$ are \emph{free over $\mathcal{B}$} if $E_{\mathcal{B}}[w_1\cdots w_m]=0$ whenever $w_j\in\mathcal{D}_{i_j}$, $E_{\mathcal{B}}[w_j]=0$ and $i_j\neq i_{j+1}$; for $\mathcal{B}=\mathbb{C}$ this is Voiculescu's freeness, the large-$p$ behaviour of independent unitarily invariant matrices \citep{mingo2017free}.
In this sense the noise $\sigma_\ell X_\ell$ of layer $\ell$ is, as $p\to\infty$, a $\mathcal{B}$-valued circular element with covariance $\eta_\ell$, free over $\mathcal{B}$ from $A_\ell\otimes I_p$ and from the other layers.

\paragraph{One layer.}
For $S=1$ and $z\in\mathbb C$ with $\operatorname{Im}z>0$, the Cauchy transform~\eqref{eq:ov-cauchy} of $J^\top J=Y_1^\top Y_1$ at $zI_n$ solves the subordination equation
\begin{equation}\label{eq:one-layer}
  G_{Y_1^\top Y_1}(zI_n)=\vartheta\,(\zeta I_n-A^\top A)^{-1},
\end{equation}
with $G_s(z)=\tfrac1n\operatorname{tr}G_{Y_1^\top Y_1}(zI_n)$, $\vartheta=1-\sigma^2 G_s(z)$ and $\zeta=z\,\vartheta^2$ as in \Cref{thm:law-response}.
This equation reduces to the Marchenko--Pastur law at $A=0$ and, diagonalising $A^\top A$, to scalar subordination over its eigenvalues\lean[Formalised]{FreeProbReal.OVDyson.ovF_eq_subordF_eigs}.
It is exact for one layer and is the base case of the multi-layer computation below.

\paragraph{Many layers.}
For $S\ge2$ the object to compute is the nested product $J=\prod_\ell(A_\ell\otimes I_p+\sigma_\ell X_\ell)$, a polynomial in free $\mathcal{B}$-valued variables, whose law follows from the linearisation of the Hermitised product~\citep{belinschi2017analytic}.
Write $n_{\text{blk}}\coloneqq n$ (for the gated models of Section~\ref{sec:gated-theory}, $\mathcal{B}$ is a larger block-diagonal algebra and $n_{\text{blk}}$ its size).
The \emph{pencil} $\mathcal P(z)$ of $J$ is the symmetric $2S\times2S$ matrix of $n_{\text{blk}}\times n_{\text{blk}}$ blocks
\begin{equation}\label{eq:pencil}
  \text{pair } m\ (\text{blocks } 2m,2m+1):\ \begin{pmatrix}\delta_{m0}\,z & -Y_{m+1}^\top\\ -Y_{m+1} & \delta_{m,S-1}\end{pmatrix},
  \quad \mathcal P(z)_{2m,2m-1}=\mathcal P(z)_{2m-1,2m}=I\ (1\le m<S),
\end{equation}
all other blocks zero; its spectral-parameter block is $\mathcal Z(z)\coloneqq\mathrm{diag}(z,0,\ldots,0)$ and its constant part $\mathcal P_0\coloneqq\mathcal Z(z)-\mathcal P(z)|_{X_\ell=0}$.
For a block vector $\varpi=(\varpi_0,\ldots,\varpi_{2S-1})$, eliminating $\varpi_{2m}=Y_m\varpi_{2m-2}$, $\varpi_{2m-1}=Y_{m+1}^\top\varpi_{2m+1}$, $\varpi_{2S-1}=Y_S\varpi_{2S-2}$ turns $\mathcal P(z)\varpi=(y,0,\ldots,0)$, for any right-hand side $y$, into $(z-J^\top J)\,\varpi_0=y$, so the corner block of $\mathcal P(z)^{-1}$ is $(z-J^\top J)^{-1}$\lean[Formalised]{FreeProbReal.OVLin.corner_pencil_inv}; for $S=1$ this is the $\bigl(\begin{smallmatrix} z & -A^\top\\ -A & 1\end{smallmatrix}\bigr)$ linearisation behind the one-layer equation~\eqref{eq:one-layer}.
The pencil is linear in the noises $X_\ell$, which is what makes the operator-valued semicircle law applicable~\citep{helton2007operator,belinschi2017analytic}: with the right and left covariance maps $\eta^{\mathrm R}_\ell(b)\coloneqq\sigma_\ell^2E_{\mathcal{B}}[X_\ell^\top bX_\ell]$ and $\eta^{\mathrm L}_\ell(b)\coloneqq\sigma_\ell^2E_{\mathcal{B}}[X_\ell bX_\ell^\top]$ ($\eta^{\mathrm R}_\ell=\eta^{\mathrm L}_\ell=\eta_\ell$ of Section~\ref{sec:spectrum} for isotropic noise), the $\mathcal{B}\otimes M_{2S}$-valued Cauchy transform solves
\begin{equation}\label{eq:lin-dyson}
  G=\mathcal F_z(G)\coloneqq\bigl(\mathcal Z_\varepsilon(z)-\mathcal P_0-\Sigma(G)\bigr)^{-1},\qquad
  \Sigma(G)_{2m}\coloneqq\eta^{\mathrm R}_{m+1}(G_{2m+1,2m+1}),\quad \Sigma(G)_{2m+1}\coloneqq\eta^{\mathrm L}_{m+1}(G_{2m,2m}),
\end{equation}
the self-energy $\Sigma$ block-diagonal\lean{FreeProb.OVLin.selfEnergy}, and $g(z)\coloneqq\tfrac1{n_{\text{blk}}}\operatorname{tr}G_{00}$ is the Stieltjes transform of the limiting law of $J^\top J$.
Here $\mathcal Z_\varepsilon(z)\coloneqq\mathrm{diag}(z,i\varepsilon,\ldots,i\varepsilon)$; $\varepsilon=0$ is the exact linearisation and $\varepsilon>0$ the regularisation of~\citet{belinschi2017analytic}.
The law itself is cited, not formalised; what Lean certifies is the algebra around it:
\begin{itemize}
  \item at $S=1$ and isotropic noise, solutions of~\eqref{eq:lin-dyson} and of the one-layer equation~\eqref{eq:one-layer} correspond\lean[Formalised]{FreeProbReal.OVLin.isSol_of_F64}\lean[Formalised]{FreeProbReal.OVLin.F64_of_isSol};
  \item at $\sigma=0$ the solution is the resolvent of $J^\top J$\lean[Formalised]{FreeProbReal.OVLin.isSol_noiseless};
  \item the $1/z^2$ coefficient of the iterates $\mathcal F_z^{[k]}(\mathcal F_z(0))$, $k\ge S$, is the nested first moment~\eqref{eq:nested-m1}\lean[Caveat]{FreeProbReal.OVLin.tendsto_firstMoment};
  \item for co-isometric mixers the equation is solved by scalar multiples on each block, and $g$ is the scalar free product of the layers\lean[Caveat]{FreeProbReal.OVLin.isSol_coiso_iff}.
\end{itemize}
The executable \lean{FreeProb.OVLin.ovLinPlainS} evaluates the first moment through these iterates, in binary64.

\paragraph{The gated block beyond the first moment.}\label{app:masked} A mask shared by all streams of a branch coordinate makes the gated model of Section~\ref{sec:gated-theory} exactly operator-valued over $M_n\otimes\mathbb{R}^{n_{\text{mask}}}$ ($n_{\text{mask}}$ mask classes), with explicit covariance maps; the linearised solver above applies verbatim.
Against Monte Carlo of the masked model it is within the Monte-Carlo floor ($W_{1,\eta_{\text{sm}}}$ $0.037$ vs.\ floor $0.059$ at $n=4$, $S=1$), while the isotropic substitution is eight times worse.

\section{The nested first moment for general branches}
\label{app:nested-branches}

Every statement of Section~\ref{sec:spectrum} reads the branch only through the covariance map $\eta_\ell$ of the layer noise, so it holds for any branch whose map is known.
The statements below are proved in Lean at every width $p$ for deterministic routings; \Cref{tab:nested-certified} collects the certified statements of the nested theory.

\begin{proof}[Proof of \Cref{thm:gated}]
  (i) For one block with branch Jacobian $Y_{\text{br},\ell}\coloneqq\sum_mR_{\ell m}\otimes X_{\ell m}$, second moments alone give $\mathbb E[Y_{\text{br},\ell}^\top(b\otimes I_p)Y_{\text{br},\ell}]=\eta_\ell(b)\otimes I_p$\lean[Formalised]{FreeProbReal.TokenFFN.integral_route_corr}.
  The pieces are centred and independent of the other blocks, so the cross terms between skip and branch vanish and the blocks compose from the output inwards, as in~\eqref{eq:nested-m1}\lean[Formalised]{FreeProbReal.TokenFFN.integral_jac_gram_corr}.
  The masked MLP is the single piece $X_\ell=D_\ell W_\ell$, whose covariance is $\sigma_{w,\ell}^2\psi_\ell$ when the mask is independent of the weights\lean[Conditional]{FreeProbReal.GatedModel.integral_jac_gram_masked}; an explicit Bernoulli-mask Gaussian model meets these hypotheses\lean[Formalised]{FreeProbReal.GatedModel.maskModel_hyps}.
  Attention is treated in the paragraph of that name below. (ii) Uncorrelated pieces give a Kraus map, which is positive\lean[Formalised]{FreeProbReal.RoutedBranch.etaKraus_posSemidef}, so $\mathcal T_\ell$ maps $b\succeq0$ to a positive matrix with $\mathbf 1^\top\mathcal T_\ell(b)\mathbf 1\ge\mathbf 1^\top H_{\text{res},\ell}^\top bH_{\text{res},\ell}\mathbf 1=\mathbf 1^\top b\,\mathbf 1$.
  Hence $\Phi_{1:S}\coloneqq\mathcal T_1\circ\cdots\circ\mathcal T_S(I_{n_q})\succeq0$ has $\mathbf 1^\top\Phi_{1:S}\mathbf 1\ge n_q$, and $\operatorname{tr}\Phi_{1:S}\ge\lambda_{\max}(\Phi_{1:S})\ge\mathbf 1^\top\Phi_{1:S}\mathbf 1/n_q\ge1$ ($\lambda_{\max}$: largest eigenvalue), i.e.\ $m_1=\tfrac1{n_q}\operatorname{tr}\Phi_{1:S}\ge1/n_q$\lean[Formalised]{FreeProbReal.RoutedBranch.nestedM_perron_kraus}. (iii) Let the gates of block $\ell$ read $x$ through a map to $\mathbb R^{r_{\text g,\ell}}$; our gates read it through their logits, the $r_{\text g,\ell}=3n^2$ outputs of $W_{\text{fused}}\,\mathrm{LayerNorm}(x_{\text{flat}})$~\eqref{eq:fused-proj} for the Cayley arm and $r_{\text g,\ell}=2n+n^2$ for the Sinkhorn arm (Appendix~\ref{sec:sinkhorn-run}).
  The chain rule splits the block's Jacobian into the frozen-gate Jacobian and a term that factors through $\mathbb R^{r_{\text g,\ell}}$, of rank at most $r_{\text g,\ell}$\lean[Formalised]{FreeProbReal.GateRank.fderiv_gated_split}.
  Ranks add along the product\lean[Formalised]{FreeProbReal.GateRank.prod_sub_prod_rank}, so $J_{\text f}-J_{\text s}$ has rank at most $\sum_\ell r_{\text g,\ell}$, and $J_{\text f}^\top J_{\text f}-J_{\text s}^\top J_{\text s}$ at most twice that\lean[Formalised]{FreeProbReal.GateRank.gram_sub_rank}.
  A rank-$k$ change of a symmetric matrix moves at most $k$ eigenvalues across any threshold\lean[Formalised]{FreeProbReal.GateRank.count_gt_rank}\lean[Formalised]{FreeProbReal.GateRank.ks_le_rank}, so for every weight draw the eigenvalue distribution functions of $J_{\text f}^\top J_{\text f}$ and $J_{\text s}^\top J_{\text s}$ differ by at most $2\sum_\ell r_{\text g,\ell}/N$ at every $\lambda$\lean[Formalised]{FreeProbReal.GateRank.ks_full_static}.
  No probability and no norm bound on the gate term enter; the bound says nothing about $m_1$ or outliers, whose size is unbounded, and is informative only when $p\gg n^2$.
\end{proof}

\paragraph{Routed branches.}
Let the noise of layer $\ell$ be $\sum_mR_{\ell m}\otimes X_{\ell m}$ with deterministic $n\times n$ routings $R_{\ell m}$ and $\mathbb E[(X_{\ell m})_{ij}(X_{\ell m})_{kl}]=\sigma_{\ell m}^2\delta_{ik}\delta_{jl}/p$.
One piece $R\otimes X$ gives $\mathbb E[(R\otimes X)^\top(b\otimes I_p)(R\otimes X)]=\sigma^2R^\top bR\otimes I_p$ with no centring, independence or Gaussianity assumed\lean[Formalised]{FreeProbReal.RoutedBranch.integral_route}, independent pieces add\lean[Formalised]{FreeProbReal.RoutedBranch.integral_kraus}, and with independent layers $\mathbb E[J^\top J]=\mathcal T_1\circ\cdots\circ\mathcal T_S(I_n)\otimes I_p$ for the Kraus maps $\eta_\ell(b)=\sum_m\sigma_{\ell m}^2R_{\ell m}^\top bR_{\ell m}$\lean[Formalised]{FreeProbReal.RoutedBranch.integral_jac_gram_kraus}.
The rank-one gate $uv^\top$ and the isotropic noise ($n^2$ matrix units of weight $\sigma^2/n$) are special cases\lean[Formalised]{FreeProbReal.RoutedBranch.etaGate_eq_etaRoute}\lean[Formalised]{FreeProbReal.RoutedBranch.etaIso_eq_etaKraus}; the map is positive\lean[Formalised]{FreeProbReal.RoutedBranch.etaKraus_posSemidef} and keeps the Perron bound $m_1\ge1/n$ for mixers fixing $\mathbf 1$\lean[Formalised]{FreeProbReal.RoutedBranch.nestedM_perron_kraus}.
A routing that is random but independent of the weights it multiplies gives the expected map\lean[Conditional]{FreeProbReal.RoutedBranch.integral_route_frozen}.
The executable\lean{FreeProb.RoutedBranch.krausM1} computes $m_1$.

\paragraph{Attention.}
On $T$ tokens the stream space is $\mathbb R^T\otimes\mathbb R^n$.
A block with per-token gates $u_t,v_t$, mixers $H_{\text{res},t}$ and a frozen attention pattern $\mathcal A$ has Jacobian $\bigl(\bigoplus_tH_{\text{res},t}\bigr)\otimes I_p+(\mathcal U\mathcal A\mathcal V^\top)\otimes W_{OV}$, with $\mathcal U=\bigoplus_tu_t$, $\mathcal V=\bigoplus_tv_t$ and $W_{OV}\coloneqq W_OW_V$ the product of the output and value projections\lean[Caveat]{FreeProbReal.SeqAttention.seqJac}: its value path is the routed branch with routing $\mathcal U\mathcal A\mathcal V^\top$.
Heads are uncorrelated pieces\lean[Formalised]{FreeProbReal.SeqAttention.integral_mha}, so a stack of layers with $n_{\text{head}}$ heads has $\mathbb E[J^\top J]=\mathcal T_1\circ\cdots\circ\mathcal T_S(I_{Tn})\otimes I_p$ with $\eta_\ell(b)=\tfrac{\sigma_\ell^2}{n_{\text{head}}}\sum_m(\mathcal U_\ell\mathcal A_{\ell m}\mathcal V_\ell^\top)^\top b\,(\mathcal U_\ell\mathcal A_{\ell m}\mathcal V_\ell^\top)$\lean[Caveat]{FreeProbReal.SeqAttention.integral_jac_gram_attn}.
With uniform gates and orthogonal $H_{\text{res},t}$ one layer has $m_1=1+\tfrac{\sigma^2}{n}\tfrac1T\operatorname{tr}(\mathcal A^\top\mathcal A)$\lean[Formalised]{FreeProbReal.SeqAttention.attn_m1_uniform_gates}; with uniform attention and orthogonal mixers fixing $\mathbf 1$, $m_1=1+\bigl(\prod_\ell(1+\sigma_\ell^2)-1\bigr)/(Tn)$\lean[Caveat]{FreeProbReal.SeqAttention.attn_uniform_perron}, the gated result of Section~\ref{sec:gated-theory} with $n$ replaced by $Tn$.
The caveat throughout is the score path $\partial\mathcal A/\partial x$, which is dropped.
Executables\lean{FreeProb.SeqAttention.attnM1}\lean{FreeProb.SeqAttention.attnJac}.

\paragraph{Token-wise FFN.}
A token-wise FFN shares its weights across tokens while token $t$ has its own activation mask $D_t$, so its pieces are correlated through the co-activation $\dot K_{tt'}\coloneqq\tfrac1{p_{\text{hid}}}\operatorname{tr}(D_tD_{t'})$, $p_{\text{hid}}$ the hidden width.
Correlated pieces with cross-covariance $\Theta\succeq0$ give the map $\eta(b)=\sum_{m,m'}\Theta_{mm'}R_m^\top bR_{m'}$, which is positive\lean[Formalised]{FreeProbReal.TokenFFN.etaCorr_posSemidef}, for one layer\lean[Formalised]{FreeProbReal.TokenFFN.integral_route_corr} and for a stack\lean[Formalised]{FreeProbReal.TokenFFN.integral_jac_gram_corr}.
For the FFN with input and output weights of variances $\sigma_{\text{in}}^2/p$ and $\sigma_{\text{out}}^2/p_{\text{hid}}$, $\Theta=\sigma_{\text{in}}^2\sigma_{\text{out}}^2\dot K$ with deterministic masks\lean[Caveat]{FreeProbReal.TokenFFN.integral_tok_mlp} and $\Theta=\sigma_{\text{in}}^2\sigma_{\text{out}}^2\,\mathbb E\dot K$ with masks independent of the weights\lean[Conditional]{FreeProbReal.TokenFFN.integral_tok_mlp_frozen}.
A block-diagonal input and the trace see only the diagonal $\dot K_{tt}$\lean[Formalised]{FreeProbReal.TokenFFN.etaTok_blockDiag}\lean[Formalised]{FreeProbReal.TokenFFN.etaTok_trace}: the off-diagonal co-activation enters a first moment only when tokens are mixed on both sides of the FFN, and every two-layer stack is blind to it.
Executable\lean{FreeProb.TokenFFN.blkM1}.

\paragraph{Product branches beyond the first moment.}
A branch $(R_{\text{out}}\otimes W_{\text{out}})\,\Delta_{\text{hid}}\,(R_{\text{in}}\otimes W_{\text{in}})$ through a hidden width, with weights $W_{\text{in}},W_{\text{out}}$, routings $R_{\text{in}},R_{\text{out}}$ into and out of the hidden blocks and a hidden factor $\Delta_{\text{hid}}$ equal to $\Delta_{\text{hid},k}$ on hidden class $k$, is linearised with hidden blocks of a different size: the pencil rows reduce to $z\varpi_0-Y^\top Y\varpi_0$\lean[Formalised]{FreeProbReal.RectPencil.rect_pencil_rows}, and composing through the hidden block gives the Kraus maps of the pieces $R_{\text{out}}\Delta_{\text{hid},k}R_{\text{in}}$, the width ratio cancelling\lean[Formalised]{FreeProbReal.RectPencil.rect_composite_etaR}\lean[Formalised]{FreeProbReal.RectPencil.rect_composite_etaL}.
The MLP and the attention value path are instances\lean[Formalised]{FreeProbReal.RectPencil.rect_mlp_etaR}\lean[Formalised]{FreeProbReal.RectPencil.rect_attn_etaR}, so the linearised equation~\eqref{eq:lin-dyson} applies to them with these maps\lean{FreeProb.RectPencil.rectSolve}.

\paragraph{Forward second moments.}
The parameters of these maps (the gain $\psi$ of Section~\ref{sec:signal-prop}, the co-activation $\dot K$) come from the forward pass.
For one gated block, exactly over its output weights (variance $\sigma_{\text{out}}^2/p_{\text{hid}}$), $\mathbb E\,S_{\text{out}}(t,i;t',j)=S_{\text{skip}}(t,i;t',j)+\sigma_{\text{out}}^2u_{t,i}u_{t',j}\tfrac1{p_{\text{hid}}}\langle\phi(h_t),\phi(h_{t'})\rangle$, with $S_{\text{out}}$ and $S_{\text{skip}}$ the second moments $\tfrac1p\sum_kx_{t,i,k}x_{t',j,k}$ of the block output $x$ and of its skip part, and $\phi(h_t)$ the hidden features of token $t$\lean[Formalised]{FreeProbReal.ForwardKernel.fwd_block_second_moment}; the same holds in expectation for an untied network at initialisation\lean[Formalised]{FreeProbReal.ForwardKernel.fwd_block_second_moment_indep}.

\paragraph{Measured.}
In the sequence probe of Appendix~\ref{app:mc} ($n=4$, paired permutation tests), attention with frozen patterns and gates also matches its exact first moment within $1.42$ standard errors, and live scores change the spectrum in $14$ of the $\mnSeqCells$ cells.
Its gate cells (two token-wise MLP blocks on $n=2$ streams of width $p\in\{16,32,64\}$, three mixers) compare the full Jacobian with the one whose gates are detached: live gates change it by $35\%$ to $69\%$ in relative norm on average.
This is a different experiment from the Cayley block of Appendix~\ref{app:mc} ($n=4$, $S=1$ to $32$), whose Jacobian changes by $\mnDetachLo\%$ to $\mnDetachHi\%$ when its gates are detached.
The rank bounds and the eigenvalue-count bound of \Cref{thm:gated}(iii) hold in every sample of the $25$ score and gate cells, checked on binary64 spectra, and a mutant pairing different samples exceeds the rank bound wherever that bound is below $N$.
The forward second moments agree with the Lean Monte Carlo in $4$ of $4$ cells, and a mutant that adds the previous layer's activation gain $\psi_{\ell-1}$ to the variance of the branch input is rejected in each.

\begin{table}[t]
  \centering\small
  \caption{Certified statements of the nested theory (tiers as in Appendix~\ref{app:lean}).}\label{tab:nested-certified}
  \begin{tabular}{@{}p{0.62\linewidth}p{0.3\linewidth}@{}}
    \toprule
    Statement & Lean, tier \\
    \midrule
    $\mathbb E\frac1N\operatorname{tr}J^\top J=\frac1n\operatorname{tr}\mathcal T_1\circ\cdots\circ\mathcal T_S(I_n)$ at every width, Gaussian noise
    & \lean[Formalised]{FreeProbReal.GaussianMoment.gauss_m1} Formalised \\
    co-isometric mixers: nested $=$ flat $=$ scalar product at every depth
    & \lean[Formalised]{FreeProbReal.GaussianMoment.iso_of_coisometric_gauss}\lean[Formalised]{FreeProbReal.GaussianMoment.nestedM_eq_flatM_of_scalar_gauss} Formalised \\
    commuting mixers are not enough; no Gram-only formula is exact
    & \lean[Formalised]{FreeProbReal.GaussianMoment.hgud3_gauss}\lean[Formalised]{FreeProbReal.GaussianMoment.gram_flat_eq_nested_ne_gauss}\lean[Formalised]{FreeProbReal.FlatNested.hgud1} Formalised \\
    mixers fixing $\mathbf 1$: $m_1\ge1/n$ at every depth; the scalar product collapses
    & \lean[Formalised]{FreeProbReal.GaussianMoment.nestedM_perron_gauss}\lean[Formalised]{FreeProbReal.FlatNested.scalar_collapse} Formalised \\
    one layer: subordination over the eigenvalues of $A^\top A$
    & \lean[Formalised]{FreeProbReal.OVDyson.ovF_eq_subordF_eigs} Formalised \\
    the corner block of the pencil is the resolvent of $J^\top J$
    & \lean[Formalised]{FreeProbReal.OVLin.corner_pencil_inv} Formalised \\
    the Dyson iterates give the nested first moment; co-isometric reduction to the scalar product
    & \lean[Caveat]{FreeProbReal.OVLin.tendsto_firstMoment}\lean[Caveat]{FreeProbReal.OVLin.isSol_coiso_iff} Caveat (cited semicircle law) \\
    gated block: rank-one noise, not isotropic
    & \lean[Formalised]{FreeProbReal.GatedModel.combine_eq}\lean[Formalised]{FreeProbReal.GatedModel.etaGate_ne_etaIso} Formalised \\
    gated block: nested moment with an independent mask
    & \lean[Conditional]{FreeProbReal.GatedModel.integral_jac_gram_masked} Conditional \\
    gated block: orthogonal mixers fixing $\mathbf 1$ isometric off the Perron line
    & \lean[Caveat]{FreeProbReal.GatedModel.m1_uniform_perron} Caveat \\
    routed branches (Kraus maps), Perron bound
    & \lean[Formalised]{FreeProbReal.RoutedBranch.integral_jac_gram_kraus}\lean[Formalised]{FreeProbReal.RoutedBranch.nestedM_perron_kraus} Formalised \\
    multi-head attention with frozen patterns
    & \lean[Caveat]{FreeProbReal.SeqAttention.integral_jac_gram_attn} Caveat (score path) \\
    token-wise FFN (correlated pieces), blind spots of the first moment
    & \lean[Formalised]{FreeProbReal.TokenFFN.integral_jac_gram_corr}\lean[Formalised]{FreeProbReal.TokenFFN.etaTok_trace} Formalised \\
    product branches: rectangular linearisation
    & \lean[Formalised]{FreeProbReal.RectPencil.rect_composite_etaR} Formalised \\
    dropped gate derivatives: pathwise spectral distance
    & \lean[Formalised]{FreeProbReal.GateRank.ks_full_static} Formalised \\
    forward second moments of a gated block
    & \lean[Formalised]{FreeProbReal.ForwardKernel.fwd_block_second_moment} Formalised \\
    \bottomrule
  \end{tabular}
\end{table}

\section{Solver}
\label{app:hrfs}

We solve~\eqref{eq:lin-dyson} by the averaged iteration of \citet{helton2007operator}, $G\leftarrow\tfrac12\bigl(G+\mathcal F_z(G)\bigr)$\lean{FreeProb.OVLin.hrfsStep}, with continuation in $\operatorname{Im}z$.

\begin{algorithm}[H]
  \caption{Linearised Dyson solve at one point\lean{FreeProb.OVLin.ovLinSolve}}
  \label{alg:ovlin}
  \begin{algorithmic}[1]
    \Require layers $(A_\ell,\eta^{\mathrm L}_\ell,\eta^{\mathrm R}_\ell)_{\ell\le S}$, $z$ with $\operatorname{Im}z>0$, continuation height $\eta_{\text{cont}}$, levels $n_{\text{lev}}$, cap $s_{\max}$, tolerance $\varepsilon_{\text{tol}}$
    \Ensure $g(z)$, residual, convergence flag
    \State $G\gets\mathcal F_{z+i\eta_{\text{cont}}}(0)$ \Comment{noiseless resolvent at the top level}
    \For{$k=0,\ldots,n_{\text{lev}}$}
      \State $z_k\gets z+i(n_{\text{lev}}-k)\,\eta_{\text{cont}}/n_{\text{lev}}$
      \For{$t=1,\ldots,s_{\max}$}
        \State $G_{\text{new}}\gets\mathcal F_{z_k}(G)$
        \If{$\max_{ij}|(G_{\text{new}}-G)_{ij}|^2<\varepsilon_{\text{tol}}^2\max(1,\max_{ij}|G_{ij}|^2)$} \textbf{break} \EndIf
        \State $G\gets\tfrac12(G+G_{\text{new}})$
      \EndFor
    \EndFor
    \State \Return $g(z)=\tfrac1{n_{\text{blk}}}\operatorname{tr}G_{00}$, the last residual, and whether it is below $\varepsilon_{\text{tol}}^2$
  \end{algorithmic}
\end{algorithm}

A density is never reported from an unconverged point: every grid point's flag is checked.

\section{Structured solve}
\label{app:structured}

$\mathcal Z_\varepsilon(z)-\mathcal P_0-\Sigma(G)$ is block-tridiagonal and $\Sigma$ reads only the diagonal blocks of $G$, so the diagonal blocks of the iterates form a closed sequence.
They follow from forward and backward Schur complements over the $S+1$ diagonal groups $\{2m-1,2m\}$ of blocks ($2n_{\text{blk}}\times2n_{\text{blk}}$): $3(S+1)$ inversions of size $2n_{\text{blk}}$ per step instead of one of size $2Sn_{\text{blk}}$\lean{FreeProb.OVLin.ovLinSolveS}.
The structured solve reproduces the dense one to $2\cdot10^{-6}$ relative on $S\in\{2,3,10\}$, all four mixers; at $S=10$, $n=4$ a grid point costs about $0.4$\,s instead of $12$\,s.
A point at which the binary32 structured solve does not converge is re-solved with the dense pencil\lean{FreeProb.OVLin.ovLinSolveC}; in Figure~\ref{fig:ov-vs-mc-panel} this happens at $2$ grid points in total, out of $\mnGridPoints$ in each of the $\mnFigCells$ panels.

\section{Density recovery}
\label{app:density}

The density on a grid is $\rho(\lambda)=-\tfrac1\pi\operatorname{Im}g(\lambda+i\eta_{\text{sm}})$, optionally refined by Richardson extrapolation over $\eta_{\text{sm}},\eta_{\text{sm}}/2,\ldots$ with weight $2$ (the $O(\eta_{\text{sm}})$ error of the Poisson smoothing)\lean{FreeProb.OVLin.ovLinDensity}\lean{FreeProb.Density.richardsonHalving}.
Weights $4^j$, appropriate for an $O(\eta_{\text{sm}}^2)$ error, extrapolate in the wrong direction here.

\section{Monte-Carlo validation}
\label{app:mc}

\begin{figure}[t]
  \centering
  \includegraphics[width=\textwidth]{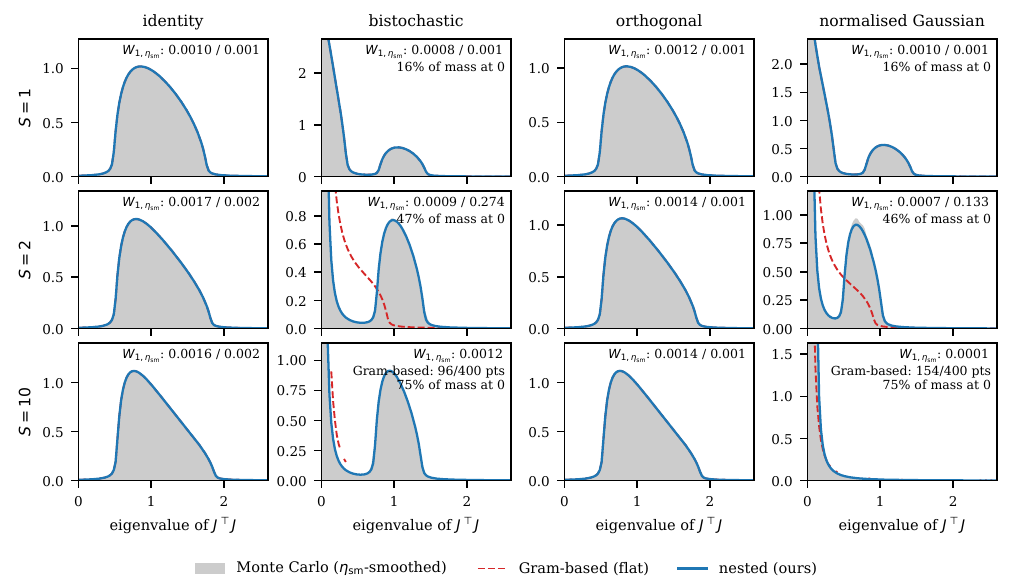}
  \caption{\textbf{Theory vs.\ Monte Carlo spectra of $J^\top J$} ($n=4$ streams, $p=25$, $S\sigma^2=0.05$, smoothing width $\eta_{\text{sm}}=0.02$) at $S\in\{1,2,10\}$ mixer applications (rows) for identity, bistochastic, orthogonal and normalised-Gaussian mixers (columns); Figure~\ref{fig:nested-l10} repeats the orthogonal and bistochastic cells at $S=10$.
  Grey: Monte Carlo ($300/200/100$ samples of $J$ at $S=1/2/10$), Poisson-smoothed at $\eta_{\text{sm}}$ like the theory curves; red dashed: Gram-based (flat) composition, drawn where its Newton solve converged; blue: nested theory (Appendix~\ref{app:pencil}).
  Annotations: 1-Wasserstein distance $W_{1,\eta_{\text{sm}}}$ of nested / Gram-based theory to Monte Carlo; where the Gram-based solve converged at only part of the grid, the count of converged points.
  Mixers are drawn once per family from a PCG32 stream with seed $42$: Sinkhorn of $\exp Z$, Haar orthogonal, $Z/\|Z\|_2$, with $Z$ i.i.d.\ $\mathcal{N}(0,1)$\lean{FreeProb.twist}.
  All panels are produced by the Lean executable\lean{FreeProb.Fig1.validate}.}
  \label{fig:ov-vs-mc-panel}
\end{figure}

Monte Carlo samples $J^\top J$ at $p=25$ with i.i.d.\ Gaussian entries from a seeded PCG32 generator\lean{FreeProb.gaussianMat} and computes its eigenvalues with the Jacobi method\lean{FreeProb.Fig1.mcPhase}.
The theory density at $\lambda+i\eta_{\text{sm}}$ is the Poisson smoothing of the limiting law at scale $\eta_{\text{sm}}$, so we compare it with the \emph{same} smoothing of the sample eigenvalues $\lambda_k$, $\rho^{\mathrm{MC}}_{\eta_{\text{sm}}}(\lambda)\coloneqq\tfrac1N\sum_k\tfrac{\eta_{\text{sm}}/\pi}{(\lambda-\lambda_k)^2+\eta_{\text{sm}}^2}$\lean{FreeProb.Fig1.poissonSmooth}, both normalised over the grid, and report the 1-Wasserstein distance $W_{1,\eta_{\text{sm}}}$ between them.
Comparing a smoothed theory with a raw sample instead puts the collapsed mass near $0$ (up to $\mnBistoTenCollapsed\%$ of the eigenvalues for the bistochastic mixer at $S=10$) on different sides of the first grid point and dominates the distance.
The acceptance rule, fixed before any nested-theory number existed: $W_{1,\eta_{\text{sm}}}$ at most twice that of the identity mixer at the same depth and at most $0.03$, and the first moment within three standard errors of Monte Carlo\lean{FreeProb.Fig1.validate}.

\paragraph{Results.}
For identity and orthogonal mixers the two compositions coincide and match Monte Carlo at every depth ($W_{1,\eta_{\text{sm}}}\le\mnIsoWoneMax$).
For bistochastic and Gaussian mixers at $S=2$ the Gram-based curves miss ($W_{1,\eta_{\text{sm}}}=\mnGramWoneBistoTwo$ and $\mnGramWoneGaussTwo$); at $S=10$ the Gram-based solve converges at only $\mnGramConvBistoTen$ and $\mnGramConvGaussTen$ of the $\mnGridPoints$ grid points, and deviates from Monte Carlo even there.
The nested theory stays at the Monte-Carlo floor ($W_{1,\eta_{\text{sm}}}\le\mnNestedWoneMax$, every grid point converged), including the collapsed mass at $0$ ($\mnBistoTenCollapsed\%$ of the eigenvalues for the bistochastic mixer at $S=10$), and all $\mnFigCells$ cells meet the acceptance rule.
Appendix~\ref{app:families} repeats the check beyond Gaussian branches: every cell passes for gated, linear, MLP and transformer branches, fewer than half for elementwise ReLU, tanh and GELU branches.

\paragraph{The gated block.}
We measure the Jacobian of a PyTorch implementation of the block of \Cref{thm:gated} (autograd) against the theory.
With the gates frozen at uniform values, the network agrees with the gated theory in all $\mnGatedCells$ mixer$\times$depth cells (permutation test on the spectra, $\ge100$ matrices per side\lean{FreeProb.Probe.permTest}), while the isotropic-noise model with the same variance is rejected in all of them.
With live gates, detaching $H_{\text{pre}},H_{\text{post}},H_{\text{res}}$ in the backward pass changes the Cayley block's Jacobian by $\mnDetachLo\%$ to $\mnDetachHi\%$ in relative norm ($S=1$ to $32$); the dropped gate derivatives carry $\mnDetachShare\%$ of its excess first moment, which reaches $\mnCayleyMoneTen$ at $S=10$.
The frozen-gate theory therefore describes the mixer and branch, not the gate dynamics.
The dropped terms are large in norm but of low rank: our gates read the input through the $r_{\text g,\ell}=3n^2$ outputs of the fused projection $W_{\text{fused}}$~\eqref{eq:fused-proj}, so by \Cref{thm:gated}(iii) they move at most $\mnGateRankPct\%$ of the spectrum per block at $n=4$, $p=512$ ($r_{\text g,\ell}=\mnGateRankR$), a bound vacuous from $\mnGateRankBlocks$ blocks on.
With frozen attention patterns and gates, attention stacks agree with the attention theory in all $\mnSeqCells$ cells (four mixers, $T\in\{2,8\}$, $S\in\{1,2\}$), and correlated-head and token-wise mutants are rejected in $\mnSeqMutants$; live scores change the Jacobian by $\mnScoreLo\%$ to $\mnScoreHi\%$ on average (Appendix~\ref{app:nested-branches}).

\section{Monte-Carlo validation beyond Gaussian branches}
\label{app:families}

Figures~\ref{fig:fam-gated}--\ref{fig:fam-transformer} repeat Figure~\ref{fig:ov-vs-mc-panel} for branches that are not a single Gaussian matrix, in the same skeleton $x_{\ell+1}=(A\otimes I)\,x_\ell+a\,F_\ell(x_\ell)$ with the four mixers of Figure~\ref{fig:ov-vs-mc-panel}\lean{FreeProb.twist}, $n=4$ streams ($p=25$; $64$ coordinates per stream for the CNN and the transformer), $S$ mixer applications ($S=L$ for $L$ layers; $S=2L$ for the threaded transformer), budget $Sa^2=0.05$, $\eta_{\text{sm}}=0.02$ and the same distance and acceptance rule.
Each panel compares Monte Carlo (grey) with our theory (blue) and with a comparator that removes the structure our theory keeps (red dashed).
Monte Carlo and both theories are computed in Lean (experiment \texttt{spectral-families} of the Lean repository); the branches with an activation take the pre-activation variance $q_\ell$ of each layer from the Monte-Carlo forward pass, which given $x_\ell$ is exact.
\Cref{tab:families} summarises the eight families.
The acceptance rule of Appendix~\ref{app:mc} is deliberately strict: fixed before any nested-theory number existed, it is set at the Monte-Carlo floor (twice the identity mixer's distance at the same depth), so it rejects cells whose theory curve agrees with Monte Carlo to the eye (Figures~\ref{fig:fam-relu}--\ref{fig:fam-gelu} and~\ref{fig:fam-cnn}).
The cells that fail in \Cref{tab:families} fail only this strict tolerance.

\paragraph{Gated branches.}
With uniform gates the layer noise is $\sigma(P\otimes X_\ell)$, with the Perron projector $P\coloneqq\ones\ones^\top/n$, and its covariance map $b\mapsto\sigma^2(\ones^\top b\ones/n)P$ is rank one (\Cref{thm:gated}).
The nested theory with this map is within $6.6\cdot10^{-5}$ to $2.5\cdot10^{-3}$ of Monte Carlo; isotropic noise with the same first moment is ten to seventy times further (Figure~\ref{fig:fam-gated}).

\paragraph{Activation branches.}
The block Jacobian is $A\otimes I_p+aD_\ell W_\ell$ with the masks $D_\ell$ of the forward pass.
The masks sit on the coordinates the mixer couples, so an exact operator-valued treatment carries every coordinate's mask history.
Our per-layer closure keeps the mask classes of one layer at a time ($\{0,1\}$ for ReLU; eight Gauss--Hermite nodes of $\phi'(\sqrt{q_\ell}Z)^2$ for tanh and GELU) in the pencil blocks that read them and sets every other layer to its gain $\psi_\ell$; this is a conjecture, and the comparator is the Gaussian surrogate of variance $a^2\psi_\ell$.
Without masks (linear branch) the two coincide and pass at every depth up to $60$.
With masks the closure is within a few percent of the surrogate except for ReLU and GELU with the orthogonal mixer, where it is up to $1.6$ times closer; these cells pass (GELU at $L=10$ misses the first-moment criterion at $3.1$ standard errors).
With the identity and bistochastic mixers both theories degrade with depth for ReLU and GELU ($W_{1,\eta_{\text{sm}}}$ up to $0.030$), and with the identity, bistochastic and Gaussian mixers the Monte-Carlo first moment of ReLU exceeds the value with masks independent of the weights (identity, $L=60$: $1.036$ against $1.025$), which no theory of independent masks reproduces. tanh shows a depth-independent gap of $0.006$--$0.007$ with the identity and orthogonal mixers under both theories (Figures~\ref{fig:fam-linear}--\ref{fig:fam-gelu}).

\paragraph{Threaded ResNet.}
For an MLP branch $aW_{\text{out}}D_\ell W_{\text{in}}$ ($p_{\text{hid}}=4N$) we write the block as the product of $[I;\,D_\ell W_{\text{in}}]$ and $[A\otimes I_p,\ W_{\text{out}}]$.
The hidden coordinates have no deterministic coupling, so their masks are exact (no closure) and their pencil blocks invert in closed form; on one stream with three mask classes this agrees with the dense rectangular pencil\lean{FreeProb.RectPencil.rectSolve} to $6\cdot10^{-6}$.
All twelve cells pass, within $10^{-3}$ of Monte Carlo (Figure~\ref{fig:fam-mlp}).
A CNN branch (circular convolution over $8$ positions, $3$ taps, $8$ channels per stream and position, weights shared by the positions) has an isotropic covariance on the stream blocks, so its theory is the activation one; it passes with the orthogonal mixer at every depth and degrades like ReLU with the others (Figure~\ref{fig:fam-cnn}).

\paragraph{Threaded transformer.}
A block is frozen uniform attention (the softmax at zero scores) over $T=8$ tokens followed by a token-wise ReLU MLP with $p_{\text{hid}}=4np$ per token, each with its own skip and mixer, $a^2=0.05/(2L)$ per sublayer.
With exchangeable tokens every block of the pencil reduces to the token mean and its complement: attention is the rank-one token covariance (the gated case over tokens), and the MLP's hidden masks are exact through the number of active tokens of each hidden unit.
At $L=10$ and $L=30$ all eight cells pass; the Gaussian surrogate of the MLP that keeps the token co-activation $\dot K$ is as close, within $2.5\cdot10^{-4}$ (Figure~\ref{fig:fam-transformer}).

\begin{table}[t]
  \centering\small
  \caption{Monte Carlo of $J^\top J$ beyond Gaussian branches: $W_{1,\eta_{\text{sm}}}$ of our theory and of the comparator over the cells of each family, and the cells that pass the acceptance rule of Appendix~\ref{app:mc}.}\label{tab:families}
  \begin{tabular}{@{}lllll@{}}
    \toprule
    Branch & Depths $L$ & Ours & Comparator & Pass \\
    \midrule
    gated, uniform gates & $1,2,10$ & $6.6\cdot10^{-5}$--$0.0025$ & $8.1\cdot10^{-4}$--$0.092$ & $12/12$ \\
    linear & $10,30,60$ & $2.4\cdot10^{-5}$--$0.0019$ & same & $12/12$ \\
    ReLU & $10,30,60$ & $2.9\cdot10^{-5}$--$0.030$ & $2.9\cdot10^{-5}$--$0.030$ & $4/12$ \\
    tanh & $10,30,60$ & $2.5\cdot10^{-5}$--$0.0072$ & $2.5\cdot10^{-5}$--$0.0071$ & $6/12$ \\
    GELU & $10,30,60$ & $3.0\cdot10^{-5}$--$0.019$ & $3.0\cdot10^{-5}$--$0.019$ & $6/12$ \\
    MLP (ReLU) & $10,30,60$ & $2.7\cdot10^{-5}$--$0.0010$ & $2.7\cdot10^{-5}$--$0.0015$ & $12/12$ \\
    CNN (ReLU) & $10,30,60$ & $3.1\cdot10^{-5}$--$0.034$ & $3.1\cdot10^{-5}$--$0.034$ & $5/12$ \\
    transformer (ReLU MLP) & $10,30$ & $2.3\cdot10^{-5}$--$0.0027$ & $2.3\cdot10^{-5}$--$0.0025$ & $8/8$ \\
    \bottomrule
  \end{tabular}
\end{table}

\newcommand{\famcaption}[2]{\caption{\textbf{Theory vs.\ Monte Carlo spectra of $J^\top J$, #1}, as Figure~\ref{fig:ov-vs-mc-panel} (columns: mixers; rows: depth $L$; $n=4$, $Sa^2=0.05$, $\eta_{\text{sm}}=0.02$).
Blue: #2; annotations: $W_{1,\eta_{\text{sm}}}$ of ours / comparator.}}

\begin{figure}[p]
  \centering
  \includegraphics[width=0.92\textwidth]{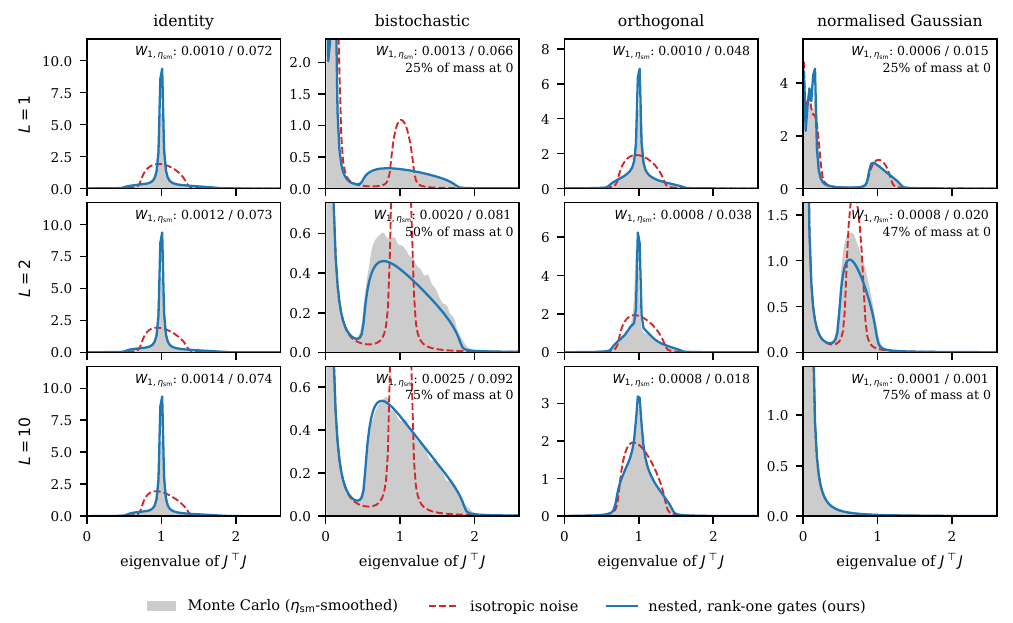}
  \famcaption{gated branch with uniform gates}{nested theory with the rank-one gate covariance; red dashed: isotropic noise of the same first moment}
  \label{fig:fam-gated}
\end{figure}

\begin{figure}[p]
  \centering
  \includegraphics[width=0.92\textwidth]{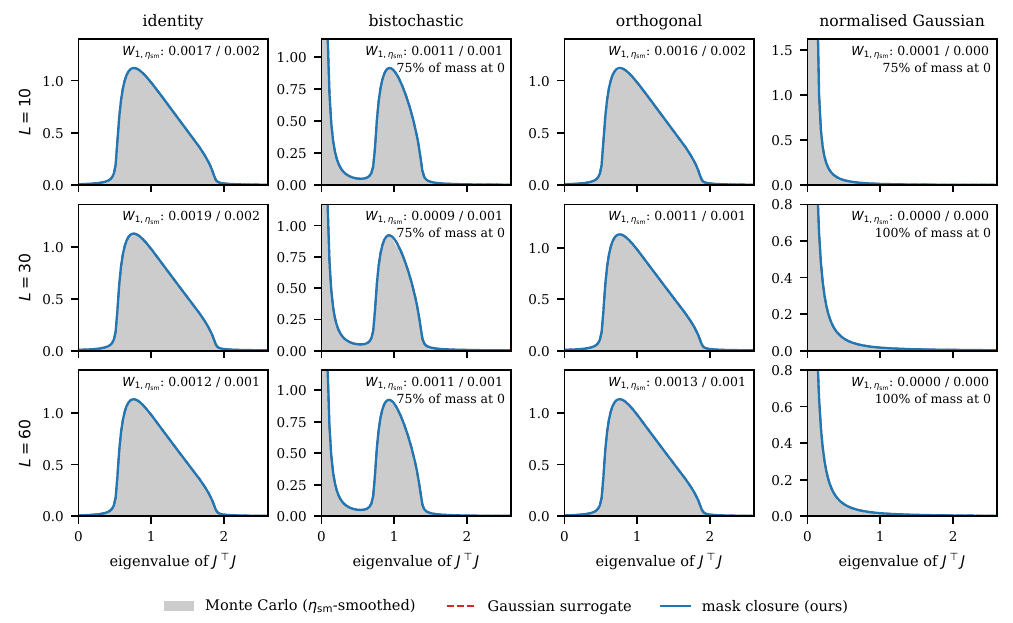}
  \famcaption{linear branch $aW_\ell$}{nested theory; red dashed: the same (no masks)}
  \label{fig:fam-linear}
\end{figure}

\begin{figure}[p]
  \centering
  \includegraphics[width=0.92\textwidth]{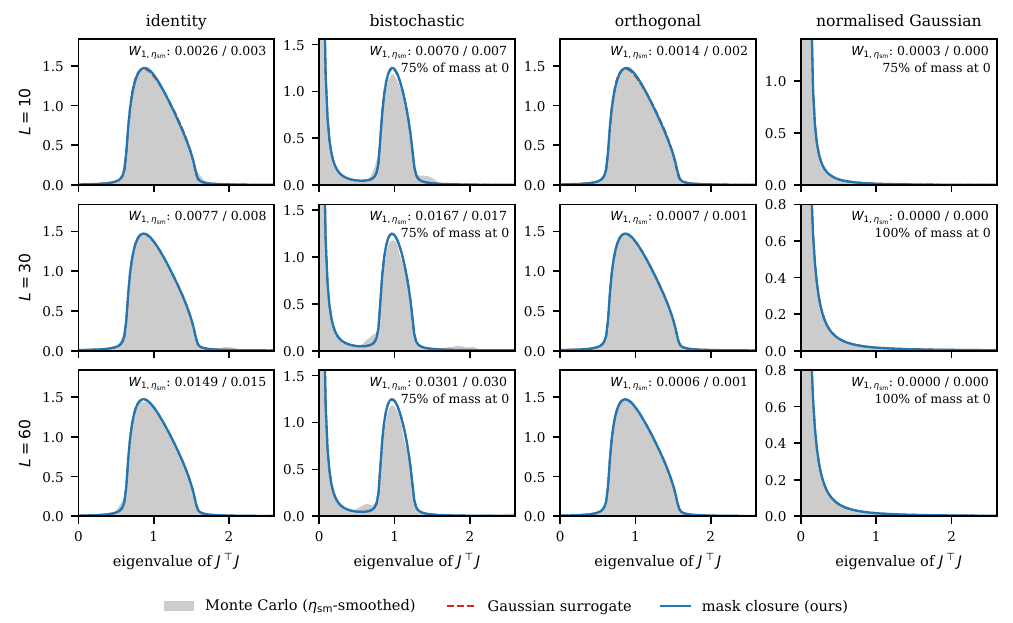}
  \famcaption{ReLU branch $aD_\ell W_\ell$}{per-layer mask closure; red dashed: Gaussian surrogate}
  \label{fig:fam-relu}
\end{figure}

\begin{figure}[p]
  \centering
  \includegraphics[width=0.92\textwidth]{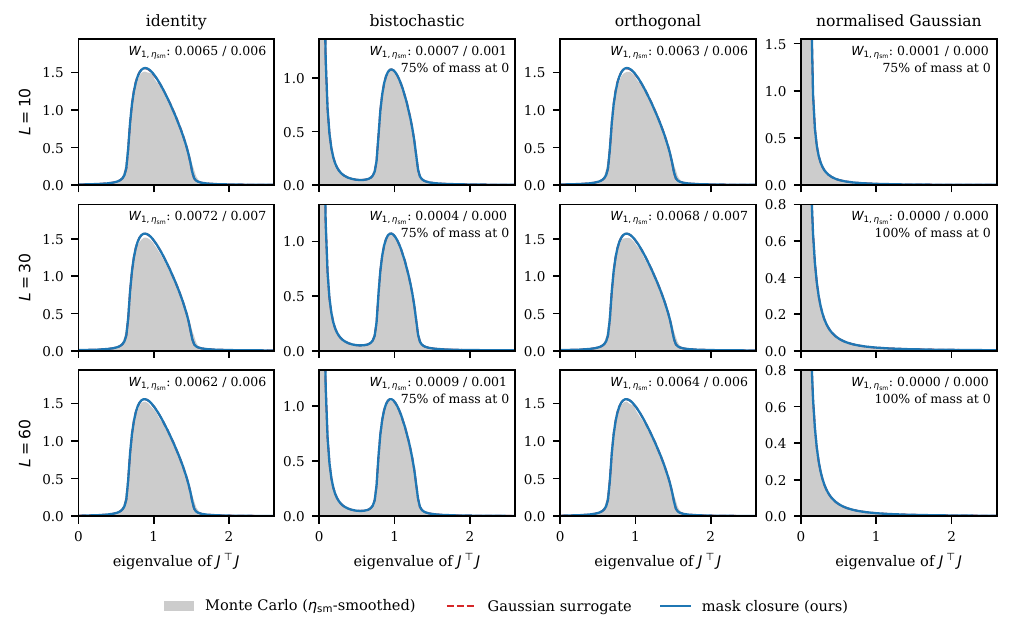}
  \famcaption{tanh branch}{per-layer mask closure; red dashed: Gaussian surrogate}
  \label{fig:fam-tanh}
\end{figure}

\begin{figure}[p]
  \centering
  \includegraphics[width=0.92\textwidth]{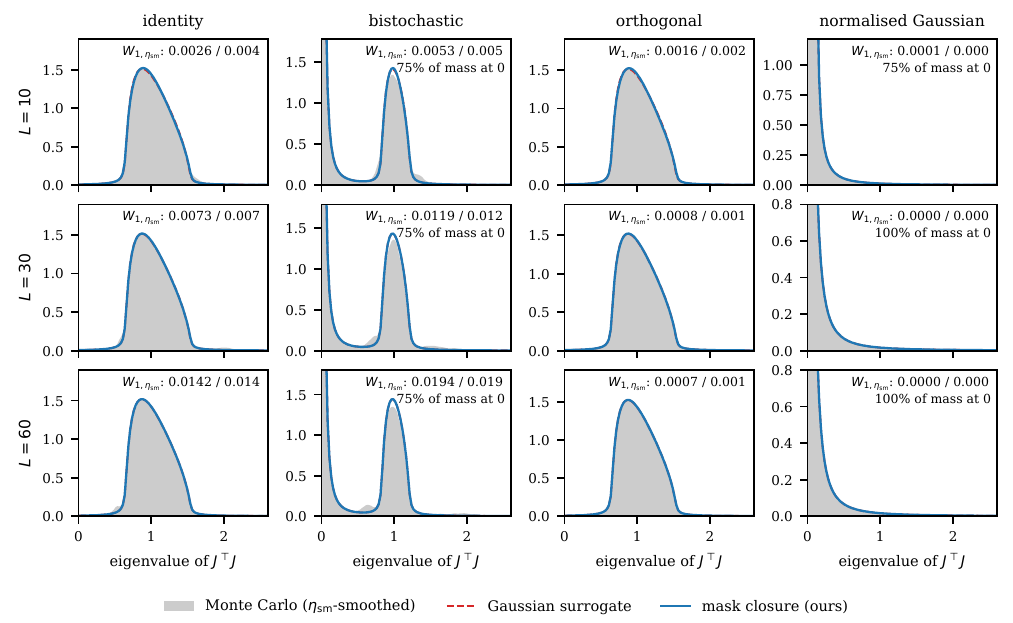}
  \famcaption{GELU branch (tanh approximation)}{per-layer mask closure; red dashed: Gaussian surrogate}
  \label{fig:fam-gelu}
\end{figure}

\begin{figure}[p]
  \centering
  \includegraphics[width=0.92\textwidth]{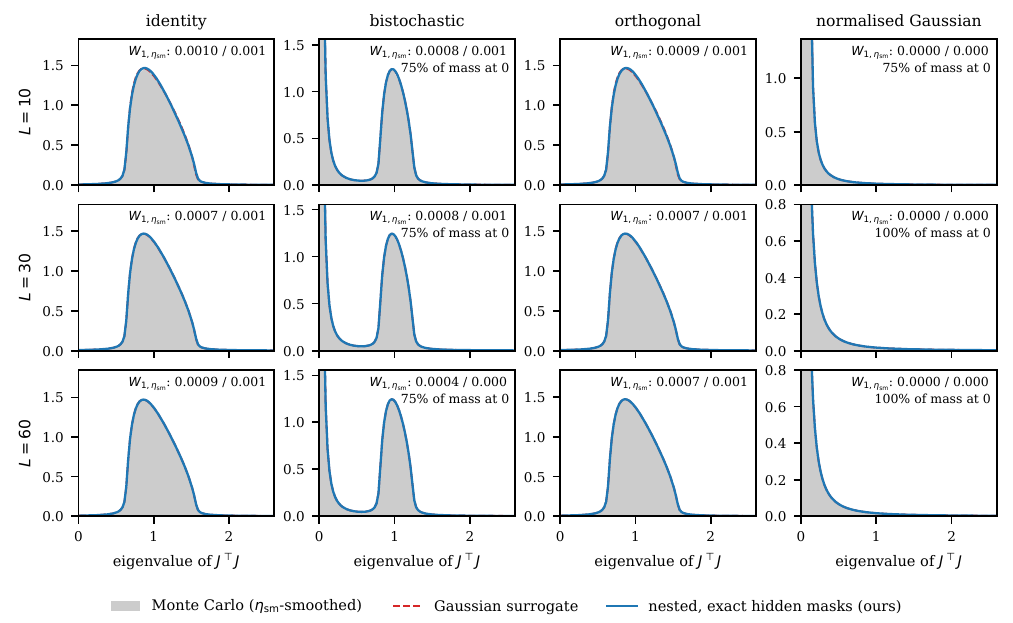}
  \famcaption{threaded ResNet, MLP branch $aW_{\text{out}}D_\ell W_{\text{in}}$ (ReLU, $p_{\text{hid}}=4N$)}{nested theory with the hidden masks exact; red dashed: Gaussian surrogate}
  \label{fig:fam-mlp}
\end{figure}

\begin{figure}[p]
  \centering
  \includegraphics[width=0.92\textwidth]{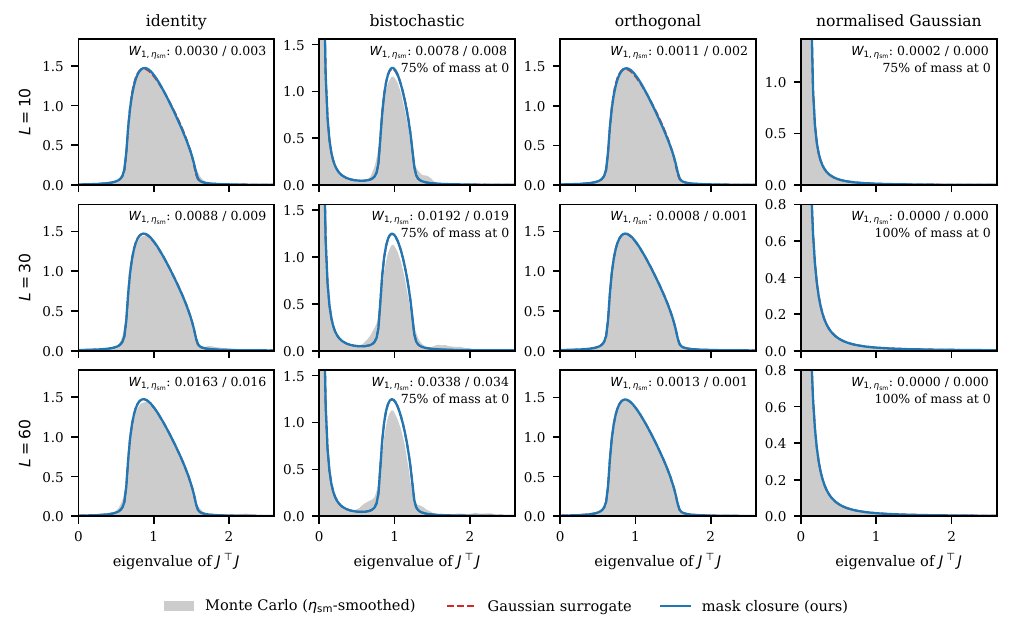}
  \famcaption{threaded ResNet, CNN branch (ReLU, circular convolution, $8$ positions, $3$ taps)}{per-layer mask closure; red dashed: Gaussian surrogate}
  \label{fig:fam-cnn}
\end{figure}

\begin{figure}[p]
  \centering
  \includegraphics[width=0.92\textwidth]{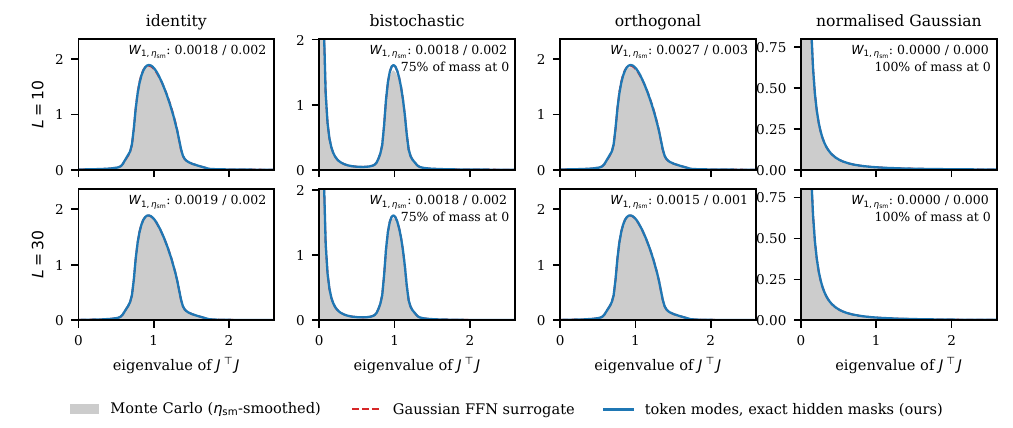}
  \famcaption{threaded transformer (uniform frozen attention over $T=8$ tokens, token-wise ReLU MLP)}{token-mode theory with the hidden masks exact; red dashed: Gaussian surrogate of the MLP with the token co-activation}
  \label{fig:fam-transformer}
\end{figure}
\clearpage

\section{Spectral and functional sensitivity of residual mixers}\label{app:sensitivity}

Sections~\ref{sec:spectrum}--\ref{sec:cayley-method} describe the Jacobian spectrum at initialisation.
This appendix asks two further questions: what training does to that spectrum, and when a mixer family reaches an unstable one.
The answers depend mainly on the mixer family; depth enters only logarithmically.
On the stream differences, training spends spectrum only through the mixers (\Cref{thm:sens-blind}): a doubly stochastic move always contracts them and the losses add over layers (\Cref{thm:sens-birkhoff,prop:sens-additive}), while an orthogonal move that fixes $\ones$ costs nothing.
At the replicated initialisation of our codebases, mixers that fix $\ones$ receive no gradient (\Cref{thm:sens-copy}); with Adam's $\varepsilon_{\text{A}}$ this keeps mHC's mixer where it started in the nano-GPT runs.
A family becomes unstable at depth when a dimensionless optimiser drive exceeds a distance that grows only logarithmically in depth (\Cref{thm:sens-adam}); weight decay instead pulls Sinkhorn logits toward the uniform matrix, which TRM's Sudoku and Maze recipes cannot escape (\Cref{thm:sens-decay}).
The criteria predict collapse for TRM (at initialisation from its default logits $b_{\text{res}} = I$, and during training from $\tau_{\text{init}} = 4$ under Sudoku's weight decay), stability for nano-GPT as configured, and collapse for nano-GPT with weight decay on its mixer logits; some of these runs are pending, and the predictions hold so far (\Cref{sec:sens-budget-values,sec:sens-ledger}).
The machinery is a pair of sensitivities.
The \emph{functional} sensitivity measures how fast a parameter group changes the loss; the \emph{spectral} sensitivity measures how fast it changes a functional of the spectrum.
Their ratio, integrated along training, is the spectrum spent per unit of loss (\Cref{thm:sens-identity}); it is set by the tangent geometry of the family (\Cref{sec:sens-geometry}) and by what the gradient can reach (\Cref{sec:sens-training}).
\Cref{sec:sens-criteria} states the depth criteria, \Cref{sec:sens-scaling} draws consequences for autoregressive and recursive (TRM) models, and \Cref{sec:sens-ledger} lists the missing theory and the registered predictions.

\paragraph{Status of the statements.}
Every theorem, proposition and corollary has a written proof except \Cref{prop:sens-thread}, whose proof is in Lean only.
A superscript \textsf{L} marks a statement, or the step a proof cites, that is proved in Lean (\Cref{app:lean}).
\emph{Measured} marks numbers from trained models or from a Float64 model of the block (linear or $\tanh$ branches, static gates); the measurements come with the repository.
Two codebases are used: a nano-GPT implementation of mHC and JPmHC (FineWeb-Edu, 2.5B tokens, $n = 4$, $L \in \{12, 24\}$) and the TRM implementation of \Cref{sec:experiments}.
Unless stated, the mixers and gates are \emph{static} (independent of the stream values), as in \Cref{sec:gated-theory}.

\subsection{Setting and notation}\label{sec:sens-setting}

A model has $n$ streams of width $p$ and $L$ blocks.
The residual mixer \emph{as applied} at its $k$-th application is $A_k \in \R^{n \times n}$, and there are $S$ applications: $S = 2L$ sublayers in the nano-GPT block, and $S = 2LH(L_c + 1)$ tied uses per supervision step in TRM, whose $L$ blocks run $H$ outer cycles of $L_c + 1$ updates.
The skip product, or stream map, is $B = A_S \cdots A_1$ (Section~\ref{sec:beyond-init}), with the Perron projector $P \coloneqq \ones\ones^\top/n$ and $Q \coloneqq I - P$.
Its mean squared singular value on $\ones^\perp$ is $t(B) \coloneqq \tr(QB^\top BQ)/(n-1)$, and its collapse exponent is $\Lambda(B)$ of~\eqref{eq:collapse-exp}; $\Lambda^\ast \coloneqq 1$ (one $e$-fold) marks instability (\Cref{sec:sens-criteria}).
The mixer families are:
\begin{itemize}[leftmargin=*,itemsep=1pt]
  \item the \emph{residual} network ($n = 1$);
  \item \emph{mHC}, $A = \mathrm{Sinkhorn}(\exp\tilde H_{\text{res}})$, doubly stochastic (\Cref{alg:sinkhorn-implicit-fwd}), with the gap $\Delta$ of~\eqref{eq:gap} and, when $\tilde H_{\text{res}} - \Delta I$ is a multiple of $\ones\ones^\top$, the exponent $c_{\text{Sk}}(\Delta)$ of~\eqref{eq:gap-star} per application (\Cref{prop:sens-additive}(c));
  \item \emph{JPmHC}, $A$ the Cayley transform of a skew generator $\Omega$, either exact (orthogonal) or the deployed two-step iteration of \Cref{alg:cayley};
  \item \emph{Cayley--Perron}, the same with $\Omega \leftarrow Q\Omega Q$, so that $\Omega\ones = 0$ and every iterate fixes $\ones$ on both sides;
  \item \emph{perm\_mix}, a convex combination of permutations (doubly stochastic).
\end{itemize}

\noindent All symbols are listed in \Cref{tab:notation}.

\subsection{The principle}\label{sec:sens-principle}

\begin{definition}[Sensitivities]\label{def:sens}
  A model with parameters $\theta$ and output $\Psi(\theta)$ is trained on the squared loss $\calL \coloneqq \tfrac12\|r_{\text{res}}\|^2$, with residual $r_{\text{res}} \coloneqq \Psi(\theta) - y_{\text{tgt}}$ for a target $y_{\text{tgt}}$, by the preconditioned flow $\dot\theta = -\mathcal K\nabla\calL$, with a preconditioner $\mathcal K \succ 0$ block-diagonal over parameter groups $\mathsf g$, with blocks $\mathcal K_{\mathsf g}$ (gradient descent; RMSprop and Adam with their preconditioner frozen over a step).
  Let $\mathcal S$ be a differentiable functional of the spectrum of $J^\top J$.
  For a group $\mathsf g$, with $\nabla_{\mathsf g}$ the gradient in its coordinates and $\|\cdot\|_{\mathcal K_{\mathsf g}}$, $\cos_{\mathcal K_{\mathsf g}}$ the norm and the cosine of the inner product $\langle\cdot, \mathcal K_{\mathsf g}\,\cdot\rangle$, set
  \begin{gather*}
    \varphi_{\mathsf g}^2 \coloneqq \frac{\langle\nabla_{\mathsf g}\calL, \mathcal K_{\mathsf g}\nabla_{\mathsf g}\calL\rangle}{\|r_{\text{res}}\|^2},\qquad
    \mathrm{share}_{\mathsf g} \coloneqq \frac{\langle\nabla_{\mathsf g}\calL, \mathcal K_{\mathsf g}\nabla_{\mathsf g}\calL\rangle}{\langle\nabla\calL, \mathcal K\nabla\calL\rangle},\\
    \varrho_{\mathsf g} \coloneqq \frac{\|r_{\text{res}}\|\,\langle\nabla_{\mathsf g}\mathcal S, \mathcal K_{\mathsf g}\nabla_{\mathsf g}\calL\rangle}{\langle\nabla_{\mathsf g}\calL, \mathcal K_{\mathsf g}\nabla_{\mathsf g}\calL\rangle}
    = \frac{\|\nabla_{\mathsf g}\mathcal S\|_{\mathcal K_{\mathsf g}}\cos_{\mathcal K_{\mathsf g}}(\nabla_{\mathsf g}\mathcal S, \nabla_{\mathsf g}\calL)}{\varphi_{\mathsf g}}.
  \end{gather*}
  Here $\varphi_{\mathsf g}$ is the \emph{functional sensitivity} of the group, $\|\nabla_{\mathsf g}\mathcal S\|_{\mathcal K_{\mathsf g}}$ its \emph{spectral sensitivity}, $\varrho_{\mathsf g}$ the \emph{sensitivity ratio}, the spectrum it spends per unit of residual, and $\mathrm{share}_{\mathsf g}$ its \emph{share}.
\end{definition}

\begin{theorem}[Sensitivity identity]\label{thm:sens-identity}
  Along the flow, while $\nabla\calL \ne 0$, $\|r_{\text{res}}\|$ decreases strictly\lean[Formalised]{FreeProbReal.SensIdentity.residual_strictAntiOn} and, with subscripts $0$ and $t_{\text{end}}$ for the start and the end of training,
  \[
    \frac{d\mathcal S}{d\|r_{\text{res}}\|} = \varrho \coloneqq \sum_{\mathsf g} \mathrm{share}_{\mathsf g}\,\varrho_{\mathsf g},\qquad
    \mathcal S_{t_{\text{end}}} - \mathcal S_0 = \int_{\|r_{\text{res},0}\|}^{\|r_{\text{res},t_{\text{end}}}\|}\varrho\; d\|r_{\text{res}}\|,\qquad
    |\mathcal S_{t_{\text{end}}} - \mathcal S_0| \le \sup|\varrho|\,\bigl(\|r_{\text{res},0}\| - \|r_{\text{res},t_{\text{end}}}\|\bigr).
  \]
  Each $\varrho_{\mathsf g}$ is invariant under a reparametrisation $\theta_{\mathsf g} = c_{\text{rp}}\tilde\theta_{\mathsf g}$ with a gain $c_{\text{rp}} \ne 0$, with the same preconditioner in the $\tilde\theta_{\mathsf g}$ coordinates\lean[Formalised]{FreeProbReal.SensIdentity.rho_reparam}; $\mathrm{share}_{\mathsf g}$ is not\lean[Formalised]{FreeProbReal.SensIdentity.share_not_invariant}.
\end{theorem}

\noindent The rate\lean[Formalised]{FreeProbReal.SensIdentity.sens_identity}, its sum over groups\lean[Formalised]{FreeProbReal.SensIdentity.share_rho_sum} and the bound\lean[Formalised]{FreeProbReal.SensIdentity.sens_bound} are proved in Lean.

\begin{proof}
  $\nabla\calL = (\partial\Psi/\partial\theta)^\top r_{\text{res}}$, so $\frac{d}{dt}\|r_{\text{res}}\| = \langle r_{\text{res}}, (\partial\Psi/\partial\theta)\,\dot\theta\rangle/\|r_{\text{res}}\| = -\langle\nabla\calL, \mathcal K\nabla\calL\rangle/\|r_{\text{res}}\| < 0$ and $\frac{d}{dt}\mathcal S = -\langle\nabla\mathcal S, \mathcal K\nabla\calL\rangle = -\sum_{\mathsf g}\langle\nabla_{\mathsf g}\mathcal S, \mathcal K_{\mathsf g}\nabla_{\mathsf g}\calL\rangle$.
  The quotient is $\sum_{\mathsf g}\mathrm{share}_{\mathsf g}\varrho_{\mathsf g}$; integrate.
  Under $\theta_{\mathsf g} = c_{\text{rp}}\tilde\theta_{\mathsf g}$ both gradients scale by $c_{\text{rp}}$, so the numerator and the denominator of $\varrho_{\mathsf g}$ scale by $c_{\text{rp}}^2$.
\end{proof}

\paragraph{Principle.}
A training run spends spectrum in proportion to the loss it buys through each parameter group, $\mathcal S_{t_{\text{end}}} - \mathcal S_0 = \int\sum_{\mathsf g}\mathrm{share}_{\mathsf g}\,\varrho_{\mathsf g}\,d\|r_{\text{res}}\|$.
The ratio $\varrho_{\mathsf g}$ belongs to the mixer \emph{family}: it is fixed by its tangent geometry (\Cref{sec:sens-geometry}), not by its parametrisation gain.
The share belongs to the parametrisation and the optimiser.
Two consequences follow. (i) Whether training ends with its initial, trainable spectrum depends on the family and on how much of the function must pass through the mixers, not on depth. (ii) Whether a family becomes unstable at depth depends on the budget the optimiser can spend on the mixer, measured against the distance to the unstable set (\Cref{thm:sens-adam,thm:sens-decay}).

\emph{Measured} (Float64 model, linear branches, static uniform gates, $n = p = 4$, $L = 4$ to $1024$, 8 seeds): $\varphi^2$ grows in proportion to $L$ ($\sensKBPhiUntied$ untied), and so does the aligned spectral rate, so $\varrho$ is flat in depth.
For the $\ones^\perp$ trace, $\varrho_0\|r_{\text{res},0}\| = \sensKBRhoSinkhorn$ for Sinkhorn mixers, below $\sensKBRhoPerron$ for Cayley--Perron mixers, and $\sensKBRhoBranch$ for the branches on the log-spectrum.
The same flatness holds with $\tanh$ branches, on and off the data span.
Depth does not make learning spectrally cheaper; no proof is known (\Cref{tab:sens-gaps}, M3).

\subsection{Geometry of the mixer families}\label{sec:sens-geometry}

\begin{proposition}[No doubly stochastic direction is isometric]\label{prop:sens-cone}
  If $A(\epsilon) = I + \epsilon\mathcal X + o(\epsilon)$ is doubly stochastic for $\epsilon \in [0, \epsilon_0)$, then $\mathcal X\ones = 0$, $\ones^\top\mathcal X = 0$ and $\mathcal X_{ij} \ge 0$ for $i \ne j$\lean[Formalised]{FreeProbReal.SensGeometry.ds_tangent_cone}.
  The matrix $\mathrm{sym}\,\mathcal X \coloneqq (\mathcal X + \mathcal X^\top)/2$ is minus the graph Laplacian with edge weights $(\mathcal X_{ij} + \mathcal X_{ji})/2 \ge 0$\lean[Formalised]{FreeProbReal.SensGeometry.sym_eq_neg_laplacian}, and $\mathrm{sym}\,\mathcal X \preceq 0$\lean[Formalised]{FreeProbReal.SensGeometry.cone_quadForm_nonpos}.
  The tangent cone of the Birkhoff polytope at $I$ meets $\mathfrak{so}(n)$ only in $0$\lean[Formalised]{FreeProbReal.SensGeometry.cone_skew_eq_zero}.
\end{proposition}

\begin{proof}
  Row and column sums equal $1$ for every $\epsilon$, so $\mathcal X\ones = 0 = \ones^\top\mathcal X$.
  The off-diagonal entries $\epsilon\mathcal X_{ij} + o(\epsilon)$ are non-negative, so $\mathcal X_{ij} \ge 0$.
  The matrix $\mathrm{sym}\,\mathcal X$ has zero row sums and non-negative off-diagonal entries, which makes it minus that Laplacian.
  If $\mathcal X$ is skew, $\mathcal X_{ij} \ge 0$ and $\mathcal X_{ji} = -\mathcal X_{ij} \ge 0$ force $\mathcal X = 0$.
\end{proof}

\noindent The orthogonal group is the case $\mathrm{sym}\,\mathcal X = 0$.
Every first-order doubly stochastic move off the identity is a contraction, consistently with $\|A\|_2 \le 1$ on the whole polytope\lean[Formalised]{FreeProbReal.Mixers.ds_l2_opNorm_le_one}.

\begin{theorem}[The price of a doubly stochastic stream map]\label{thm:sens-birkhoff}
  Let $B$ be doubly stochastic, $\mathcal X \coloneqq B - I$, $m_{\text{off}} \coloneqq \sum_{i\ne j}B_{ij}$ its off-diagonal mass and $\delta \coloneqq \|\mathcal X\|_F$.
  Then
  \[
    t(B) = 1 + \frac{\delta^2 - 2m_{\text{off}}}{n-1},\qquad m_{\text{off}} \ge \delta,\qquad\text{and}\qquad t(B) \le 1 - \frac{\delta(2-\delta)}{n-1}.
  \]
  The identity\lean[Formalised]{FreeProbReal.SensGeometry.tPerp_ds_eq}, the inequality $m_{\text{off}} \ge \delta$\lean[Formalised]{FreeProbReal.SensGeometry.frobSq_le_offMass_sq} and the bound\lean[Formalised]{FreeProbReal.SensGeometry.tPerp_ds_le} hold for every doubly stochastic $B$; the last two are attained at $B = (1-\epsilon)I + \epsilon\,\Pi_{\text{swap}}$, $\Pi_{\text{swap}}$ the permutation matrix of a transposition\lean[Formalised]{FreeProbReal.SensGeometry.tPerp_transposition}.
\end{theorem}

\noindent An orthogonal $B$ has $t(B) = 1$\lean[Formalised]{FreeProbReal.SensGeometry.tPerp_orth}.

\begin{proof}
  $B\ones = \ones = B^\top\ones$ gives $B^\top B\ones = \ones$, so $\tr(QB^\top BQ) = \tr(B^\top B) - 1$, and $\tr(B^\top B) = n + 2\tr\mathcal X + \delta^2 = n - 2m_{\text{off}} + \delta^2$.
  Let $m_{\text{off},i} \coloneqq 1 - B_{ii}$, the off-diagonal mass of row $i$ and of column $i$.
  For $i \ne j$, $\mathcal X_{ij} \le \min(m_{\text{off},i}, m_{\text{off},j})$, so $\delta^2 = \sum_i m_{\text{off},i}^2 + \sum_{i\ne j}\mathcal X_{ij}^2 \le \sum_i m_{\text{off},i}^2 + \sum_{i \ne j}m_{\text{off},i}m_{\text{off},j} = m_{\text{off}}^2$.
  For $(1-\epsilon)I + \epsilon\,\Pi_{\text{swap}}$, $m_{\text{off}} = \delta = 2\epsilon$.
\end{proof}

\noindent Any stream map at Frobenius distance $\delta$ from the identity that is realised by doubly stochastic mixers costs at least $\delta(2-\delta)/(n-1)$ of the mean squared singular value on $\ones^\perp$.
This holds at any depth, tied or untied, and under any parametrisation.
Orthogonal mixers that fix $\ones$ realise every rotation of $\ones^\perp$ at no cost.

\begin{proposition}[The exponent is one-signed and adds]\label{prop:sens-additive}
  \begin{enumerate}[label=(\alph*),leftmargin=*,itemsep=1pt]
    \item For doubly stochastic factors, $0 \le \Lambda(B)$\lean[Formalised]{FreeProbReal.SensExponent.perpLower_ds_sq_le_one} $\le \sum_k\Lambda(A_k)$\lean[Formalised]{FreeProbReal.SensExponent.perpLower_layerProd}, with equality for identical symmetric factors.
      For symmetric $A$ with $A\ones = \ones$, $\sigma_{\min}(A|_{\ones^\perp})$ is the least modulus of an eigenvalue with an eigenvector in $\ones^\perp$\lean[Formalised]{FreeProbReal.SensSymm.perpLower_symm}, attained on that eigenvector\lean[Formalised]{FreeProbReal.SensSymm.exists_perp_eig_min}, and $\Lambda(A^L) = L\,\Lambda(A)$\lean[Formalised]{FreeProbReal.SensSymm.exponent_pow_symm}.
    \item To first order, if $A_k(\epsilon) = I + \epsilon\mathcal X_k + O(\epsilon^2)$ are doubly stochastic with off-diagonal masses $m_{\text{off},k}$, then $t(B) = 1 - \frac{2\epsilon}{n-1}\sum_k m_{\text{off},k} + O(\epsilon^2)$\lean[Formalised]{FreeProbReal.SensExponent.tPerp_first_order}.
    \item A Sinkhorn mixer with symmetric logits of gap $\Delta$ has off-diagonal entries $1/(e^\Delta + n - 1)$\lean[Formalised]{FreeProbReal.SensExponent.piB_symLogits} and acts on $\ones^\perp$ as the scalar $1 - n/(e^\Delta + n - 1)$\lean[Formalised]{FreeProbReal.SensExponent.sinkSym_mulVec_perp}.
      Its exponent per application is $c_{\text{Sk}}(\Delta) = -\log\bigl(1 - n/(e^\Delta + n - 1)\bigr) \approx n e^{-\Delta}$\lean[Formalised]{FreeProbReal.SensExponent.cSink_bounds}, and $S$ applications give $\Lambda = S\,c_{\text{Sk}}(\Delta)$\lean[Formalised]{FreeProbReal.SensExponent.sinkSym_pow_mulVec_perp}.
    \item If the entries of every factor are at least $a_{\min} > 0$, the product contracts $\ones^\perp$ like $(1 - 2a_{\min})^L$\lean[Formalised]{FreeProbReal.SpectralContraction.layerProd_perp_contract}.
  \end{enumerate}
\end{proposition}

\noindent By (b), the losses add over layers and cannot cancel.
Skew directions give no first-order loss\lean[Formalised]{FreeProbReal.SensExponent.offMass_skew}, and the exact Cayley map keeps $t \equiv 1$\lean[Formalised]{FreeProbReal.SensGeometry.tPerp_orth}.
The uniform bound $a_{\min}$ of (d) is needed, and per-factor positivity is not enough: by (a), factors with summable exponents never collapse, and an explicit positive instance is checked in Lean\lean[Formalised]{FreeProbReal.SpectralContraction.slowDS_no_collapse}.
The exponent $\Lambda = \sum_k\Lambda(A_k)$ is the quantity that decides between the two.

\begin{proof}
  (a) $B$ maps $\ones^\perp$ to itself and $B|_{\ones^\perp} = \prod_k A_k|_{\ones^\perp}$.
  Then $\sigma_{\min}(B|_{\ones^\perp}) \ge \prod_k\sigma_{\min}(A_k|_{\ones^\perp})$, and $\sigma_{\min}(B|_{\ones^\perp}) \le \prod_k\|A_k\|_2 \le 1$.
  For symmetric $A$ with $A\ones = \ones$, also $\ones^\top A = \ones^\top$, so $A$ commutes with $Q$: the $\ones^\perp$ part of an eigenvector is again an eigenvector with the same eigenvalue, $A$ has an orthonormal eigenbasis made of $\ones/\sqrt n$ and of vectors in $\ones^\perp$, and on $\ones^\perp$ the squared norm $\|A\,\cdot\,\|^2$ involves only the eigenvalues $\lambda_i$ of the latter; on the minimising eigenvector the bound of the product is attained\lean[Formalised]{FreeProbReal.SensExponent.perpLower_layerProd_sharp}. (b) $B^\top B = I + 2\epsilon\,\mathrm{sym}(\sum_k \mathcal X_k) + O(\epsilon^2)$, and $\tr(Q\,\mathrm{sym}\mathcal X_k\,Q) = \tr\mathcal X_k = -m_{\text{off},k}$ because $\mathrm{sym}\mathcal X_k\ones = 0$. (c) The logits $\Delta I$ plus a constant give, after one row normalisation, $(e^\Delta I + \ones\ones^\top - I)/(e^\Delta + n - 1)$, whose rows and columns already sum to one; $\ones\ones^\top$ vanishes on $\ones^\perp$.
\end{proof}

\begin{proposition}[Low-rank mixers]\label{prop:sens-grassmann}
  A mixer of rank at most $n - 2$ annihilates a nonzero vector of $\ones^\perp$\lean[Formalised]{FreeProbReal.SensRank.exists_perp_kernel}; $\sigma_{\min}(\,\cdot\,|_{\ones^\perp}) = 0$ and $\Lambda = \infty$, for the mixer\lean[Formalised]{FreeProbReal.SensRank.perpLower_eq_zero_of_rank} and for every product that contains it\lean[Formalised]{FreeProbReal.SensRank.perpLower_eq_zero_of_factor}.
  A Grassmann mixer, an orthogonal projector $\Pi_{\text G}$ of rank $n_{\text G} \le n - 2$, has $t(\Pi_{\text G}) = (n_{\text G} - \ones^\top\Pi_{\text G}\ones/n)/(n-1) \in \bigl[(n_{\text G} - 1)/(n-1),\, n_{\text G}/(n-1)\bigr]$\lean[Formalised]{FreeProbReal.SensRank.tPerp_grassProj_bounds}.
\end{proposition}

\noindent In particular a Grassmann mixer collapses $\ones^\perp$ at its first use, whatever the other mixers\lean[Formalised]{FreeProbReal.SensRank.grass_collapse}.
The rank bound cannot be raised: the rank-$(n-1)$ projector $Q$ is isometric on $\ones^\perp$\lean[Formalised]{FreeProbReal.SensRank.perpLower_perpQ}.

\begin{proof}
  The kernel has dimension at least $2$, so it meets the hyperplane $\ones^\perp$ in a nonzero vector, and the rank of a product is at most that of each factor.
  For the projector, $\Pi_{\text G}^\top\Pi_{\text G} = \Pi_{\text G}$, so $\tr(Q\Pi_{\text G}Q) = \tr\Pi_{\text G} - \ones^\top\Pi_{\text G}\ones/n$, and $0 \le \ones^\top\Pi_{\text G}\ones/n \le 1$.
\end{proof}

\noindent The Grassmann variant is outside the perturbative picture for a structural reason.
At rank $n_{\text G} \le n - 2$ ($n_{\text G} \le 2$ for $n = 4$) it reaches $\Lambda = \infty$ without any training dynamics, and its mean squared singular value on $\ones^\perp$ is at most $n_{\text G}/(n-1)$ at every use ($2/3$ for $n = 4$, $n_{\text G} = 2$).

\begin{proposition}[The deployed JPmHC mixer contracts]\label{prop:sens-cayley2}
  Let $\Omega$ be skew, $\alpha > 0$, $\Cay_0 = I + \alpha\Omega$ and $\Cay_{k+1} = I + \tfrac\alpha2\Omega(I + \Cay_k)$.
  The two-step map $\Cay_2 = I + \alpha\Omega + \alpha^2\Omega^2/2 + \alpha^3\Omega^3/4$ satisfies\lean[Formalised]{FreeProbReal.SensExponent.cayley2_gram}
  \[
    \Cay_2^\top\Cay_2 = I - 4\Upsilon^2(I - \Upsilon),\qquad \Upsilon \coloneqq (\tfrac\alpha2\Omega)^\top(\tfrac\alpha2\Omega).
  \]
  On a rotation plane of $\Omega$ with rate $\omega$\lean[Formalised]{FreeProbReal.SensExponent.cayU_rotation_plane}, the squared singular values are $1 - (\alpha\omega)^4/4 + (\alpha\omega)^6/16$\lean[Formalised]{FreeProbReal.SensExponent.cayley2_plane}.
  This is a one-signed contraction for $\alpha\omega < 2$\lean[Formalised]{FreeProbReal.SensExponent.cayley2_contract}.
  Over $S$ uses at rate $\omega$ the exponent is $S\,c_{\text{Cay}}(\alpha\omega)$, where $c_{\text{Cay}}(\alpha\omega) \coloneqq -\tfrac12\log\bigl(1 - (\alpha\omega)^4/4 + (\alpha\omega)^6/16\bigr) = (\alpha\omega)^4/8 + O\bigl((\alpha\omega)^6\bigr)$\lean[Formalised]{FreeProbReal.SensExponent.c2_bounds}.
\end{proposition}

\noindent The iteration is that of \Cref{alg:cayley}, whose deployed map is $\Cay_2$.
By the expansion of $c_{\text{Cay}}$, the instability budget $\Lambda^\ast$ is reached at the threshold rate $\omega^\ast$ with $\alpha\omega^\ast \approx (8\Lambda^\ast/S)^{1/4}$.
The exact Cayley map and the exact Cayley--Perron map are orthogonal at every depth\lean[Formalised]{FreeProbReal.SpectralContraction.orth_layerProd_mem}, so $\Lambda \equiv 0$ for them.

\begin{proof}
  For every $s$, $(I + \Upsilon)(\Cay_s^\top\Cay_s - I) = 4(-1)^{\lfloor s/2\rfloor}\Upsilon^{\lfloor(s+3)/2\rfloor} + 4\Upsilon^{s+2}$\lean[Formalised]{FreeProbReal.CayleyTolerance.cayleyIter_orth_defect_eq}.
  At $s = 2$ the right side is $-4\Upsilon^2 + 4\Upsilon^4 = -4\Upsilon^2(I - \Upsilon)(I + \Upsilon)$, and $I + \Upsilon$ is invertible and commutes with $\Upsilon$.
  On the plane of rate $\omega$, $\Upsilon$ has eigenvalue $(\alpha\omega/2)^2$.
\end{proof}

\noindent\emph{Measured} on the 15 real nano-GPT runs: every deployed JPmHC layer follows the law to a relative error of $\sensProbeLawErr$, and at the end of training every (token, layer) has $\alpha\omega < 2$.
The deployed mixer is therefore a pure contraction (registered prediction X3, \Cref{tab:sens-ledger}).
Its exponent at the end of training is $\sensProbeJpTwelve$ at $L = 12$ and $\sensProbeJpTwentyfour$ at $L = 24$. mHC's exponent stays at its step-0 value, at most $\sensProbeMhcTwelve$ and $\sensProbeMhcTwentyfour$.
In these runs the arm that loses stream-difference spectrum is the ``orthogonal'' one, yet that loss is not what separates it from mHC (K8).

\begin{proposition}[Law response at the identity]\label{prop:sens-law}
  \begin{enumerate}[label=(\alph*),leftmargin=*,itemsep=1pt]
    \item One layer $Y = A\otimes I_p + \sigma X$ with $A = I + \epsilon\mathcal X$, $\mathcal X \in M_n(\C)$.
      Fix $z$, let $G_s(\epsilon)$ solve the trace of \eqref{eq:one-layer}, $G_s = \frac1n\operatorname{tr}\,\vartheta(\zeta I_n - A^\top A)^{-1}$, for $\epsilon$ near $0$ and be differentiable at $0$, and let $\vartheta$, $\zeta$ take their values at $\epsilon = 0$, with $\zeta \ne 1$ and $(\zeta - 1)^2 \ne \sigma^2(\zeta + 1)$.
      Then, with the dot denoting $d/d\epsilon$ at $\epsilon = 0$,
      \[
        \dot G_s = \frac{2\vartheta\operatorname{tr}\mathcal X}{n\bigl((\zeta - 1)^2 - \sigma^2(\zeta + 1)\bigr)}\lean[Caveat]{FreeProbReal.LawResponse.law_response}.
      \]
    \item Any depth, any width.
      Let $J = \prod_\ell(A_\ell\otimes I_p + \sigma_\ell X_\ell)$ with independent Gaussian noise of i.i.d.\ entries, $A_\ell = I$ except one mixer $A_k = I + \epsilon\mathcal X$, $\mathcal X$ real.
      Then $\frac{d}{d\epsilon}\big|_{\epsilon = 0}\mathbb E\,\frac1N\operatorname{tr}(z - J^\top J)^{-1}$ is $\operatorname{tr}\mathcal X$ times a factor that does not depend on $\mathcal X$.
  \end{enumerate}
\end{proposition}

\noindent The law moves only through $\operatorname{tr}\mathcal X$: moves of equal trace have equal response\lean[Caveat]{FreeProbReal.LawResponse.law_response_trace_only}, skew-symmetric moves (tangent to $O(n)$) have none\lean[Caveat]{FreeProbReal.LawResponse.law_response_skew}, and a real move with $\mathcal X\ones = 0$, such as a doubly stochastic tangent, has $\operatorname{tr}\mathcal X = -m_{\text{off}}$ and $\dot G_s = -2\vartheta m_{\text{off}}/\bigl(n((\zeta - 1)^2 - \sigma^2(\zeta + 1))\bigr)$\lean[Caveat]{FreeProbReal.LawResponse.law_response_ds}.

\begin{proof}
  (a) At $\epsilon = 0$, $\zeta I - A^\top A = (\zeta - 1)I$, with $\epsilon$-derivative $\dot\zeta I - \mathcal X - \mathcal X^\top$.
  The derivative of an inverse is minus the inverse, times the derivative, times the inverse, so differentiating $G_s = \frac1n\vartheta\operatorname{tr}(\zeta I - A^\top A)^{-1}$ gives $\dot G_s = \dot\vartheta/(\zeta - 1) - \vartheta(n\dot\zeta - 2\operatorname{tr}\mathcal X)/\bigl(n(\zeta - 1)^2\bigr)$.
  Insert $\dot\vartheta = -\sigma^2\dot G_s$ and $\vartheta\dot\zeta = 2z\vartheta^2\dot\vartheta = 2\zeta\dot\vartheta$ and solve for $\dot G_s$.
  The corollaries use $\operatorname{tr}\mathcal X^\top = \operatorname{tr}\mathcal X$ and, for zero row sums, $\operatorname{tr}\mathcal X = -m_{\text{off}}$. (b) The resolvent is bounded by $1/\operatorname{Im}z$ and its $\epsilon$-derivative by a polynomial in the Gaussian entries, so the derivative may be taken under $\mathbb E$.
  It is linear in $\mathcal X$, so it equals $\operatorname{tr}(\mathcal R\mathcal X)$ for some $\mathcal R \in M_n(\C)$.
  For $O \in O(n)$, $(O\otimes I_p)X_\ell(O\otimes I_p)^\top$ has the law of $X_\ell$.
  With identity mixers, replacing $\mathcal X$ by $O\mathcal XO^\top$ and every $X_\ell$ by its conjugate conjugates $J$ by $O\otimes I_p$, which leaves $\operatorname{tr}(z - J^\top J)^{-1}$ unchanged.
  Hence the derivative is the same at $O\mathcal XO^\top$ and at $\mathcal X$, so $O^\top\mathcal RO = \mathcal R$ for every $O$: diagonal sign matrices force $\mathcal R$ to be diagonal, and permutations force its diagonal to be constant.
\end{proof}

\noindent\emph{Measured} (\texttt{torch/law\_response.py}, $n = 4$, $\sigma^2 = 0.3$).
Central differences of the solver of \eqref{eq:one-layer} match (a) to a relative error of $\sensLawSolverErr$ at $\sensLawSolverSettings$ values of $(z, \sigma^2)$.
For a Gaussian $\mathcal X$, the per-sample derivative of $\frac1N\operatorname{tr}(z - Y^\top Y)^{-1}$, averaged and extrapolated in width ($p = \sensLawWidths$), is $\sensLawMCRe$ (real part) and $\sensLawMCIm$ (imaginary part) at $z = 1.5 + 0.7i$, against $\sensLawPredRe$ and $\sensLawPredIm$; over two values of $z$ the deviation is at most $\sensLawMCZ$ standard errors.
A missing $\sigma^2(\zeta + 1)$ term, a missing factor $\vartheta$, the opposite sign and $\ones^\top\mathcal X\ones$ in place of $\operatorname{tr}\mathcal X$ are each rejected by at least $\sensLawMutantZ$ standard errors.
For (b), three layers with the middle mixer moved, at $p = \sensLawThreeWidths$: a doubly stochastic tangent and the scalar move of equal trace have the same mean response (within $\sensLawThreeEqZ$ standard errors), a skew move has none (within $\sensLawThreeSkewZ$), and a zero response or the opposite trace is rejected by at least $\sensLawThreeMutantZ$.
At the identity the law therefore sees a doubly stochastic move only through its off-diagonal mass, as $t$ does (\Cref{prop:sens-additive}(b)), and an orthogonal move not at all; the density, recovered by \eqref{eq:stieltjes}, inherits both.

\subsection{Training dynamics: what the gradient can reach}\label{sec:sens-training}

\begin{theorem}[Branch blindness]\label{thm:sens-blind}
  With uniform static gates and mixers that fix $\ones$ on both sides, $J^\top J = P \otimes J_{\text{br}}^\top J_{\text{br}} + (QB^\top BQ) \otimes I_p$, where $J_{\text{br}}$ is the single-stream product of the branch Jacobians\lean[Caveat]{FreeProbReal.GatedModel.jac_uniform_perron}.
  For every spectral functional $\mathcal S$ of the stream-difference block $(QB^\top BQ) \otimes I_p$, $\nabla_{\mathrm{branch}}\mathcal S = 0$ and $\varrho_{\mathrm{branch}} = 0$ for the group of branch parameters.
\end{theorem}

\noindent The stream-difference block depends only on the mixers.
Mixers that fix $\ones$ on both sides include the doubly stochastic ones and the orthogonal ones with $A\ones = \ones$.

\begin{proof}
  With uniform static gates the write and read vectors are multiples of $\ones$, so a block has Jacobian $A_k\otimes I_p + P\otimes Y_{\text{br},k}$, with $Y_{\text{br},k}$ the branch Jacobian times the write gain.
  A mixer that fixes $\ones$ on both sides has $PA_k = A_kP = P$, so it splits as $A_k = P + QA_kQ$.
  Each factor is then $P\otimes(I_p + Y_{\text{br},k}) + QA_kQ\otimes I_p$, and the cross products vanish since $PQ = 0$.
  Hence $J = P\otimes J_{\text{br}} + QBQ\otimes I_p$ with $J_{\text{br}}$ the ordered product of the $I_p + Y_{\text{br},k}$, and $J^\top J$ splits as stated.
  The second block contains no branch parameter, so its gradient with respect to them vanishes.
\end{proof}

\begin{theorem}[Copy-mode zero]\label{thm:sens-copy}
  Consider the stream recursion $x_k \coloneqq A_k x_{k-1} + u_k\,F_k(v_k^\top x_{k-1})^\top$, $x_k \in \R^{n\times p}$, with any branches $F_k$, any read vectors $v_k$ and a loss $\calL$ that depends on $x_S$.
  Suppose three things: the input is replicated, $x_0 = \ones x_{\text{emb}}^\top$; every write vector satisfies $u_k \propto \ones$; and every mixer as applied fixes $\ones$, $A_k\ones = \ones$.
  Then every $x_k$ has equal rows, $x_k = \ones\bar x_k^\top$, for all $k$\lean[Formalised]{FreeProbReal.SensTraining.copy_replicated}.
  Moreover, for any parametrisation $A_k = A_k(\theta)$ with $A_k(\theta)\ones = \ones$ for all $\theta$, the gradient vanishes exactly: $\partial\calL/\partial\theta = 0$\lean[Formalised]{FreeProbReal.SensTraining.copy_grad_zero}.
  If the mixers fix $\ones$ at the current point but the family's tangent directions do not, the gradient is $\partial\calL/\partial\theta = (\nabla_{x_k}\calL\,\bar x_{k-1})^\top(\partial_\theta A_k)\ones$\lean[Formalised]{FreeProbReal.SensTraining.copy_grad_pairing}, with $\nabla_{x_k}\calL$ the gradient with respect to $x_k$; it is nonzero in general\lean[Formalised]{FreeProbReal.SensTraining.copy_grad_witness}.
\end{theorem}

\noindent The last case covers full $O(n)$ Cayley at $A = I$.

\begin{proof}
  By induction, $A_k\ones\bar x_{k-1}^\top = \ones\bar x_{k-1}^\top$ and $u_kF_k(\cdot)^\top \propto \ones F_k(\cdot)^\top$.
  Then $\partial\calL/\partial A_k = \nabla_{x_k}\calL\,x_{k-1}^\top = (\nabla_{x_k}\calL\,\bar x_{k-1})\ones^\top$.
  Hence $\partial\calL/\partial\theta = \langle(\nabla_{x_k}\calL\,\bar x_{k-1})\ones^\top, \partial_\theta A_k\rangle = (\nabla_{x_k}\calL\,\bar x_{k-1})^\top(\partial_\theta A_k)\ones = 0$, because differentiating $A_k(\theta)\ones = \ones$ gives $(\partial_\theta A_k)\ones = 0$.
\end{proof}

\noindent At the replicated, uniform-write initialisation, Sinkhorn (mHC), perm\_mix and Cayley--Perron mixers receive no gradient at all, whatever their parametrisation gain.
This initialisation is used by the nano-GPT runs and by TRM, whose inputs are repeated to the streams.
The trap opens only once the write gates become non-uniform.
A truncated Sinkhorn iteration satisfies $A\ones = \ones$ only up to its residual, and leaks gradient of that size.
\emph{Measured}: in the Float64 model the mixer gradient of these families is at most $\sensKBCopyLean$ relative to the branch gradient, against $0.30$ for full $O(n)$ Cayley.
On the nano-GPT code in float64, the doubly stochastic part of $\partial\calL/\partial H_{\text{res}}$ is $\sensKBCopyCode$ under the replicated stream initialisation and $0.30$ under the orthogonal one.
In training it lifts in proportion to the non-uniformity of the write gates (log--log slope $\sensKBCopySlope$).

Two further suppressions act on mHC's mixer logits in the nano-GPT recipe.
The first is the parametrisation gain $\|\partial H_{\text{res}}/\partial\tilde H_{\text{res}}\|_2$: $\sensKBGainCode$ at the code's initialisation, against $\sensKBGainDeepSeek$ at DeepSeek's and $\sensKBGainCayley$ for Cayley.
The second is Adam's $\varepsilon_{\text{A}}$: in training runs of the same code at small width, it leaves a per-coordinate step of $\sensKBEpsStep$ of the learning rate (JPmHC: $\sensKBEpsJp$), and with $\varepsilon_{\text{A}} = 10^{-12}$ the logits move $\sensKBEpsMutant\times$ further.
On the real final checkpoints, the median AdamW step ratio on the mixer logits is $\sensKBStepRatioMhc$ for mHC and $\sensKBStepRatioJp$ for JPmHC (all \emph{measured}).

\begin{proposition}[What a mean readout sees; gate families]\label{prop:sens-readout}
  Under a replicated input $x_0 = \ones x_{\text{emb}}^\top$ and the recursion of \Cref{thm:sens-copy}, the stream mean after $S$ applications is exactly\lean[Formalised]{FreeProbReal.SensTraining.readout_unroll}
  \[
    \tfrac1n\ones^\top x_S = \nu\,x_{\text{emb}}^\top + \sum_k (a_{\text{out},k}u_k)\,F_k(v_k^\top x_{k-1})^\top,\qquad
    a_{\text{out},k} \coloneqq \tfrac1n\ones^\top A_S\cdots A_{k+1},\quad \nu \coloneqq a_{\text{out},0}\ones,
  \]
  with readout rows $a_{\text{out},k}$, \emph{output gains} $a_{\text{out},k}u_k$ and \emph{embedding retention} $\nu$.
  \begin{enumerate}[label=(\alph*),leftmargin=*,itemsep=1pt]
    \item If every mixer fixes $\ones^\top$, $\ones^\top A_k = \ones^\top$, then $a_{\text{out},k} = \ones^\top/n$, $\nu = 1$ and the output gain is $\frac1n\ones^\top u_k$, the \emph{write gain}\lean[Formalised]{FreeProbReal.SensTraining.readout_fix_one}.
    \item If the mixers are orthogonal, then $\|a_{\text{out},k}\| = 1/\sqrt n$ and $\nu \in [-1, 1]$\lean[Formalised]{FreeProbReal.SensTraining.readout_orth}.
    \item Sigmoid writes $u = 2\,\mathrm{sig}(\cdot)$, with $\mathrm{sig}$ the logistic function, leave the write gain free in $(0, 2)$ for every token\lean[Formalised]{FreeProbReal.SensTraining.gain_sigmoid}.
      Softmax writes $u = n\cdot\mathrm{softmax}(\cdot)$ fix it at $1$\lean[Formalised]{FreeProbReal.SensTraining.gain_softmax}.
  \end{enumerate}
\end{proposition}

\noindent If the mixers are only orthogonal, the direction of $a_{\text{out},k}$ is not kept, and both the embedding and the early branches can drain from the readout.
Doubly stochastic and Cayley--Perron mixers fix $\ones^\top$, as (a) requires.
TRM's Cayley variant, whose softmax writes are scaled by its write scale $c = 2$, fixes the write gain at $2$.

\begin{proof}
  Unroll $x_S = A_S\cdots A_1x_0 + \sum_k A_S\cdots A_{k+1}u_kF_k(v_k^\top x_{k-1})^\top$. (a) is immediate. (b) $\|\ones^\top A_S\cdots A_{k+1}\| = \|\ones\|$ for orthogonal mixers, and Cauchy--Schwarz bounds $\nu$. (c) $\sum_i n\cdot\mathrm{softmax}_i = n$.
\end{proof}

\noindent\emph{Measured} on the 9 real final nano-GPT checkpoints, the decomposition reconstructs the model's stream mean to a relative error of $\sensWGRecon$. mHC's write gain is free and used: it is $\sensWGMhcEarly$ in the first two thirds of depth, with a static part of $\sensWGMhcStatic$ and a per-token standard deviation of $\sensWGMhcTokenSd$; and $\nu \equiv 1$.
Deployed JPmHC has write gain $1$, a mid-depth output gain of $\sensWGJpMidTwelve$ at $L = 12$, and a median embedding retention between $\sensWGJpRetentionLo$ and $\sensWGJpRetentionHi$ across runs.
The comparison of \Cref{sec:experiments} and the nano-GPT comparison therefore confound the mixer with the write gain; family G of \Cref{tab:sens-ledger} swaps the gate families to separate them.

\noindent\emph{Measured, tier 3} (family P).
With the orthogonal stream initialisation at $L = 12$, where the zero gradient of \Cref{thm:sens-copy} does not apply, the Cayley--Perron arm keeps the mean readout to $\sensTierPerronDriftMax$ at every token ($\sensTierPerronSeeds$ seeds), while deployed JPmHC moves it by a median $\sensTierJorthoDriftMed$ ($\sensTierJorthoSeeds$ seeds): P1 passes.
The projection keeps $\ones$ fixed but does not make the deployed map orthogonal.
The Perron mixers rotate faster ($\alpha\omega$ at the 90th percentile $\sensTierPerronOmega$, against $\sensTierJorthoOmega$), so their two-step map contracts more (\Cref{prop:sens-cayley2}): $\Lambda = \sensTierPerronLam$ against $\sensTierJorthoLam$, above $\Lambda^\ast = 1$.
The mean validation loss is $\sensTierPerronVal$ for Cayley--Perron, $\sensTierJorthoVal$ for deployed JPmHC and $\sensTierMorthoVal$ for mHC ($\sensTierMorthoSeeds$ seeds): Perron's gap to mHC is within half of deployed JPmHC's, so P2 passes.
An exponent above $\Lambda^\ast$ at no visible loss cost is consistent with M6 of \Cref{tab:sens-gaps}: $\Lambda^\ast$ is a convention, not a loss threshold.

\begin{proposition}[Live gates: the score terms]\label{prop:sens-score}
  Let the mixer $A$, the write vector $u$ and the branch output $F = F(v^\top x)$ of the block $x \mapsto Ax + uF^\top$ vary along a differentiable curve in $\epsilon$, with the state through the gates' dependence on $x$, or with the parameters; dots are $\epsilon$-derivatives.
  Then $\frac{d}{d\epsilon}(Ax + uF^\top) = A\dot x + u\dot F^\top + \dot Ax + \dot uF^\top$\lean[Formalised]{FreeProbReal.SensScore.block_deriv}.
  The \emph{score terms} are $\dot Ax$, $\dot uF^\top$ and, on the read path, $\dot v^\top x$ inside $\dot F$.
  \begin{enumerate}[label=(\alph*),leftmargin=*,itemsep=1pt]
    \item If $A\ones = \ones$ for every value of the gates, then $\dot A\ones = 0$\lean[Formalised]{FreeProbReal.SensScore.deriv_mulVec_one} and $\dot Ax = \dot AQx$\lean[Formalised]{FreeProbReal.SensScore.score_mixer_perp}; at a state with equal rows, $x = \ones\bar x^\top$, $\dot Ax = 0$\lean[Formalised]{FreeProbReal.SensScore.score_mixer_replicated} and the block derivative is $A\dot x + u\dot F^\top + \dot uF^\top$\lean[Formalised]{FreeProbReal.SensScore.block_deriv_replicated}.
      A read vector of constant sum, $\ones^\top v = 1$, has $\ones^\top\dot v = 0$\lean[Formalised]{FreeProbReal.SensScore.deriv_sum_const}, $\dot v^\top x = \dot v^\top Qx$\lean[Formalised]{FreeProbReal.SensScore.score_read_perp} and $\dot v^\top x = 0$ at equal rows\lean[Formalised]{FreeProbReal.SensScore.score_read_replicated}.
    \item If $\ones^\top A = \ones^\top$ for every value of the gates and the write vector has a constant sum, then $\ones^\top\dot A = 0$\lean[Formalised]{FreeProbReal.SensScore.deriv_one_vecMul} and $\ones^\top\dot u = 0$, the mixer and write score terms have zero stream sum\lean[Formalised]{FreeProbReal.SensScore.score_mixer_sum}\lean[Formalised]{FreeProbReal.SensScore.score_write_sum}, and at every state $\frac{d}{d\epsilon}\ones^\top(Ax + uF^\top) = \ones^\top\dot x + (\ones^\top u)\dot F^\top$\lean[Formalised]{FreeProbReal.SensScore.block_sum_deriv}.
  \end{enumerate}
\end{proposition}

\noindent The score terms are what the frozen-gate statements drop.
In (a) the mixer score term sees only the stream differences, and in (b) the stream sum moves like a single-stream residual block with branch gain $\ones^\top u$.
Sigmoid writes $u = 2\,\mathrm{sig}(\cdot)$ break (b): their sum moves, at rate $n/2$ at zero logits\lean[Formalised]{FreeProbReal.SensScore.sigmoid_write_witness}.
The hypotheses on $A$ hold for doubly stochastic and Cayley--Perron mixers, the constant sum of $v$ for softmax reads, and that of $u$ for softmax writes.

\begin{proof}
  The first identity is the product rule.
  Differentiating $A\ones = \ones$, $\ones^\top A = \ones^\top$, $\ones^\top v = 1$ and $\ones^\top u = \text{const}$ gives $\dot A\ones = 0$, $\ones^\top\dot A = 0$, $\ones^\top\dot v = 0$ and $\ones^\top\dot u = 0$.
  Then $\dot AP = 0$ gives $\dot Ax = \dot AQx$ and $\dot A\ones\bar x^\top = 0$, likewise for $\dot v$, and $\ones^\top\dot Ax = 0$, $\ones^\top\dot uF^\top = 0$.
\end{proof}

\noindent\emph{Measured} (\texttt{torch/law\_response.py}, central differences, $n = 4$).
At equal rows, the derivative of a Sinkhorn block and the score term of a softmax read, and at a generic state the stream sum under softmax writes, match these forms to $\sensScoreResid$.
A column-normalised mixer, sigmoid writes and a sigmoid read miss them by at least $\sensScoreMutantMin$.
The frozen-gate statements thus hold with live gates in two places: at the replicated initialisation of \Cref{thm:sens-copy}, up to the write score term, and at every state in the stream-sum sector of doubly stochastic mixers with softmax writes.
Elsewhere the score terms enter through the stream differences $Qx$ only.

\subsection{Depth criteria: when does a doubly stochastic family become unstable?}\label{sec:sens-criteria}

Call the mixers \emph{unstable} when $\Lambda \ge \Lambda^\ast$.
Stream differences and their gradients are then attenuated by at least $e^{-\Lambda^\ast}$.
Over $S$ applications at a common gap, $\Lambda = S\,c_{\text{Sk}}(\Delta)$, and the gap at which $S$ applications reach $\Lambda^\ast$ is $\Delta^\ast(S)$ of~\eqref{eq:gap-star}\lean[Formalised]{FreeProbReal.SensExponent.gapStar_spec} $\approx \log(nS/\Lambda^\ast)$\lean[Formalised]{FreeProbReal.SensExponent.exp_gapStar_bounds}.

\begin{remark}[The heuristic and its homogeneous form]\label{rem:sens-heuristic}
  A tempting criterion is ``number of steps $\times$ learning rate $\times$ functional sensitivity $>$ distance to the unstable set $\times$ spectral sensitivity''.
  Under Adam it is not dimensionally consistent.
  Adam's step $\gamma\,\hat m/(\sqrt{\hat v} + \varepsilon_{\text{A}})$, with learning rate $\gamma$, bias-corrected moment estimates $\hat m$, $\hat v$ of the gradient and Adam's epsilon $\varepsilon_{\text{A}}$, is homogeneous of degree zero in the gradient, so the functional sensitivity enters only through the coherence $s_{\text{coh}}$ of the steps toward mixing and the ratio $\varepsilon_{\text{A}}/\Gamma$, with $\Gamma$ the loss gradient per unit of off-diagonal mixing (both defined below).
  The spectral sensitivity of a Sinkhorn mixer is proportional to the spectrum itself ($\partial c_{\text{Sk}}/\partial\Delta \approx -c_{\text{Sk}}$), so the distance integrates to a logarithm.
  The theorem below is the homogeneous form.
\end{remark}

\begin{theorem}[Instability criterion under Adam]\label{thm:sens-adam}
  Assume the following. (i) All $S$ mixers share the logit gap $\Delta$. (ii) The loss gradient on the gap is $|\partial\calL/\partial\Delta| = \Gamma e^{-\Delta}$ with a constant $\Gamma$. (iii) Adam's normalised step on the gap coordinates has a constant coherence $s_{\text{coh}} \in (0, A_{\text{Adam}}]$ toward mixing, $A_{\text{Adam}} \coloneqq \sup|\hat m|/\sqrt{\hat v}$, and moves the gap by $\kappa_{\text{gap}} \le 2$ per unit coordinate step. (iv) $\sqrt{\hat v}$ tracks the gradient magnitude.
  With per-coordinate step $\gamma_t\hat m/(\sqrt{\hat v} + \varepsilon_{\text{A}})$ at learning rate $\gamma_t$, the gap obeys $d\Delta = -\kappa_{\text{gap}}s_{\text{coh}}\,\chi(\Delta)\,\gamma_t$ with the damping $\chi(\Delta) \coloneqq 1/(1 + (\varepsilon_{\text{A}}/\Gamma)e^\Delta)$.
  From an initial gap $\Delta_0$ and with the budget $\gamma_\Sigma = \sum_t\gamma_t$, the mixers end training unstable at $S$ applications if and only if\lean[Caveat]{FreeProbReal.SensCriteria.adam_unstable_iff_depth}
  \[
    \kappa_{\text{gap}}s_{\text{coh}}\,\gamma_\Sigma \;\ge\; \bigl(\Delta_0 - \Delta^\ast(S)\bigr) + \frac{\varepsilon_{\text{A}}}{\Gamma}\bigl(e^{\Delta_0} - e^{\Delta^\ast(S)}\bigr).
  \]
\end{theorem}

\noindent Equivalently, the end-of-training gap $\Delta_{t_{\text{end}}}$ determines the unstable depth $S^\ast \coloneqq \Lambda^\ast/c_{\text{Sk}}(\Delta_{t_{\text{end}}})$, i.e. $L^\ast \coloneqq S^\ast/2$ blocks.
The constant $\Gamma$ of (ii) is the loss per unit of off-diagonal mixing.

\begin{proof}
  $\chi = \sqrt{\hat v}/(\sqrt{\hat v} + \varepsilon_{\text{A}})$ with $\sqrt{\hat v} = \Gamma e^{-\Delta}$.
  Separating variables, $(1 + (\varepsilon_{\text{A}}/\Gamma)e^\Delta)\,d\Delta = -\kappa_{\text{gap}}s_{\text{coh}}\,\gamma_t$.
  Integrating from $\Delta_0$ to $\Delta_{t_{\text{end}}}$ gives $(\Delta_0 - \Delta_{t_{\text{end}}}) + (\varepsilon_{\text{A}}/\Gamma)(e^{\Delta_0} - e^{\Delta_{t_{\text{end}}}}) = \kappa_{\text{gap}}s_{\text{coh}}\,\gamma_\Sigma$\lean[Caveat]{FreeProbReal.SensCriteria.adam_gap_conserved}.
  The left side increases as $\Delta_{t_{\text{end}}}$ decreases, so $\Delta_{t_{\text{end}}} \le \Delta^\ast(S)$ exactly when $\kappa_{\text{gap}}s_{\text{coh}}\gamma_\Sigma$ reaches its value at $\Delta_{t_{\text{end}}} = \Delta^\ast(S)$.
\end{proof}

\noindent Four readings follow. (a) Every term is a pure number: a gain $c_{\text{rp}}$ on the logits, $\tilde H_{\text{res}} = c_{\text{rp}}\tilde\theta$, maps $\gamma_\Sigma \mapsto c_{\text{rp}}\gamma_\Sigma$ and $\varepsilon_{\text{A}}/\Gamma \mapsto \varepsilon_{\text{A}}/(c_{\text{rp}}\Gamma)$. (b) As $\varepsilon_{\text{A}} \to 0$, the condition is $\kappa_{\text{gap}}s_{\text{coh}}\gamma_\Sigma \ge \Delta_0 - \Delta^\ast \approx \log\bigl(c_{\text{Sk}}(\Delta^\ast)/c_{\text{Sk}}(\Delta_0)\bigr)$: the distance is logarithmic in the exponent. (c) Plain SGD gives $\kappa_{\text{gap}}s_{\text{coh}}\Gamma\gamma_\Sigma \ge e^{\Delta_0} - e^{\Delta^\ast}$, an exponential barrier\lean[Caveat]{FreeProbReal.SensCriteria.sgd_gap_conserved}; $\varepsilon_{\text{A}}$-dominated Adam behaves like SGD at rate $\gamma/\varepsilon_{\text{A}}$. (d) Depth enters only through $\Delta^\ast(S) \approx \log(nS/\Lambda^\ast)$, and $n$ and $S$ enter only through their product.
An instability that is exponential in depth is therefore logarithmic in the training budget.

\emph{Measured inputs} come from the 5 real mHC runs (fp32 final checkpoints with their AdamW state; initial gap $\Delta_0 = \sensAHGapZero$, i.e.\ $2\tau_{\text{init}}$ at the code's identity-initialisation scale $\tau_{\text{init}} = 8$ of the logits):
\begin{itemize}[leftmargin=*,itemsep=1pt]
  \item budget $\gamma_\Sigma = \sensAHBudget$ and $A_{\text{Adam}} = \sensAHAdamSup$;
  \item exponent per application $c_{\text{Sk}}(\Delta_0) = \sensAHcZero$, and $\sensAHcEnd$ at the end of training, identical at $L = 12$ and $L = 24$, so $\Lambda \propto L$;
  \item $\sqrt{\hat v} = \sensAHSqrtV$ on the mixer logits, far below $\varepsilon_{\text{A}} = 10^{-8}$, so $\chi = \sensAHDamping$: AdamW takes about one percent of its step;
  \item $\Gamma = \sensAHGamma$ and a realised $\kappa_{\text{gap}}s_{\text{coh}} = \sensAHKappaS$.
\end{itemize}
At $L = 12$ the right side is $\sensAHLogDist$ (log distance) $+\ \sensAHEpsTerm$ ($\varepsilon_{\text{A}}$ term).
The realised left side is $\sensAHDrive$; the worst case $2A_{\text{Adam}}\gamma_\Sigma$ is $\sensAHWorstDrive$.
The unstable depth is $L^\ast = \sensAHLstarReal$ blocks for the realised drive, $\sensAHLstarWorst$ for any drive AdamW could follow at its $\varepsilon_{\text{A}}$, and $\sensAHLstarNoEps$ with $\varepsilon_{\text{A}} \to 0$.
The input-dependent part of the logits is bounded by $\sensAHDynLogit$.

\begin{table}[ht]
  \centering
  \small
  \caption{Largest initial $\tau_{\text{init}}$ (gap $\Delta_0 = 2\tau_{\text{init}}$) for which mHC's mixers are unstable ($\Lambda \ge 1$) by the end of the nano-GPT budget, per depth $L$ (\Cref{thm:sens-adam}). ``Init only'': no drive. ``Realised'': the measured $\kappa_{\text{gap}}s_{\text{coh}}$. ``Worst'': $\kappa_{\text{gap}}s_{\text{coh}} = 2A_{\text{Adam}}$ at the measured $\varepsilon_{\text{A}}/\Gamma$.
  The drive was measured at $n = 4$, so the $n = 8, 16$ columns are init only.
  The init-only boundary depends on $nL$ alone, to within $0.01$ in $\tau_{\text{init}}$: doubling the streams costs what doubling the depth costs.
  DeepSeek's initialisation (gap 8) is unstable from initialisation alone at $L \approx \sensAHDeepSeekL$.}\label{tab:sens-boundary}
  \begin{tabular}{@{}rrrrrrr@{}}
    \toprule
    $L$ & $\Delta^\ast(2L)$ & init only & realised & worst & init only, $n = 8$ & $n = 16$ \\
    \midrule
    \sensAHBoundaryRows
    \bottomrule
  \end{tabular}
\end{table}

\begin{theorem}[TRM: weight decay sets an equilibrium]\label{thm:sens-decay}
  Consider AdamAtan2 with decoupled weight decay $\lambda_{\text{wd}}$ \citep{everett2024scaling,loshchilov2019adamw} and scale $a_{\text{at}} = 1.27$, $\theta \leftarrow \theta(1 - \gamma\lambda_{\text{wd}}) - \gamma a_{\text{at}}\,\mathrm{atan2}(\hat m, \sqrt{\hat v})$, applied to the mixer logits.
  Under \Cref{thm:sens-adam}(i) and a mean drive $\bar u \in (-\pi/2, \pi/2)$\lean[Formalised]{FreeProbReal.SensCriteria.mean_drive_mem} on the gap, the gap at budget $\gamma_\Sigma$ is\lean[Caveat]{FreeProbReal.SensCriteria.decay_gap_solution}
  \[
    \Delta_{t_{\text{end}}} = \Delta_0e^{-\lambda_{\text{wd}}\gamma_\Sigma} + \frac{\kappa_{\text{gap}}a_{\text{at}}\bar u}{\lambda_{\text{wd}}}\bigl(1 - e^{-\lambda_{\text{wd}}\gamma_\Sigma}\bigr).
  \]
\end{theorem}

\noindent The static gap therefore relaxes\lean[Formalised]{FreeProbReal.SensCriteria.decay_gap_tendsto} to $\kappa_{\text{gap}}a_{\text{at}}\bar u/\lambda_{\text{wd}} < \kappa_{\text{gap}}a_{\text{at}}\pi/(2\lambda_{\text{wd}})$\lean[Formalised]{FreeProbReal.SensCriteria.decay_equilibrium_lt}, and the mixers are stable over $S$ uses if and only if $\Delta_{t_{\text{end}}} > \Delta^\ast(S)$\lean[Formalised]{FreeProbReal.SensCriteria.decay_stable_iff}; at $\Delta_{t_{\text{end}}} = \Delta^\ast(S)$, $\Lambda = \Lambda^\ast$ is already unstable.

\begin{proof}
  The gap is linear in the logits, so $d\Delta/d\gamma_\Sigma = -\lambda_{\text{wd}}\Delta + \kappa_{\text{gap}}a_{\text{at}}\bar u$; solve.
  Since $\mathrm{atan2}$ is homogeneous of degree zero in $(\hat m, \sqrt{\hat v})$, the Sinkhorn gain $e^{-\Delta}$ cancels exactly\lean[Formalised]{FreeProbReal.SensCriteria.atan2Pos_homog}, and there is no $\varepsilon_{\text{A}}$ barrier.
\end{proof}

\begin{corollary}[Decay destabilises doubly stochastic mixers and stabilises orthogonal ones]\label{cor:sens-decay}
  The attractor of the decay is the zero logit.
  For Sinkhorn this is the uniform matrix $\ones\ones^\top/n$\lean[Formalised]{FreeProbReal.SensExponent.sinkSym_zero}, with $\Lambda = \infty$\lean[Formalised]{FreeProbReal.SensExponent.perronProj_kills_perp}.
  For perm\_mix it is the uniform mixture of its permutations.
  For a Cayley mixer it is $\Omega = 0$, the identity\lean[Formalised]{FreeProbReal.SensExponent.iterMat_zero_eq_one}, with $\Lambda = 0$.
\end{corollary}

\begin{proof}
  Without a drive, decoupled decay multiplies every logit by $1 - \gamma_t\lambda_{\text{wd}}$ at each step, so the logits tend to $0$ (for the gap, \Cref{thm:sens-decay} at $\bar u = 0$).
  Zero Sinkhorn logits have gap $0$, and \Cref{prop:sens-additive}(c) at $\Delta = 0$ gives off-diagonal entries $1/n$, so $A = \ones\ones^\top/n$, which acts on $\ones^\perp$ as $0$.
  Zero perm\_mix logits weight its permutations equally. $\Omega = 0$ gives $\Cay_s = I$ for every iterate $s$ of \Cref{alg:cayley}, and the exact Cayley map is $I$ as well.
\end{proof}

\noindent The corollary fixes the attractor, not how fast a deployed map approaches it: D6 failed on deployed JPmHC (\Cref{sec:sens-ledger}).

\noindent The TRM recipes give the following.
\begin{itemize}[leftmargin=*,itemsep=1pt]
  \item \emph{Sudoku-Extreme and Maze-Hard} ($\lambda_{\text{wd}} = 1$, $\gamma_\Sigma = \sensTRMSudBudget$): the ceiling $\sensTRMSudCeiling$ is below $\Delta^\ast(84) = \sensTRMGapStarSup$ (Sudoku) and below $\Delta^\ast(\sensTRMMazeUses) = \sensTRMMazeGapStar$ (Maze).
    No drive keeps Sinkhorn open: starting from gap 8, the exponent $\Lambda_{\mathrm{sup}}$ of the $S$ mixer factors of one supervision step is $\sensTRMSudEndBest$ even as $\bar u \to \pi/2$.
  \item \emph{ARC-AGI-1} ($\lambda_{\text{wd}} = 0.1$, $\gamma_\Sigma = \sensTRMArcBudget$): the ceiling is $\sensTRMArcCeiling$, and a drive $\bar u \ge \sensTRMUstarArc$ is needed.
    With the copy-mode zero (\Cref{thm:sens-copy}) the expected drive is $\bar u \approx 0$, which gives $\Lambda_{\mathrm{sup}} \approx \sensTRMArcEndZero$ from $b_{\text{res}} = I$.
  \item \emph{Measured at initialisation} (TRM's own modules, 84 tied uses): $\Lambda_{\mathrm{sup}} = \sensTRMInitSinkhorn$ for Sinkhorn ($b_{\text{res}} = I$), $\sensTRMInitPerm$ for perm\_mix, and $0$ for Cayley ($\alpha = 0.1$).
\end{itemize}

\begin{corollary}[Training budget before collapse]\label{cor:sens-budget}
  Under \Cref{thm:sens-adam}(i)--(iv), let decoupled weight decay at rate $\lambda_{\text{wd}} \ge 0$ act on the logits and a constant drive $\mathsf d$ per unit budget push the gap toward mixing: $\mathsf d = \kappa_{\text{gap}}s_{\text{coh}}$ for Adam, $\mathsf d = -\kappa_{\text{gap}}a_{\text{at}}\bar u$ for AdamAtan2.
  From an initial gap $\Delta_0$, $S$ applications become unstable at the budget before collapse $\gamma^\ast_\Sigma$, the budget at which the gap reaches $\Delta^\ast(S)$.
  It is $\gamma^\ast_\Sigma = 0$ if $\Delta_0 \le \Delta^\ast(S)$.
  Otherwise, for $\lambda_{\text{wd}} = 0$ and $\mathsf d > 0$, $\gamma^\ast_\Sigma = [(\Delta_0 - \Delta^\ast(S)) + (\varepsilon_{\text{A}}/\Gamma)(e^{\Delta_0} - e^{\Delta^\ast(S)})]/\mathsf d$; for $\lambda_{\text{wd}} > 0$ with $\varepsilon_{\text{A}} = 0$ or AdamAtan2, $\gamma^\ast_\Sigma$ is the decayed form of~\eqref{eq:budget} if $\Delta^\ast(S) + \mathsf d/\lambda_{\text{wd}} > 0$.
  In the remaining cases of these two ($\mathsf d \le 0$, or $\Delta^\ast(S) + \mathsf d/\lambda_{\text{wd}} \le 0$) the mixers never become unstable, $\gamma^\ast_\Sigma = \infty$.
\end{corollary}

\noindent For $\lambda_{\text{wd}} = 0$ the budget is the undecayed form of~\eqref{eq:budget} plus Adam's barrier.
As $\lambda_{\text{wd}} \to 0$ at $\varepsilon_{\text{A}} = 0$ the decayed budget tends to the undecayed one.

\begin{proof}
  Without decay, the end-of-training gap of \Cref{thm:sens-adam} decreases in $\gamma_\Sigma$ and reaches $\Delta^\ast(S)$ exactly at the stated budget\lean[Caveat]{FreeProbReal.SensCriteria.adam_unstable_iff_depth}.
  With decay the gap solves $d\Delta/d\gamma_\Sigma = -\lambda_{\text{wd}}\Delta - \mathsf d$, so $\Delta = (\Delta_0 + \mathsf d/\lambda_{\text{wd}})e^{-\lambda_{\text{wd}}\gamma_\Sigma} - \mathsf d/\lambda_{\text{wd}}$\lean[Caveat]{FreeProbReal.SensCriteria.decay_gap_solution} (\Cref{thm:sens-decay} with its drive $\kappa_{\text{gap}}a_{\text{at}}\bar u$ replaced by $-\mathsf d$), which moves monotonically to $-\mathsf d/\lambda_{\text{wd}}$.
  It reaches $\Delta^\ast(S)$ from $\Delta_0 > \Delta^\ast(S)$ iff $-\mathsf d/\lambda_{\text{wd}} < \Delta^\ast(S)$, at the stated budget, and a gap at or below $\Delta^\ast(S)$ is unstable\lean[Formalised]{FreeProbReal.SensCriteria.decay_stable_iff}.
  Finally $\lambda_{\text{wd}}^{-1}\log\bigl(1 + \lambda_{\text{wd}}(\Delta_0 - \Delta^\ast)/(\lambda_{\text{wd}}\Delta^\ast + \mathsf d)\bigr) \to (\Delta_0 - \Delta^\ast)/\mathsf d$ as $\lambda_{\text{wd}} \to 0$.
\end{proof}

\noindent The corollary is the full statement of \Cref{thm:budget}(ii) of Section~\ref{sec:beyond-init}; at $\varepsilon_{\text{A}} = 0$ it reduces to~\eqref{eq:budget}, and \Cref{sec:sens-budget-values} evaluates it.

\begin{corollary}[The static budget bounds the gated one]\label{cor:sens-gated}
  Let the logits of application $k$ depend on the token $x$, with $\tilde H_{\text{res},k}(x) - \Delta_k(x)I$ a multiple of $\ones\ones^\top$ and $\Delta_k(x) > 0$, and let $\bar\Delta$ be the mean of $\Delta_k(x)$ over the $S$ applications and over a finite set of tokens.
  Then $\Lambda(x) = \sum_k c_{\text{Sk}}(\Delta_k(x))$ for each token, and $\mathbb E_x\Lambda(x) \ge S\,c_{\text{Sk}}(\bar\Delta)$.
  If $\bar\Delta$ follows the dynamics of \Cref{cor:sens-budget}, the token-averaged exponent reaches $\Lambda^\ast$ at a budget no larger than $\gamma^\ast_\Sigma$.
\end{corollary}

\begin{proof}
  Each factor acts on $\ones^\perp$ as the scalar $1 - n/(e^{\Delta_k(x)} + n - 1)$ (\Cref{prop:sens-additive}(c))\lean[Formalised]{FreeProbReal.SensExponent.sinkSym_mulVec_perp}, so the factors commute there and the exponents of a token add. $c_{\text{Sk}}$ is convex on $\Delta > 0$\lean[Formalised]{FreeProbReal.SensDiscrete.cSink_convexOn}, and Jensen's inequality over the pairs (application, token) gives the bound\lean[Formalised]{FreeProbReal.SensDiscrete.perLayer_jensen}.
  At the budget $\gamma^\ast_\Sigma$ of \Cref{cor:sens-budget}, $S\,c_{\text{Sk}}(\bar\Delta) = \Lambda^\ast$, so $\mathbb E_x\Lambda(x) \ge \Lambda^\ast$ there or earlier.
\end{proof}

\Cref{thm:sens-adam,thm:sens-decay} integrate an ODE in the budget $\gamma_\Sigma$.
The optimisers take discrete steps, and the criteria survive them.

\begin{proposition}[Discrete steps and per-layer gaps]\label{prop:sens-discrete}
  Keep assumptions (i)--(iv) of \Cref{thm:sens-adam}, but let the coherence $s_{\text{coh},t}$ vary from step to step: Adam's step $t$ moves the gap to $\Delta_{t+1} = \Delta_t - \mathsf d_t\,\chi(\Delta_t)$ with the drive $\mathsf d_t \coloneqq \kappa_{\text{gap}}s_{\text{coh},t}\gamma_t$.
  Let $\mathcal E(\Delta) \coloneqq \Delta + (\varepsilon_{\text{A}}/\Gamma)e^\Delta$.
  \begin{enumerate}[label=(\alph*),leftmargin=*,itemsep=1pt]
    \item Each step lowers $\mathcal E$ by $\mathsf d_t$, up to a remainder in $[0, \mathsf d_t^2/8]$; the lower bound holds for either sign of $\mathsf d_t$\lean[Caveat]{FreeProbReal.SensDiscrete.adam_step_pot_ge}\lean[Caveat]{FreeProbReal.SensDiscrete.adam_step_pot_le}.
      Over $t_{\text{end}}$ steps, $\mathcal E(\Delta_0) - \sum_t\mathsf d_t \le \mathcal E(\Delta_{t_{\text{end}}}) \le \mathcal E(\Delta_0) - \sum_t\mathsf d_t + \frac18\sum_t\mathsf d_t^2$\lean[Caveat]{FreeProbReal.SensDiscrete.adam_discrete_pot_ge}\lean[Caveat]{FreeProbReal.SensDiscrete.adam_discrete_pot_le}.
    \item The mixers end training stable at $S$ applications if $\sum_t\mathsf d_t < \mathcal E(\Delta_0) - \mathcal E(\Delta^\ast(S))$, whatever the signs of the $\mathsf d_t$\lean[Caveat]{FreeProbReal.SensDiscrete.adam_discrete_stable_depth}, and unstable if $\mathcal E(\Delta_0) - \mathcal E(\Delta^\ast(S)) \le \sum_t\mathsf d_t - \frac18\sum_t\mathsf d_t^2$ with every $\mathsf d_t \ge 0$\lean[Caveat]{FreeProbReal.SensDiscrete.adam_discrete_unstable_depth}.
    \item AdamAtan2 with decoupled decay moves the gap to $\Delta_{t+1} = (1 - \lambda_{\text{wd}}\gamma_t)\Delta_t + \gamma_t\mathsf u_t$, with a drive $\mathsf u_t \le \kappa_{\text{gap}}a_{\text{at}}\pi/2$ at every step.
      If $\lambda_{\text{wd}}\gamma_t \le 1$, then $\Delta_{t_{\text{end}}} \le \Pi_{\text{dec}}\Delta_0 + (1 - \Pi_{\text{dec}})\,\kappa_{\text{gap}}a_{\text{at}}\pi/(2\lambda_{\text{wd}})$ with $\Pi_{\text{dec}} \coloneqq \prod_t(1 - \lambda_{\text{wd}}\gamma_t)$\lean[Caveat]{FreeProbReal.SensDiscrete.decay_discrete_ceiling} and $0 \le \Pi_{\text{dec}} \le e^{-\lambda_{\text{wd}}\gamma_\Sigma}$\lean[Formalised]{FreeProbReal.SensDiscrete.prod_one_sub_le_exp}.
      A constant drive gives the discrete solution exactly\lean[Caveat]{FreeProbReal.SensDiscrete.decay_discrete_const}.
    \item The exponent $c_{\text{Sk}}(\Delta)$ is convex for $\Delta > 0$\lean[Formalised]{FreeProbReal.SensDiscrete.cSink_convexOn}, and per-layer gaps $\Delta_k$ give $\sum_k c_{\text{Sk}}(\Delta_k) \ge S\,c_{\text{Sk}}(\bar\Delta)$ at their mean $\bar\Delta$\lean[Formalised]{FreeProbReal.SensDiscrete.perLayer_jensen}.
      If the per-layer values of $\mathcal E$ average to $\mathcal E(\Delta_{\text{unif}})$, then $\sum_k c_{\text{Sk}}(\Delta_k) \ge S\,c_{\text{Sk}}(\Delta_{\text{unif}})$\lean[Formalised]{FreeProbReal.SensDiscrete.perLayer_exponent_ge}.
    \item If an optimiser's per-coordinate step is at most $\gamma_t$ in modulus, the gap moves by at most $2\gamma_t$ per step.
      The mixers then stay stable whenever $2\gamma_\Sigma < \Delta_0 - \Delta^\ast(S)$, whatever the drive\lean[Caveat]{FreeProbReal.SensDiscrete.bounded_step_gap_ge}\lean[Caveat]{FreeProbReal.SensDiscrete.bounded_step_stable_depth}, and with decoupled decay (c) applies with $\mathsf u_t \le 2$, a ceiling $2/\lambda_{\text{wd}}$.
  \end{enumerate}
\end{proposition}

\noindent In (c), for $\Delta_0 \ge \kappa_{\text{gap}}a_{\text{at}}\pi/(2\lambda_{\text{wd}})$ the gap therefore stays below the ceiling trajectory of \Cref{thm:sens-decay}\lean[Caveat]{FreeProbReal.SensDiscrete.decay_discrete_le_ode}.
The function $\mathcal E$ is the quantity conserved by the ODE.
The per-layer values of $\mathcal E$ in (d) average to $\mathcal E(\Delta_{\text{unif}})$ under equal mean drives.
The step bound of (e) holds for sign-based Lion and signSGD, and for Muon when its update has spectral norm at most 1, as an exactly orthogonalised update does, since entries are bounded by the spectral norm.

\begin{proof}
  (a) Write $\chi_t \coloneqq \chi(\Delta_t)$, so that $(\varepsilon_{\text{A}}/\Gamma)e^{\Delta_t} = \chi_t^{-1} - 1$ and the step is $\Delta_{t+1} - \Delta_t = -\mathsf d_t\chi_t$.
  Then $\mathcal E(\Delta_{t+1}) - \mathcal E(\Delta_t) + \mathsf d_t = (\chi_t^{-1} - 1)\bigl(e^{-\mathsf d_t\chi_t} - 1 + \mathsf d_t\chi_t\bigr)$.
  This is non-negative for every step.
  For $\mathsf d_t \ge 0$ the step is non-positive, where the exponential lies below its second-order Taylor polynomial, so it is at most $(\chi_t^{-1} - 1)\,\mathsf d_t^2\chi_t^2/2 = \mathsf d_t^2\,\chi_t(1 - \chi_t)/2 \le \mathsf d_t^2/8$.
  Sum over $t$. (b) $\mathcal E$ is increasing; conclude with \Cref{thm:sens-adam}'s $\Delta^\ast(S)$. (c) With $\Delta_{\text{ceil}} \coloneqq \kappa_{\text{gap}}a_{\text{at}}\pi/(2\lambda_{\text{wd}})$, $\Delta_{t+1} - \Delta_{\text{ceil}} \le (1 - \lambda_{\text{wd}}\gamma_t)(\Delta_t - \Delta_{\text{ceil}})$; and $0 \le 1 - \lambda_{\text{wd}}\gamma_t \le e^{-\lambda_{\text{wd}}\gamma_t}$ for $\lambda_{\text{wd}}\gamma_t \le 1$. (d) Jensen's inequality for $c_{\text{Sk}}$ and for the convex $\mathcal E$, and $c_{\text{Sk}}$ decreases. (e) $\Delta_{t_{\text{end}}} \ge \Delta_0 - 2\gamma_\Sigma$; conclude as in (b).
\end{proof}

\noindent In the nano-GPT runs $\mathsf d_t \le 2A_{\text{Adam}}\gamma_t$ and $\gamma_t \le \gamma_{\max} = \sensDiscPeakLr$, so the margin of (b) is at most $(2A_{\text{Adam}})^2\gamma_{\max}\gamma_\Sigma/8 = \sensDiscMargin$.
Since $\mathcal E' \ge 1$, it moves an initial gap on the boundary of \Cref{tab:sens-boundary} by at most that much, which is below the table's resolution.
The stable verdicts (mHC at its initialisation) hold for any sign pattern of the per-step coherence.
For TRM, (c) needs no mean drive: no sequence of AdamAtan2 steps keeps a Sinkhorn gap above the ceiling trajectory, so the Sudoku and Maze predictions hold step by step.
By (d), heterogeneous per-layer gaps can only add to the exponent: the unstable verdicts of the uniform-gap criterion stand, and its stable verdicts assume equal gaps.
By (e), at the nano-GPT budget $2\gamma_\Sigma = \sensDiscBoundedDrive$, against a log distance of $\sensAHLogDist$ at $L = 12$: mHC's mixers would stay stable under Lion or Muon as well.

\subsection{The budgets of \Cref{tab:budget}}\label{sec:sens-budget-values}

Every number below is generated from the result files (\texttt{budget\_predictions.py} in the repository).

\paragraph{The assumptions of \Cref{thm:budget} in detail.}
Along the budget the gap obeys $d\Delta/d\gamma_\Sigma=-\lambda_{\text{wd}}\Delta-\mathsf d\,\chi(\Delta)$ with constant $\mathsf d$, $\Gamma>0$ and $\lambda_{\text{wd}}\ge0$ (\Cref{thm:sens-adam,cor:sens-budget}).
Equation~\eqref{eq:budget} solves it at $\varepsilon_{\text{A}}=0$, and \Cref{cor:sens-budget} also at $\lambda_{\text{wd}}=0$ with Adam's barrier; otherwise $\gamma^\ast_\Sigma$, the budget at which the gap reaches $\Delta^\ast(S)$, is the root of the integrated equation.
The static logits measured in our experiments ($b_{\text{res}}$ alone) have the mean gap $\bar\Delta$ of \Cref{thm:budget}(iii) when the input-dependent part averages to zero over tokens.

By the copy-mode zero (\Cref{thm:sens-copy}) the drive starts at $0$, so $\mathsf d$ is an average over the run.
When the drive fades as mixing grows, a $\mathsf d$ measured far from collapse makes (C2) err toward predicting collapse.

\paragraph{nano-GPT.}
The A100 recipe starts at gap $16$ with $\varepsilon_{\text{A}}=10^{-8}$.
Per real run, its own drive and $\varepsilon_{\text{A}}/\Gamma$ (\Cref{sec:sens-criteria}) give $\gamma^\ast_\Sigma=\sensBudAHStar$, $\sensBudAHHeadroom$ times its budget $\sensBudAHBudget$: Adam's $\varepsilon_{\text{A}}$, not the drive, keeps mHC open, and the measured exponent stays below $\sensBudAHLambdaEnd$.
The showcase arms with mixer weight decay test the decayed branch of~\eqref{eq:budget}.
At $\lambda_{\text{wd}}=\sensBudDecayWd$ the criterion's discrete model puts the crossing at $\gamma_\Sigma=\sensBudDecayPredGamma$ (step $\sensBudDecayPredStep$), and both seeds cross at $\sensBudDecayMeasGamma$ (step $\sensBudDecayMeasStep$); at $\lambda_{\text{wd}}=\sensBudDecayLoWd$ the gap ends at $\sensBudDecayLoEnd$, above $\Delta^\ast(24)=\sensBudGapStarTwelve$.
The static gap stays within $\sensBudDecayGapErr$ of the prediction at every snapshot.
These readings are post hoc: D1 registered only the final exponents.
The drive fades as mixing grows: the static drift per unit budget is $\sensBudDriveSixteen$ at gap $16$ without the $\varepsilon$ barrier (arm \texttt{eps}), $\sensBudDriveSix$ at gap $6$ and $\sensBudDriveThree$ at gap $3.6$ (arms \texttt{init}).
For the other nano-GPT and the DeepSeek-V3 rows, the ranges of \Cref{tab:budget} take $\mathsf d$ from $0$ (decayed rows) or from the gap-$6$ drive (undecayed rows) to the largest realised drive $\max\kappa_{\text{gap}}s_{\text{coh}}$ of the real runs.
Registered before their runs (\Cref{tab:sens-budget-ledger}): B1, mHC from DeepSeek's gap $8$ with the decay $\sensBudArmWd$ that crosses $\Delta^\ast(24)$ at half the budget, at $L=12$ (steps $\sensBudArmWinTwelve$) and $L=24$ (steps $\sensBudArmWinTwentyfour$; windows widened by $\sensBudArmTol\%$); B2, the drive at gap $8$ without decay or $\varepsilon$ barrier, predicted in $\sensBudDriveWin$.

\paragraph{Gated mixers.} mHC's logits are input-dependent.
On every recorded checkpoint of the nano-GPT mHC runs ($\sensBudGatedCkpts$ checkpoints of $\sensBudGatedRuns$ runs), the median per-token exponent is at least the exponent of the static logits $b_{\text{res}}$ alone, as \Cref{cor:sens-gated} predicts for the token average.
In the $\sensBudGatedCrossRuns$ runs where either exponent reaches $1$, the per-token one does so at the same snapshot or earlier, and in $\sensBudGatedStaticNever$ of them the static one never does (the init arm at gap $6$, the decay arms at $\lambda_{\text{wd}}=\sensBudDecayLoWd$ and the $\varepsilon$ arm): the static budget is an upper bound, and in these runs not a tight one.

\paragraph{TRM.}
At its default $b_{\text{res}}=I$, $\Delta_0=1<\Delta^\ast(84)=\sensBudGapStarTRM$: the budget is $0$, and the measured exponent of one supervision step is $\sensTRMInitSinkhorn$ at initialisation, against $0$ for Cayley (T1).
From $\tau_{\text{init}}=4$ ($\Delta_0=8$) under Sudoku's weight decay $1$ (\Cref{thm:sens-decay}), the crossing lies at $\gamma_\Sigma=\sensBudTRMGamma$ (steps $\sensBudTRMStep$ of $\sensBudTRMSteps$) for every mean AdamAtan2 drive from toward mixing, at A100's largest realised drive, to $\pi/2$ toward the identity (B3).
Under ARC's decay $0.1$, Sinkhorn reopens only with a sustained drive $\bar u\ge\sensTRMUstarArc$ toward the identity (K5).
The ARC gap of Section~\ref{sec:results} is consistent with, but does not test, this: the two variants also differ in their write gates (\Cref{prop:sens-readout}).

\paragraph{DeepSeek-V3 with mHC (prediction).}
The inputs are DeepSeek-V3's $\sensBudDSLayers$ layers ($S=\sensBudDSUses$) and pretraining schedule \citep{deepseekv3}: $14.8$T tokens of length 4K, warmup over 2K steps to $2.2\cdot10^{-4}$, constant to $10$T tokens, cosine decay to $2.2\cdot10^{-5}$ over $4.3$T, then $333$B tokens at $2.2\cdot10^{-5}$ and $167$B at $7.3\cdot10^{-6}$, with the batch raised from $3072$ to $15360$ sequences over the first $469$B tokens ($\sensBudDSSteps$ steps, $\gamma_\Sigma=\sensBudDSBudget$); mHC's $n=4$, AdamW $\varepsilon_{\text{A}}=10^{-20}$ and weight decay $0.1$ \citep{xie2025mhc}; DeepSeek's initial gap $8$.
Adam's barrier (\Cref{cor:sens-budget}) is then negligible, and~\eqref{eq:budget} applies.
With the static logits decayed, collapse is expected at $\gamma^\ast_\Sigma=\sensBudDSDec$, after $\sensBudDSDecTok$T tokens, for any drive toward mixing.
Undecayed, it is expected at $\gamma^\ast_\Sigma=\sensBudDSUndec$ ($\sensBudDSUndecFrac$ of the budget, $\sensBudDSUndecTok$T tokens), and only a drive below $\sensBudDSDriveNeeded$ avoids it; B2 measures the drive at gap $8$.
The drive was measured at $124$M parameters, and the input-dependent logits add to the static exponent.

\subsection{Consequences for scaling}\label{sec:sens-scaling}

\paragraph{Autoregressive language models.}
\begin{enumerate}[leftmargin=*,itemsep=2pt]
  \item \emph{Depth.}
    At a fixed gap, the per-application exponent of a doubly stochastic mixer does not depend on depth (measured: $\sensAHcEnd$ at $L = 12$ and $24$).
    So $\Lambda$ grows linearly in $L$, while the gap that keeps it below $\Lambda^\ast$ grows only as $\log(2nL/\Lambda^\ast)$ (\Cref{tab:sens-boundary}).
    At the nano-GPT initialisation ($\tau_{\text{init}} = 8$), mHC is stable to $L^\ast = \sensAHLstarReal$ blocks.
    The same gap keeps it inert: its gain, Adam's $\varepsilon_{\text{A}}$ and the copy-mode zero (\Cref{sec:sens-training}) leave the mixer where it started.
    Depth alone does not wake it (prediction K6, at $L = 48$).
    A less identity-biased initialisation is a different regime: DeepSeek's (gap 8) is unstable from initialisation at $L \approx \sensAHDeepSeekL$.
  \item \emph{The deployed Cayley map} is not orthogonal (\Cref{prop:sens-cayley2}).
    Its threshold rate falls like $L^{-1/4}$: $\omega^\ast = \sensCayOmegaTwelve$ at $L = 12$ and $\sensCayOmegaTwentyfour$ at $L = 24$.
    In the real runs, however, the per-layer mixing halved when the depth doubled (X4 failed), which kept $\Lambda$ roughly depth-neutral; K7 registers the replication at $L = 48$.
    The exact Cayley and Cayley--Perron maps have $\Lambda \equiv 0$ at every depth.
  \item \emph{Streams.}
    The number of streams enters the instability only through $nL$ (\Cref{tab:sens-boundary}), and the Birkhoff price of a stream map scales as $1/(n-1)$ (\Cref{thm:sens-birkhoff}).
    The copy-mode zero does not depend on $n$ (K11, at $n = 16$).
  \item \emph{Width and model size.}
    \Cref{thm:sens-identity,thm:sens-birkhoff,thm:sens-copy,thm:sens-adam} involve only the $n \times n$ mixer and the optimiser.
    Width enters through the share (the parametrisation) and through $1/p$ corrections to the moments (the mask--weight correlation, with a measured finite-width coefficient of $\sensKBMaskSplit$ for $p \le 200$), not through $\varrho_{\mathsf g}$.
    The gate effects are therefore predicted to keep their sign at 410M parameters (K10).
  \item \emph{Training budget.}
    The drive is linear in $\gamma_\Sigma$, while the distance is logarithmic in depth.
    Longer training reaches the unstable set only through a coherent drive toward mixing, against the barrier $e^{\Delta_0}$ (\Cref{thm:sens-adam}), or through weight decay on the logits (\Cref{cor:sens-budget}).
    DeepSeek-V3 with mHC removes the barrier ($\varepsilon_{\text{A}} = 10^{-20}$), and its budget exceeds $\gamma^\ast_\Sigma$ with or without decay on the logits (\Cref{tab:budget}, \Cref{sec:sens-budget-values}).
  \item \emph{What the loss sees} (\Cref{tab:sens-real}).
    At $nL \le 96$ neither arm is near the unstable set. mHC is $\sensGainMhcTwelve$ nats better than the residual network at $L = 12$, and JPmHC $\sensGainJpTwelve$ (gap $\sensGapJpMhcTwelve \pm \sensGapJpMhcTwelveSe$).
    \Cref{prop:sens-readout} attributes this difference to the write gain and the readout, not to the mixer: mHC's mixer is frozen in these runs, so mHC acts as a gated residual network over $n$ streams.
    Exact Cayley, with $\Lambda \equiv 0$, keeps at least half of deployed JPmHC's gap to mHC in the showcase runs (K8 passes): the contraction of the deployed map is not its main cause.
    Swapping the gates (family G) shows how much of the difference the gates carry: softmax gates cost mHC at least half of its gain over the residual network (G2 passes), and with sigmoid gates on both arms mHC still leads JPmHC ($\sensGatefairMhcVal$ against $\sensGatefairJgsigVal$, \Cref{fig:nested-l10}); whether sigmoid gates close at least half of JPmHC's gap (G3) is not yet resolved.
\end{enumerate}

\begin{table}[ht]
  \centering
  \small
  \caption{The real nano-GPT runs (2.5B tokens, $n = 4$): mean validation loss (seeds) and end-of-training exponent.}
  \label{tab:sens-real}
  \begin{tabular}{@{}lccc@{}}
    \toprule
    & val.\ loss, $L = 12$ & val.\ loss, $L = 24$ & $\Lambda$ at the end, $L = 12$ / $24$ \\
    \midrule
    residual & $\sensValResidualTwelve$ ($\sensSeedsResidualTwelve$) & $\sensValResidualTwentyfour$ ($\sensSeedsResidualTwentyfour$) & -- \\
    mHC ($\tau_{\text{init}} = 8$) & $\sensValMhcTwelve$ ($\sensSeedsMhcTwelve$) & $\sensValMhcTwentyfour$ ($\sensSeedsMhcTwentyfour$) & $\le\sensProbeMhcTwelve$ / $\le\sensProbeMhcTwentyfour$ \\
    JPmHC (two-step Cayley) & $\sensValJpmhcTwelve$ ($\sensSeedsJpmhcTwelve$) & $\sensValJpmhcTwentyfour$ ($\sensSeedsJpmhcTwentyfour$) & $\sensProbeJpTwelve$ / $\sensProbeJpTwentyfour$ \\
    \bottomrule
  \end{tabular}
\end{table}

\paragraph{Recursive models (TRM).}
\begin{enumerate}[leftmargin=*,itemsep=2pt]
  \item \emph{Recursion depth.}
    One supervision step makes $S = 2LH(L_c + 1)$ tied mixer uses ($H$ outer and $L_c$ inner cycles, each update applying the $L = 2$ blocks of two sublayers).
    \Cref{tab:sens-recursion} lists $\Delta^\ast(S)$, the drive it requires, and Cayley's threshold.
    The doubly stochastic stability margin grows as $\log S$, but the decay ceiling is fixed by $\lambda_{\text{wd}}$.
    At $\lambda_{\text{wd}} = 1$ (Sudoku, Maze) the ceiling $\sensTRMSudCeiling$ lies below $\Delta^\ast(S)$ for every $S > \sensTRMUsesAtCeiling$.
    Sinkhorn and perm\_mix TRMs are therefore predicted to collapse at any recursion depth and from any initialisation (T6, K1, K2, K4).
    At $\lambda_{\text{wd}} = 0.1$ (ARC) the required drive grows with $\log S$.
  \item \emph{Cayley} at $\alpha = 0.1$ tolerates rates $\omega \le \omega^\ast(S) \approx (8\Lambda^\ast/S)^{1/4}/\alpha$: $\sensCayOmegaSup$ at 84 uses and $\sensCayOmegaAct$ over 16 supervision steps.
    Weight decay alone pulls $\omega$ toward 0 (\Cref{cor:sens-decay}), but on nano-GPT the mixer decay that collapses mHC lowered the deployed exponent only partially (D6 fails), so deeper recursion stays safe only while the learned rates stay below $\omega^\ast(S)$ (K13).
  \item \emph{Streams.}
    A collapsed TRM is a single-stream TRM with extra parameters: the extra streams cost compute and buy nothing (T2, K1).
    The ARC comparison of \Cref{sec:results} agrees with this reading but does not test it: there is no $n = 1$ arm and a single seed per arm.
  \item \emph{Tied weights.}
    TRM reuses its blocks, and the nested first moment of \Cref{sec:gated-theory} assumes independent layers.
    It underestimates a tied stack: one weight reused $L$ times with noise variance $\sigma^2$ has first-order term $L^2\sigma^2$ instead of $L\sigma^2$, and the stack stays bounded only if the branch scale falls as $1/L$, not $1/\sqrt L$ (\Cref{prop:sens-tied}).
\end{enumerate}

\begin{table}[ht]
  \centering
  \small
  \caption{TRM recursion scaling (\Cref{thm:sens-decay,prop:sens-cayley2}): tied mixer uses per supervision step, the gap that keeps $\Lambda_{\mathrm{sup}} \le 1$, the AdamAtan2 drive it requires at weight decay 0.1 and 1 ($^{*}$: above $\pi/2$, unreachable), and Cayley's threshold rate at $\alpha = 0.1$, exact and from the $(\alpha\omega)^4/8$ approximation.}
  \label{tab:sens-recursion}
  \begin{tabular}{@{}lrrrrrr@{}}
    \toprule
    $H$, $L_c$ & $S$ & $\Delta^\ast(S)$ & $\bar u^\ast$, $\lambda_{\text{wd}} = 0.1$ & $\bar u^\ast$, $\lambda_{\text{wd}} = 1$ & $\omega^\ast$ & $(8/S)^{1/4}/\alpha$ \\
    \midrule
    \sensTRMRecursionRows
    \bottomrule
  \end{tabular}
\end{table}

\begin{proposition}[Tied weights]\label{prop:sens-tied}
  Let $J \coloneqq Y_0Y_1\cdots Y_{L-1}$ with $Y_\ell \coloneqq A_\ell\otimes I_p + Y_{\text{br},\iota(\ell)}$: layers with the same label $\iota(\ell)$ share one Gaussian branch Jacobian, with covariance $\mathbb E[Y_{\text{br},m}^\top(b\otimes I_p)Y_{\text{br},m}] = \eta_m(b)\otimes I_p$ for label $m$.
  Define $\mathcal C(0, 0) \coloneqq I$, $\mathcal C(i, 0) \coloneqq A_{i-1}^\top\mathcal C(i-1, 0)$, $\mathcal C(0, j) \coloneqq \mathcal C(0, j-1)A_{j-1}$ and
  \[
    \mathcal C(i, j) \coloneqq A_{i-1}^\top\mathcal C(i-1, j) + \mathcal C(i, j-1)A_{j-1} - A_{i-1}^\top\mathcal C(i-1, j-1)A_{j-1}
    + [\iota(i-1) = \iota(j-1)]\;\eta_{\iota(i-1)}\bigl(\mathcal C(i-1, j-1)\bigr).
  \]
  \begin{enumerate}[label=(\alph*),leftmargin=*,itemsep=1pt]
    \item As $p \to \infty$, $\mathbb E\,\frac1{np}\tr J^\top J \to \frac1n\tr\mathcal C(L, L)$.
      For untied layers ($\iota$ injective) $\mathcal C(L, L)$ is the nested moment of \Cref{sec:gated-theory}\lean[Formalised]{FreeProbReal.TiedMoment.tiedF_untied}.
    \item One weight reused at every layer, with identity mixers: $\mathcal C(L, L) = \sum_k\binom Lk^2\eta^k(I)$\lean[Caveat]{FreeProbReal.TiedMoment.tiedOne_eq}, against $\sum_k\binom Lk\eta^k(I)$ untied\lean[Formalised]{FreeProbReal.TiedMoment.nestedM_one_eq}.
      If $\eta$ keeps positive semidefinite matrices, the tied trace is the larger\lean[Caveat]{FreeProbReal.TiedMoment.trace_nestedM_le_tiedOne}\lean[Formalised]{FreeProbReal.TiedMoment.trace_pow_one_nonneg}.
    \item Isotropic noise of variance $\sigma^2$: tied, $m_1 = \sum_k\binom Lk^2\sigma^{2k}$\lean[Caveat]{FreeProbReal.TiedMoment.m1_tied_iso}; untied, $m_1 = (1 + \sigma^2)^L$\lean[Formalised]{FreeProbReal.TiedMoment.m1_nested_iso}.
      Moreover $(1 + \sigma)^{2L}/(L + 1) \le \sum_k\binom Lk^2\sigma^{2k} \le (1 + \sigma)^{2L}$\lean[Formalised]{FreeProbReal.TiedMoment.tiedSum_bounds} and $\sum_k\binom Lk^2\sigma^{2k} \ge 1 + L^2\sigma^2$\lean[Formalised]{FreeProbReal.TiedMoment.tiedSum_ge_first}.
  \end{enumerate}
\end{proposition}

\noindent By (c), a tied stack stays bounded in depth if and only if $L\sigma$ does, an untied one if and only if $L\sigma^2$ does.
The covariance maps $\eta_m$ are those of \Cref{sec:gated-theory}, whose nested moment, the untied case of (a), is exact at every width.
Gated\lean[Formalised]{FreeProbReal.FlatNested.etaGate_posSemidef} and isotropic\lean[Formalised]{FreeProbReal.FlatNested.etaIso_posSemidef} noise keep positive semidefinite matrices, as (b) requires.

\begin{proof}
  (a) By Wick's theorem, $\mathbb E\tr J^\top J$ sums over the pairings of the branch Jacobians in the word $Y_{L-1}^\top\cdots Y_0^\top Y_0\cdots Y_{L-1}$.
  Only pairings of a $Y_{\text{br},m}^\top$ with a $Y_{\text{br},m}$ of the same label survive the limit, and only non-crossing ones; pairings of two $Y_{\text{br}}$'s or of two $Y_{\text{br}}^\top$'s and crossing pairings are $O(1/p)$.
  Every $Y_{\text{br}}^\top$ lies to the left of every $Y_{\text{br}}$, so a non-crossing pairing is nested.
  Split on the two outermost letters, $Y_{i-1}^\top$ and $Y_{j-1}$: if either is a bare mixer it factors out (the first three terms of the recursion, by inclusion--exclusion); otherwise the two branch Jacobians must pair with each other, which needs the same label and contributes $\eta(\mathcal C(i-1, j-1))$.
  For injective $\iota$, by induction $\mathcal C(i, j+1) = \mathcal C(i, j)A_j$ for $i \le j$ and $\mathcal C(i+1, j) = A_i^\top\mathcal C(i, j)$ for $j \le i$, so $\mathcal C(i+1, i+1) = A_i^\top\mathcal C(i, i)A_i + \eta(\mathcal C(i, i))$, the nested layer map. (b) With $A = I$ the recursion is $\mathcal C(i, j) = \mathcal C(i-1, j) + \mathcal C(i, j-1) - \mathcal C(i-1, j-1) + \eta(\mathcal C(i-1, j-1))$, which $\sum_k\binom ik\binom jk\eta^k(I)$ solves by Pascal's rule; the untied sum is the binomial expansion of $(I + \eta)^L$.
  Term by term, $\binom Lk \le \binom Lk^2$ and $\tr\eta^k(I) \ge 0$. (c) The numbers $\binom Lk\sigma^k$ sum to $(1 + \sigma)^L$ and their squares to $\sum_k\binom Lk^2\sigma^{2k}$, so Cauchy--Schwarz and $\sum_k\binom Lk^2\sigma^{2k} \le \bigl(\sum_k\binom Lk\sigma^k\bigr)^2$ give the bounds; the terms $k = 0, 1$ give $1 + L^2\sigma^2$.
\end{proof}

\noindent \emph{Measured} (\texttt{torch/tied\_moment.py}, Monte Carlo extrapolated in width by a fit $m_1 \approx m_{1,\infty} + b_{\text{fw}}/p$).
A TRM-like stack ($n = 4$, two blocks reused six times, asymmetric mixers and gates, $\sigma^2 = \sensTiedTrmX$, $p = \sensTiedTrmWidths$) gives $m_{1,\infty} = \sensTiedTrmMC$ against $\frac1n\tr\mathcal C(12, 12) = \sensTiedTrmPred$ (deviation $\sensTiedTrmZ$ standard errors); untied layers would give $\sensTiedTrmNested$.
One isotropic weight reused 12 times ($\sigma^2 = \sensTiedIsoX$, $p = \sensTiedIsoWidths$) gives $m_{1,\infty} = \sensTiedIsoMC$ against $\sum_k\binom{12}k^2\sigma^{2k} = \sensTiedIsoPred$ (deviation $\sensTiedIsoZ$ standard errors), untied $\sensTiedIsoNested$.
The nested pairing, transposed mixers, swapped gate vectors and a missing inclusion--exclusion term are each rejected by at least $\sensTiedMutantZ$ standard errors.
The finite-width term is positive, with $b_{\text{fw}} \le \sensTiedBMax$ in these cells.

\paragraph{Threaded attention.}
When each stream holds a slice of the sequence, the mixers carry information between slices, and causality constrains them.

\begin{proposition}[Mixers for threaded attention]\label{prop:sens-thread}
  \begin{enumerate}[label=(\alph*),leftmargin=*,itemsep=1pt]
    \item A causal (lower-triangular) doubly stochastic mixer is the identity\lean[Formalised]{FreeProbReal.ThreadAttention.lowerTri_doublyStochastic_eq_one}; a causal orthogonal mixer is a diagonal of signs\lean[Formalised]{FreeProbReal.ThreadAttention.lowerTri_orthogonal}.
    \item With interleaved slices and a one-step shift, every input read is an earlier or equal token\lean[Formalised]{FreeProbReal.ThreadAttention.interleaved_shift_causal}.
    \item Averaging (doubly stochastic, uniform) radix blocks erase every zero-sum thread signal\lean[Formalised]{FreeProbReal.ThreadAttention.radix_average_kills_perp}.
      With a frozen uniform in-thread attention pattern, any orthogonal mixers keep $m_1 = 1 + (\prod_\ell(1 + \sigma_\ell^2) - 1)/n_{\text{tok}}$, where $n_{\text{tok}}$ is the number of tokens per thread\lean[Caveat]{FreeProbReal.ThreadJacobian.thread_m1_orth}.
    \item On $n_{\text{thr}}$ threads of at least two tokens with mixers of degree at most $n_{\text{deg}}$, full dependence requires $n_{\text{thr}} \le n_{\text{deg}}^L$ and $n_{\text{thr}} \le Ln_{\text{deg}}^{L-1}$\lean[Formalised]{FreeProbReal.ThreadReach.reach_lower_bound}; radix stacks reach it in $m + 1$ layers for $n_{\text{thr}} = n_{\text{deg}}^m$\lean[Formalised]{FreeProbReal.ThreadReach.radix_full_dependence}.
  \end{enumerate}
\end{proposition}

\noindent Consequently, a doubly stochastic mixer on contiguous slices of a causal model cannot mix at all.
With interleaved slices and a one-step shift, any mixer is causal by (b).
Slice-mode threaded attention needs interleaved slices, and once slices are interleaved, doubly stochastic mixers average away the slice-specific signal that orthogonal mixers keep (K12).
Re-threading the tokens between layers lets attention itself carry information across threads: with threads of $n_{\text{tok}}$ tokens, threading by all digits but one reaches all $T=n_{\text{tok}}^m$ tokens in $m$ layers\lean[Formalised]{FreeProbReal.ThreadRethread.attnReach_digit}; under a causal mask two layers serve exactly $2n_{\text{tok}}-1$ tokens\lean[Formalised]{FreeProbReal.ThreadCausalDepth.causal_two_layers_exact}, and a sweep schedule reaches every token's full past in $2m-1$ layers\lean[Formalised]{FreeProbReal.ThreadCausal.sweep_full_past}.

\subsection{Ledger}\label{sec:sens-ledger}

\paragraph{What the theory is missing.}
\Cref{tab:sens-gaps} lists, for each gap, its status, what the gap puts at risk, and what would close it.

{\small
\begin{longtable}{@{}l p{0.26\linewidth} p{0.31\linewidth} p{0.30\linewidth}@{}}
  \caption{Theory ledger: what is not yet covered.}\label{tab:sens-gaps}\\
  \toprule
  & Gap and status & Consequence & What would close it \\
  \midrule
  \endfirsthead
  \toprule
  & Gap and status & Consequence & What would close it \\
  \midrule
  \endhead
  M1 & Laws, not moments. One layer at the identity: closed. The first-order law response is
  $2\vartheta\operatorname{tr}\mathcal X/\bigl(n((\zeta - 1)^2 - \sigma^2(\zeta + 1))\bigr)$ (\Cref{prop:sens-law}), checked
  against the solver and against Monte Carlo extrapolated in width. Any depth with identity mixers: a multiple of
  $\operatorname{tr}\mathcal X$ (written proof, Monte Carlo). Open: that multiple for $L \ge 2$, base points away from $I$,
  edges and $\log\det$. & Away from the identity the statements concern $t$, $\Lambda$ and $m_1$ only. & Implicit
  differentiation of the pencil Dyson equation of \Cref{sec:ov-free-prob} at a general base point, checked against Monte
  Carlo. \\
  M2 & Depth-continuum limit (mixing $c_{\text{mix}}/L$, noise $\sigma/\sqrt L$). Open. & The flatness of $\varrho$ in depth
  stays empirical. &
  A drift--diffusion limit in the spirit of \citet{noci2023shaped} and \citet{yang2023depthmup}, with a phase diagram in
  $(c_{\text{mix}}, \sigma^2)$. \\
  M3 & $\varrho$ flat in depth: measured (Float64 model, linear and $\tanh$ branches, $L \le 1024$), no proof. & Depth-scaling
  claims rest on a measurement. & A proof that the per-use gradients of a near-identity stack stay coherent. \\
  M4 & Tied weights: closed in the planar limit by a recursion over pairs of uses (\Cref{prop:sens-tied}), checked against
  Monte Carlo extrapolated in width. The $O(1/p)$ correction is measured, not derived. & TRM moment predictions need the
  recursion; the TRM block (attention, nonlinear MLP) enters only through its static gated linearisation. & The non-planar
  $1/p$ term, and the TRM block beyond static gates. \\
  M5 & Dynamic gates and mixers: the theorems hold the gates $H_{\text{res}}$, $H_{\text{pre}}$, $H_{\text{post}}$ fixed.
  Narrowed: the score terms are explicit (\Cref{prop:sens-score}). The mixer and read terms vanish at equal rows, and the
  mixer and write terms drop out of the stream sum for doubly stochastic mixers with softmax writes. Open: their size on
  $\ones^\perp$ away from equal rows. & Live-gate effects on the stream differences are measured, not derived. & A
  gated-model law that keeps the $Qx$-dependent score terms. \\
  M6 & No theorem links $\Lambda$ (or $t$) to the loss; $\Lambda^\ast = 1$ is a convention. & Loss predictions (A2, A3, D4,
  D5) are empirical calibrations; D4 failed, D5 passed. & The grid runs (A2, A3). \\
  M7 & \Cref{thm:sens-adam} assumes a constant $\Gamma$. Closed: a path-dependent coherence $s_{\text{coh},t}$ and per-layer
  gaps, which are at least as unstable as the uniform gap (\Cref{prop:sens-discrete}(b), (d)). Out of sample,
  $s_{\text{coh}}$ is init-invariant only to within a factor $\sensKBSInvariance$. & The realised boundary of
  \Cref{tab:sens-boundary} moves by up to $\sensKBSInvTau$ in $\tau_{\text{init}}$; its stable verdicts assume equal
  per-layer gaps. & A time-varying $\Gamma$, and per-layer
  gaps measured in D1. \\
  M8 & Optimisers covered: Adam(W) with $\varepsilon_{\text{A}}$, AdamAtan2 with decoupled decay, and every optimiser whose
  per-coordinate step is bounded by its learning rate (Lion, signSGD, and Muon with an update of spectral norm at most 1;
  \Cref{prop:sens-discrete}(e)). Not covered:
  momentum SGD beyond the limit case, and Shampoo. & For bounded-step optimisers only the stable side is covered: whether
  they drive Sinkhorn logits toward mixing is not predicted. & Their coherence $s_{\text{coh},t}$, measured as for Adam. \\
  M9 & The gate swap of family G moves reads and writes together; only the write gain is predicted to matter
  (\Cref{prop:sens-readout}). & The G2 pass, like any G3 verdict, cannot separate reads from writes. & A write-only swap. \\
  M10 & The mask correlation follows the fixed direction $\ones$ (finite-width coefficient $\sensKBMaskSplit$), but a mixer
  fixing a vector other than $\ones$ is untested. & ``Fixed direction'' and ``fixes $\ones$'' are not separated. & One such
  twist in the same cells. \\
  M11 & Copy-mode exit: the trap opens as the write gates break symmetry (slope $\sensKBCopySlope$), with no theory for how
  long it lasts. & The drive $\bar u$ in TRM (K5) is an expectation, not a derivation. & A two-time-scale analysis of the
  write-gate dynamics. \\
  M12 & Threaded attention: in-thread attention is frozen (score path not modelled), sharp selection is conditional on the
  Orthogonal Vectors Conjecture \citep{zaheer2020bigbird}, and nothing has been trained. & K12 tests the structure only. &
  An implementation and the K12 runs. \\
  M13 & Lean coverage: every theorem, proposition and corollary carries a Lean proof (\Cref{app:lean}) except
  \Cref{prop:sens-law}(b), which has a written proof. This includes the
  optimisers' discrete steps (\Cref{prop:sens-discrete}) and the equality case of \Cref{prop:sens-additive}(a). Not
  formalised: assumptions (i)--(iv) of \Cref{thm:sens-adam}, which reduce Adam to a step map on the gap (tier Caveat); the
  planar limit of \Cref{prop:sens-tied}(a) and the one-layer law~\eqref{eq:one-layer} behind \Cref{prop:sens-law}(a) (tier
  Caveat); and the readings marked $\approx$, which Lean replaces by
  two-sided bounds. & The criteria are exact for the step maps, not for Adam's moment estimates $\hat m$, $\hat v$. & A
  derivation of (i)--(iv) from Adam's moment estimates, checked against D1. \\
  M14 & Grassmann mixers (a TRM variant). Closed by a rank argument: rank $n_{\text G} \le n - 2$ gives $\Lambda = \infty$
  for every product that contains the mixer, and $t \le n_{\text G}/(n-1)$ (\Cref{prop:sens-grassmann}). & The variant
  collapses $\ones^\perp$ at its first use; no perturbative prediction is needed. & Rank-$(n-1)$ projectors other than $Q$
  ($\Lambda$ finite, not computed). \\
  M15 & Whether extra streams help a task: the theory says which TRMs collapse, not whether more isometric streams improve
  accuracy. & K1--K4 read N/A if Cayley does not beat $n = 1$. & The TRM grid, with seed counts from the power analysis
  below. \\
  \bottomrule
\end{longtable}
}

\paragraph{Registered predictions.}
Each prediction in \Cref{tab:sens-ledger} was committed before any of its runs, together with its check function and thresholds (the commit is listed).
A gate pins by digest the text, the check, every shared helper and the run plans.
It also requires each check to pass a clean fixture and to fail a mutant.
The families are:
\begin{itemize}[leftmargin=*,itemsep=1pt]
  \item X: the 15 real nano-GPT runs;
  \item A: a nano-GPT exponent grid;
  \item T: the TRM Sudoku grid;
  \item D: a nano-GPT ablation in which each arm moves one lever of \Cref{thm:sens-adam}, with its knob set by the criterion;
  \item P: Cayley--Perron under an orthogonal stream initialisation;
  \item G: the gate-family swap;
  \item K: task-level success or failure, registered with this appendix;
  \item B: the training budget before collapse (\Cref{cor:sens-budget}), registered with Section~\ref{sec:beyond-init} in \Cref{tab:sens-budget-ledger}.
\end{itemize}

The ``at least half'' predictions (G1--G3, K1--K4, K6, K8) use a seed-matched contrast, effect $-$ reference$/2$.
They read N/A, which is not a verdict, in two cases: the reference is within 2 standard errors of 0, or the contrast is within 2 standard errors of the threshold.

Power was computed with the registered rule, as the effect at which the probability of a pass reaches $0.8$ at the theory's predicted point.
For TRM, with a per-seed standard deviation of $\sensPowTrmSd$ accuracy points, a Cayley $-$ $(n = 1)$ advantage is resolved at $\sensPowTrmThree$ points with 3 seeds and $\sensPowTrmEight$ points with 8.
For nano-GPT with 3 seeds, the resolvable reference effect is $\sensPowSixThree$ nats for K6 and $\sensPowEightThree$ for K8, and the resolvable sign effect for K10 is $\sensPowTenThree$.
The references at $L = 12$ are $\sensGainMhcTwelve$ and $\sensGapJpMhcTwelve$ nats.

\Cref{tab:sens-ledger} records the verdicts.
X4 fails: the deployed JPmHC exponent did not compound with depth, because the per-layer mixing halved when the depth doubled.
D4 fails: across the mHC arms the loss gap to the residual network does not rise with the applied exponent at the registered rank correlation, so $\Lambda^\ast = 1$ remains a spectral threshold, not a loss threshold (M6 of \Cref{tab:sens-gaps}).
D6 fails on its JPmHC half: the mixer decay that collapses mHC lowers deployed JPmHC's exponent only partially, not to half of its seed-matched baseline, while exact Cayley stays at $\Lambda \approx 0$; the deployed map does not reach the attractor of \Cref{cor:sens-decay} within the budget.
K8 passes: exact Cayley, with $\Lambda \equiv 0$, keeps at least half of deployed JPmHC's loss gap to mHC, so the two-step contraction is not the main cause of that gap.
G1 and G3 are pending: their arms exist, but the data do not yet resolve them.
In \Cref{tab:sens-ledger}, $n_{\mathrm{eff}}$ is the effective number of streams (the entropy rank of the per-token stream Gram matrix) and $\Lambda_{\mathrm{st}}$ the exponent of the mixers' static logits.

{\small
\begin{longtable}{@{}l p{0.14\linewidth} p{0.45\linewidth} p{0.11\linewidth} l l@{}}
\caption{Prediction ledger. Every row is a registered prediction with its frozen text and check (pinned digest in the gate); the summary is checked against it. Status: verdict on the data recorded so far, or \emph{pending} while its arms have not run. Commit: registration. $\dagger$: arms not yet planned or built.}\label{tab:sens-ledger}\\
\toprule ID & Setting & Prediction & Basis & Status & Commit \\ \midrule \endfirsthead
\toprule ID & Setting & Prediction & Basis & Status & Commit \\ \midrule \endhead
X1 & nano-GPT, 15 real runs & mHC ($\tau_{\text{init}}$ = 8): $\Lambda$ $\le$ $10^{-3}$ (median), $\le$ $10^{-2}$ (max), off-diagonal $\le$ $10^{-5}$, every checkpoint & \ref{prop:sens-additive}, \ref{thm:sens-adam} & \textbf{pass} & \texttt{39a3dd6} \\
X2 & nano-GPT, 15 real runs & mHC: $t(B)\le1-\delta(2-\delta)/(n-1)$ per token (instrument) & \ref{thm:sens-birkhoff} & \textbf{pass} & \texttt{39a3dd6} \\
X3 & nano-GPT, 15 real runs & deployed JPmHC: squared singular values $1-(\alpha\omega)^4/4+(\alpha\omega)^6/16$, and $\ge$ 95\% of (token, layer) at $\alpha\omega<2$ & \ref{prop:sens-cayley2} & \textbf{pass} & \texttt{39a3dd6} \\
X4 & nano-GPT, 15 real runs & deployed JPmHC: $\Lambda$ compounds with depth, $\Lambda$(L24) $\ge$ 1.5 $\Lambda$(L12) & \ref{prop:sens-additive} & \textbf{fail} & \texttt{39a3dd6} \\
X5 & nano-GPT, 15 real runs & product bounds on the stream-difference singular values hold per token (instrument) & \ref{prop:sens-additive} & \textbf{pass} & \texttt{39a3dd6} \\
X6 & nano-GPT, 15 real runs & exploratory: stream-difference gradient attenuation, mHC vs JPmHC & -- & info & \texttt{39a3dd6} \\
A1 & nano-GPT grid & mHC step-0 $\Lambda$ equals the Sinkhorn table within 5\% (instrument) & \ref{prop:sens-additive} & pending & \texttt{39a3dd6} \\
A2 & nano-GPT grid & mHC's loss gap to the residual rises with the final $\Lambda$ (Spearman $\ge$ 0.7) & -- & pending & \texttt{39a3dd6} \\
A3 & nano-GPT grid & cells matched in $\Lambda$ across depths have equal gaps (within 2 $\times$ seed noise) & -- & pending & \texttt{39a3dd6} \\
A4 & nano-GPT grid & exact Cayley $\Lambda$ $\le$ $10^{-4}$ at every depth; deployed $\Lambda$ deepest $\ge$ 2 $\times$ shallowest & \ref{prop:sens-cayley2} & pending & \texttt{39a3dd6} \\
T1 & TRM Sudoku grid & init $\Lambda$ over a supervision step $\ge$ 30 (Sinkhorn, perm\_mix), $\le$ $10^{-3}$ (Cayley) & \ref{prop:sens-additive} & \textbf{pass} & \texttt{39a3dd6} \\
T2 & TRM Sudoku grid & arms with $\Lambda_{\mathrm{sup}}$ $\ge$ 5 end single-stream: $n_{\mathrm{eff}}$ $\le$ 1.5, accuracy within 2 se of n = 1 & \ref{thm:sens-copy}, \ref{thm:sens-decay} & pending & \texttt{39a3dd6} \\
T3 & TRM Sudoku grid & Sinkhorn with $\tau_{\text{init}}$ = 4 init: $n_{\mathrm{eff}}$ $\ge$ 2 and $\Lambda_{\mathrm{sup}}$ $\le$ 3 (T6 implies a failure) & \ref{thm:sens-decay} & pending & \texttt{39a3dd6} \\
T4 & TRM Sudoku grid & Cayley: $\Lambda_{\mathrm{sup}}$ $\le$ 0.5 and $n_{\mathrm{eff}}$ $\ge$ 2 & \ref{prop:sens-cayley2}, \ref{cor:sens-decay} & pending & \texttt{39a3dd6} \\
T5 & TRM Sudoku grid & Spearman($\Lambda_{\mathrm{sup}}$, $n_{\mathrm{eff}}$) $\le$ -0.7 over all multi-stream checkpoints & -- & pending & \texttt{39a3dd6} \\
T6 & TRM Sudoku grid & weight decay 1.0: every decayed Sinkhorn / perm\_mix arm ends with $\Lambda_{\mathrm{sup}}$ $\ge$ 5 & \ref{thm:sens-decay} & pending & \texttt{585d068} \\
T7 & TRM Sudoku grid & undecayed Sinkhorn logits ($\tau_{\text{init}}$ = 4): $\Lambda_{\mathrm{sup}}$ $\le$ 3 and below the decayed twin & \ref{cor:sens-decay} & pending & \texttt{585d068} \\
D1 & showcase tiers 1-2 & the criterion predicts each init / decay arm's final static $\Lambda$ within $\times2$ & \ref{thm:sens-adam} & \textbf{pass} & \texttt{a225c20} \\
D2 & showcase tiers 1-2 & mixer $\varepsilon_{\text{A}}$ $10^{-12}$ multiplies the static log drift $\ge$ $10\times$, inside Adam's bound & \ref{thm:sens-adam} & \textbf{pass} & \texttt{a225c20} \\
D3 & showcase tiers 1-2 & $\Lambda_{\mathrm{st}}$(L24) / $\Lambda_{\mathrm{st}}$(L12) in [1.4, 2.8] & \ref{thm:sens-adam} & pending & \texttt{a225c20} \\
D4 & showcase tiers 1-2 & loss gap to the residual rises with the applied $\Lambda$ across mHC arms (Spearman $\ge$ 0.7) & -- & \textbf{fail} & \texttt{a225c20} \\
D5 & showcase tiers 1-2 & at matched final $\Lambda$ the decay arm's gap $\le$ the init arm's + 2 $\times$ seed noise & -- & \textbf{pass} & \texttt{a225c20} \\
D6 & showcase tiers 1-2 & the decay that collapses mHC halves deployed JPmHC's $\Lambda$; exact Cayley $\Lambda$ $\le$ $10^{-4}$ & \ref{cor:sens-decay} & \textbf{fail} & \texttt{a225c20} \\
P1 & tier 3 (ortho) & Cayley-Perron fixes the mean readout (drift $\le$ $10^{-5}$); deployed JPmHC moves it ($\ge$ 0.1) & \ref{prop:sens-readout} & \textbf{pass} & \texttt{2f5613e} \\
P2 & tier 3 (ortho) & Perron's loss gap to mHC $\le$ half of deployed JPmHC's & \ref{prop:sens-readout} & \textbf{pass} & \texttt{2f5613e} \\
P3 & tier 3 (ortho) & exploratory: applied $\Lambda$, readout drift, loss of the three arms & -- & info & \texttt{2f5613e} \\
G1 & tier 4 (gates) & gated residual (mHC at n = 1) recovers $\ge$ half of mHC's gain over the residual & \ref{prop:sens-readout} & \textbf{fail} & \texttt{50b5553} \\
G2 & tier 4 (gates) & softmax gates (write gain 1) cost mHC $\ge$ half of that gain & \ref{prop:sens-readout} & \textbf{pass} & \texttt{50b5553} \\
G3 & tier 4 (gates) & sigmoid gates close $\ge$ half of JPmHC's gap to mHC & \ref{prop:sens-readout} & pending & \texttt{50b5553} \\
G4 & tier 4 (gates) & instrument: mHC-softmax write gain 1 to $10^{-5}$; sigmoid arms use their gain (sd $\ge$ 0.01) & \ref{prop:sens-readout} & \textbf{pass} & \texttt{50b5553} \\
G5 & tier 4 (gates) & exploratory: Cayley-Perron with sigmoid gates; write gains & -- & info & \texttt{50b5553} \\
K1 & TRM Sudoku grid & Sinkhorn and perm\_mix TRMs end at most half-way from n = 1 to Cayley (exact accuracy) & \ref{thm:sens-copy}, \ref{thm:sens-decay} & pending & \texttt{9f7527a} \\
K2 & TRM Sudoku grid & identity-biased inits ($\Lambda$ 6.0 / 0.83 / 0.11 / 0.15) do not rescue them & \ref{thm:sens-decay} & pending & \texttt{9f7527a} \\
K3 & TRM Sudoku grid & undecayed mixer logits recover $\ge$ half of Cayley's advantage & \ref{cor:sens-decay} & pending & \texttt{9f7527a} \\
K4 & TRM Maze-Hard\textsuperscript{$\dagger$} & ceiling 3.99 $<$ 5.48: Sinkhorn at most half-way from n = 1 to Cayley & \ref{thm:sens-decay} & pending & \texttt{9f7527a} \\
K5 & TRM ARC-AGI-1 & weight decay 0.1: Sinkhorn ends with $\Lambda_{\mathrm{sup}}$ $\ge$ 5, $n_{\mathrm{eff}}$ $\le$ 1.5; Cayley $n_{\mathrm{eff}}$ $\ge$ 2 & \ref{thm:sens-copy}, \ref{thm:sens-decay} & pending & \texttt{9f7527a} \\
K6 & nano-GPT L48\textsuperscript{$\dagger$} & mHC $\Lambda$ $\le$ $10^{-4}$; the gated residual recovers $\ge$ half of mHC's gain & \ref{thm:sens-adam}, \ref{prop:sens-readout} & pending & \texttt{9f7527a} \\
K7 & nano-GPT L48\textsuperscript{$\dagger$} & post hoc: deployed JPmHC's $\Lambda$ within [0.5, 2] $\times$ its L12 value & \ref{prop:sens-cayley2} & pending & \texttt{9f7527a} \\
K8 & showcase tier 1 & exact Cayley keeps $\ge$ half of deployed JPmHC's gap to mHC & \ref{prop:sens-readout}, \ref{prop:sens-cayley2} & \textbf{pass} & \texttt{9f7527a} \\
K9 & showcase tier 1 & mixer decay on deployed JPmHC halves its readout drift & \ref{cor:sens-decay}, \ref{prop:sens-readout} & pending & \texttt{9f7527a} \\
K10 & nano-GPT 410M\textsuperscript{$\dagger$} & softmax gates raise mHC's loss, sigmoid gates lower JPmHC's (each by more than 2 se) & \ref{prop:sens-readout} & pending & \texttt{9f7527a} \\
K11 & nano-GPT n = 16\textsuperscript{$\dagger$} & DeepSeek init: $n_{\mathrm{eff}}$ $\le$ 1.1 (replicate), $\ge$ 2 (ortho) & \ref{thm:sens-copy} & pending & \texttt{9f7527a} \\
K12 & threaded attention\textsuperscript{$\dagger$} & causal doubly stochastic slice mixer (= I) at chance; orthogonal interleaved threads $\ge$ chance + 0.2 from L = m + 1 & \ref{prop:sens-thread} & pending & \texttt{9f7527a} \\
K13 & TRM Sudoku grid\textsuperscript{$\dagger$} & exploratory: Cayley at twice the recursion (156 uses) & \ref{prop:sens-cayley2} & pending & \texttt{9f7527a} \\
\bottomrule
\end{longtable}
}

{\small
\begin{longtable}{@{}l p{0.14\linewidth} p{0.45\linewidth} p{0.11\linewidth} l l@{}}
  \caption{Budget predictions (family B), registered with the plan of their arms before any run, in the same gate as
  \Cref{tab:sens-ledger}.}\label{tab:sens-budget-ledger}\\
  \toprule ID & Setting & Prediction & Basis & Status & Commit \\ \midrule \endfirsthead
  \sensBudLedgerRows
  \bottomrule
\end{longtable}
}

\section{Mixer sets and parameter counts}\label{app:mixers}

This appendix defines the two mixer sets compared in the paper, the doubly stochastic matrices of mHC and the orthogonal matrices reached by the Cayley map (Section~\ref{sec:cayley-method}), and counts the parameters of one mixer module in the ARC runs; the algorithms are in Appendices~\ref{sec:cayley} and~\ref{sec:sinkhorn}.

\subsection{Doubly stochastic matrices and the Birkhoff polytope}

\begin{definition}[Doubly-Stochastic Matrix]
  A matrix $\bP \in \R^{n \times n}$ is \emph{doubly stochastic} if $\bP_{ij} \geq 0$ for all $i,j$, $\bP \ones = \ones$, and $\bP^\top \ones = \ones$.
\end{definition}
Bold $\bP$ denotes a doubly stochastic matrix (in Appendix~\ref{sec:sinkhorn}, the Sinkhorn output); plain $P$ is the Perron projector.

The set of all $n \times n$ doubly-stochastic matrices forms the \emph{Birkhoff polytope} $\mathrm{Bis}_n$.
By the Birkhoff-von Neumann theorem, $\mathrm{Bis}_n$ is the convex hull of the $n!$ permutation matrices:
\begin{equation}
  \mathrm{Bis}_n = \mathrm{conv}\{\text{permutation matrices of size } n\},
\end{equation}
a permutation matrix having exactly one entry $1$ in every row and every column and zeros elsewhere.
The Sinkhorn-Knopp algorithm~\citep{sinkhorn1967} maps a positive matrix to $\mathrm{Bis}_n$ by alternating row and column normalisation; its log-space implementation, with the clamp of the logits, is \Cref{alg:sinkhorn-implicit-fwd}, and its backward pass \Cref{alg:sinkhorn-implicit-bwd}.

\subsection{Orthogonal matrices and the Cayley image}

For a step $\alpha>0$, the Cayley transform maps skew-symmetric matrices $\Omega$ to orthogonal matrices:
\begin{equation}
  \label{eq:cayley-closed}
  \Cay(\Omega) = (I - \tfrac{\alpha}{2}\Omega)^{-1}(I + \tfrac{\alpha}{2}\Omega),\qquad \Omega = -\Omega^\top,
\end{equation}
the two factors commuting.
Its image lies in $SO(n)$\lean[Formalised]{FreeProbReal.Mixers.cayley_det}: each eigenvalue $i\omega$ of $\Omega$ ($\omega\in\R$) becomes the eigenvalue $(1+\tfrac{\alpha}{2}i\omega)/(1-\tfrac{\alpha}{2}i\omega)\neq-1$ of $\Cay(\Omega)$, and conversely every rotation without eigenvalue $-1$ is reached, so the image is the set of rotations without eigenvalue $-1$.
For $n\ge2$ the map is therefore not onto $SO(n)$: it misses $-I_n$ for even $n$ and, for every $n$, every rotation by $\pi$ in a plane; restricting $\Omega$, e.g.\ to a Frobenius sphere, only shrinks the image and does not change this.
The iteration that approximates it is \Cref{alg:cayley}.

\subsection{Parameter counts}\label{app:params}

\begin{table}[h]
  \centering
  \caption{Parameters of one mixer module in the ARC runs ($n = 4$, $p = 512$, $np = 2048$), counted from the training code.
  RMSNorm has no learned scale; LayerNorm has a scale and a bias. $b_{\text{pre}},b_{\text{post}},b_{\text{res}}$ and $g_{\text{pre}},g_{\text{post}},g_{\text{res}}$ are the logit biases and scalar gains of the Sinkhorn arm (Appendix~\ref{sec:sinkhorn-run}).}
  \label{tab:param-counts}
  \begin{tabular}{lrl}
    \toprule
    Variant & Parameters & Composition \\
    \midrule
    Sinkhorn & $49{,}179$ & $W_{\text{fused}}$: $(2n+n^2)\times np=49{,}152$; $b_{\text{pre}},b_{\text{post}},b_{\text{res}}$: $24$; $g_{\text{pre}},g_{\text{post}},g_{\text{res}}$: $3$ \\
    Cayley & $102{,}400$ & $W_{\text{fused}}$: $3n^2\times np=98{,}304$; LayerNorm: $2\,np=4{,}096$ \\
    \bottomrule
  \end{tabular}
\end{table}

The model has 4 mixer modules (2 per block, shared across the recursion), so the mixers add $196{,}716$ (Sinkhorn) and $409{,}600$ (Cayley) parameters to the $6.83$M of the two blocks, the token embedding and the heads, i.e.\ $2.9\%$ and $6.0\%$; the puzzle embeddings are not counted.

\section{Cayley Transform Variant}
\label{sec:cayley}

This appendix details the Cayley mixer of Section~\ref{sec:cayley-method}: the gates of the trained arm, the fixed-point iteration, its orthogonality defect and the number of iterations $s$ it requires, its implementation in the ARC runs of Section~\ref{sec:experiments}, and the evaluation of both ARC runs at every checkpoint.
The mixer is orthogonal, $H_{\text{res}}^\top H_{\text{res}}=I_n$, instead of doubly stochastic, so the skip path preserves norms, $\|H_{\text{res}}x\|=\|x\|$.

\subsection{Gate computation}\label{app:cayley-gates}

The Cayley arm replaces the vector gates of HC by $n\times n$ softmax gates.
It computes them from the layer's input streams $x_{\text{in}}\in\mathbb R^{n\times p}$, flattened to $x_{\text{flat}}\in\mathbb R^{np}$, with one bias-free projection $W_{\text{fused}}\in\mathbb R^{3n^2\times np}$ of $\mathrm{LayerNorm}(x_{\text{flat}})$ into three $n\times n$ logit matrices, a softmax temperature $\tau$ and the $s$-step Cayley iteration $\Cay_s$ of Section~\ref{sec:cayley-method}:
\begin{equation}\label{eq:fused-proj}
  \begin{aligned}
    x_{\text{flat}} &\coloneqq \mathrm{Flatten}(x_{\text{in}}),\\
    (\tilde H_{\text{res}},\tilde H_{\text{pre}},\tilde H_{\text{post}}) &\coloneqq \mathrm{Split}\bigl(W_{\text{fused}}\,\mathrm{LayerNorm}(x_{\text{flat}})\bigr),\\
    H_{\text{pre}} &\coloneqq \softmax(\tilde H_{\text{pre}}/\tau,\ \dim{=}{-}1),\qquad
    H_{\text{post}} \coloneqq \softmax(\tilde H_{\text{post}}/\tau,\ \dim{=}{-}2),\\
    H_{\text{res}} &= \Cay_s(\Omega),\qquad \Omega=\tilde H_{\text{res}}-\tilde H_{\text{res}}^\top,
  \end{aligned}
\end{equation}
so each row of $H_{\text{pre}}$ and each column of $H_{\text{post}}$ is a softmax.
The branch reads the stream mean of $H_{\text{pre}}x$, $\tfrac1n\ones^\top H_{\text{pre}}x=v^\top x$ with $v=H_{\text{pre}}^\top\ones/n$, and writes its output to stream $i$ with $c$ times the $i$-th row sum of $H_{\text{post}}$, i.e.\ through $u=c\,H_{\text{post}}\ones$; the write scale is $c=2$ in our experiments.
Only these two vectors of $H_{\text{pre}}$ and $H_{\text{post}}$ reach the output, so the block is $x\mapsto H_{\text{res}}x+u\,F_\ell(v^\top x)^\top$, the form of Section~\ref{sec:spectrum}.
In the runs, as in TRM, each sublayer is followed by a per-stream RMSNorm, which Section~\ref{sec:spectrum} leaves out.

\subsection{Iterative Cayley transform}

The closed form~\eqref{eq:cayley-closed} needs a matrix inverse per token; \Cref{alg:cayley} replaces it by the fixed-point iteration of \citet{li2020efficient} (Section~\ref{sec:cayley-method}).
Its start $\Cay_0(\Omega)=I_n+\alpha\Omega$ (default $\alpha=0.1$) is one step from the identity and saves the multiplication of a general orthogonal starting point by $\alpha\Omega$.

\begin{algorithm}[H]
  \caption{Iterative Cayley Transform Projection\lean{FreeProb.cayleyIter}\lean[Formalised]{FreeProbReal.CayleySpec.toMat_cayleyIter}}
  \label{alg:cayley}
  \begin{algorithmic}[1]
    \Require Unconstrained logits $\tilde H_{\text{res}} \in \R^{n \times n}$, step-size $\alpha$, iterations $s$
    \Ensure Approximately orthonormal matrix $\Cay_s(\Omega) \in \R^{n \times n}$ ($\Cay_s^\top \Cay_s \approx I$)
    \State $\Omega \leftarrow \tilde H_{\text{res}} - \tilde H_{\text{res}}^\top$ \Comment{Skew-symmetrize}
    \State $\Cay_0(\Omega) \leftarrow I_n + \alpha \Omega$ \Comment{Initialize (saves one matmul: start at $I_n$)}
    \For{$i = 1, \ldots, s$}
      \State $\Cay_i(\Omega) \leftarrow I_n + \frac{\alpha}{2} \Omega (I_n + \Cay_{i-1}(\Omega))$ \Comment{Fixed-point step via \texttt{baddbmm}}
    \EndFor
    \State \Return $\Cay_s(\Omega)$
  \end{algorithmic}
\end{algorithm}

\subsection{Convergence, orthogonality defect and the number of iterations}\label{sec:cayley-tolerance}

All properties below hold under the convergence condition $r=\|\tfrac{\alpha}{2}\Omega\|_2<1$, which cannot be dropped: at $r=2$ the iteration does not converge to the rotation.
Under it, \Cref{alg:cayley} converges geometrically to the Cayley transform\lean[Formalised]{FreeProbReal.CayleySpec.cayleyIter_error_le}, norms are preserved in the limit\lean[Formalised]{FreeProbReal.CayleySpec.cayleyIter_norm_preservation}, and $\det\Cay_s(\Omega)\to1$\lean[Formalised]{FreeProbReal.CayleySpec.cayleyIter_det_tendsto}.

\paragraph{The defect at finite $s$.}
It is known in closed form: with $\Upsilon\coloneqq(\tfrac{\alpha}{2}\Omega)^\top(\tfrac{\alpha}{2}\Omega)$ and $\Cay_s\coloneqq\Cay_s(\Omega)$,
\begin{equation}\label{eq:cayley-defect}
  (I+\Upsilon)\bigl(\Cay_s^\top\Cay_s-I\bigr)=4(-1)^{\lfloor s/2\rfloor}\Upsilon^{\lfloor \frac{s+3}{2}\rfloor}+4\Upsilon^{s+2}
\end{equation}
\lean[Formalised]{FreeProbReal.CayleyTolerance.cayleyIter_orth_defect_eq}, hence, in spectral norm and for every $n$,
\begin{equation}\label{eq:cayley-defect-bounds}
  \frac{4\,\bigl|(-1)^{\lfloor s/2\rfloor}r^{2\lfloor (s+3)/2\rfloor}+r^{2s+4}\bigr|}{1+r^2}\;\le\;\|\Cay_s^\top\Cay_s-I\|_2\;\le\;4\,r^{2\lfloor (s+3)/2\rfloor}+4\,r^{2s+4}
\end{equation}
\lean[Formalised]{FreeProbReal.CayleyTolerance.cayleyIter_orth_defect_ge}\lean[Formalised]{FreeProbReal.CayleyTolerance.cayleyIter_orth_defect_le_parity}.
Two consequences bear on the deployed $s=2$.
First, steps $2m-1$ and $2m$ share their leading term: $\|\Cay_1^\top\Cay_1-I\|_2=4r^4$ and $\|\Cay_2^\top\Cay_2-I\|_2=4r^4(1-r^2)$, so $s=2$ buys almost nothing over $s=1$.
In particular $\|\Cay_2^\top\Cay_2-I\|_{\max}<10^{-3}$ (largest absolute entry) holds only for small logits $\tilde H_{\text{res}}$ (entries at most $1$: $2.5\cdot10^{-4}$) and fails for entries up to $5$.
Second, no fixed $s$ is dimension-free: at fixed $\alpha$, $r$ grows with $n$ (with unit-variance $\tilde H_{\text{res}}$ and $\alpha=0.1$ we measure $r=0.14,\,0.32,\,0.49$ for $n=4,8,16$, and the $s=2$ defect is $1.5\cdot10^{-3}$, $3.6\cdot10^{-2}$, $0.17$).

\paragraph{The number of iterations.}
We therefore choose $s$ from a bound $r\le r_{\max}<1$: for every tolerance $\varepsilon_{\text{orth}}>0$ on the orthogonality defect,
\begin{equation}\label{eq:cayley-rule}
  s\;\ge\;2\Bigl\lceil \frac{\log(\varepsilon_{\text{orth}}/8)}{2\log r_{\max}}\Bigr\rceil-3
  \quad\Longrightarrow\quad \|\Cay_s^\top\Cay_s-I\|_2\le\varepsilon_{\text{orth}}
\end{equation}
\lean[Formalised]{FreeProbReal.CayleyTolerance.cayleyIter_orth_defect_le_of_tolS}, stated directly for the executable iteration\lean[Formalised]{FreeProbReal.CayleyTolerance.cayleyIter_exec_orth_defect_le_of_tolS} (for $\varepsilon_{\text{orth}}=10^{-3}$ the rule gives $s=3,5,11$ at the radii above).

\paragraph{Over depth.} $S$ factors with defect at most $\varepsilon_{\text{orth}}\le1$ compose to defect at most $(1+\varepsilon_{\text{orth}})^S-1$, and every singular value of the product lies in $[(1-\varepsilon_{\text{orth}})^{S/2},(1+\varepsilon_{\text{orth}})^{S/2}]$\lean[Formalised]{FreeProbReal.CayleyTolerance.orth_defect_list_prod}\lean[Formalised]{FreeProbReal.CayleyTolerance.list_prod_apply_sq_bounds}; both ends are attained (the factors with $s\equiv 1,0 \bmod 4$ expand, those with $s\equiv 2,3$ contract).
This is the sense in which a Cayley stack within tolerance is co-isometric, the case where the nested and Gram-based compositions of Section~\ref{sec:spectrum} agree exactly\lean[Formalised]{FreeProbReal.GaussianMoment.iso_of_coisometric_gauss}; the step from a defect $\varepsilon_{\text{orth}}$ to a bound on the difference of the two compositions is not formalised.
The deployed $s=2$ map is analysed in Section~\ref{sec:beyond-init} and \Cref{prop:sens-cayley2}.

\subsection{Implementation in the runs}

The iteration runs as batched matrix products (\texttt{baddbmm}) over all tokens, with the identity held in a non-persistent buffer so that no tensor is allocated per call.
Its backward pass is standard autograd through these $s$ matrix products.
The model was compiled with \texttt{torch.compile} in \texttt{default} mode, which captures no CUDA graphs.
The map is evaluated in bfloat16 (\Cref{tab:training-config}): rounding to bfloat16 adds an orthogonality defect of its own, which no choice of $s$ removes.

\subsection{Comparison: Doubly-Stochastic vs.\ Orthonormal Mixing}

\Cref{tab:ds-vs-orth} compares the two mixer sets.

\begin{table}[h]
  \centering
  \caption{Properties of doubly-stochastic (Sinkhorn) vs.\ orthonormal (Cayley) mixing of $H_{\text{res}}$; $s_{\text{Sk}}$ is the number of Sinkhorn iterations (Appendix~\ref{sec:sinkhorn}).}
  \label{tab:ds-vs-orth}
  \begin{tabular}{lcc}
    \toprule
    Property & Doubly-Stochastic ($\mathrm{Bis}_n$) & Orthonormal ($O(n)$) \\
    \midrule
    Norm behavior & Contractive ($\|H_{\text{res}}x\| \leq \|x\|$) & Preserving ($\|H_{\text{res}}x\| = \|x\|$) \\
    Entries & Non-negative & Unconstrained sign \\
    Row/col sums & Both equal $\ones$ & Not constrained \\
    $\det H_{\text{res}}$ & $|\det H_{\text{res}}|\le1$ & $+1$ (Cayley image) \\
    Convex hull & Permutation matrices & Not a convex set \\
    Products & Contract $\ones^\perp$ (positive $H_{\text{res}}$) & Orthogonal \\
    Backward cost & $\mathcal{O}(s_{\text{Sk}})$ or implicit & Standard autograd \\
    \bottomrule
  \end{tabular}
\end{table}

Products of orthogonal mixers keep the skip path isometric at every recursion depth; the block is not, because of its branch (Section~\ref{sec:cayley-method}).

\subsection{Evaluation progression of the ARC runs}\label{app:eval-progression}

\Cref{tab:eval-progression,tab:sinkhorn-progression} list the evaluation metrics of the two ARC runs at each checkpoint, taken from the training logs; the best value per metric is the one reported in \Cref{fig:nested-l10} and Section~\ref{sec:results}.

\begin{table}[h]
  \centering
  \caption{Cayley evaluation metrics at each checkpoint.
  Bold indicates the best value for each metric.
  Training ran for ${\sim}516$K steps.}
  \label{tab:eval-progression}
  \begin{tabular}{rcccccc}
    \toprule
    Step & Exact Acc. & Pass@1 & Pass@2 & Pass@10 & Pass@100 & Pass@1000 \\
    \midrule
    68,951 & 10.6\% & --- & --- & --- & --- & --- \\
    81,903 & 12.7\% & --- & --- & --- & --- & --- \\
    94,854 & 15.7\% & --- & --- & --- & --- & --- \\
    107,805 & 17.6\% & --- & --- & --- & --- & --- \\
    161,952 & 23.4\% & 35.4\% & 40.1\% & 47.4\% & 52.7\% & 54.1\% \\
    174,904 & 24.5\% & 35.6\% & 40.3\% & 47.7\% & 53.9\% & 57.1\% \\
    188,951 & 25.1\% & 36.3\% & 39.6\% & 45.4\% & 51.2\% & 52.9\% \\
    201,902 & 25.4\% & 36.8\% & 40.7\% & 46.1\% & 53.8\% & 56.9\% \\
    214,853 & 26.3\% & 36.3\% & 40.4\% & 47.5\% & 54.9\% & 58.0\% \\
    231,951 & 26.0\% & 36.9\% & 39.8\% & 47.4\% & 52.5\% & 53.8\% \\
    244,902 & 26.4\% & 37.3\% & 41.1\% & 48.9\% & 55.5\% & 57.6\% \\
    257,853 & 27.8\% & 38.5\% & 41.6\% & 50.0\% & 56.2\% & 58.5\% \\
    270,804 & 27.0\% & 37.7\% & 41.7\% & 50.5\% & 57.4\% & 59.5\% \\
    283,755 & 26.9\% & 38.5\% & 43.1\% & 52.0\% & 58.0\% & 60.5\% \\
    296,707 & 28.1\% & 38.7\% & 43.8\% & 51.9\% & \textbf{59.2\%} & 62.0\% \\
    309,659 & 28.3\% & 39.0\% & 43.6\% & 52.6\% & \textbf{59.2\%} & \textbf{62.7\%} \\
    328,951 & 28.9\% & 36.4\% & 40.3\% & 46.8\% & 52.9\% & 54.4\% \\
    341,902 & 27.3\% & 37.9\% & 42.5\% & 49.6\% & 56.4\% & 58.1\% \\
    354,853 & 28.6\% & 38.1\% & 43.8\% & 52.2\% & 58.4\% & 60.1\% \\
    367,804 & 29.0\% & 39.5\% & 43.5\% & 52.7\% & 58.6\% & 60.6\% \\
    380,755 & 29.6\% & \textbf{40.5\%} & 43.6\% & \textbf{53.1\%} & 58.9\% & 61.4\% \\
    392,952 & 29.3\% & 37.7\% & 42.1\% & 49.6\% & 55.1\% & 56.6\% \\
    405,903 & 29.2\% & 38.7\% & 43.8\% & 51.7\% & 57.7\% & 60.5\% \\
    418,854 & \textbf{31.4\%} & 38.2\% & \textbf{45.4\%} & 52.5\% & 58.7\% & 61.6\% \\
    431,806 & 31.0\% & 38.9\% & 45.1\% & 52.1\% & 59.0\% & 62.6\% \\
    456,952 & 30.9\% & 38.5\% & 43.4\% & 48.4\% & 55.0\% & 55.3\% \\
    469,904 & 31.0\% & 38.6\% & 43.3\% & 49.6\% & 55.9\% & 57.9\% \\
    482,855 & 31.4\% & 38.6\% & 43.9\% & 50.0\% & 56.4\% & 58.4\% \\
    502,951 & 31.3\% & 36.4\% & 41.5\% & 47.5\% & 53.1\% & 53.3\% \\
    515,902 & 31.4\% & 36.9\% & 42.5\% & 49.3\% & 54.0\% & 55.6\% \\
    \bottomrule
  \end{tabular}
\end{table}

\begin{table}[h]
  \centering
  \caption{Sinkhorn evaluation metrics at each checkpoint.
  Bold indicates the best value for each metric.
  Training ran for ${\sim}511$K steps; Sinkhorn continued to improve in exact-match accuracy through the final checkpoint.}
  \label{tab:sinkhorn-progression}
  \begin{tabular}{rcccccc}
    \toprule
    Step & Exact Acc. & Pass@1 & Pass@2 & Pass@10 & Pass@100 & Pass@1000 \\
    \midrule
    52,951 & 1.4\% & 5.2\% & 7.5\% & 12.1\% & 14.9\% & 16.1\% \\
    65,902 & 3.1\% & 9.0\% & 11.6\% & 16.2\% & 21.4\% & 22.6\% \\
    78,854 & 5.0\% & 11.6\% & 16.2\% & 21.1\% & 26.8\% & 29.5\% \\
    99,951 & 8.5\% & 19.1\% & 22.8\% & 29.4\% & 36.6\% & 39.6\% \\
    112,902 & 9.9\% & 22.2\% & 26.2\% & 33.1\% & 40.6\% & 43.9\% \\
    125,853 & 11.2\% & 23.7\% & 28.0\% & 35.5\% & 42.8\% & 45.5\% \\
    138,805 & 13.0\% & 26.6\% & 30.3\% & 38.6\% & 44.9\% & 48.8\% \\
    151,757 & 14.3\% & 28.6\% & 31.7\% & 40.3\% & 46.0\% & 49.8\% \\
    164,709 & 15.9\% & 29.0\% & 33.0\% & 41.2\% & 47.4\% & 51.5\% \\
    177,661 & 16.2\% & 30.0\% & 34.6\% & 41.6\% & 47.5\% & 51.7\% \\
    190,612 & 17.4\% & 30.8\% & 34.7\% & 42.3\% & 48.5\% & 53.0\% \\
    203,563 & 17.1\% & 31.0\% & 35.5\% & 42.8\% & 48.8\% & 53.8\% \\
    216,514 & 18.2\% & 31.3\% & 35.5\% & 42.4\% & 50.5\% & 54.1\% \\
    229,465 & 18.6\% & 31.7\% & 36.0\% & 43.0\% & 50.9\% & 55.4\% \\
    242,416 & 18.7\% & 31.2\% & 35.6\% & 43.4\% & 51.1\% & \textbf{56.1\%} \\
    258,951 & 19.3\% & 30.8\% & 33.4\% & 42.3\% & 47.5\% & 48.0\% \\
    271,902 & 19.4\% & 31.9\% & 34.5\% & 44.1\% & 48.9\% & 51.0\% \\
    284,853 & 20.3\% & 33.6\% & 37.6\% & 44.7\% & 50.1\% & 52.2\% \\
    297,805 & 20.1\% & 34.1\% & 38.0\% & 45.5\% & 51.4\% & 54.3\% \\
    322,952 & 21.2\% & 31.5\% & 35.0\% & 42.5\% & 48.1\% & 49.1\% \\
    335,903 & 22.0\% & 33.6\% & 37.6\% & 44.6\% & 50.1\% & 51.4\% \\
    348,854 & 22.2\% & 34.0\% & 38.5\% & 45.6\% & 50.2\% & 53.0\% \\
    361,805 & 22.6\% & 34.7\% & 39.5\% & 46.4\% & 51.7\% & 54.3\% \\
    374,756 & 23.0\% & 36.3\% & 40.5\% & 45.9\% & \textbf{51.9\%} & 54.5\% \\
    393,952 & 22.9\% & 33.8\% & 38.5\% & 43.9\% & 48.8\% & 49.8\% \\
    406,903 & 24.0\% & \textbf{36.5\%} & 41.1\% & 46.3\% & 51.1\% & 52.6\% \\
    419,854 & 26.0\% & 36.0\% & \textbf{41.7\%} & \textbf{46.8\%} & 51.9\% & 53.4\% \\
    434,952 & 26.8\% & 35.7\% & 40.1\% & 45.8\% & 49.5\% & 49.8\% \\
    459,952 & 26.7\% & 35.0\% & 40.4\% & 43.6\% & 48.0\% & 48.2\% \\
    472,904 & 27.7\% & 36.1\% & 41.2\% & 45.6\% & 48.9\% & 49.4\% \\
    485,855 & 27.6\% & 35.9\% & 41.2\% & 46.4\% & 49.3\% & 50.6\% \\
    510,951 & \textbf{27.9\%} & 36.4\% & 40.6\% & 45.1\% & 48.9\% & 49.5\% \\
    \bottomrule
  \end{tabular}
\end{table}

\section{Sinkhorn Variant with Implicit Differentiation}
\label{sec:sinkhorn}

This appendix specifies the bistochastic arm of Section~\ref{sec:experiments}: the log-domain Sinkhorn forward pass, its implicit backward pass and the accuracy of that backward pass, the gates of the trained arm, and the cost of the two backward passes.
The implicit backward pass replaces differentiation through the $s_{\text{Sk}}$ unrolled iterations, whose autograd graph stalls distributed training.
Our contribution is its implementation from standard methods: the closed-form gradient of \citet{eisenberger2022unified}, a Gauss--Seidel solve of its linear system ($s_{\text{GS}}=16$ iterations at $n=4$), and fused forward and backward Triton kernels.

\subsection{Problem: Autograd Graph Explosion}

The standard implementation records all $s_{\text{Sk}}$ Sinkhorn iterations ($s_{\text{Sk}}$: the Sinkhorn iteration count) in the PyTorch autograd graph.
Each iteration involves two log-sum-exp operations, each of which PyTorch decomposes into multiple elementary ops (exp, sum, log, subtract).
For $s_{\text{Sk}}=20$ iterations on $n \times n = 4 \times 4$ matrices, this creates approximately $128{,}000$ backward nodes.

These nodes produce only \emph{microsecond-scale} GPU kernels---far too small to overlap with DDP AllReduce communication.
Our profiling on NVIDIA B200 GPUs revealed that 55\% of the total backward pass time was spent on DDP synchronization stalls, waiting for AllReduce operations to complete because no substantial compute was available to overlap with them.

\subsection{Solution: Custom Autograd Function with Implicit Differentiation}

We implement a custom \texttt{torch.autograd.Function} that decouples the forward and backward passes:

\paragraph{Forward Pass.}
The Sinkhorn iterations run under \texttt{torch.no\_grad()}, so no autograd graph is recorded.
Only the final doubly-stochastic matrix $\bP$ is saved for the backward pass:

\begin{algorithm}[H]
  \caption{Implicit Sinkhorn --- Forward Pass\lean{FreeProb.SinkhornLog.logSinkhorn}\lean[Formalised]{FreeProbReal.SinkhornLogSpec.logSinkhorn_eq}}
  \label{alg:sinkhorn-implicit-fwd}
  \begin{algorithmic}[1]
    \Require Logit matrix $\tilde H_{\text{res}} \in \R^{n \times n}$, iterations $s_{\text{Sk}}$
    \Ensure $\bP \in \mathrm{Bis}_n$ (doubly-stochastic)
    \State \textbf{with} \texttt{torch.no\_grad()}: \Comment{No autograd recording}
    \State \quad $\log \mathbf M_{\text{Sk}} \leftarrow \mathrm{clamp}(\tilde H_{\text{res}}, -10, 10)$ \Comment{Log-domain iterate; entrywise clamp}
    \For{$t = 1, \ldots, s_{\text{Sk}}$}
      \State $\log \mathbf M_{\text{Sk}} \leftarrow \log \mathbf M_{\text{Sk}} - \lse_{\text{row}}(\log \mathbf M_{\text{Sk}})$ \Comment{$\lse_{\text{row}}$: log-sum-exp of each row}
      \State $\log \mathbf M_{\text{Sk}} \leftarrow \log \mathbf M_{\text{Sk}} - \lse_{\text{col}}(\log \mathbf M_{\text{Sk}})$ \Comment{$\lse_{\text{col}}$: of each column}
    \EndFor
    \State $\bP \leftarrow \exp(\log \mathbf M_{\text{Sk}})$
    \State \textbf{save} $\bP$ for backward
    \State \Return $\bP$
  \end{algorithmic}
\end{algorithm}

\paragraph{Backward Pass.}
Gradients are computed via implicit differentiation of the fixed-point conditions, following \citet{eisenberger2022unified}.
At the doubly-stochastic fixed point, the constraints are:
\begin{equation}
  \bP \ones = \ones \quad \text{and} \quad \bP^\top \ones = \ones.
  \label{eq:ds-constraints}
\end{equation}

\begin{proposition}[Implicit Sinkhorn Gradient~\citep{eisenberger2022unified}]
  \label{prop:implicit-grad}
  Let $\bP \coloneqq \Pi_{\mathrm{Bis}}(\tilde H_{\text{res}})$ be the Sinkhorn projection of the logits $\tilde H_{\text{res}}$ onto the Birkhoff polytope ($\Pi_{\mathrm{Bis}}$: the Sinkhorn map of \Cref{alg:sinkhorn-implicit-fwd}).
  Given the upstream gradient $\nabla_{\bP}\calL$ of the training loss $\calL$, the gradient with respect to the input is, with $\odot$ the entrywise product:
  \begin{equation}
    \frac{\partial \calL}{\partial \tilde H_{\text{res}}} = \bP \odot \left(\nabla_{\bP}\calL - \mathbf a_{\text{row}} \ones^\top - \ones \mathbf a_{\text{col}}\right),
    \label{eq:implicit-grad}
  \end{equation}
  where $\mathbf a_{\text{row}} \in \R^{n \times 1}$ and $\mathbf a_{\text{col}} \in \R^{1 \times n}$ solve the coupled linear system:
  \begin{align}
    \bP \mathbf a_{\text{col}}^\top + \mathbf a_{\text{row}} &= \bh_{\text{row}}, \label{eq:coupled-u} \\
    \bP^\top \mathbf a_{\text{row}} + \mathbf a_{\text{col}}^\top &= \bh_{\text{col}}^\top, \label{eq:coupled-v}
  \end{align}
  with
  \begin{equation}
    \bh_{\text{row}} \coloneqq (\bP \odot \nabla_{\bP}\calL) \ones, \quad
    \bh_{\text{col}} \coloneqq \ones^\top (\bP \odot \nabla_{\bP}\calL).
    \label{eq:h-row-col}
  \end{equation}
\end{proposition}

\noindent The vectors $\mathbf a_{\text{row}}$ and $\mathbf a_{\text{col}}$ are the dual variables of the row and column constraints~\eqref{eq:ds-constraints}.

\subsection{Gauss-Seidel Solver for the Coupled System}

The coupled system~\eqref{eq:coupled-u}--\eqref{eq:coupled-v} can be solved efficiently via Gauss-Seidel iteration.
Starting from $\mathbf a_{\text{col}}^{(0)} \coloneqq \zeros$:

\begin{algorithm}[H]
  \caption{Implicit Sinkhorn --- Backward Pass (Gauss-Seidel)\lean{FreeProb.sinkhornBackward}\lean[Formalised]{FreeProbReal.SinkhornSpec.toMat_sinkhornBackward}}
  \label{alg:sinkhorn-implicit-bwd}
  \begin{algorithmic}[1]
    \Require Saved $\bP \in \mathrm{Bis}_n$, upstream gradient $\nabla_{\bP}\calL$
    \Ensure $\nabla_{\tilde H_{\text{res}}}\calL$
    \State $\mathbf M_{\text{GS}} \leftarrow \bP \odot \nabla_{\bP}\calL$ \Comment{Elementwise product}
    \State $\bh_{\text{row}} \leftarrow \mathbf M_{\text{GS}} \ones$ \Comment{Row sums, shape $(\ldots, n, 1)$}
    \State $\bh_{\text{col}} \leftarrow \ones^\top \mathbf M_{\text{GS}}$ \Comment{Column sums, shape $(\ldots, 1, n)$}
    \State $s_{\text{GS}} \leftarrow$ fixed iteration count \Comment{Accuracy: \Cref{sec:gs-convergence}}
    \State $\mathbf a_{\text{col}} \leftarrow \zeros$ \Comment{Shape $(\ldots, 1, n)$}
    \For{$i = 1, \ldots, s_{\text{GS}}$}
      \State $\mathbf a_{\text{row}} \leftarrow \bh_{\text{row}} - \bP \mathbf a_{\text{col}}^\top$ \Comment{Shape $(\ldots, n, 1)$}
      \State $\mathbf a_{\text{col}} \leftarrow \bh_{\text{col}} - (\bP^\top \mathbf a_{\text{row}})^\top$ \Comment{Shape $(\ldots, 1, n)$}
    \EndFor
    \State $\nabla_{\tilde H_{\text{res}}}\calL \leftarrow \mathbf M_{\text{GS}} - \mathbf a_{\text{row}} \odot \bP - \mathbf a_{\text{col}} \odot \bP$
    \State \Return $\nabla_{\tilde H_{\text{res}}}\calL$
  \end{algorithmic}
\end{algorithm}

\subsection{Convergence Analysis}\label{sec:gs-convergence}

The system~\eqref{eq:coupled-u}--\eqref{eq:coupled-v} is singular but consistent: $(\mathbf a_{\text{row}}+t\ones,\mathbf a_{\text{col}}-t\ones^\top)$ solves it for every real $t$ whenever $(\mathbf a_{\text{row}},\mathbf a_{\text{col}})$ does, and a solution exists for every doubly stochastic $\bP$\lean[Formalised]{FreeProbReal.SinkhornSpec.implicit_system_solvable}.
Accordingly the Gauss--Seidel iteration matrix $\bP^\top\bP$ fixes $\ones$: its spectral radius is $1$\lean[Formalised]{FreeProbReal.SinkhornSpec.gs_iterMatrix_hasEigenvalue_one}, and no contraction factor $1-1/n$ holds.
The gauge does not reach the gradient~\eqref{eq:implicit-grad}, which converges geometrically for every entrywise-positive $\bP$\lean[Formalised]{FreeProbReal.SinkhornSpec.backward_tendsto}; the measured rate is the squared second singular value of $\bP$.
It degrades as $\bP$ approaches a permutation matrix: $0.81$ at $\bP=\bigl(\begin{smallmatrix}.95&.05\\.05&.95\end{smallmatrix}\bigr)$, where a factor $1-1/n$ would give $0.5$, and at $n=4$ with $s_{\text{GS}}=16$ iterations a near-identity doubly stochastic $\bP$ leaves a relative residual of $0.26$.
The iteration count of our training runs ($s_{\text{GS}}=16$ at $n=4$) is therefore not a guarantee of gradient accuracy for strongly concentrated mixers.

Two further properties of the implemented pair (Algorithms~\ref{alg:sinkhorn-implicit-fwd} and~\ref{alg:sinkhorn-implicit-bwd}) qualify the gradient.
First, at every iteration count the backward pass returns exactly $\bP\odot(\nabla_{\bP}\calL-\mathbf a_{\text{row}}\ones^\top-\ones\mathbf a_{\text{col}})$ for the current $(\mathbf a_{\text{row}},\mathbf a_{\text{col}})$\lean[Formalised]{FreeProbReal.SinkhornLogSpec.backward_factor}, so entries with $\bP_{ij}=0$ get zero gradient whatever the accuracy of $\mathbf a_{\text{row}},\mathbf a_{\text{col}}$.
Second, the backward pass does not mask the entries clamped by the forward pass: where $|(\tilde H_{\text{res}})_{ij}|>10$ the forward output does not depend on $(\tilde H_{\text{res}})_{ij}$, yet the returned gradient is non-zero\lean[Formalised]{FreeProbReal.SinkhornLogSpec.errata17_instance}.
In binary32 it is this clamp, not the log-domain form, that keeps the forward pass finite.
Clamping $\bP$ at $10^{-8}$ preserves the direction of the gradient but not its norm (a factor $4.85$ on the near-permutation test case), so the two variants are not interchangeable under gradient clipping.

\subsection{The trained bistochastic arm}\label{sec:sinkhorn-run}

The ARC run of Section~\ref{sec:experiments} uses mHC's parameterisation~\citep{xie2025mhc}.
With $\hat x\coloneqq\mathrm{RMSNorm}(x_{\text{flat}})$ (no learned scale) and one fused projection $[\Xi_{\text{pre}}\,|\,\Xi_{\text{post}}\,|\,\Xi_{\text{res}}]\coloneqq W_{\text{fused}}\hat x$, $W_{\text{fused}}\in\R^{(2n+n^2)\times np}$ ($\Xi_{\text{pre}},\Xi_{\text{post}}\in\R^n$; $\Xi_{\text{res}}$ reshaped to $n\times n$), the read vector $v$, the write vector $u$ and the mixer are
\begin{equation}
  v\coloneqq\mathrm{sig}(g_{\text{pre}}\Xi_{\text{pre}}+b_{\text{pre}}),\quad
  u\coloneqq2\,\mathrm{sig}(g_{\text{post}}\Xi_{\text{post}}+b_{\text{post}}),\quad
  H_{\text{res}}\coloneqq\Pi_{\mathrm{Bis}}(\tilde H_{\text{res}}),\quad \tilde H_{\text{res}}\coloneqq g_{\text{res}}\Xi_{\text{res}}+b_{\text{res}},
\end{equation}
with $\mathrm{sig}(t)\coloneqq1/(1+e^{-t})$ the logistic function (entrywise), learned scalar gains $g_{\text{pre}},g_{\text{post}},g_{\text{res}}$ initialised at $0.01$, logit biases $b_{\text{pre}}=b_{\text{post}}=0$, $b_{\text{res}}=I_n$ and $s_{\text{Sk}}=20$ iterations; the block is $x\mapsto H_{\text{res}}x+u\,F(v^\top x)^\top$, followed by the per-stream RMSNorm of Appendix~\ref{app:cayley-gates}.
The read and write gates are thus vectors with a free per-token gain, unlike the softmax gates of the Cayley arm.

The run was resumed several times and its code changed while it trained; the checkpoints are lost, so the steps at which the changes took effect are not known.
The first segments iterated in the linear domain (alternating row and column normalisation of $\exp(\mathrm{clamp}(\tilde H_{\text{res}},-10,10))$) and differentiated through the unrolled iterations, then in the log domain, still unrolled; the implicit backward pass of \Cref{alg:sinkhorn-implicit-bwd} ($s_{\text{GS}}=16$) replaced unrolling from mid-February, and a fused GPU kernel evaluating the same two algorithms from late February.

\subsection{Complexity Comparison}

\Cref{tab:sinkhorn-complexity} compares the costs of the unrolled and the implicit backward pass.

\begin{table}[h]
  \centering
  \caption{Backward pass complexity: standard vs.\ implicit Sinkhorn.}
  \label{tab:sinkhorn-complexity}
  \begin{tabular}{lcc}
    \toprule
    & Standard Sinkhorn & Implicit Sinkhorn \\
    \midrule
    Autograd nodes per call & $\mathcal{O}(s_{\text{Sk}})$ & $\mathcal{O}(1)$ \\
    Backward kernels & $\sim\!128$K microsecond-scale & $\sim\!20$ millisecond-scale \\
    Memory (saved tensors) & $\mathcal{O}(s_{\text{Sk}} \cdot n^2)$ & $\mathcal{O}(n^2)$ (just $\bP$) \\
    Gradient accuracy & Exact & not guaranteed at $s_{\text{GS}}=16$ (\Cref{sec:gs-convergence}) \\
    \bottomrule
  \end{tabular}
\end{table}

\section{Training and Architecture Details}
\label{app:training-details}

This appendix gives the architecture, hyperparameters, data, metrics and run history of the two ARC-AGI runs of Section~\ref{sec:experiments}; the nano-GPT runs are described in Appendix~\ref{app:sensitivity}.

\subsection{TRM Architecture Configuration}

\begin{table}[h]
  \centering
  \caption{TRM architecture configuration of the ARC runs.}
  \label{tab:trm-config}
  \begin{tabular}{ll}
    \toprule
    Parameter & Value \\
    \midrule
    Parameters & $6.83$M (blocks, token embedding, heads) + mixers (\Cref{tab:param-counts}) \\
    Hidden dimension (stream width) $p$ & 512 \\
    Number of streams $n$ & 4 \\
    Effective hidden dim $np$ & 2048 \\
    Weight-tied layers $L$ & 2 (attention and SwiGLU MLP sublayers) \\
    Cycles $H$, $L_c$ & 3, 6 \\
    Network applications per supervision step & $H(L_c+1)=21$, gradient through the last $L_c+1=7$ \\
    Mixer applications per supervision step & $S=2LH(L_c+1)=84$, gradient through the last $2L(L_c+1)=28$ \\
    Attention heads & 8 \\
    MLP & SwiGLU, expansion 4 (inner width 1536) \\
    Norm & per-stream RMSNorm after each sublayer (no learned scale) \\
    Halting & Adaptive Computation Time (ACT) \\
    ACT max supervision steps & 16 \\
    ACT exploration probability & 0.1 \\
    Mixer modules & 4 (one per sublayer, shared across the recursion) \\
    Cayley arm & $\alpha=0.1$, $s=2$, softmax temperature $\tau=1$, write scale $c=2$ \\
    Sinkhorn arm & $s_{\text{Sk}}=20$ iterations, $s_{\text{GS}}=16$ Gauss--Seidel steps (\Cref{sec:sinkhorn-run}) \\
    Input and readout & embedding replicated to the streams; heads read the stream mean \\
    Positional encoding & RoPE \\
    Flash Attention & Enabled (with SDPA fallback) \\
    \bottomrule
  \end{tabular}
\end{table}

\subsection{Training Hyperparameters}

\begin{table}[h]
  \centering
  \caption{Training hyperparameters (shared across all variants).}
  \label{tab:training-config}
  \begin{tabular}{ll}
    \toprule
    Hyperparameter & Value \\
    \midrule
    Optimizer & AdamAtan2~\citep{everett2024scaling}, moment decay rates $(\beta_1,\beta_2)=(0.9,\,0.95)$ \\
    Learning rate & $1 \times 10^{-4}$ \\
    Global batch size & 768 \\
    Weight decay & 0.1 (decoupled, all dense parameters) \\
    Gradient clipping & global norm 1.0 \\
    LR schedule & Step decay at 80\% and 90\% of training \\
    LR decay factors & $0.316\times$ and $0.1\times$ \\
    Warmup steps & 2000 \\
    Training steps & $515{,}902$ (Cayley), $510{,}951$ (Sinkhorn) \\
    Precision & bfloat16 mixed precision; mixer parameters stored in bfloat16 \\
    Framework & PyTorch DDP (Accelerate), \texttt{torch.compile} (\texttt{default}) \\
    Hardware & NVIDIA B200 192GB GPUs ($\times 8$) \\
    Data & ARC-AGI-1 training + evaluation + ConceptARC, 1000 augmentations \\
    Puzzle embedding optimizer & SignSGD (lr=$10^{-2}$, weight decay 0.1) \\
    Puzzle embedding dim & 512 ($\times$ 16 tokens) \\
    \bottomrule
  \end{tabular}
\end{table}

\paragraph{AdamAtan2.}
With bias-corrected moments $\hat m_t,\hat v_t$ of the gradient at step $t$ (decay rates $\beta_1,\beta_2$), learning rate $\gamma_t$ and weight decay $\lambda_{\text{wd}}$, the update of a parameter coordinate $\theta$ is $\theta\leftarrow(1-\gamma_t\lambda_{\text{wd}})\,\theta-\gamma_t\,a_{\text{at}}\,\mathrm{atan2}(\hat m_t,\,b_{\text{at}}\sqrt{\hat v_t})$ with $a_{\text{at}}=1.27$, $b_{\text{at}}=1$~\citep{everett2024scaling}: no constant is added to $\sqrt{\hat v_t}$, and the atan2 step of every coordinate is at most $a_{\text{at}}\pi\gamma_t/2$.

\paragraph{Data and metrics.}
Each task presents a few demonstration input--output grid pairs and one or more test inputs; the solver must infer the latent rule and produce the exact output grid (\Cref{fig:arc-examples}), up to $30\times30$ cells of ten colours, reproduced cell for cell.
Every task of ARC-AGI-1 (training and evaluation) and ConceptARC enters in 1000 augmented copies (colour permutations fixing the background, dihedral maps, translations of the training grids); the demonstration pairs of every task are training data, and the test inputs of the 400 evaluation tasks are held out.
Pass@$k$ is, averaged over the evaluation tasks, the fraction of a task's test inputs whose output is among the $k$ predictions most voted across its augmented copies (ties broken by the mean halting confidence); grid-wise and token-wise accuracies are the exact-match and per-cell rates of single predictions.

\begin{figure}[h]
  \centering
  \begin{subfigure}{0.18\linewidth}
    \includegraphics[width=\linewidth]{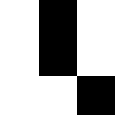}
    \caption{1: input}
  \end{subfigure}\hfill
  \begin{subfigure}{0.18\linewidth}
    \includegraphics[width=\linewidth]{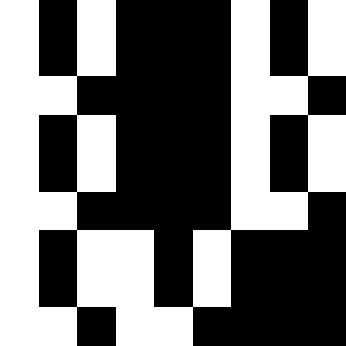}
    \caption{1: output}
  \end{subfigure}\hfill
  \begin{subfigure}{0.18\linewidth}
    \includegraphics[width=\linewidth]{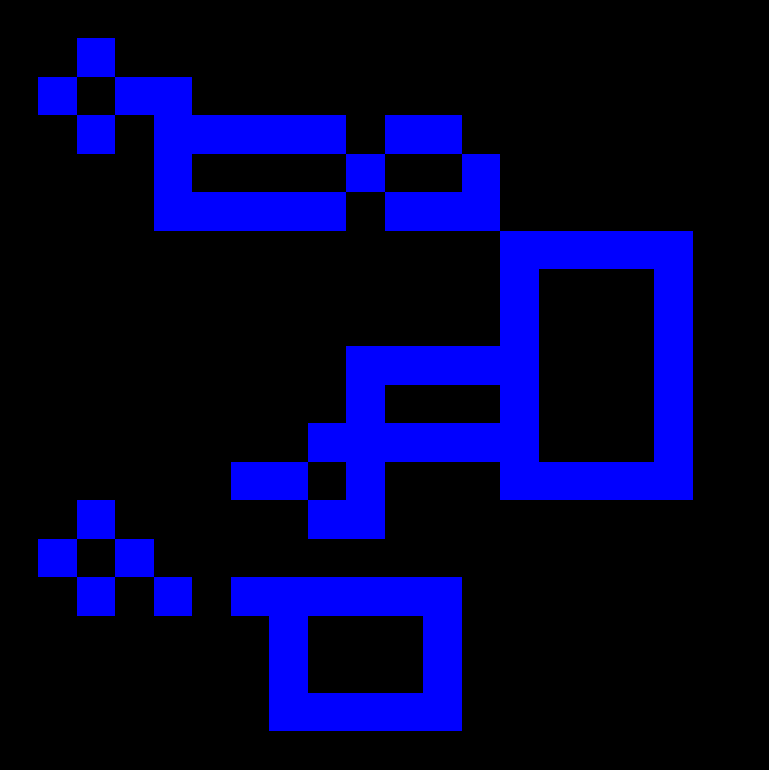}
    \caption{2: input}
  \end{subfigure}\hfill
  \begin{subfigure}{0.18\linewidth}
    \includegraphics[width=\linewidth]{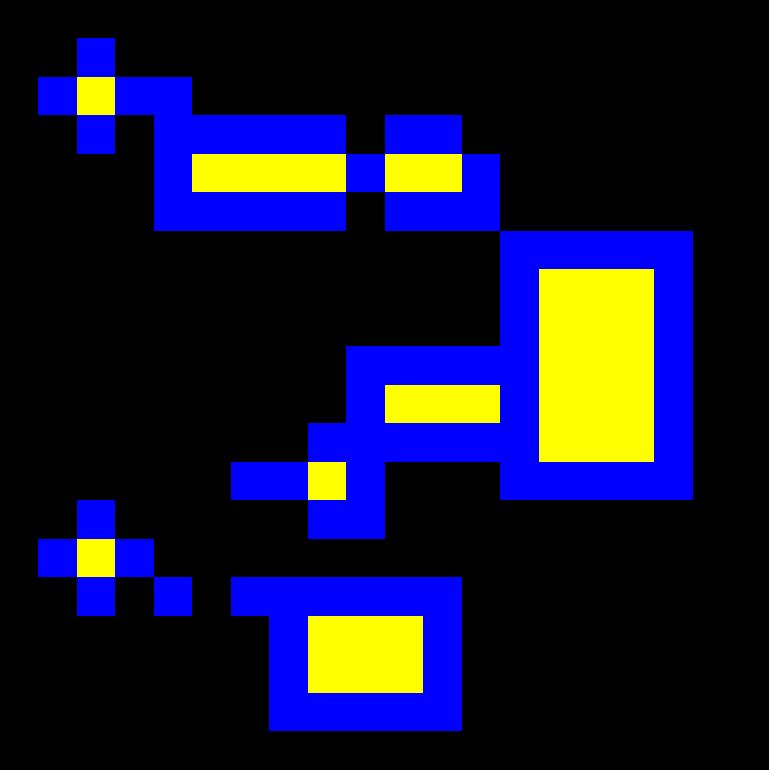}
    \caption{2: output}
  \end{subfigure}
  \caption{Two ARC-AGI tasks: a latent rule (pattern tiling, region filling) is inferred from a few demonstrations and applied to a new input.}
  \label{fig:arc-examples}
\end{figure}

\paragraph{Runs.}
Each arm was trained once (seed 0) with the hyperparameters above, under PyTorch DDP through Accelerate (no DeepSpeed), and evaluated every 2500 steps; there is no weight averaging (EMA).
Both runs were interrupted and resumed from checkpoints; \texttt{torch.compile} (\texttt{default} mode) was enabled from the first resumption on.
The Cayley arm's code did not change in function during its run; the Sinkhorn arm's did (\Cref{sec:sinkhorn-run}).
The evaluation LM loss of both arms is lowest near 420K steps and rises afterwards, while exact match plateaus (Cayley) or keeps improving (Sinkhorn) and pass@1000 falls from its peak (\Cref{fig:eval-accuracy,tab:eval-progression,tab:sinkhorn-progression}): mild overfitting, which reduces sampling diversity while greedy accuracy holds.

\begin{figure*}[t]
\centering
\begin{tikzpicture}
\begin{axis}[
    name=evalacc,
    width=0.48\textwidth,
    height=0.29\textwidth,
    xlabel={Training Steps ($\times 10^3$)},
    ylabel={Eval LM Loss (per token)},
    legend pos=north east,
    grid=major,
    xmin=0, xmax=520,
    ymin=0.5, ymax=1.5,
]
\addplot[blue, thick, mark=*, mark size=1.5pt] coordinates {
    (69, 1.238) (82, 1.132) (95, 1.093) (108, 1.023) (162, 0.901) (175, 0.857) (189, 0.851) (202, 0.862) (215, 0.856) (232, 0.855) (245, 0.751) (258, 0.695) (271, 0.673) (284, 0.676) (297, 0.704) (310, 0.646) (329, 0.697) (342, 0.646) (355, 0.671) (368, 0.681) (381, 0.715) (393, 0.736) (406, 0.677) (419, 0.643) (432, 0.688) (457, 0.697) (470, 0.702) (483, 0.685) (503, 0.739) (516, 0.767)
};
\addlegendentry{Cayley}

\addplot[red, thick, mark=square*, mark size=1.5pt] coordinates {
    (53, 1.415) (66, 1.194) (79, 1.099) (100, 0.976) (113, 0.974) (126, 0.972) (139, 0.968) (152, 0.929) (165, 0.951) (178, 0.927) (191, 0.912) (204, 0.904) (217, 0.909) (229, 0.897) (242, 0.898) (259, 0.880) (272, 0.898) (285, 0.906) (298, 0.845) (323, 0.842) (336, 0.874) (349, 0.819) (362, 0.827) (375, 0.832) (394, 0.824) (407, 0.875) (420, 0.817) (435, 0.852) (460, 0.873) (473, 0.873) (486, 0.888) (511, 0.902)
};
\addlegendentry{Sinkhorn}

\end{axis}

\begin{axis}[
    at=(evalacc.right of south east), anchor=left of south west,
    xshift=1.2cm,
    width=0.48\textwidth,
    height=0.29\textwidth,
    xlabel={Training Steps ($\times 10^3$)},
    ylabel={Exact Accuracy (\%)},
    title={\small (b) Exact-Match (Full Grid)},
    legend pos=south east,
    grid=major,
    xmin=0, xmax=520,
    ymin=0, ymax=35,
]
\addplot[blue, thick, mark=*, mark size=1.5pt] coordinates {
    (69, 10.6) (82, 12.7) (95, 15.7) (108, 17.6) (162, 23.4) (175, 24.5) (189, 25.1) (202, 25.4) (215, 26.3) (232, 26.0) (245, 26.4) (258, 27.8) (271, 27.0) (284, 26.9) (297, 28.1) (310, 28.3) (329, 28.9) (342, 27.3) (355, 28.6) (368, 29.0) (381, 29.6) (393, 29.3) (406, 29.2) (419, 31.4) (432, 31.0) (457, 30.9) (470, 31.0) (483, 31.4) (503, 31.3) (516, 31.4)
};
\addlegendentry{Cayley}

\addplot[red, thick, mark=square*, mark size=1.5pt] coordinates {
    (53, 1.4) (66, 3.1) (79, 5.0) (100, 8.5) (113, 9.9) (126, 11.2) (139, 13.0) (152, 14.3) (165, 15.9) (178, 16.2) (191, 17.4) (204, 17.1) (217, 18.2) (229, 18.6) (242, 18.7) (259, 19.3) (272, 19.4) (285, 20.3) (298, 20.1) (323, 21.2) (336, 22.0) (349, 22.2) (362, 22.6) (375, 23.0) (394, 22.9) (407, 24.0) (420, 26.0) (435, 26.8) (460, 26.7) (473, 27.7) (486, 27.6) (511, 27.9)
};
\addlegendentry{Sinkhorn}

\end{axis}
\end{tikzpicture}
\caption{Evaluation curves of the two ARC-AGI runs. \textbf{(a)}~Evaluation LM loss per token: both arms are lowest near 420K steps and rise afterwards. \textbf{(b)}~Exact match on full grids: the Cayley arm leads throughout; the Sinkhorn arm is still improving at the end of its run.}
\label{fig:eval-accuracy}
\end{figure*}

\section{Effective spectra under finite-precision gradient accumulation}\label{app:effective-spectra}

Section~\ref{sec:spectrum} computes the law of the eigenvalues $\lambda=\varsigma^2$ of $J^\top J$, $\varsigma$ a singular value of $J$, in exact arithmetic: its first moment~\eqref{eq:nested-m1} and its density at resolution $\eta_{\text{sm}}$ (Figure~\ref{fig:ov-vs-mc-panel}).
Training does not see that law directly.
In the singular basis of $J$ the backpropagated gradient of a sample is the loss gradient with its $i$-th component multiplied by the $i$-th singular value $\varsigma_i$, and a batch sums these per-sample gradients in floating point, where an addend much smaller than the running sum is rounded off.
This appendix asks which part of the law reaches the optimiser, and measures it.
What reaches it, the \emph{effective spectrum}, is, up to rounding, the exact law with part of its mass moved to an atom at zero.
A law concentrated at one value is treated almost the same wherever it sits, because the floating-point grid is relative; the spread of the law is what costs, and the spread is what contracting mixers create with depth (\Cref{sec:fc-results}).
On the weight gradients of the trained nano-GPT models the atom hardly depends on the mixer at these depths; it is set by the bf16 cast of the stored gradient, not by the fp32 accumulation (\Cref{sec:fc-measured}).
The mechanism is a cutoff relative to the running sum at each accumulation step (\Cref{thm:fc-step}); along a batch the cutoff follows the sampled values, their signs and their order (\Cref{thm:fc-sum,cor:fc-batch}).

\paragraph{Status of the statements.}
Statements marked \textsf{L} are proved in Lean (\Cref{app:lean}).
Remarks marked \emph{derived} are arguments on paper, which the Monte Carlo of \Cref{sec:fc-mc} tests.
Every number in this appendix is generated from the result files of that Monte Carlo and of the measurements of \Cref{sec:fc-measured} (repository directory \texttt{experiments/float-cutoff}).

\subsection{Setting}\label{sec:fc-setting}

\paragraph{Sampling model.}
Fix one parameter coordinate.
In a batch of $B_{\text{b}}$ samples (the batch size), sample $k$ contributes $\pm g_{\text{scale}}\,\varsigma_k$ to the gradient of that coordinate: $\varsigma_k$ is the distortion its backpropagated gradient met, drawn from the law of the singular values of $J$; $g_{\text{scale}}>0$ is a scale; the sign is $+$ for every sample (\emph{coherent}: all samples push the same way) or an independent fair sign (\emph{Rademacher}).
The model reduces each backpropagated vector to one singular direction.
It is the hypothesis this appendix tests, not a theorem.

\paragraph{Floating-point accumulation.}
The contributions are stored in a format $\mathrm{fl}_{\mathrm{in}}$ as the addends $v^{\mathrm{in}}_k\coloneqq\mathrm{fl}_{\mathrm{in}}(\pm g_{\text{scale}}\,\varsigma_k)$, summed left to right in a format $\mathrm{fl}_{\mathrm{acc}}$ into the partial sums $S^{\mathrm{acc}}_0\coloneqq0$, $S^{\mathrm{acc}}_k\coloneqq\mathrm{fl}_{\mathrm{acc}}(S^{\mathrm{acc}}_{k-1}+v^{\mathrm{in}}_k)$, and the gradient is stored as $g_{\text{st}}\coloneqq\mathrm{fl}_{\mathrm{out}}(S^{\mathrm{acc}}_{B_{\text{b}}})$.
We use bf16, fp16 or fp32 for all three, and the mixed-precision configurations bf16-acc32 and fp16-acc32 (low-precision values and stored gradient, fp32 accumulator, as in tensor-core matrix products).

\begin{definition}[Effective spectrum]\label{def:fc-effective}
  The \emph{step transmission} of addend $k$ is $\delta^{\mathrm{step}}_k\coloneqq S^{\mathrm{acc}}_k-S^{\mathrm{acc}}_{k-1}$ (T1); its \emph{leave-one-out transmission} is $\delta^{\mathrm{loo}}_k\coloneqq g_{\text{st}}-g_{\text{st}}^{(-k)}$, where $g_{\text{st}}^{(-k)}$ is the same computation without $v^{\mathrm{in}}_k$ (T2).
  The effective singular value of the addend is $|\delta^{\mathrm{step}}_k|/g_{\text{scale}}$ (T1) or $|\delta^{\mathrm{loo}}_k|/g_{\text{scale}}$ (T2), and the \emph{effective spectrum} is the law $\mu_{\mathrm{eff}}$ of its square over the addends of all batches.
  It decomposes as $\mu_{\mathrm{eff}}=a_{\text{atom}}\,\delta_0+(1-a_{\text{atom}})\,\mu_{\mathrm{tr}}$, with $\delta_0$ the unit point mass at zero; we call $a_{\text{atom}}$ the \emph{atom}.
  The transmitted part $\mu_{\mathrm{tr}}$ is the exact law $\mu_{J^\top J}$ of $\lambda=\varsigma^2$ over the surviving addends up to rounding, which moves an addend $v^{\mathrm{in}}$ by a relative amount of order $u_{\text{fl}}\,|S^{\mathrm{acc}}|/|v^{\mathrm{in}}|$, with $S^{\mathrm{acc}}$ the running sum and $u_{\text{fl}}$ the unit roundoff of the format (\Cref{def:fc-format}).
\end{definition}

\noindent T1 is what each step added; T2 is what sample $k$ changed in the gradient the optimiser receives.
An addend kept by its step can still leave $g_{\text{st}}$ unchanged (\Cref{rem:fc-order}).
The T2 atom is the fraction of the spectrum that stalls for this mechanism: the directions whose gradient never reaches the optimiser.

\subsection{Floating-point formats and one accumulation step}\label{sec:fc-step}

\begin{definition}[Formats]\label{def:fc-format}
  For a precision $p_{\text{fl}}\ge1$, a real number $\xi\in\R$ is a \emph{$p_{\text{fl}}$-bit float} if $\xi=m_{\text{fl}}\,2^{e_{\text{fl}}}$ with significand and exponent $m_{\text{fl}},e_{\text{fl}}\in\mathbb Z$, $|m_{\text{fl}}|<2^{p_{\text{fl}}}$; if moreover $e_{\text{fl}}\ge e_{\min}$, a minimum exponent, it lies in the bounded format $\mathbb F_{p_{\text{fl}},e_{\min}}$ (gradual underflow: the subnormals are the $m_{\text{fl}}\,2^{e_{\min}}$).
  A map $\mathrm{rnd}$ is a \emph{round-to-nearest} onto a set $\mathbb F\subset\R$ if $\mathrm{rnd}(\xi)\in\mathbb F$ and $|\mathrm{rnd}(\xi)-\xi|\le|\xi'-\xi|$ for every $\xi\in\R$ and $\xi'\in\mathbb F$ (any tie rule).
  The unit roundoff is $u_{\text{fl}}\coloneqq2^{-p_{\text{fl}}}$; we write $c_{\text{fl}}\coloneqq(1-u_{\text{fl}})\,u_{\text{fl}}/2$.
\end{definition}

\begin{table}[h]
  \centering\small
  \begin{tabular}{@{}lccccccc@{}}
    \toprule
    format & $p_{\text{fl}}$ & $e_{\min}$ & $u_{\text{fl}}$ & $c_{\text{fl}}=(1-u_{\text{fl}})u_{\text{fl}}/2$ & smallest normal
    & rounds to $0$ & $B_{\text{b}}^\ast$ \\
    \midrule
    bf16 & $8$ & $-133$ & $2^{-8}$ & \fcCBf & $2^{-126}$ & $|\xi|<2^{-134}$ & \fcNBf \\
    fp16 & $11$ & $-24$ & $2^{-11}$ & \fcCFh & $2^{-14}$ & $|\xi|<2^{-25}$ & \fcNFh \\
    fp32 & $24$ & $-149$ & $2^{-24}$ & \fcCFs & $2^{-126}$ & $|\xi|<2^{-150}$ & \fcNFs \\
    \bottomrule
  \end{tabular}
  \caption{The three formats. $B_{\text{b}}^\ast$: the largest batch size at which ascending order provably absorbs nothing (\Cref{cor:fc-batch}).}\label{tab:fc-formats}
\end{table}

\begin{theorem}[One step]\label{thm:fc-step}
  Let $\mathrm{rnd}$ be a round-to-nearest onto $\mathbb F_{p_{\text{fl}},e_{\min}}$, $p_{\text{fl}}\ge1$, let $S^{\mathrm{acc}}\in\mathbb F_{p_{\text{fl}},e_{\min}}$ be a running sum and $v^{\mathrm{in}}\in\R$ an addend.
  \begin{enumerate}
    \item Two distinct $p_{\text{fl}}$-bit floats $\xi\ne\xi'$ satisfy $|\xi-\xi'|\ge(1-u_{\text{fl}})\,u_{\text{fl}}\,|\xi|$.\lean[Formalised]{FreeProbReal.FloatCutoff.float_gap}
    \item If $|v^{\mathrm{in}}|<c_{\text{fl}}\,|S^{\mathrm{acc}}|$ then $\mathrm{rnd}(S^{\mathrm{acc}}+v^{\mathrm{in}})=S^{\mathrm{acc}}$.
      Conversely, if $\mathrm{rnd}(S^{\mathrm{acc}}+v^{\mathrm{in}})=S^{\mathrm{acc}}$ and $|S^{\mathrm{acc}}+v^{\mathrm{in}}|\ge2^{e_{\min}+p_{\text{fl}}-1}$, then $(1-u_{\text{fl}})\,|v^{\mathrm{in}}|\le u_{\text{fl}}\,|S^{\mathrm{acc}}|$.\lean[Formalised]{FreeProbReal.FloatCutoff.step_band}
    \item In the normal range $|\xi|\ge2^{e_{\min}+p_{\text{fl}}-1}$, $|\mathrm{rnd}(\xi)-\xi|\le u_{\text{fl}}\,|\xi|$; onto the $p_{\text{fl}}$-bit floats (unbounded exponent) this holds for every $\xi$.\lean[Formalised]{FreeProbReal.FloatCutoff.rnE_rel_err}\lean[Formalised]{FreeProbReal.FloatCutoff.rn_rel_err}
    \item If $|\xi|<2^{e_{\min}}/2$ then $\mathrm{rnd}(\xi)=0$.\lean[Formalised]{FreeProbReal.FloatCutoff.rnE_underflow}
  \end{enumerate}
\end{theorem}

\noindent An addend $v^{\mathrm{in}}$ with $\mathrm{rnd}(S^{\mathrm{acc}}+v^{\mathrm{in}})=S^{\mathrm{acc}}$ is \emph{absorbed}.

\begin{proof}[Sketch]
  (1) With $e_{\text{fl}}$, $e_{\text{fl}}'$ the exponents of $\xi$, $\xi'$, both floats are integer multiples of $2^{\min(e_{\text{fl}},e_{\text{fl}}')}$, so they differ by at least that much, while $u_{\text{fl}}|\xi|<2^{e_{\text{fl}}}$. (2) If $\mathrm{rnd}(S^{\mathrm{acc}}+v^{\mathrm{in}})\ne S^{\mathrm{acc}}$, both $S^{\mathrm{acc}}$ and $\mathrm{rnd}(S^{\mathrm{acc}}+v^{\mathrm{in}})$ lie within $|v^{\mathrm{in}}|$ of $S^{\mathrm{acc}}+v^{\mathrm{in}}$, hence within $2|v^{\mathrm{in}}|<(1-u_{\text{fl}})u_{\text{fl}}|S^{\mathrm{acc}}|$ of each other, contradicting (1).
  The converse applies (3) at $S^{\mathrm{acc}}+v^{\mathrm{in}}$. (3) The grid point nearest to $\xi$ at the quantum of $|\xi|$ is within half a quantum, which is at most $u_{\text{fl}}|\xi|$. (4) Every nonzero element of the format has magnitude at least $2^{e_{\min}}$, so $0$ is nearer.
\end{proof}

\begin{remark}[A relative cutoff, derived]\label{rem:fc-dirac}
  By (2) an addend is absorbed below $c_{\text{fl}}\,|S^{\mathrm{acc}}|$ and survives above $u_{\text{fl}}\,|S^{\mathrm{acc}}|/(1-u_{\text{fl}})$: the cutoff is a band a factor $\approx2$ wide, relative to the running sum and not to the addend's own size.
  Scaling every addend by a power of two scales every partial sum and every rounding by the same power while all values stay in the normal range, so a law concentrated at one value is treated the same wherever it sits in that range.
  Statement (4) is the exception: fp16's normal range starts at $2^{-14}$ and everything below $2^{-25}$ is lost, which loss scaling exists to prevent; bf16 has fp32's exponent range.
\end{remark}

\subsection{Accumulation along a batch}\label{sec:fc-sum}

Write $\mathrm{psum}(v^{\mathrm{in}}_1,\dots,v^{\mathrm{in}}_{n_{\text{add}}})\coloneqq S^{\mathrm{acc}}_{n_{\text{add}}}$ for the recursive sum of $n_{\text{add}}$ addends, and say that step $k$ \emph{absorbs} if $v^{\mathrm{in}}_k\ne0$ and $S^{\mathrm{acc}}_k=S^{\mathrm{acc}}_{k-1}$.

\begin{theorem}[Recursive sums]\label{thm:fc-sum}
  Let $\mathrm{rnd}$ satisfy $|\mathrm{rnd}(\xi)-\xi|\le u'_{\text{fl}}|\xi|$ for all $\xi$, with a relative error bound $u'_{\text{fl}}\ge0$.
  \begin{enumerate}
    \item $|S^{\mathrm{acc}}_{n_{\text{add}}}|\le(1+u'_{\text{fl}})^{n_{\text{add}}}\sum_k|v^{\mathrm{in}}_k|$ and $\bigl|S^{\mathrm{acc}}_{n_{\text{add}}}-\sum_kv^{\mathrm{in}}_k\bigr|\le\bigl((1+u'_{\text{fl}})^{n_{\text{add}}}-1\bigr)\sum_k|v^{\mathrm{in}}_k|$.%
      \lean[Formalised]{FreeProbReal.FloatCutoff.psum_abs_le}\lean[Formalised]{FreeProbReal.FloatCutoff.psum_err_le}
    \item (\emph{Ascending order.})
      If $v^{\mathrm{in}}\ne0$, every earlier addend has $|v^{\mathrm{in}}_j|\le|v^{\mathrm{in}}|$, and $u'_{\text{fl}}n_{\text{add}}(1+u'_{\text{fl}})^{n_{\text{add}}}<1-u'_{\text{fl}}$, then $v^{\mathrm{in}}$ is not absorbed: $\mathrm{rnd}(S^{\mathrm{acc}}_{n_{\text{add}}}+v^{\mathrm{in}})\ne S^{\mathrm{acc}}_{n_{\text{add}}}$.\lean[Formalised]{FreeProbReal.FloatCutoff.ascending_transmits}
  \end{enumerate}
  Let instead $\mathrm{rnd}$ be a round-to-nearest onto a set $\mathbb F\ni0$ of $p_{\text{fl}}$-bit floats.
  \begin{enumerate}
    \setcounter{enumi}{2}
    \item (\emph{Coherent addends, any order.})
      If $v^{\mathrm{in}}_1,\dots,v^{\mathrm{in}}_{n_{\text{add}}}\in\mathbb F$ are nonnegative, every partial sum dominates every addend so far, and an addend $v^{\mathrm{in}}$ with $|v^{\mathrm{in}}|<c_{\text{fl}}\,v^{\mathrm{in}}_j$ for some earlier $v^{\mathrm{in}}_j$ is absorbed: $\mathrm{rnd}(S^{\mathrm{acc}}_{n_{\text{add}}}+v^{\mathrm{in}})=S^{\mathrm{acc}}_{n_{\text{add}}}$.\lean[Formalised]{FreeProbReal.FloatCutoff.psum_ge}\lean[Formalised]{FreeProbReal.FloatCutoff.running_max_absorb}
  \end{enumerate}
\end{theorem}

\begin{proof}[Sketch]
  (1) Induction on $n_{\text{add}}$, one rounding per step. (2) The earlier addends sum to at most $n_{\text{add}}|v^{\mathrm{in}}|$ in absolute value, so by (1) $u'_{\text{fl}}|S^{\mathrm{acc}}_{n_{\text{add}}}|\le u'_{\text{fl}}n_{\text{add}}(1+u'_{\text{fl}})^{n_{\text{add}}}|v^{\mathrm{in}}|<(1-u'_{\text{fl}})|v^{\mathrm{in}}|$, which contradicts the converse of \Cref{thm:fc-step}(2) (proved with $u'_{\text{fl}}$ in place of $u_{\text{fl}}$). (3) Round-to-nearest is monotone and fixes $\mathbb F$, so $S^{\mathrm{acc}}_k\ge S^{\mathrm{acc}}_{k-1}$ and $S^{\mathrm{acc}}_k\ge v^{\mathrm{in}}_k$; then \Cref{thm:fc-step}(2) applies with $|S^{\mathrm{acc}}_{n_{\text{add}}}|\ge v^{\mathrm{in}}_j$.
\end{proof}

\begin{corollary}[A batch]\label{cor:fc-batch}
  \emph{Worst case.}
  In a coherent batch $v^{\mathrm{in}}_1,\dots,v^{\mathrm{in}}_{B_{\text{b}}}\in\mathbb F$ led by its largest addend, the number of absorbing steps is at least the number of nonzero addends below $c_{\text{fl}}\,v^{\mathrm{in}}_1$.\lean[Formalised]{FreeProbReal.FloatCutoff.absorbs_of_below}\lean[Formalised]{FreeProbReal.FloatCutoff.atom_ge_mass_below_max} \emph{Best case.}
  If $|v^{\mathrm{in}}_1|\le\dots\le|v^{\mathrm{in}}_{B_{\text{b}}}|$ and $u'_{\text{fl}}B_{\text{b}}(1+u'_{\text{fl}})^{B_{\text{b}}}<1-u'_{\text{fl}}$, no step absorbs.\lean[Formalised]{FreeProbReal.FloatCutoff.ascending_no_atom}
\end{corollary}

\noindent In the worst case, the atom contains every sampled value below $c_{\text{fl}}$ times the sampled maximum; a batch in descending order is led by its largest addend.
In the best case, with $u'_{\text{fl}}=u_{\text{fl}}$, as for round-to-nearest onto the $p_{\text{fl}}$-bit floats (\Cref{thm:fc-step}(3)), no step absorbs up to $B_{\text{b}}=B_{\text{b}}^\ast$ of Table~\ref{tab:fc-formats}.

\begin{remark}[The size of the cutoff, derived]\label{rem:fc-scale}
  Within a batch the threshold is $c_{\text{fl}}\,|S^{\mathrm{acc}}_{k-1}|$, so it depends on the sampled values through the partial sums.
  Let $\varsigma_{\mathrm{rms}}$ and $\bar\varsigma$ be the root mean square and the mean of the law of the $\varsigma_k$, and $\varsigma_{\max}\coloneqq\max_k\varsigma_k$ the sampled maximum of the batch.
  With Rademacher signs $|S^{\mathrm{acc}}_k|\approx\sqrt k\,\varsigma_{\mathrm{rms}}\,g_{\text{scale}}$; with coherent signs $S^{\mathrm{acc}}_k\approx k\,\bar\varsigma\,g_{\text{scale}}$.
  Relative to the sampled maximum the cutoff is of order $c_{\text{fl}}\sqrt{B_{\text{b}}}\,\varsigma_{\mathrm{rms}}/\varsigma_{\max}$ (Rademacher) or $c_{\text{fl}}\,B_{\text{b}}\,\bar\varsigma/\varsigma_{\max}$ (coherent): it grows with the batch, and for coherent sums it reaches the sampled values themselves.
  A law concentrated at one value is therefore not immune to the batch size: a coherent bf16 sum of equal addends stops growing once half an ulp of $S^{\mathrm{acc}}$ exceeds the addend (in bf16, $1000$ ones sum to $256$).
\end{remark}

\begin{remark}[Order and the stored gradient, derived]\label{rem:fc-order}
  Ascending order postpones absorption (\Cref{cor:fc-batch}), but the stored gradient still has the resolution of its final format: an addend kept by its step changes $g_{\text{st}}$ only if it moves the sum across a rounding boundary of $\mathrm{fl}_{\mathrm{out}}$, roughly if it exceeds half an ulp of $g_{\text{st}}$.
  Ascending order can therefore remove the T1 atom without removing the T2 atom.
\end{remark}

\begin{remark}[The first moment does not see the atom, derived]\label{rem:fc-m1}
  An absorbed addend has $|v^{\mathrm{in}}|\le u_{\text{fl}}\,|S^{\mathrm{acc}}|/(1-u_{\text{fl}})$, so the absorbed addends carry a small fraction of the squared mass of the batch: the effective first moment $\sum_k(\delta^{\mathrm{step}}_k)^2\big/\bigl(g_{\text{scale}}^2\sum_k\varsigma_k^2\bigr)$ stays close to $1$ even when most addends are absorbed.
  The moments of the free-probability analysis, and its density at resolution $\eta_{\text{sm}}\gg u_{\text{fl}}^2\lambda_{\max}$, with $\lambda_{\max}$ the largest eigenvalue of $J^\top J$, cannot tell the effective spectrum from the exact one.
  The atom counts the directions that reach the optimiser, not their energy.
\end{remark}

\begin{proposition}[The stagnation criteria of \citet{mikaitis2026stagnation}]\label{prop:fc-mw}
  Let $\mathrm{ulp}(\xi)\coloneqq2^{\lfloor\log_2|\xi|\rfloor-p_{\text{fl}}+1}$ and let $\mathrm{rnd}$ be a round-to-nearest onto a set of $p_{\text{fl}}$-bit floats.
  \begin{enumerate}
    \item (\emph{Same signs.})
      If $\xi>0$ is a float and $0\le\xi'<\mathrm{ulp}(\xi)/2$, then $\mathrm{rnd}(\xi+\xi')=\xi$\lean[Formalised]{FreeProbReal.FloatCutoff.mw_absorb_same_sign}; likewise for $\xi<0$, $\xi'\le0$ in a symmetric format\lean[Formalised]{FreeProbReal.FloatCutoff.mw_absorb_same_sign_neg}.
    \item (\emph{Opposite signs.})
      For $p_{\text{fl}}\ge2$, $\xi'=-3\cdot2^{-p_{\text{fl}}-2}$ has $|\xi'|<\mathrm{ulp}(1)/2$, yet $\mathrm{rnd}(1+\xi')\ne1$ for every round-to-nearest\lean[Formalised]{FreeProbReal.FloatCutoff.mw_absorb_opposite_fails}.
    \item (\emph{Total stagnation.}) ``A decreasing positive sequence of addends cannot stagnate totally if $v^{\mathrm{in}}_{k+1}>k\,\mathrm{ulp}(v^{\mathrm{in}}_1)/2$'' fails for every $p_{\text{fl}}\ge4$: four decreasing positive floats meet the condition and the fourth is absorbed\lean[Formalised]{FreeProbReal.FloatCutoff.mw_prop4_counterexample}.
      If every addend is at most $|v^{\mathrm{in}}_1|$, in any order and with any signs, $v^{\mathrm{in}}_{k+1}$ is not absorbed if $k(1+u_{\text{fl}})^k\,\mathrm{ulp}(v^{\mathrm{in}}_1)\le(1-u_{\text{fl}})|v^{\mathrm{in}}_{k+1}|$\lean[Formalised]{FreeProbReal.FloatCutoff.mw_prop4_corrected}.
      If the $n_{\text{add}}$ earlier addends are at most $v^{\mathrm{in}}_{\max}$ in magnitude, $v^{\mathrm{in}}$ is not absorbed while $u_{\text{fl}}\,n_{\text{add}}(1+u_{\text{fl}})^{n_{\text{add}}}v^{\mathrm{in}}_{\max}<(1-u_{\text{fl}})|v^{\mathrm{in}}|$\lean[Formalised]{FreeProbReal.FloatCutoff.transmits_of_bound}.
  \end{enumerate}
\end{proposition}

\noindent Item (1) is Def.~4 of \citet{mikaitis2026stagnation}, and the quoted statement in (3) is their Prop.~4, which holds with a factor $2$ when every addend is at most $|v^{\mathrm{in}}_1|$, as a case of the bound on the earlier addends.
\Cref{thm:fc-step}(2) absorbs below $c_{\text{fl}}\,|\xi|\le\mathrm{ulp}(\xi)/2$ ($\xi$ the running sum) for addends of either sign; (1) moves the edge up to $\mathrm{ulp}(\xi)/2$, and (2) shows that this needs equal signs: the criterion fails at a power of two.
The emulator checks (1) at every step where it applies, \fcMwSame{} steps in the Monte Carlo of \Cref{sec:fc-mc} and \fcMeasMwSame{} on the measured gradients of \Cref{sec:fc-measured}, with no failure; with opposite signs the criterion applied \fcMwOpp{} and \fcMeasMwOpp{} times and the addend was not absorbed in \fcMwOppMiss{} and \fcMeasMwOppMiss{} of them, as (2) allows.

\subsection{Monte Carlo}\label{sec:fc-mc}

\paragraph{Spectra.}
\fcSpectra{} mixing cases: the four mixers of Figure~\ref{fig:ov-vs-mc-panel} (identity, bistochastic, orthogonal, Gaussian; $n=4$, $p=25$, $L\sigma^2=0.05$) and the gated block of Section~\ref{sec:gated-theory} with the frozen $H_{\text{res}}$ of the default-initialised Cayley (JPmHC), Sinkhorn (mHC), PermMix (a softmax mixture of six fixed permutation matrices) and Grassmann (a rank-2 orthogonal projector) stacks, each at $L\in\{1,2,10\}$ with $300/200/100$ samples of $J$; the last two are mixer families of our code base that no reported training run uses.
The matrices are those of the Float32 Monte Carlo of Section~\ref{sec:validation}, recomputed in binary64 from the same random streams; their mean $\tfrac1N\operatorname{tr}J^\top J$ reproduces the Float32 value to \fcBridgeMax{} relative.
Singular values come from one-sided Jacobi on $J$, never from $J^\top J$, whose error of order $2^{-26}\sigma_{\max}(J)\approx 10^{-8}\sigma_{\max}(J)$ (the square root of the binary64 machine epsilon $2^{-52}$ times the largest singular value) is the size of the fp32 cutoff; over the computed singular values $\varsigma$ of $J$, $\bigl|\sum\varsigma^2-\|J\|_F^2\bigr|/\|J\|_F^2\le\fcSvdMax$.

\paragraph{Emulator.}
Round to nearest even with gradual underflow and overflow, evaluated on binary64.
It is bit-identical to an exact rational oracle, numpy and PyTorch on \fcGoldenRound{} roundings and \fcGoldenAdd{} sums, and to native Float32 on $400\,000$ values and $400\,000$ sums; four mutated roundings (ties away from zero, truncation, one significand bit short, no gradual underflow) are rejected.

\paragraph{Grid.}
Five precision configurations $\times$ three orders (ascending $|v^{\mathrm{in}}|$, the best case of \Cref{cor:fc-batch}; shuffled, the order in which samples arrive; descending, the worst case) $\times$ two sign models $\times$ $B_{\text{b}}\in\{64,256,1024\}$, with \fcBatches{} batches per cell and the same draws for every precision, order and sign model; and a scale sweep $g_{\text{scale}}\in\{2^{-8},2^{-16},2^{-24}\}$ for the absolute cutoff.
In all, \fcCells{} cells, \fcSteps{} accumulation steps and \fcAbsorbed{} absorptions.

\paragraph{Certificates.}
Every statement of \Cref{thm:fc-step,thm:fc-sum} is evaluated at every emulated step where its hypothesis holds, and none fails: the absorption bound of \Cref{thm:fc-step}(2) applied \fcLowForced{} times, its converse \fcUpChecked{} times, the running-maximum bound \fcRmForced{} times and ascending transmission \fcAscForced{} times.
A band computed from the exact prefix sums alone (\Cref{thm:fc-sum}(1) with \Cref{thm:fc-step}(2), no emulated value) brackets the emulated absorbed fraction in every cell (\fcLawSteps{} steps).
It is tight where a few large addends dominate the sum (fp32, bistochastic, $L=10$: \fcLawBistoTenFsLo--\fcLawBistoTenFsHi\,\%, emulated \fcLawBistoTenFsAbs\,\%) and loose in bf16 (identity: $0$--\fcLawIdBfHi\,\%, emulated $0$).

\subsection{Results}\label{sec:fc-results}

\begin{figure}[p]
  \centering
  \includegraphics[width=\textwidth]{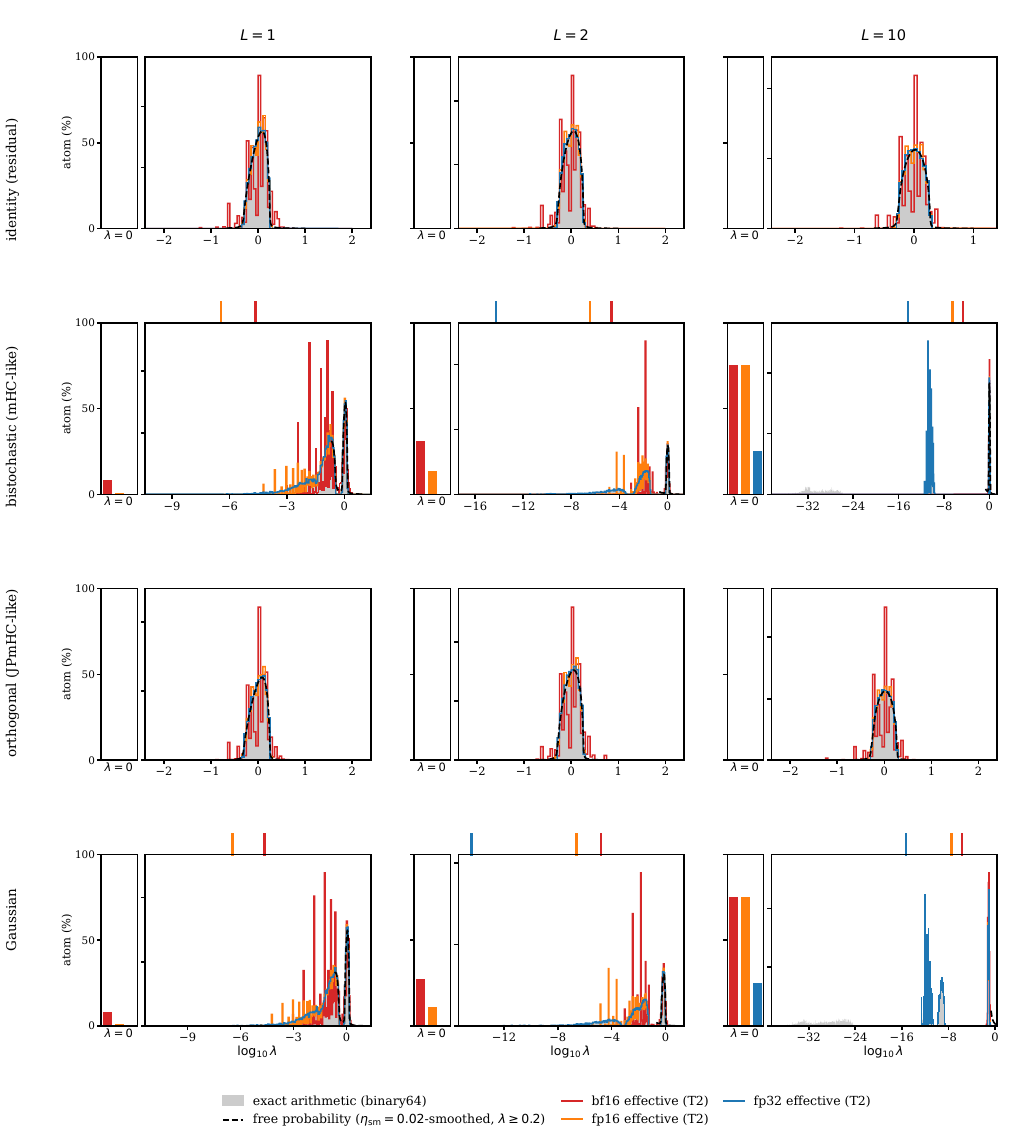}
  \caption{\textbf{Exact versus effective spectra}, the four mixers of Figure~\ref{fig:ov-vs-mc-panel} (rows) at $L\in\{1,2,10\}$ (columns).
  Density of $\log_{10}\lambda$ (fraction of all addends per decade, on a common linear scale within each panel).
  Grey: exact arithmetic (binary64 singular values of the Monte-Carlo Jacobians).
  Black dashed: the nested free-probability density of Figure~\ref{fig:ov-vs-mc-panel}, smoothed at $\eta_{\text{sm}}$ and drawn for $\lambda\ge10\eta_{\text{sm}}$; its excess at the left end of that range is the smoothed image of the mass below $\eta_{\text{sm}}$.
  Coloured: effective spectrum (T2) in bf16, fp16 and fp32 ($B_{\text{b}}=256$, shuffled order, Rademacher signs); bars at the left: its atom at zero.
  Ticks above a panel: $u_{\text{fl}}^2\lambda_{\max}$ of each format, $\lambda_{\max}$ the panel's largest exact eigenvalue.}
  \label{fig:fc-law-fig1}
\end{figure}

\begin{figure}[p]
  \centering
  \includegraphics[width=\textwidth]{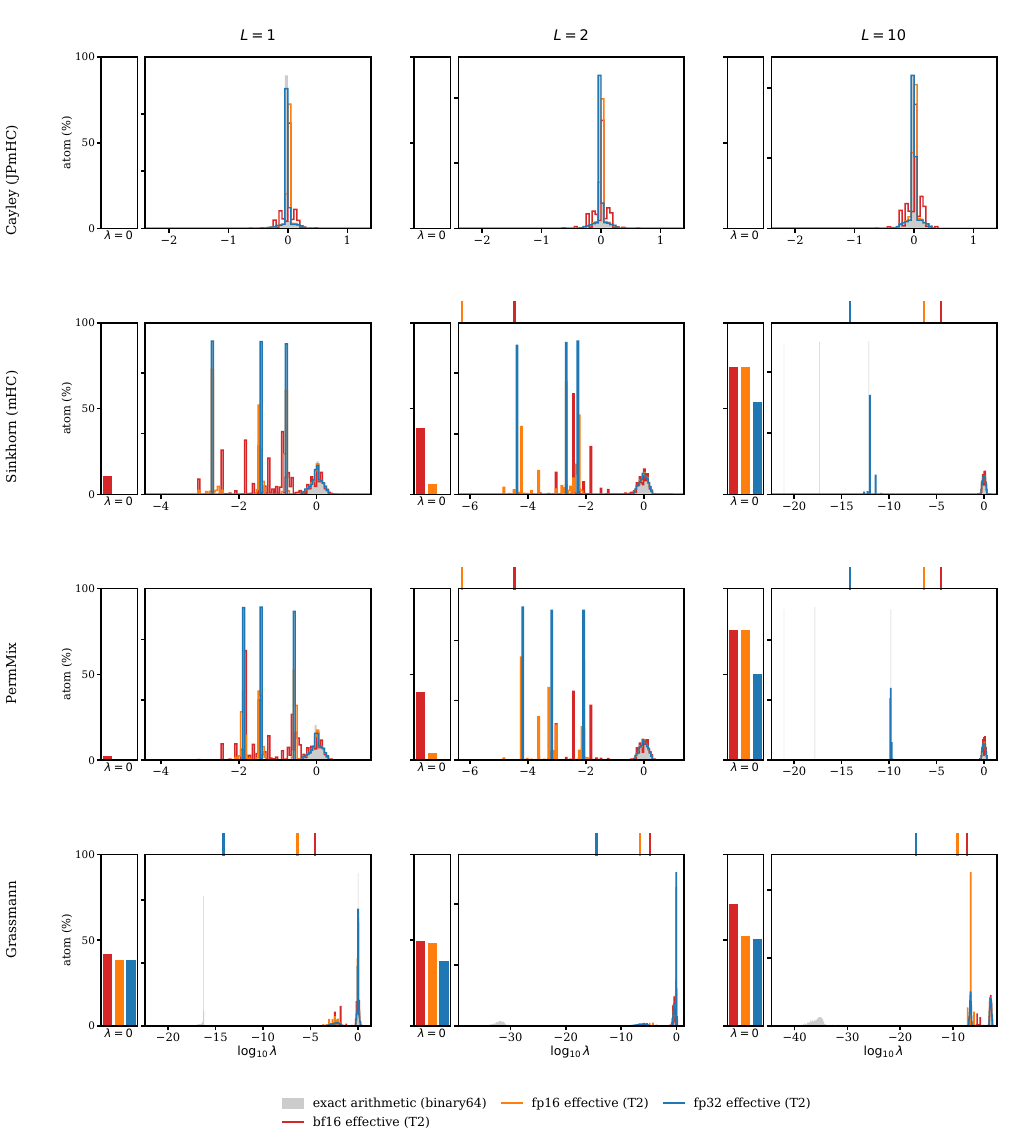}
  \caption{\textbf{Exact versus effective spectra of the gated block} (Section~\ref{sec:gated-theory}) with the frozen $H_{\text{res}}$ of the default-initialised Cayley, Sinkhorn, PermMix and Grassmann stacks; as in Figure~\ref{fig:fc-law-fig1}.
  No free-probability density is drawn: the gated solver's density is computed only up to $L=2$ and is not kept with the probe data.}
  \label{fig:fc-law-gated}
\end{figure}

\begin{table}[t]
  \centering\small
  \begin{tabular}{@{}llccccccc@{}}
\toprule
 & & \multicolumn{3}{c}{atom T1 / T2 (\%)} & \multicolumn{2}{c}{atom T2 (\%)} & $\varsigma<u_{\text{fl}}\,\sigma_{\max}(J)$ (\%) \\ \cmidrule(lr){3-5}\cmidrule(lr){6-7}\cmidrule(l){8-8}
mixer & $L$ & bf16 & fp16 & fp32 & bf16-acc32 & fp16-acc32 & bf16 / fp16 / fp32 \\
\midrule
identity & 1 & 0 / 0 & 0 / 0 & 0 / 0 & 0 & 0 & 0 / 0 / 0 \\
identity & 2 & 0 / 0 & 0 / 0 & 0 / 0 & 0 & 0 & 0 / 0 / 0 \\
identity & 10 & 0 / 0 & 0 / 0 & 0 / 0 & 0 & 0 & 0 / 0 / 0 \\
bistochastic & 1 & 2.23 / 8.64 & 0.336 / 1.07 & 0 / 0 & 3.52 & 0.482 & 0.947 / 0.163 / 0 \\
bistochastic & 2 & 18.2 / 31.0 & 6.81 / 13.4 & 0.0183 / 0.0427 & 20.1 & 8.80 & 13.6 / 3.95 / 0.0100 \\
bistochastic & 10 & 74.6 / 75.3 & 74.4 / 75.3 & 24.9 / 25.0 & 75.3 & 74.9 & 75.0 / 75.0 / 25.0 \\
orthogonal & 1 & 0 / 0 & 0 / 0 & 0 / 0 & 0 & 0 & 0 / 0 / 0 \\
orthogonal & 2 & 0 / 0.116 & 0 / 0 & 0 / 0 & 0 & 0 & 0 / 0 / 0 \\
orthogonal & 10 & 0 / 0 & 0 / 0 & 0 / 0 & 0 & 0 & 0 / 0 / 0 \\
Gaussian & 1 & 2.19 / 8.05 & 0.348 / 1.06 & 0 / 0 & 3.14 & 0.537 & 0.887 / 0.163 / 0 \\
Gaussian & 2 & 16.6 / 27.5 & 5.71 / 10.9 & 0.0244 / 0.0610 & 19.0 & 7.21 & 11.2 / 3.21 / 0 \\
Gaussian & 10 & 74.2 / 75.4 & 73.0 / 75.3 & 24.9 / 25.2 & 75.2 & 74.0 & 75.0 / 75.0 / 25.0 \\
\midrule
gated Cayley & 1 & 0 / 0 & 0 / 0 & 0 / 0 & 0 & 0 & 0 / 0 / 0 \\
gated Cayley & 2 & 0 / 0 & 0 / 0 & 0 / 0 & 0 & 0 & 0 / 0 / 0 \\
gated Cayley & 10 & 0 / 0 & 0 / 0 & 0 / 0 & 0 & 0 & 0 / 0 / 0 \\
gated Sinkhorn & 1 & 0.452 / 10.8 & 0 / 0 & 0 / 0 & 3.06 & 0 & 0 / 0 / 0 \\
gated Sinkhorn & 2 & 17.8 / 38.5 & 0.360 / 6.26 & 0 / 0 & 21.3 & 2.44 & 0 / 0 / 0 \\
gated Sinkhorn & 10 & 73.5 / 74.5 & 74.1 / 74.5 & 48.8 / 54.0 & 74.2 & 74.3 & 75.0 / 75.0 / 50.0 \\
gated Grassmann & 1 & 38.7 / 42.2 & 38.5 / 38.5 & 37.3 / 38.4 & 39.8 & 38.5 & 37.7 / 37.7 / 37.7 \\
gated Grassmann & 2 & 48.6 / 49.4 & 45.3 / 48.5 & 37.8 / 38.0 & 48.9 & 45.8 & 50.0 / 42.1 / 38.6 \\
gated Grassmann & 10 & 61.4 / 70.9 & 50.7 / 52.7 & 50.1 / 50.7 & 62.3 & 50.8 & 50.0 / 50.0 / 50.0 \\
gated PermMix & 1 & 0 / 2.15 & 0 / 0 & 0 / 0 & 0.153 & 0 & 0 / 0 / 0 \\
gated PermMix & 2 & 15.1 / 39.6 & 0 / 3.99 & 0 / 0 & 20.4 & 0.446 & 0 / 0 / 0 \\
gated PermMix & 10 & 74.8 / 75.7 & 75.3 / 75.7 & 49.7 / 50.5 & 75.6 & 75.4 & 75.0 / 75.0 / 50.0 \\
\bottomrule
\end{tabular}

  \caption{Atom at zero (\%) of the effective spectrum, reference cell ($B_{\text{b}}=256$, shuffled order, Rademacher signs, $g_{\text{scale}}=1$): T1 / T2 for the pure formats, T2 for the mixed-precision ones; last column: fraction of the exact singular values $\varsigma$ below $u_{\text{fl}}\,\sigma_{\max}(J)$ of each format ($\sigma_{\max}(J)$ per matrix $J$).}
  \label{tab:fc-atoms}
\end{table}

\begin{figure}[t]
  \centering
  \includegraphics[width=\textwidth]{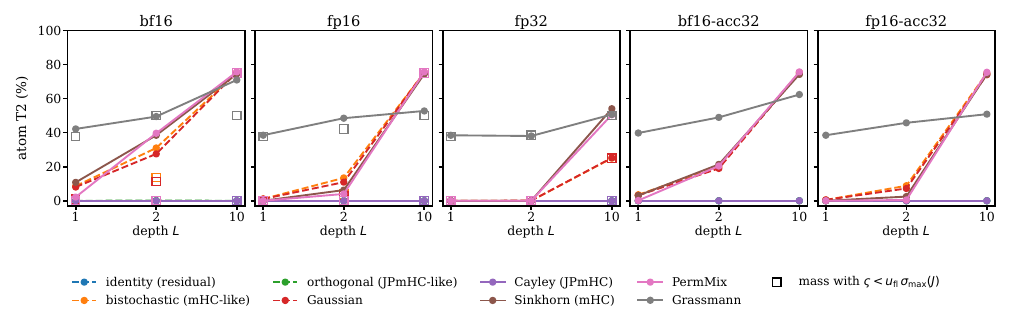}
  \caption{\textbf{Atom (T2) versus depth}, reference cell, every mixing case (dashed: Figure~\ref{fig:ov-vs-mc-panel}'s mixers; solid: gated block) and precision configuration.
  Hollow squares: fraction of the exact singular values below $u_{\text{fl}}\,\sigma_{\max}(J)$ of the panel's format.}
  \label{fig:fc-depth}
\end{figure}

\begin{figure}[t]
  \centering
  \includegraphics[width=\textwidth]{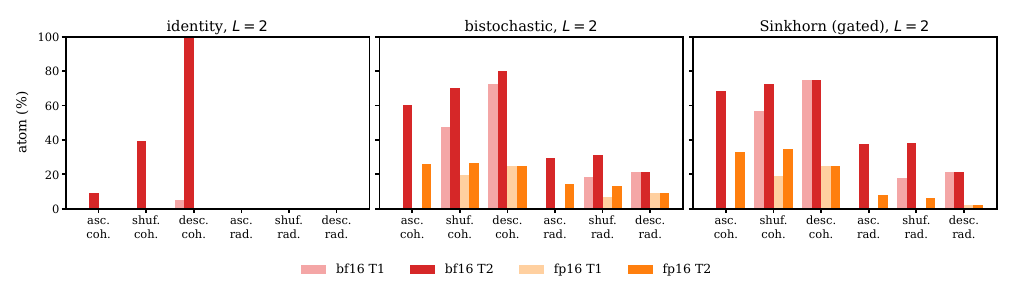}
  \caption{\textbf{Order and signs.}
  Atom T1 (light) and T2 (dark) in bf16 and fp16, $B_{\text{b}}=256$, by order (ascending, shuffled, descending) and sign model (coherent, Rademacher).}
  \label{fig:fc-order}
\end{figure}

Figures~\ref{fig:fc-law-fig1}--\ref{fig:fc-order} and Table~\ref{tab:fc-atoms} show the effective spectra; numbers below are percentages of the addends.

\paragraph{Spread, not level.}
The laws without spread (identity, orthogonal and Cayley; every singular value at least \fcFlatSigmaMin) have no T1 atom in the reference cell, in any format and at any depth, and a T2 atom of at most \fcFlatRefTTwoMax.
The contracting mixers (bistochastic, Gaussian, and the gated Sinkhorn and PermMix stacks) spread the law over tens of decades as depth grows, and the atom follows: in bf16 the bistochastic T1 atom is \fcBistoOneBfTOne, \fcBistoTwoBfTOne{} and \fcBistoTenBfTOne{} at $L=1,2,10$, and at $L=10$ fp32 loses \fcBistoTenFsTOne, where $10\%$ of the eigenvalues lie below \fcQTenBistoTen.

\paragraph{The cutoff follows the running sum.}
The fp32 atom is the mass below $u_{\text{fl}}\,\sigma_{\max}(J)$: \fcBistoTenFsTOne{} against \fcMassBistoTenFs{} (bistochastic, $L=10$), \fcSinkTenFsTOne{} against \fcMassSinkTenFs{} (Sinkhorn), \fcPermTenFsTOne{} against \fcMassPermTenFs{} (PermMix), \fcGrassOneFsTOne{} against \fcMassGrassOneFs{} (Grassmann, $L=1$).
The bf16 and fp16 atoms exceed that mass (bistochastic, $L=2$: \fcBistoTwoBfTOne{} against \fcMassBistoTwoBf{} in bf16), because the threshold is $c_{\text{fl}}\,|S^{\mathrm{acc}}|$ (\Cref{rem:fc-scale}).
For the contracting mixers at $L\le2$ the largest absorbed $\varsigma$ is, in the median batch, \fcCutRadSpread{} times the batch maximum with Rademacher signs (Figure~\ref{fig:ov-vs-mc-panel}'s mixers; \fcCutRadFactorLo--\fcCutRadFactorHi{} times $c_{\text{fl}}$) and \fcCutCohSpread{} times it with coherent signs.

\paragraph{Scale and batch size.}
The bf16 atoms are identical at $g_{\text{scale}}=2^{-8}$, $2^{-16}$ and $2^{-24}$ (\Cref{rem:fc-dirac}). fp16 is not: at $g_{\text{scale}}=2^{-24}$ every value with $\varsigma<1/2$ underflows, which is \fcUnderSpread{} of the stored values for the contracting and Grassmann stacks (bistochastic $L=1$: \fcUnderBistoOne; Sinkhorn $L=1$: \fcUnderSinkOne; Grassmann $L=10$: \fcUnderGrassTen) and none for identity, orthogonal and Cayley.
The batch size matters even without spread (\Cref{rem:fc-scale}): for identity, orthogonal and Cayley, the descending coherent bf16 batch absorbs \fcDescFlatTOne{} of the addends and no single sample changes the stored gradient; in the shuffled coherent batch the largest absorbed $\varsigma$ is \fcCutCohFlat{} times the maximum (median batch); and with Rademacher signs at $B_{\text{b}}=1024$ the T2 atom is \fcBigFlatTTwo.
For the bistochastic mixer at $L=2$ the bf16 T2 atom is \fcBBistoTwoSmallTTwo, \fcBistoTwoBfTTwo{} and \fcBBistoTwoLargeTTwo{} at $B_{\text{b}}=64,256,1024$.

\paragraph{The first moment.}
Across all reference cells the effective first moment differs from the exact one by at most \fcMOneTOneMax\,\% (T1) and \fcMOneTTwoMax\,\% (T2) (\Cref{rem:fc-m1}), including the cells where most addends are lost (bistochastic, $L=10$, bf16: T1 \fcBistoTenBfTOne, T2 \fcBistoTenBfTTwo).

\paragraph{Best case and mixed precision.}
In ascending coherent order no bf16 step absorbs, for every spectrum (the corollary guarantees it only up to $B_{\text{b}}=\fcNBf$), yet the T2 atom remains: \fcAscBistoTwoBfTTwo{} for the bistochastic mixer at $L=2$, \fcAscSinkTenBfTTwo{} for Sinkhorn at $L=10$, and \fcAscSinkTenFsTTwo{} for the latter in fp32 (\Cref{rem:fc-order}, Figure~\ref{fig:fc-order}).
An fp32 accumulator removes the T1 atom of low-precision values except the fp32 tail, but storing the gradient in bf16 keeps part of the T2 atom: bistochastic $L=2$, \fcBistoTwoBfTOne{} / \fcBistoTwoBfTTwo{} (T1 / T2) in bf16 against \fcBistoTwoBfAccTOne{} / \fcBistoTwoBfAccTTwo{} in bf16-acc32 and \fcBistoTwoFhAccTOne{} / \fcBistoTwoFhAccTTwo{} in fp16-acc32.

\paragraph{Structural atoms.}
The Grassmann $H_{\text{res}}$ is a rank-2 projector; its null directions are lost in every format (fp32: \fcGrassOneFsTOne, \fcGrassTwoFsTOne, \fcGrassTenFsTOne{} at $L=1,2,10$), and finite precision adds \fcGrassExcessOne, \fcGrassExcessTwo{} and \fcGrassExcessTen{} points in bf16.

\paragraph{The panel of \Cref{fig:effective-main}.}
It uses the generator of Figure~\ref{fig:ov-vs-mc-panel} ($n=4$, $p=25$, $S\sigma^2=0.05$, $100$ samples of $J$ per depth) with every mixer replaced by the symmetric Sinkhorn mixer of gap $\Delta=4$ (\Cref{prop:sens-additive}(c)) or by the orthogonal mixer of that figure, at $S\in\{10,84\}$ applications; its effective spectra are those of the reference cell ($B_{\text{b}}=256$, shuffled order, Rademacher signs).
Two thresholds meet in it.
A gap at or below $\Delta^\ast(S)$ means $\Lambda\ge1$, the failure of criterion (C1) of Section~\ref{sec:beyond-init}; bf16 rounds a stream-difference direction off once the mixers contract it below the relative cutoff of \Cref{rem:fc-scale}, a small multiple of $c_{\text{fl}}$ times the batch maximum.
At $\Delta=4$ the Sinkhorn product passes neither threshold at $S=10$ and both at $S=84$, where bf16 loses most of its samples and fp32 none; the orthogonal mixer loses almost none at either depth.
The two thresholds coincide here by the choice of $\Delta$ and $S$, not in general.

\subsection{Measured gradients of trained language models}\label{sec:fc-measured}

The sampling model of \Cref{sec:fc-setting} draws independent singular values of random Jacobians.
For a linear layer the weight gradient of one micro-batch is itself a sum over tokens, $\partial\mathcal L/\partial(W_{\text{lin}})_{ij}=\sum_tg_{\text{lin},t}[i]\,x_{\text{lin},t}[j]$ ($W_{\text{lin}}$ the layer's weight, $x_{\text{lin},t}$ its input and $g_{\text{lin},t}$ the loss gradient at its output, at token $t$), so one forward and backward pass exports the actual addends.
We measure them on the trained nano-GPT models of \Cref{app:sensitivity} (final checkpoints of all \fcMeasRuns{} runs: residual, mHC and JPmHC at $L\in\{12,24\}$): one micro-batch of $T_{\text{mb}}=\fcMeasTokens$ evaluation tokens per run, for every attention and feed-forward projection, the output head and the mixer logits, \fcMeasGroups{} parameter groups in all.
Every exported coordinate is bridged to autograd: its addends sum to the autograd gradient within \fcMeasBridgeLo--\fcMeasBridgeHi{} relative, inside the bar $(1+2^{-24})^{T_{\text{mb}}+1}-1=\fcMeasBridgeBar$ of autograd's own fp32 sum, while transposed or shifted pairings miss by at least \fcMeasMutantMin.
The addends run through the emulator of \Cref{sec:fc-mc} in every configuration of \Cref{tab:fc-measured}, in token order (the GEMM's order), a seeded random permutation of the tokens, and ascending and descending $|v^{\mathrm{in}}|$.
The training itself used bf16 autocast, in which a weight-gradient GEMM multiplies bf16 operands, accumulates in fp32 and stores a bf16 result: configuration bf16-acc32.
The predictions quoted below (registered in the experiment as FC-P1--FC-P6 and FC-R1--FC-R3, unrelated to the ledger of \Cref{tab:sens-ledger}) were committed before any analysis of the exports; the appendix generator fails if a registered verdict moves.

\begin{table}[h]
  \centering\small
  \begin{tabular}{@{}lccc@{}}
    \toprule
    configuration & atom T1 (\%) & atom T2 (\%) & input underflow (\%) \\
    \midrule
    \fcMeasPrecRows
    \bottomrule
  \end{tabular}
  \caption{Measured weight gradients of the trained nano-GPT models: atom at zero of the per-token contributions of the linear layers, token order, pooled over the linear groups (output head included) and averaged over seeds; range over the six (topology, depth) cohorts.}\label{tab:fc-measured}
\end{table}

\paragraph{On the A100 path the atom forms mainly in the stored result, not in the sum.}
In bf16-acc32 the T1 atom of the attention and feed-forward projections is \fcMeasPathHidTOne\,\% (registered: below $1\%$), and \fcMeasPathHeadTOne\,\% for the output head.
Their T2 atom, the share of tokens that cannot individually change the stored gradient, is \fcMeasPathHidTTwo\,\% (registered: at least $25\%$), and \fcMeasPathHeadTTwo\,\% for the output head.
With the result stored in fp32 the pooled T2 atom falls to \fcMeasFsTTwo\,\%: the bf16 cast of the GEMM output, not the accumulation, sets the atom.
The ratio of effective to exact first moment is \fcMeasPathMOne{} per module type and run (\Cref{rem:fc-m1}).
Pure bf16 accumulation would erase \fcMeasBfHidTOne\,\% of the steps of the hidden projections (registered: at least $25\%$). fp16 without loss scaling underflows \fcMeasFhUnder\,\% of the contributions (registered: at least $10\%$); with scale $2^{16}$, \fcMeasFhLsUnder\,\% (registered: at most $1\%$).

\paragraph{Topology does not matter here.}
At the same depth, mHC, JPmHC and the residual network agree within \fcMeasTopoGap{} percentage points in every module type (bf16-acc32 T2; registered: within $5$).
The large differences between mixers in \Cref{sec:fc-results} come from the Jacobians of random mixers at depth; the trained models' contributions do not inherit them, consistently with \Cref{app:sensitivity}, where neither trained mixer is near collapse at these depths.
The mixer logits differ: in fp32, their path, JPmHC loses at most \fcMeasMixJpTOneMax\,\% of the steps, while mHC loses \fcMeasMixMhcResTOne\,\% of those of $b_{\text{res}}$ (T2 \fcMeasMixMhcResTTwo\,\%), on top of the suppressions of \Cref{sec:sens-training}.
JPmHC's diagonal logits get no gradient at all: the skew generator does not depend on them.

\paragraph{Order.}
Token order matters in pure bf16: in \fcMeasOrderBfMore{} of \fcMeasOrderBfGroups{} linear groups it absorbs more steps than a random permutation of the same tokens, and fewer in \fcMeasOrderBfFewer{} (exact sign test, p-value \fcMeasOrderBfP).
The exchangeability registered as FC-P1 is therefore refuted for pure bf16, and the i.i.d.\ draws of \Cref{sec:fc-mc} are optimistic for real token streams.
With an fp32 accumulator the effect is gone (\fcMeasOrderAccMore{} against \fcMeasOrderAccFewer, p-value \fcMeasOrderAccP).
Ascending order with any fp32 accumulator absorbs none of \fcMeasAscSteps{} steps, as \Cref{cor:fc-batch} requires; descending bf16 absorbs \fcMeasDescBfTOne\,\%.

\paragraph{Reductions.}
Against an fp32-stored result, with bf16 operands, the pooled T2 atom is \fcMeasRedSeq\,\% for the left-to-right bf16 sum, \fcMeasRedPair\,\% pairwise, \fcMeasRedKahan\,\% with Kahan compensation, \fcMeasRedBlock\,\% for blocks of 128 promoted to fp32, and \fcMeasRedSr\,\% with stochastic rounding, which makes each addition unbiased but keeps the atom; a 14-bit accumulator gives \fcMeasRedAccFourteen\,\%, \fcMeasRedAccFourteenBlock\,\% with the same promotion, and fp32 \fcMeasRedFs\,\%.
Pairwise and Kahan summation beat left to right in every hidden projection but not always in the output head (\fcMeasNotBelowPair{} and \fcMeasNotBelowKahan{} run cells), so the registered ``in every module type'' (FC-R1, FC-R3) fails there.
Promotion lowers the 14-bit accumulator's median relative error, but the ratio with and without it reaches \fcMeasRTwoRatio{} in some run, above the registered $0.1$ (FC-R2).
Every left-to-right, pairwise and blocked sum (\fcMeasTree) lies inside the error bound of its summation tree, and the certificates of \Cref{thm:fc-step,thm:fc-sum} hold at all \fcMeasSteps{} emulated steps.
The tree bound and the stochastic-rounding statements are proved in the experiment's Lean module, not in the audited library of \Cref{app:lean}.

\subsection{Relation to other cutoffs, and limits}\label{sec:fc-limits}

The same rounding also acts when an update is added to a weight in low precision: an update below $u_{\text{fl}}\,|\theta|$ is lost, $\theta$ the weight.
With fp32 master weights, the standard mixed-precision setup, that step is exact enough, and low precision enters through the accumulation studied here, in the values, the sums or the stored gradient (the acc32 configurations).
Loss scaling multiplies $g_{\text{scale}}$: it removes fp16's absolute cutoff but not the relative one (\Cref{rem:fc-dirac}).
The Monte Carlo does not model stochastic rounding, blocked or tree-shaped reductions, or compensated sums; \Cref{sec:fc-measured} measures them on trained models, where each shrinks the bf16 atom and none removes it.
The sampling model takes one coordinate and independent $\varsigma_k$ from the singular-value law of the Jacobians at initialisation.
The measured addends of \Cref{sec:fc-measured} come from one micro-batch of one checkpoint per run, with the forward and backward passes in fp32 and the factors rounded afterwards, and leave out the trainer's sums over micro-batches and devices.

\section{Lean certification}\label{app:lean}
Every statement or algorithm marked \textsuperscript{\textsf{L}} in the text is a Lean 4 declaration of the accompanying repository. \emph{Formalised}: proved, depending only on the standard axioms. \emph{Conditional}: proved from the named hypothesis. \emph{Caveat}: proved, but its relevance rests on the stated unformalised assumption. \emph{Executable}: Lean code that computes the paper's algorithm.

{\small
\begin{longtable}{p{0.37\linewidth}p{0.13\linewidth}p{0.42\linewidth}}
\toprule Lean declaration & Tier & Cited in / file \\ \midrule \endhead
\hypertarget{lean:FreeProb.Density.richardsonHalving}{\texttt{FreeProb.\allowbreak{}Density.\allowbreak{}richardsonHalving}} & Executable & Appendix/computational\_\allowbreak{}pipeline.\allowbreak{}tex:229; \texttt{FreeProb/Num} \\
\hypertarget{lean:FreeProb.Fig1.mcPhase}{\texttt{FreeProb.\allowbreak{}Fig1.\allowbreak{}mcPhase}} & Executable & Appendix/computational\_\allowbreak{}pipeline.\allowbreak{}tex:248; \texttt{FreeProb/Num} \\
\hypertarget{lean:FreeProb.Fig1.poissonSmooth}{\texttt{FreeProb.\allowbreak{}Fig1.\allowbreak{}poissonSmooth}} & Executable & Appendix/computational\_\allowbreak{}pipeline.\allowbreak{}tex:249; \texttt{FreeProb/Num} \\
\hypertarget{lean:FreeProb.Fig1.validate}{\texttt{FreeProb.\allowbreak{}Fig1.\allowbreak{}validate}} & Executable & Appendix/computational\_\allowbreak{}pipeline.\allowbreak{}tex:244; \texttt{FreeProb/Num} \\
\hypertarget{lean:FreeProb.FlatNested.nestedM}{\texttt{FreeProb.\allowbreak{}FlatNested.\allowbreak{}nestedM}} & Executable & Appendix/computational\_\allowbreak{}pipeline.\allowbreak{}tex:11; \texttt{FreeProb/Num} \\
\hypertarget{lean:FreeProb.OVLin.hrfsStep}{\texttt{FreeProb.\allowbreak{}OVLin.\allowbreak{}hrfsStep}} & Executable & Appendix/computational\_\allowbreak{}pipeline.\allowbreak{}tex:191; \texttt{FreeProb/Num} \\
\hypertarget{lean:FreeProb.OVLin.ovLinDensity}{\texttt{FreeProb.\allowbreak{}OVLin.\allowbreak{}ovLinDensity}} & Executable & Appendix/computational\_\allowbreak{}pipeline.\allowbreak{}tex:229; \texttt{FreeProb/Num} \\
\hypertarget{lean:FreeProb.OVLin.ovLinPlainS}{\texttt{FreeProb.\allowbreak{}OVLin.\allowbreak{}ovLinPlainS}} & Executable & Appendix/computational\_\allowbreak{}pipeline.\allowbreak{}tex:81; \texttt{FreeProb/Num} \\
\bottomrule
\end{longtable}
}

\end{document}